\documentclass[journal]{IEEEtran}
\usepackage[numbers]{natbib}

\newcommand{\toggleformat}[2]{#1}
\newcommand{\togglevspace}[2]{\vspace{\toggleformat{#1}{#2}}}

\usepackage{multicol}

\usepackage{amsmath,amsfonts,amssymb,amsthm}
\usepackage{float}
\usepackage{subfig}
\usepackage{bm}
\usepackage{graphicx}
\usepackage{epstopdf}
\usepackage{epsfig}
\usepackage{xspace}
\usepackage{multirow}
\usepackage{hhline}
\usepackage[export]{adjustbox}
\usepackage{csvsimple}
\usepackage[dvipsnames]{xcolor}

\usepackage{graphicx}
\usepackage{xr}

\usepackage{tabularx}
\usepackage[colorlinks=true, allcolors=blue, bookmarks=true]{hyperref}

\usepackage{comment}
\usepackage[binary-units]{siunitx}
\usepackage{relsize}
\usepackage{ifthen}
\usepackage[colorinlistoftodos]{todonotes}

\usepackage[vlined,ruled,linesnumbered]{algorithm2e}
\usepackage{graphics} \usepackage{rotating}
\usepackage{color}
\usepackage{enumerate}
\usepackage[T1]{fontenc}
\usepackage{psfrag}
\usepackage{epsfig} \usepackage{booktabs}
\usepackage{graphicx,url}
\usepackage{multirow}
\usepackage{array}
\usepackage{latexsym}
\usepackage{amsfonts}
\usepackage{amsmath}
\usepackage{amssymb}
\usepackage{xstring}
\usepackage[noend]{algorithmic}
\usepackage{multirow}
\usepackage{xcolor}
\usepackage{prettyref}
\usepackage{flexisym}
\usepackage{bigdelim}
\usepackage{breqn} \usepackage{listings}
\let\labelindent\relax
\usepackage{enumitem}
\usepackage{xspace}
\usepackage{bm}
\graphicspath{{./figures/}}
\usepackage{tikz}
\usetikzlibrary{matrix,calc}

\usepackage{mdwlist}

\let\itemize\undefined
\makecompactlist{itemize}{stditemize}

 \usepackage[capitalise]{cleveref}

\newrefformat{prob}{Problem\,\ref{#1}}
\newrefformat{def}{Definition\,\ref{#1}}
\newrefformat{sec}{Section\,\ref{#1}}
\newrefformat{sub}{Section\,\ref{#1}}
\newrefformat{prop}{Proposition\,\ref{#1}}
\newrefformat{app}{Appendix\,\ref{#1}}
\newrefformat{alg}{Algorithm\,\ref{#1}}
\newrefformat{cor}{Corollary\,\ref{#1}}
\newrefformat{thm}{Theorem\,\ref{#1}}
\newrefformat{lem}{Lemma\,\ref{#1}}
\newrefformat{fig}{Fig.\,\ref{#1}}
\newrefformat{tab}{Table\,\ref{#1}}

\newtheorem{theorem}{Theorem}

\newtheorem{lemma}[theorem]{Lemma}

\newtheorem{definition}[theorem]{Definition}
\newtheorem{proposition}[theorem]{Proposition}
\newtheorem{remark}[theorem]{Remark}

\newcommand{\beq}{\begin{equation}}
\newcommand{\eeq}{\end{equation}}
\newcommand{\bea}{\begin{eqnarray}}
\newcommand{\eea}{\end{eqnarray}}
\newcommand{\ba}{\begin{array}}

\newcommand{\calB}{{\cal B}}

\newcommand{\calE}{{\cal E}}
\newcommand{\calF}{{\cal F}}
\newcommand{\calG}{{\cal G}}
\newcommand{\calI}{{\cal I}}
\newcommand{\calM}{{\cal M}}
\newcommand{\calO}{{\cal O}}
\newcommand{\calT}{{\cal T}}
\newcommand{\calV}{{\cal V}}

\newcommand{\smallheading}[1]{\textit{#1}: }

\newcommand{\eg}{\emph{e.g.,}\xspace}
\newcommand{\ie}{\emph{i.e.,}\xspace}
\newcommand{\myParagraph}[1]{{\bf #1.}\xspace}

\newcommand{\M}[1]{{\bm #1}} \renewcommand{\boldsymbol}[1]{{\bm #1}}

\newcommand{\hide}[1]{}

\newcommand{\hiddenText}{{\color{gray} hidden text.}}
\newcommand{\hideWithText}[1]{\hiddenText}

\newcommand{\Natural}[1]{ { {\mathbb N}^{#1} } }

\DeclareMathOperator*{\argmin}{arg\,min}

\newcommand{\normsq}[2]{\left\|#1\right\|^2_{#2}}
\newcommand{\norm}[1]{\left\| #1 \right\|}

\newcommand{\tran}{^{\mathsf{T}}}
\newcommand{\diag}[1]{\mathrm{diag}\left(#1\right)}
\newcommand{\trace}[1]{\mathrm{tr}\left(#1\right)}

\newcommand{\inv}{^{-1}}

\newcommand{\setdef}[2]{ \{#1 \; {:} \; #2 \} }

\newcommand{\SEthree}{\ensuremath{\mathrm{SE}(3)}\xspace}

\newcommand{\ME}{\M{E}}
\newcommand{\MM}{\M{M}}

\newcommand{\MT}{\M{T}}
\newcommand{\MX}{\M{X}}
\newcommand{\MOmega}{\M{\Omega}}

\newcommand{\vxx}{\boldsymbol{x}}

\newcommand{\blue}[1]{{\color{blue}#1}}

\newcommand{\linkToPdf}[1]{\href{#1}{\blue{(pdf)}}}
\newcommand{\linkToPpt}[1]{\href{#1}{\blue{(ppt)}}}
\newcommand{\linkToCode}[1]{\href{#1}{\blue{(code)}}}
\newcommand{\linkToWeb}[1]{\href{#1}{\blue{(web)}}}
\newcommand{\linkToVideo}[1]{\href{#1}{\blue{(video)}}}
\newcommand{\linkToMedia}[1]{\href{#1}{\blue{(media)}}}
\newcommand{\award}[1]{\xspace}

\newcommand{\TD}[1]{\text{TD}\ensuremath{\left[ #1 \right]}}
\newcommand{\layers}{\ell}
\newcommand{\CliqueSetOf}[1]{\ensuremath{\mathcal{C}\left(#1\right)}}

\newcommand{\subgraph}{sub-graph\xspace}
\newcommand{\subgraphs}{sub-graphs\xspace}

\newcommand{\Graph}{\ensuremath{\mathcal{G}}\xspace}
\newcommand{\Nodes}{\ensuremath{\mathcal{V}}}
\newcommand{\Edges}{\ensuremath{\mathcal{E}}}
\newcommand{\bag}{\calB}
\newcommand{\Tree}{\ensuremath{\calT}}
\newcommand{\treeDecomposition}{tree decomposition\xspace}
\newcommand{\TreeDecomposition}{Tree Decomposition\xspace}
\newcommand{\Treedecomposition}{Tree decomposition\xspace}
\newcommand{\treeDecompositions}{tree decompositions\xspace}
\newcommand{\treewidth}[1]{\ensuremath{\text{tw}\left[ #1 \right]}}
\newcommand{\JTH}[1]{\ensuremath{\mathcal{J}_{#1}}}

\newcommand{\TSDF}{TSDF\xspace}
\newcommand{\ESDF}{ESDF\xspace}
\newcommand{\GVD}{GVD\xspace}
\newcommand{\ESDFs}{ESDFs\xspace}
\newcommand{\name}{Hydra\xspace}

\newcommand{\igx}{\emph{GT-Trajectory}\xspace}
\newcommand{\ivx}{\emph{VIO}\xspace}
\newcommand{\ivl}{\emph{VIO+V-LC}\xspace}
\newcommand{\ivd}{\emph{VIO+SG-LC}\xspace}
\newcommand{\ivn}{\emph{VIO+GNN-LC}\xspace}

\newcommand{\percFound}{\emph{\% Found}\xspace}
\newcommand{\percCorrect}{\emph{\% Correct}\xspace}
\newcommand{\positionError}{\emph{Position Error}\xspace}

\newcommand{\bigo}[1]{\ensuremath{\mathcal{O}\left( #1\right)}}
\newcommand{\nrSymbols}{L}

\newcommand{\dplace}{d^p}
\newcommand{\placeGraph}{\calG_p}
\newcommand{\sgf}{SceneGraphFusion}
\newcommand{\HydraGT}{\name{} (GT)}
\newcommand{\HydraBest}{\name{} (OneFormer)}

\newcommand{\mb}[1]{\SI{#1}{\mebi\byte}}
\newcommand{\semantic}[1]{{\emph{#1}}}

\newcommand{\mlp}[1]{$\mathcal{M}\left(#1\right)$}

\newcommand{\hydraURL}{\url{https://github.com/MIT-SPARK/Hydra}\xspace}
\newcommand{\videoURL}{\url{https://youtu.be/AEaBq2-FeY0}\xspace}

\toggleformat{}{\setcounter{secnumdepth}{3}}

\begin{document}

\toggleformat{
\title{Foundations of Spatial Perception for Robotics: \\
  {Hierarchical Representations and Real-time Systems} }
  \author{Nathan Hughes, Yun Chang, Siyi Hu, Rajat Talak, Rumaisa Abdulhai, Jared Strader, Luca Carlone\thanks{The authors are with the Laboratory for Information \& Decision Systems (LIDS), Massachusetts Institute of Technology, Cambridge, USA\@.
          {Email: \{na26933, yunchang, siyi, talak, rumaisa, jstrader, lcarlone\}@mit.edu}}
  \thanks{This work was partially funded by the AIA CRA FA8750-19-2-1000,
    ARL DCIST CRA W911NF-17-2-0181,
    ONR RAIDER N00014-18-1-2828,
    MIT Lincoln Laboratory's Autonomy al Fresco program,
    Lockheed Martin Corporation’s Neural Prediction in 3D Dynamic Scene
    Graphs program, and by Carlone's Amazon Research Award.
}
}
}{
\title{Foundations of spatial perception for robotics: Hierarchical representations and real-time systems}
\author{Nathan Hughes, Yun Chang, Siyi Hu, Rajat Talak, Rumaia Abdulhai, Jared Strader, and Luca Carlone}
\affiliation{The authors are with the Laboratory for Information and Decision Systems (LIDS), Massachusetts Institute of Technology, Cambridge, USA.}
\corrauth{Nathan Hughes, Laboratory for Information and Decision Systems, 77 Massachusetts Avenue, Cambridge, MA 02139}
\email{na26933@mit.edu}
\runninghead{Hughes et al.}

\begin{abstract}
3D spatial perception is the problem of building and maintaining an actionable and persistent representation of the environment in real-time using sensor data and prior knowledge.
Despite the fast-paced progress in robot perception, most existing methods either build purely geometric maps (as in traditional SLAM) or ``flat'' metric-semantic maps that do not scale to large environments or large dictionaries of semantic labels.
The first part of this paper
is concerned with representations:
we show that scalable representations for spatial perception need to be \emph{hierarchical} in nature.
Hierarchical representations are efficient to store, and lead to layered graphs with small \emph{treewidth}, which enable provably efficient inference.
We then introduce an example of hierarchical representation for indoor environments, namely a \emph{3D scene graph}, and discuss its structure and properties.
The second part of the paper focuses on algorithms to incrementally construct a 3D scene graph as the robot explores the environment. Our algorithms combine 3D geometry (\eg to cluster the free space into a graph of places), topology (to cluster the places into rooms), and
geometric deep learning (\eg to classify the type of rooms the robot is moving across).
The third part of the paper focuses on algorithms to maintain and correct 3D scene graphs during long-term operation.
We propose hierarchical descriptors for loop closure detection and describe how to correct a scene graph in response to loop closures, by solving a \emph{3D scene graph optimization problem}.
We conclude the paper by combining the proposed perception algorithms into \emph{Hydra}, a real-time spatial perception system that builds a 3D scene graph  from visual-inertial data in real-time.
We showcase Hydra's performance in photo-realistic simulations and real data collected by a Clearpath Jackal robots and a Unitree A1 robot. We release an open-source implementation of Hydra at \url{https://github.com/MIT-SPARK/Hydra}.
\end{abstract}

}

\maketitle

\toggleformat{

\begin{abstract}
3D spatial perception is the problem of building and maintaining an actionable and persistent representation of the environment in real-time using sensor data and prior knowledge.
Despite the fast-paced progress in robot perception, most existing methods either build purely geometric maps (as in traditional SLAM) or ``flat'' metric-semantic maps that do not scale to large environments or large dictionaries of semantic labels.
The first part of this paper
is concerned with representations:
we show that scalable representations for spatial perception need to be \emph{hierarchical} in nature.
Hierarchical representations are efficient to store, and lead to layered graphs with small \emph{treewidth}, which enable provably efficient inference.
We then introduce an example of hierarchical representation for indoor environments, namely a \emph{3D scene graph}, and discuss its structure and properties.
The second part of the paper focuses on algorithms to incrementally construct a 3D scene graph as the robot explores the environment. Our algorithms combine 3D geometry (\eg to cluster the free space into a graph of places), topology (to cluster the places into rooms), and
geometric deep learning (\eg to classify the type of rooms the robot is moving across).
The third part of the paper focuses on algorithms to maintain and correct 3D scene graphs during long-term operation.
We propose hierarchical descriptors for loop closure detection and describe how to correct a scene graph in response to loop closures, by solving a \emph{3D scene graph optimization problem}.
We conclude the paper by combining the proposed perception algorithms into \emph{Hydra}, a real-time spatial perception system that builds a 3D scene graph  from visual-inertial data in real-time.
We showcase Hydra's performance in photo-realistic simulations and real data collected by a Clearpath Jackal robots and a Unitree A1 robot. We release an open-source implementation of Hydra at \url{https://github.com/MIT-SPARK/Hydra}.
\end{abstract}

 }{}

\section{Introduction}\label{sec:introduction}
The next generation of robots and autonomous systems will need to build
actionable, metric-semantic, multi-resolution, persistent representations of large-scale unknown environments in real-time.
\emph{Actionable} representations are required for a robot to understand and execute complex humans instructions  (\eg{} ``bring me the cup of tea I left on the dining room table'').
These representations include both \emph{geometric and semantic} aspects of the environment (\eg{} to plan a path to the dining room, and understand where the table is); moreover, they should allow reasoning over relations between objects (\eg{} to understand what it means for the cup of tea to be \emph{on} the table).
These representations need to be \emph{multi-resolution}, in that they might need to capture information at multiple levels of abstractions (\eg{} from objects to rooms, buildings, and cities) to interpret human commands  and enable fast planning (\eg{} by allowing planning over compact abstractions rather than dense low-level geometry).
Such representations must be built in \emph{real-time} to support just-in-time decision-making.
Finally, these representations must be \emph{persistent} to support long-term autonomy:
(i) they need to scale to large environments,
(ii) they should allow fast inference and corrections as new evidence is collected by the robot, and
(iii) their size should only grow with the size of the environment they model.

\begin{figure}
    \centering
\includegraphics[trim={3mm, 0, 2mm, 0}, clip, width=1.02\columnwidth]{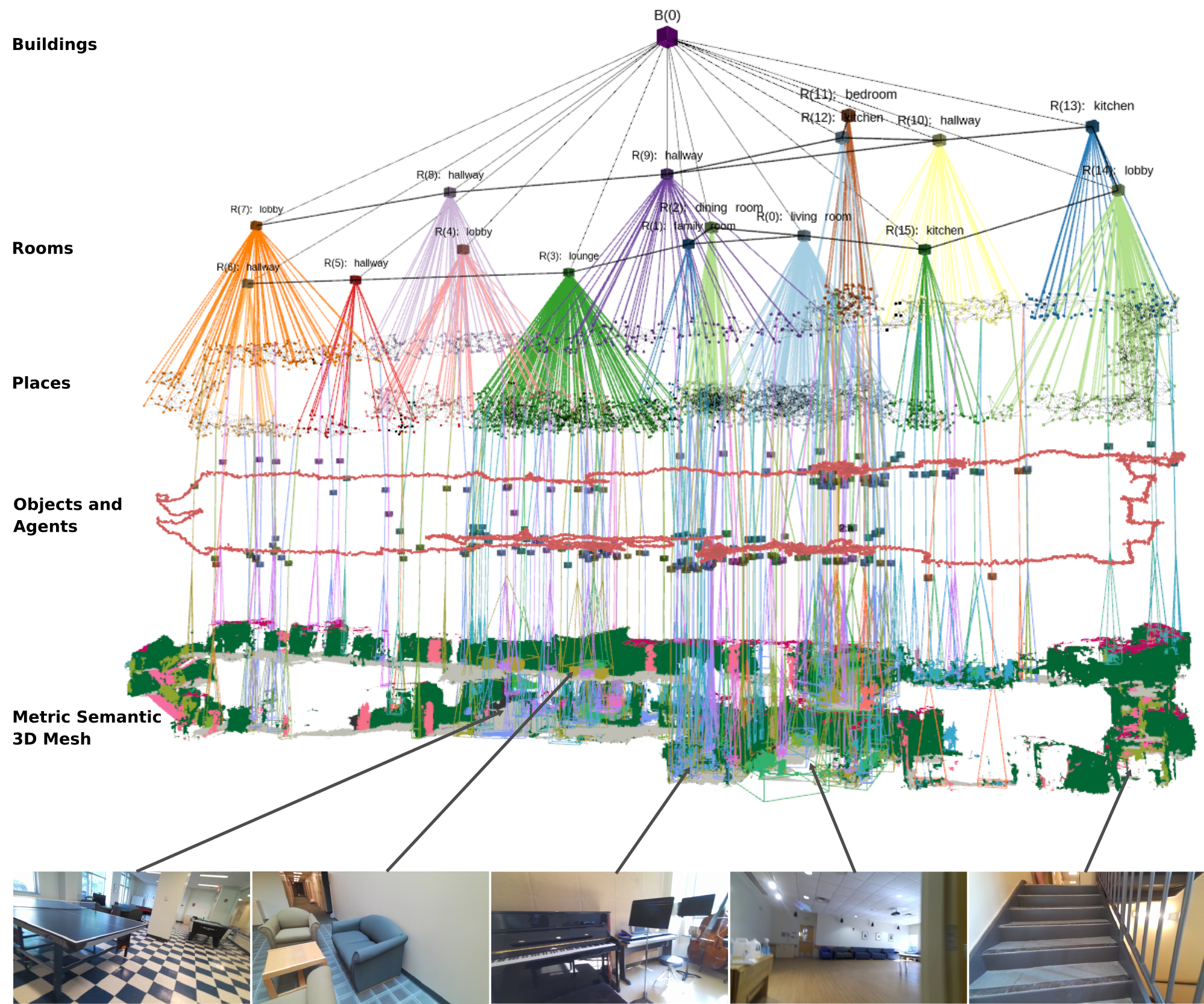}
\caption{We introduce \emph{\name}, a highly parallelized system that builds 3D scene graphs from sensor data in real-time, by combining geometric reasoning (e.g., to build a 3D mesh and cluster the free space into a graph of places), with topology (to cluster the places into rooms), and geometric deep learning (e.g., to classify the type of rooms the robot is moving across).
    The figure shows sample input data and the 3D scene graph created by \name{} in a large-scale real environment.}\label{fig:intro}
     \togglevspace{0mm}{0mm}
\end{figure}

{\bf 3D spatial perception} (or Spatial AI~\citep{Davison18-futuremapping}) is the problem of building actionable and persistent representations from sensor data and prior knowledge in real-time.
This problem is the natural evolution of Simultaneous Localization and Mapping (SLAM), which also focuses on building persistent map representations in real-time, but is typically limited to geometric understanding.
In other words, if the task assigned to the robot is purely geometric (\eg{} ``go to position [X, Y, Z]''), then spatial perception reduces to SLAM and 3D reconstruction, but as the task specifications become more advanced (\eg{} including semantics, relations, and affordances), we obtain a much richer problem space and SLAM becomes only a component of a more elaborate spatial perception system.

The pioneering work~\citep{Davison18-futuremapping} has introduced the notion of spatial AI.
Indeed the requirements that spatial AI has to build actionable and persistent representations can  already be found in~\cite{Davison18-futuremapping}. In this paper we take a step further by arguing that
such representations must be \emph{hierarchical} in nature, since a suitable hierarchical organization reduces storage during long-term operation and leads to provably efficient inference.
Moreover, going beyond the vision in~\cite{Davison18-futuremapping}, we discuss how to
combine different tools (metric-semantic SLAM, 3D geometry, topology, and geometric deep learning) to implement  a real-time spatial perception system for indoor environments.

\myParagraph{Hierarchical Representations}
3D Scene Graphs~\citep{Rosinol21ijrr-Kimera,Armeni19iccv-3DsceneGraphs,Rosinol20rss-dynamicSceneGraphs,Kim19tc-3DsceneGraphs,Wald20cvpr-semanticSceneGraphs,Wu21cvpr-SceneGraphFusion} have recently emerged as expressive hierarchical representations of 3D environments.
A 3D scene graph (\eg \cref{fig:intro}) is a layered graph where nodes represent spatial concepts at multiple levels of abstraction (from low-level geometry to objects, places, rooms, buildings, etc.) and edges represent relations between concepts.
\citet{Armeni19iccv-3DsceneGraphs} pioneered the use of 3D scene graphs in computer vision and proposed the first algorithms to parse a metric-semantic 3D mesh into a 3D scene graph.
\citet{Kim19tc-3DsceneGraphs} reconstruct a 3D scene graph of objects and their relations.
\citet{Rosinol21ijrr-Kimera,Rosinol20rss-dynamicSceneGraphs} propose a novel 3D scene graph model that
(i) is built directly from sensor data,
(ii) includes a \subgraph of places (useful for robot navigation),
(iii) models objects, rooms, and buildings, and
(iv) captures moving entities in the environment.
Recent work~\citep{Wald20cvpr-semanticSceneGraphs, Wu21cvpr-SceneGraphFusion,Izatt21tr-sceneGraphs,Gothoskar21arxiv-3dp3} infers objects and relations from point clouds, RGB-D sequences, or object detections. In this paper, we formalize the intuition behind these works that suitable representations for robot perception have to be hierarchical in nature, and discuss guiding principles behind the choice of ``symbols'' we have to include in these representations.

\myParagraph{Real-time Systems}
While 3D scene graphs can serve as an advanced ``mental model'' for robots, methods to build such a rich representation in real-time remain largely unexplored.
The works~\citep{Kim19tc-3DsceneGraphs,Wald20cvpr-semanticSceneGraphs, Wu21cvpr-SceneGraphFusion} allow real-time operation but are restricted to ``flat'' 3D scene graphs that only include objects
and their relations while disregarding the top layers in \cref{fig:intro}.
The works~\citep{Rosinol21ijrr-Kimera,Armeni19iccv-3DsceneGraphs,Rosinol20rss-dynamicSceneGraphs}, which focus on building truly hierarchical representations, run offline and require several minutes to build a 3D scene graph (\cite{Armeni19iccv-3DsceneGraphs} even assumes the availability of a complete metric-semantic mesh of the environment built beforehand).
Extending our prior works~\citep{Rosinol21ijrr-Kimera,Rosinol20rss-dynamicSceneGraphs} to operate in real-time is non-trivial.
These works utilize an Euclidean Signed Distance Function (\ESDF) of the entire environment to build the scene graph.
Unfortunately, \ESDFs{} memory requirements  scale poorly in the size of the environment; see~\cite{Oleynikova17iros-voxblox} and \cref{sec:symbolGroundingAndHieRep}.
Moreover, the extraction of places and rooms in~\cite{Rosinol21ijrr-Kimera,Rosinol20rss-dynamicSceneGraphs} involves batch algorithms that process the entire \ESDF, whose computational cost grows over time and is incompatible with real-time operation.
Finally, the \ESDF{} is reconstructed from the robot trajectory estimate which keeps changing in response to loop closures.
The approaches~\citep{Rosinol21ijrr-Kimera,Rosinol20rss-dynamicSceneGraphs} would therefore need to rebuild the scene graph from scratch after every loop closure, clashing with real-time operation.

The present paper extends our prior work~\citep{Hughes22rss-hydra} and proposes the first
real-time system to build hierarchical 3D scene graphs of large-scale environments.
Following~\citep{Hughes22rss-hydra}, recent works has explored constructing \emph{situational graphs}~\citep{Bavle22ral-SGraph,Bavle22arxiv-SGraphPlus}, a hierarchical representation for scene geometry with layers describing free-space traversability, walls, rooms, and floors.
While related to this research line, the works~\citep{Bavle22ral-SGraph,Bavle22arxiv-SGraphPlus} focus on LIDAR-based systems, which mostly reason over geometric features (\eg{} walls, rooms, floors), but lack the rich semantics we consider in this paper (\eg{} objects and room labels).
We postpone a more extensive literature review to \cref{sec:relatedWork}.

\myParagraph{Contribution 1: Foundations of Hierarchical Representations (\cref{sec:symbolGroundingAndHieRep})}
We start by observing that flat metric-semantic representations scale poorly in the size of the environment and the size of the vocabulary of semantic labels the robot has to incorporate in the map.
For instance, a voxel-based metric-semantic map picturing the floor of an office building with ${40}$ semantic labels per voxel (as the one underlying the approaches of~\cite{Rosinol21ijrr-Kimera} and~\cite{Grinvald19ral-voxbloxpp}) already requires roughly \mb{450} to be stored.
Envisioning future robots to operate on much large scales (\eg{} an entire city) and using a much larger vocabulary (\eg{} the English dictionary includes roughly 500,000 words), we argue that research should move beyond flat representations.
We show that hierarchical representations allow to largely reduce the memory requirements, by enabling loss-less compression of semantic information into a layered graph, as well as lossy compression of the geometric information into meshes and graph-structured representations of the free space.
Additionally, we show that hierarchical representations are amenable for efficient inference.
In particular, we prove that the layered graphs appearing in hierarchical map representations have small \emph{treewidth}, a property that enables efficient inference;
for instance, we conclude that the treewidth of the scene graph modeling an indoor environment does not scale with the number of nodes in the graph (\ie{} roughly speaking, the number of nodes is related to the size of the environment), but rather with the maximum number of objects in each room. While most of the results above are general and apply to a broad class of hierarchical representations, we conclude this part by introducing a specific hierarchical representation for indoor environments, namely \emph{3D scene graphs}.

\myParagraph{Contribution 2: Real-time Incremental 3D Scene Graph Construction (\cref{sec:incrementalLayers})}
After establishing the importance of hierarchical representations, we move to developing a suite of algorithms to estimate 3D scene graphs from sensor data. In particular, we develop real-time algorithms that can incrementally estimate a 3D scene graph of an unknown building from visual-inertial sensor data.
We use existing methods for geometric processing to incrementally build a metric-semantic mesh of the environment and reconstruct a sparse graph of ``places'';
intuitively, the mesh describes the occupied space (including objects in the environment), while the graph of places provides a succinct description of the free space. Then we propose novel algorithms to efficiently cluster the places into rooms; here we use tools from topology, and in particular the notion of \emph{persistent homology}~\citep{Huber21idsc-persistent}.
Finally, we use novel architectures for geometric deep learning, namely \emph{neural trees}~\citep{Talak21neurips-neuralTree}, to infer the semantic labels of each room (\eg{} bedroom, kitchen, etc.) from the labels of the object within.
Towards this goal, we show that our 3D scene graph representation allows to quickly and incrementally compute a tree-decomposition of the 3D scene graph, which ---together with our bound on the scene graph treewidth--- ensures that the neural tree runs in real-time on an embedded GPU.

\myParagraph{Contribution 3: Persistent Representations via Hierarchical Loop Closure Detection and 3D Scene Graph Optimization (\cref{sec:LCD-and-SGO})}
Building a persistent map representation requires the robot to recognize it is revisiting a location it has seen before, and correcting the map accordingly.
We propose a novel hierarchical approach for loop closure detection: the approach involves
(i)~a \emph{top-down loop closure detection} that uses hierarchical descriptors ---capturing statistics across layers in the scene graph--- to find putative loop closures, and
(ii)~a \emph{bottom-up geometric verification} that attempts estimating the loop closure pose by registering putative matches.
Then, we propose the first algorithm to optimize a 3D scene graph in response to loop closures;
our approach relies on \emph{embedded deformation graphs}~\citep{Sumner07siggraph-embeddedDeformation} to simultaneously and consistently correct all the layers of the scene graph,
 including the 3D mesh, objects, places, and the robot trajectory.

\myParagraph{Contribution 4: \name, a Real-Time Spatial Perception System  (\cref{sec:hydra})}
We conclude the paper by integrating the proposed algorithms into a highly parallelized perception system, named \emph{\name}, that combines fast early and mid-level perception processes (\eg{} local mapping) with slower high-level perception (\eg{} global optimization of the scene graph).
We demonstrate Hydra in challenging simulated and real datasets, across a variety of environments,
including an apartment complex, an office building, and two student residences.
Our experiments (\cref{sec:experiments}) show that
(i) we can reconstruct 3D scene graphs of large, real environments in real-time,
(ii) our online algorithms achieve an accuracy comparable to batch offline methods and build a richer representation compared to competing approaches~\citep{Wu21cvpr-SceneGraphFusion}, and
(iii) our loop closure detection approach outperforms standard approaches based on bag-of-words and visual-feature matching in terms of quality and quantity of detected loop closures.
The source code of \name{} is publicly available at \hydraURL{}.

\myParagraph{Novelty with respect to~\cite{Hughes22rss-hydra,Talak21neurips-neuralTree}}
This paper builds on our previous conference papers~\citep{Hughes22rss-hydra,Talak21neurips-neuralTree} but includes several novel findings.
First, rather than postulating a 3D scene graph structure as done in~\cite{Hughes22rss-hydra}, we formally show that  hierarchical representations are crucial to achieve scalable scene understanding and fast inference.
Second, we propose a novel room segmentation method based on the notion of persistent homology, as a replacement for the heuristic method in~\cite{Hughes22rss-hydra}.
Third, we develop novel learning-based hierarchical descriptors for place recognition, that further improve performance with respect to the handcrafted hierarchical descriptors in~\cite{Hughes22rss-hydra}.
Fourth, the real-time system described in this paper is able to also assign room labels, leveraging the neural tree architecture from~\cite{Talak21neurips-neuralTree}; while~\cite{Talak21neurips-neuralTree} uses the neural tree over small graphs corresponding to a single room, in this paper we provide an efficient way to obtain a tree-decomposition of the top layers of the scene graph, hence extending~\cite{Talak21neurips-neuralTree} to simultaneously operate over \emph{all} rooms and account for their relations.
Finally, this paper includes further experiments and evaluations on real robots (Clearpath Jackal robots and Unitree A1 robots) and comparisons with recent scene graph construction baselines~\citep{Wu21cvpr-SceneGraphFusion}.

\begin{figure*}[t]
\centering
\subfloat[\label{fig:map1}]{\includegraphics[trim=160mm 25mm 215mm 10mm, clip, width=0.333\textwidth]{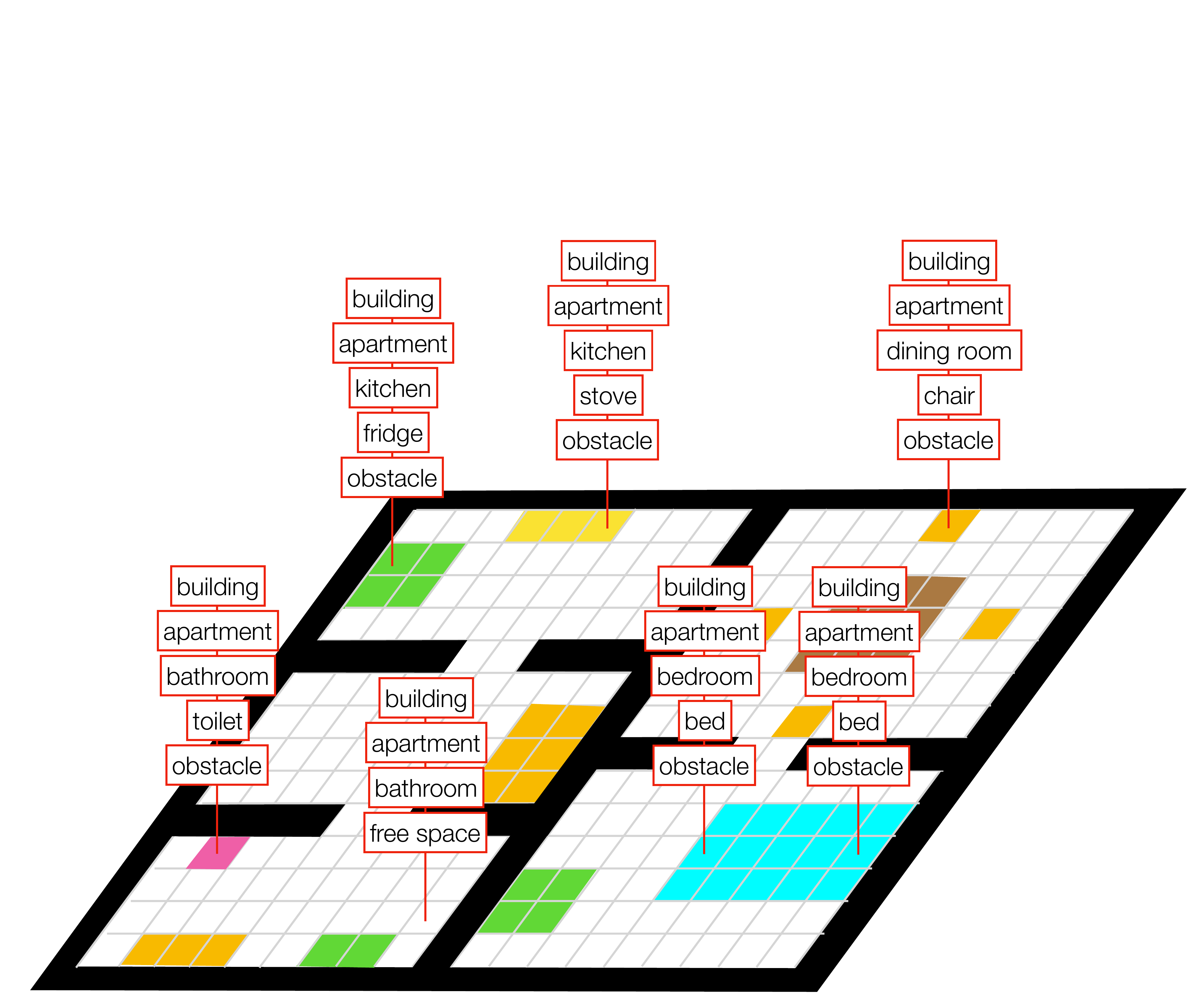}}
\subfloat[\label{fig:map2}]{\includegraphics[trim=160mm 25mm 215mm 10mm, clip, width=0.333\textwidth]{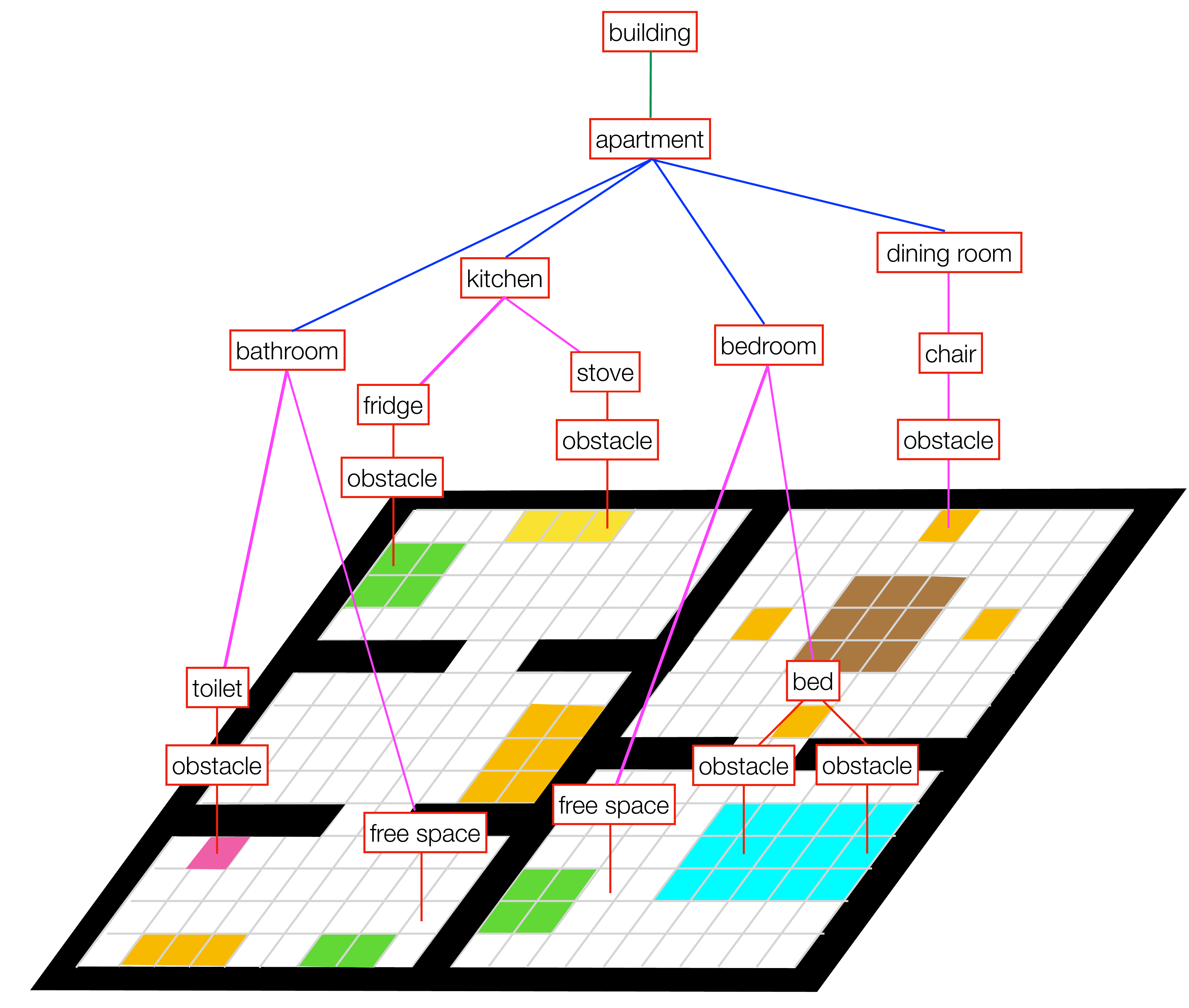}}
\subfloat[\label{fig:map3}]{\includegraphics[trim=160mm 25mm 215mm 10mm, clip, width=0.333\textwidth]{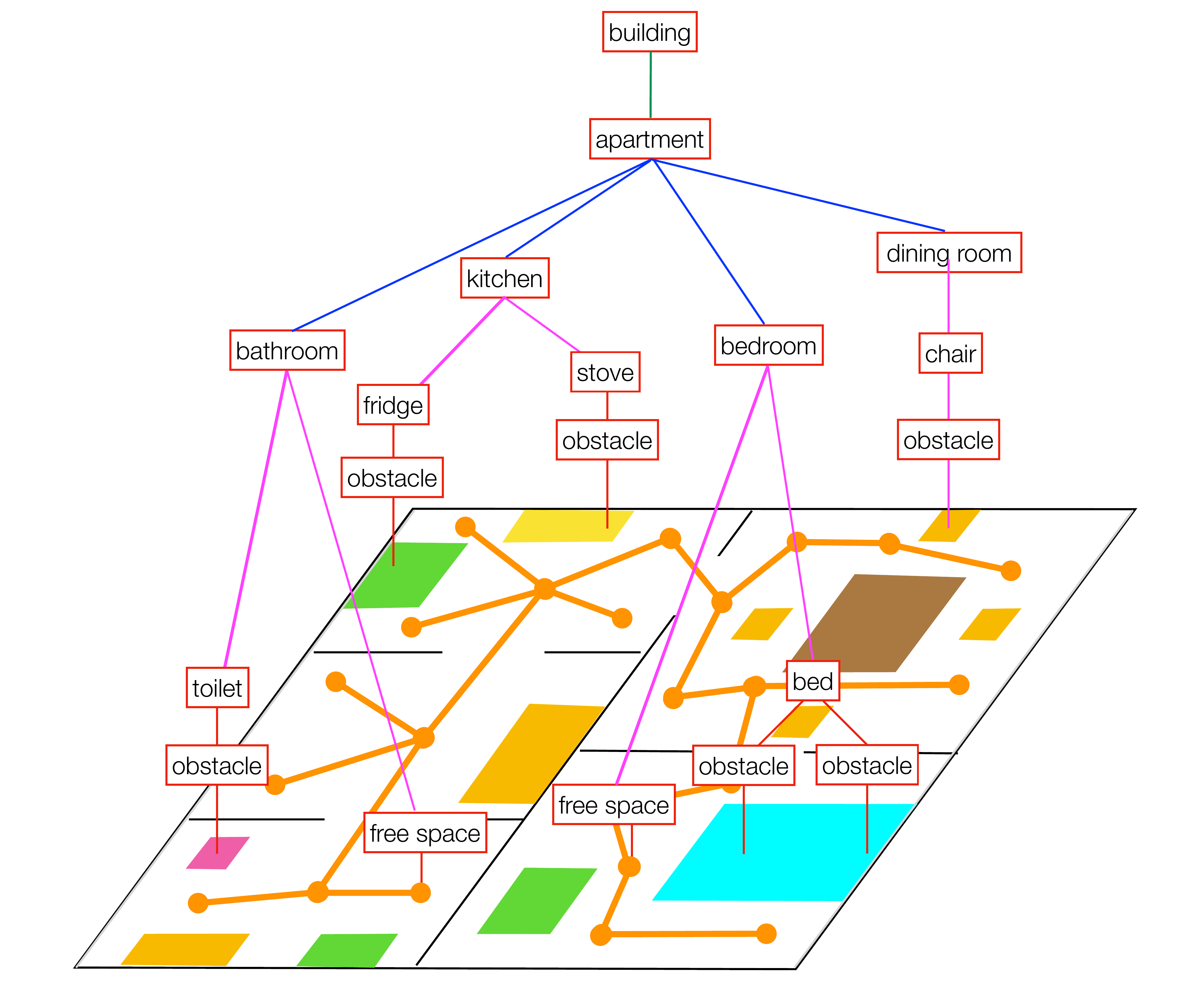}}
\caption{(a) Example of flat metric-semantic representation. We show labels only for a subset of the cells in the grid map for the sake of readability. (b) Example of hierarchical metric-semantic representation with a grid-map as sub-symbolic layer. (c) A hierarchical metric-semantic representation with compressed sub-symbolic layer. The cells representing each object are compressed into bounding polygons, and the free-space cells are compressed into a sparse graph where each node and edge are also assigned a radius that defines a circle of free-space around it.}\label{fig:fig_map}
\end{figure*}

\section{Symbol Grounding and the Need for Hierarchical Representations}\label{sec:symbolGroundingAndHieRep}
The goal of this section is twofold. First, we remark that in order to support the execution of complex instructions,
map representations must be metric-semantic, and hence ground symbolic knowledge (\ie{} semantic aspects in the scene) into the map geometry (\cref{sec:symbol_grounding}). Second, we observe that ``flat'' representations, which naively store semantic attributes for each geometric entity in the map (\eg{} attach semantic labels to each voxel in an ESDF) scale poorly in the size of the environment; on the other hand, hierarchical representations scale better to large environments (\cref{sec:hieRep_scalability}) and enable efficient inference (\cref{sec:hieRep_inference}).
We conclude the section by introducing the notion of \emph{3D scene graph}, an example of hierarchical representation for indoor environments, and discussing its structure and properties  (\cref{sec:hieRep_sceneGraphs}).

\subsection{Symbols and Symbol Grounding}\label{sec:symbol_grounding}

High-level instructions issued by humans (\eg{} ``go and pick up the chair'') involve symbols.
A \emph{symbol} is the representation of a concept that has a specific meaning for a human (\eg{} the word ``chair'' is a symbol in the sense that, as humans, we understand the nature of a chair).
For a symbol to be useful to guide robot action, it has to be \emph{grounded} into the physical world.
For instance, for the robot to correctly execute the instruction ``go and pick up the chair'', the robot needs to ground the symbol ``chair'' to the physical location the chair is situated in.\footnote{In general, this is referred to as the \emph{symbol grounding} problem~\citep{Harnard90physica-symbolGrounding} in artificial intelligence.}
In principle, we could design the perception system of our robot to ground symbols directly in the sensor data.
For instance, a robot with a camera could map image pixels to appropriate symbols (\eg{} ``chair'', ``furniture'', ``kitchen'', ``apartment''), as commonly done in 2D semantic segmentation~\citep{GarciaGarcia17arxiv}.
However,
grounding symbols directly into sensor data is not scalable: sensor data (\eg{} images) is collected at high rates and is relatively expensive to store.
This is neither convenient nor efficient when grounding symbols for long-term operation.
Instead, we need intermediate (or \emph{sub-symbolic}) representations that compress the raw sensor data into a more compact format, and that can be used to ground symbols.
Traditional geometric map models used in robotics (\eg{} 2D occupancy grids or 3D voxel-based maps) can be understood as sub-symbolic representations: each cell in a voxel grid does not represent a semantic concept, and instead is used to ground two symbols: ``obstacle'' and ``free-space''.
Therefore, this paper is concerned with building \emph{metric-semantic} representations, which
ground symbolic knowledge into geometric (\ie{} sub-symbolic) representations.

{\bf Which Symbols Should a Map Contain?}
In this paper, we directly specify the set of symbols the robot is interested in (\ie{} certain classes of objects and rooms, to support navigation in indoor environments).
However, it is worth discussing how to choose the symbols more generally.
 The choice of symbols clearly depends on the task the robot has to execute.
A mobile robot may implement a path planning query by just using the ``free-space'' symbol and the corresponding grounding, while a more complex domestic robot may instead need additional symbols (\eg{} ``cup of coffee'', ``table'', and ``kitchen'') and their groundings to execute instructions such as ``bring me the cup of coffee on the table in the kitchen''.
A potential way to choose the symbols to include in the map therefore consists of inspecting the set of instructions the robot has to execute, and then extracting the list of symbols (\eg{} objects and relations of interest) these instructions involve.
For instance, in this paper, we mostly focus on indoor environments and assume the robot is tasked with navigating to certain objects and rooms in a building.
Therefore, the symbols we include in our maps include free-space, objects, rooms, and buildings.
After establishing the list of relevant symbols, the goal is for the robot
 to build a compact sub-symbolic representation (\ie{} a geometric map), and then ground these symbols into the sub-symbolic representation.
A pictorial representation of this idea is given in \cref{fig:map1}:
the robot builds an occupancy map (sub-symbolic representation) and attaches a number of semantic labels (symbols) to each cell.
We refer to this representation as ``flat'' since each cell is assigned a list of symbols of interest.
As we will see shortly, there are more clever ways to store the same information.

\subsection{Hierarchical Representations Enable Large-Scale Mapping}\label{sec:hieRep_scalability}
While the flat representation of \cref{fig:map1} may contain all the information the robot needs to complete its task, it is unnecessarily expensive to store.
Here we show that the symbolic knowledge underlying such representation naturally lends itself to be efficiently stored using a hierarchical representation.

Consider the case where we use a flat representation (as shown in \cref{fig:map1}) to represent a 3D scene and assume our dictionary contains $\nrSymbols$ symbols of interest.
Let $\delta$ be the cell (or voxel) size and call $V$ the volume of the scene the representation describes.
Then, the corresponding metric-semantic representation would require a memory of:
\begin{equation}
\label{eq:naive-memory}
m = \bigo{\nrSymbols \cdot V / \delta^3}
\end{equation}
to store $\nrSymbols$ labels for each voxel.
Note that $m$ grows with the number of symbols $\nrSymbols$, multiplied by the size of the sub-symbolic representation $V / \delta^3$ (\ie{} the number of voxels in the volume).
When mapping large environments the term $V / \delta^3$ quickly becomes unsustainably large; for instance, covering a {$10$}{km} $\times$ $10$km area with a $10$cm grid resolution, even disregarding the vertical dimension, would require $10^{10}$ voxels. In addition, it would be desirable for the next generation of robots to execute a large set of tasks, hence requiring them to ground a large dictionary of symbols $\nrSymbols$.\footnote{For reference, the English dictionary includes around $500,000$ words.}
The size of the sub-symbolic representation, the number of symbols, and their multiplicative relation in~\eqref{eq:naive-memory} make a flat metric-semantic representation unsuitable for large-scale spatial perception.

A key observation here is that multiple voxels encode the same grounded symbols (\eg{} a chair).
In addition, many symbols naturally admit a hierarchical organization where
higher-level concepts (\ie{} buildings or rooms for indoor environments) contain lower-level symbols (\eg{} objects).
This suggests that we can more efficiently store information \emph{hierarchically}, where all voxels
associated with a certain object (\eg{} all voxels belonging to a chair) are mapped to the same symbolic node (\eg{} an object node with the semantic label ``chair''), object nodes are associated with the room they belong to, room nodes are associated to the apartment unit they belong to, and so on.
This transforms the flat representation of \cref{fig:map1} into the hierarchical model of \cref{fig:map2}, where only the lowest level symbols are directly grounded into voxels, while the top layers are arranged hierarchically.
This hierarchical representation is more parsimonious and reduces the required memory to
\begin{equation}
\label{eq:naive-orb-memory}
  m =  \bigo{V / \delta^3  + N_{\text{objects}}  + N_{\text{rooms}} + \cdots + N_{\text{buildings}}},
\end{equation}
where $N_{\text{layer}}$ (with $\text{layer} \in \{\text{objects}, \text{rooms}, \ldots\}$) is the number of symbols at the specific layer of the hierarchy.
This representation is more memory efficient than~\eqref{eq:naive-memory} because it decouples
the number of symbols from the size $V / \delta^3$ of the sub-symbolic representation. For instance, the scene in \cref{fig:fig_map} has 336 voxels, and we assume $L = 5$.\footnote{The symbols are stored as building type, apartment type, room type, object type, and free-space/obstacle.}
Then, the flat representation of \cref{fig:map1} involves storing $1680$ symbols, while \cref{fig:map2} only requires storing $355$ symbols and $354$ edges describing their hierarchical relations.
Crucially, the compression we performed when moving from \cref{fig:map1} to \cref{fig:map2} is \emph{lossless}: the two representations contain exactly the same amount of information.

While we ``compressed'' the symbolic representation using a hierarchical data structure,
the term $V / \delta^3$ in~\eqref{eq:naive-orb-memory} that corresponds to the sub-symbolic layer is still impractically large for many applications.
Therefore, our robots will also typically need some compression mechanism for the sub-symbolic layer that provides a more succinct description of the occupied and free space as compared to voxel-based maps.
Fortunately, the mapping literature already offers alternatives to standard 3D voxel-based maps such as OctTree~\citep{Zeng2013graphical-octree} or neural implicit representations~\citep{Park19cvpr-deepSDF}.
In general, this compression reduces the memory requirement to
\begin{equation}
\label{eq:total-compression}
  m =  \bigo{N_{\text{sub-sym}} + N_{\text{objects}} + N_{\text{rooms}} + \cdots + N_{\text{buildings}}},
\end{equation}
where $N_{\text{sub-sym}}$ is the size of the compressed sub-symbolic representation.
The behavior of $N_{\text{sub-sym}}$ is driven by the compression approach used, but in general $N_{\text{sub-sym}}$ ends up being much smaller than $V / \delta^3$;
in the ideal case $N_{\text{sub-sym}}$ would grow gracefully with respect to the complexity of the scene and the resolution required by the task. We show a notional example of a compressed sub-symbolic layer in \cref{fig:map3}.
This can represent the 336 original cells of the sub-symbolic layer of \cref{fig:map2} with roughly 23 nodes, 22 edges, and a 2D polygonal mesh with roughly 132 vertices.
Note that moving from \cref{fig:map2} to \cref{fig:map3} may entail a \emph{lossy} compression of the sub-symbolic layer, \ie{} the compressed representation may only be an approximate representation of the geometry of the environment.
We consider this to be a feature rather than a bug: the general idea is that we can always compress the sub-symbolic representation to fit into the available memory of our robot, although such a compression might induce some performance degradation in the execution of the task (\eg{} coarser maps might lead to longer paths in motion planning).

\subsection{Hierarchical Representations Enable Efficient Inference}\label{sec:hieRep_inference}

While above we showed that hierarchical representations are more scalable in terms of memory requirements,
this section shows that the graphs underlying hierarchical representations also enable fast inference.
Specifically, we show that these hierarchical graphs have small \emph{treewidth}: their treewidth does not grow with the size of the graph (\ie{} proportionally to the size of the explored scene), but rather with the treewidth of each layer in the hierarchy.
The treewidth is a well-known measure of complexity for many problems on graphs~\citep{Bodlaender2006c-Chapter-TreewidthComputation, Dechter07ai, Feder93stoc, Grohe18talgo-GraphIsomorphism-BndTreeWidth}.
\citet{Venkat12uai} show that the graph treewidth is the only structural parameter that influences tractability of probabilistic inference on graphical models: while inference is NP-hard in general for inference on graphical models~\citep{Cooper90ai}, proving that a graph has small treewidth opens the door to efficient, polynomial-time inference algorithms.
Additionally, in our previous work we have shown that the treewidth is also the main factor impacting the expressive power and tractability of novel graph neural tree architectures, namely \emph{neural trees}~\citep{Talak21neurips-neuralTree}. The results in this section therefore open the door to the efficient use of powerful tools for learning over graphs; see \cref{sec:rooms_classification}.

In the rest of this section, we formalize the notion of hierarchical graph and show that the treewidth of a hierarchical graph is always bounded by the maximum treewidth of each of its layers. We do this by proving that the \treeDecomposition{} of the hierarchical graph can be obtained by a suitable concatenation of the \treeDecompositions{} of its layers. The results in this section are general and apply to arbitrary hierarchical representations (as defined below) beyond the representations of indoor environments we consider later in the paper.
\begin{definition}[Hierarchical Graph]\label{def:hieGraph}
    A graph $\calG = (\calV, \calE)$ is said to be a \emph{hierarchical graph} if the set of nodes $\calV$ can be partitioned into $\layers$ layers, \ie{} $\calV = \cup_{i=1}^{\layers} \calV_i$, such that
    \begin{enumerate}
        \item \emph{single parent:} each node $v \in \calV_i$ at layer $i$ shares an edge with at most one node in the layer $\calV_{i+1}$ above,
        \item \emph{locality:} each $v \in \calV_i$ can only share edges with nodes in the same or adjacent layers, \ie{} $\calV_{i-1}$, $\calV_i$, or $\calV_{i+1}$,
        \item \emph{disjoint children:}  for any $u, v \in \calV_i$ the children of $u$ and $v$, namely $C(v)$ and $C(u)$ are disjoint (\ie{} they share no nodes or edges), where
            \begin{equation}
                C(u) \triangleq \left\{ w \in \calV_{i-1}~|~(w, u) \in \calE \right\}
            \end{equation}
            denotes the children of $u$, \ie{} the set of nodes in layer $\calV_{i-1}$ sharing an edge with $u$.
\end{enumerate}
    We refer to $\calG$ as an $\layers$-layered hierarchical graph and denote with $\calV_i$  the set of nodes in layer $i$.
    Moreover, we conventionally order the layers from the bottom to the top, \ie{} the lowest layer is $\calV_1$, while the top layer is $\calV_\layers$.
\end{definition}

To gain some insight into \cref{def:hieGraph}, consider a hierarchical graph where the bottom layer $\calV_1$ describes the objects in the environment, while the higher layers $\calV_2$ and $\calV_3$ describe room and buildings, respectively.
Then, the first condition in \cref{def:hieGraph} requires that each object belongs to a single room, and each room belongs to a single building.
The second condition restricts the inter-layer edges to connect objects to rooms, and rooms to buildings.
Finally, the third condition says that objects in a room are not connected to objects in another room, and that rooms in a building do not share edges with rooms in other buildings.
These conditions are relatively mild; we note that they are easily met when the edges in the graph represent inclusion or adjacency (\ie{} for the graphs considered in the rest of this paper).

We now show that a \treeDecomposition{} of a hierarchical graph can be constructed by concatenating the \treeDecomposition{} of each of its layers.
This result will allow us to obtain the treewidth bound in \cref{prop:tw-hierarchy}.
The resulting \treeDecomposition{} algorithm (\cref{algo:td-hierarchical}) will also be used in the neural tree approach used for room classification in \cref{sec:rooms_classification}.
The interested reader can find a refresher about \treeDecomposition{} and treewidth in \cref{sec:treeDecomposition}\footnote{We leave the proof of~\cref{thm:td-hierarchy} in the main text since it contains a description of \cref{algo:td-hierarchical}, while we postpone other proofs to the appendix.}
and an example of execution of \cref{algo:td-hierarchical} in~\cref{fig:tree_decomposition}.
\begin{theorem}[\TreeDecomposition{} of Hierarchical Graph]\label{thm:td-hierarchy}
    {Let $\calG = (\calV\!=\!\cup_{i=1}^{\layers} \calV_i, \calE)$ be an $\layers$-layered hierarchical graph.
    Let
    $T$ be the \treeDecomposition{} of the \subgraph spanned by $\calV_\layers$ (top layer) and
    $T_v$ be the \treeDecomposition{} of the \subgraph spanned by the children $C_i(v)$ of $v$, for $v \in \calV_i$ and $i=\layers,\ldots,2$.
    Then a \treeDecomposition{} for graph $\calG$ can be constructed from $\{ T_v~|~v \in \calV_i$ and $i=\layers,\ldots,2 \}$ and $T$, by the simple concatenation procedure described in \cref{algo:td-hierarchical}.
    }
\end{theorem}
\begin{proof}
\smallheading{(i)~Notation} We denote the \treeDecomposition{} of a graph $\calG$ by $T = (B, E) = \TD{\calG}$, where $B$ denotes the collection of bags and $E$ denotes the set of edges in the \treeDecomposition{} graph $T$.
    For a \treeDecomposition{} graph $T = (B, E)$ and any node $w$ (not already included in any of the bags $b \in B$), we denote $T + \{ w\}$ to be the \treeDecomposition{} $T$, after adding the element $w$ to every bag in $T$.
    Finally, given two disjoint trees $T, T'$ and an edge $(b, b')$ with $b \in T$ and $b' \in T'$, we use $T \leftarrow T \oplus T' \oplus {(b, b')}$ to indicate updating graph $T$ by adding graph $T'$, along with the edge $(b, b')$ to it.

\begin{algorithm}[t]
    \caption{\mbox{Tree-Decomposition of Hierarchical Graphs}}\label{algo:td-hierarchical}
\KwData{$\layers$-layered hierarchical graph $\calG = (\calV\!=\!\cup_{i=1}^{\layers} \calV_i, \calE)$}
    \KwResult{\Treedecomposition $T = (B, E)$}
    \tcc{extract \subgraph with nodes $\calV_\layers$}
    $\calG_{\layers} \gets \calG[\calV_\layers]$\;
    \tcc{\treeDecomposition of $\calG_\layers$}
    $T = (B, E) \gets \TD{G_\layers}$\; \label{line:initializeT}
    \For{each layer $i = \layers,\ldots,2$}{
        \For{each node $v$ in $\calV_i$}{
            $T_v = (B_v, L_v) \gets \TD{\calG[C(v)]}$\; \label{line:doChildrenTD}
            ${T^{'}_v \gets T_v + \{ v \}}$\; \label{line:addNodeToBags}
            Pick any bag $b$ in $T$ such that $v \in b$\;
            Pick any bag $b'$ in $T^{'}_v$\;
            $T \gets T \oplus T' \oplus {(b, b')}$\; \label{line:addEdgeToT}
        }
    }
\end{algorithm}

\smallheading{(ii)~Intuitive Explanation of \cref{algo:td-hierarchical}} We form the \treeDecomposition $T$ of the hierarchical graph sequentially.
    We initialize $T$ with the \treeDecomposition of the top layer graph $\calG_\layers = \calG[\calV_\layers]$.
    We then augment $T$ with the \treeDecomposition graphs $T_v$ of $C(v)$, for each $v \in \calV_\layers$.
    We augment bags of $T_v$ with element $\{v\}$ to mark the fact that $v$ connects each node in $C(v)$ (in graph $\calG$).
    This procedure is carried out for the remaining layers $i = \layers-1,\ldots,2$ and all nodes $v \in \calV_i$.

  \Cref{fig:tree_decomposition} shows an example of execution of~\cref{algo:td-hierarchical} for the graph in~\cref{fig:tree_decomposition/new_full_graph}.
  \Cref{fig:tree_decomposition/new_panel_b} shows the (disconnected) \treeDecompositions produced by line~\ref{line:initializeT} (which produces the single green bag B1) and the first execution of line~\ref{line:doChildrenTD} (\ie for $i=\layers$, which produces the two red bags).
  \Cref{fig:tree_decomposition/new_panel_c} shows the result produced by the first execution of line~\ref{line:addNodeToBags} and line~\ref{line:addEdgeToT}, which adds B1 to all the red bags, and then connects the two \treeDecompositions with an edge, respectively. \Cref{fig:tree_decomposition/new_panel_d,fig:tree_decomposition/new_panel_e} show the result produced by the next iteration of the algorithm.

\smallheading{(iii)~Proof} Recall that a \treeDecomposition $T = (B, E)$ of a graph $\calG = (\calV, \calE)$ must satisfy the following conditions:
\begin{description}
\item[C1] $T$ must be a tree;
\item[C2] Every bag $b \in B$ must be a subset of nodes of the given graph $\calG$, \ie $b \subseteq \calV$;
\item[C3] For all edges $(u, v) \in \calE$ in the given graph $\calG$, there must exist a bag $b \in B$ such that $u, v \in b$;
\item[C4] For every node $v \in \calV$ in the given graph $\calG$, the set of bags $\{ b \in B~|~v \in b\}$ must form a connected component in $T$;
\item[C5] For every node $v \in \calV$ in the given graph $\calG$, there must exist a bag $b \in B$ that contains it, \ie $v \in b$. To state this in another way: $\cup_{b \in B} \{v \in b\} = \calV$.
\end{description}

\Cref{algo:td-hierarchical} constructs a \treeDecomposition $T$ of $\calG$ sequentially by exploring the graph, by layers (going from layer-$\layers$ to layer $1$) and by nodes $v$ in each layer. Let $\calG^{e}_t$ denote the \subgraph of $\calG$ that is explored by \cref{algo:td-hierarchical}, until iteration $t$. Here, we count an iteration to be the execution of \cref{line:addEdgeToT}, and
count the initialization at \cref{line:initializeT} as the first iteration (\ie the first iteration initializes the \treeDecomposition, while the subsequent iterations update the \treeDecomposition).

\begin{figure}[t]
    \centering
    \subfloat[\label{fig:tree_decomposition/new_full_graph}]{\includegraphics[trim=8mm 9mm 8mm 8mm,clip,width=0.9\columnwidth]{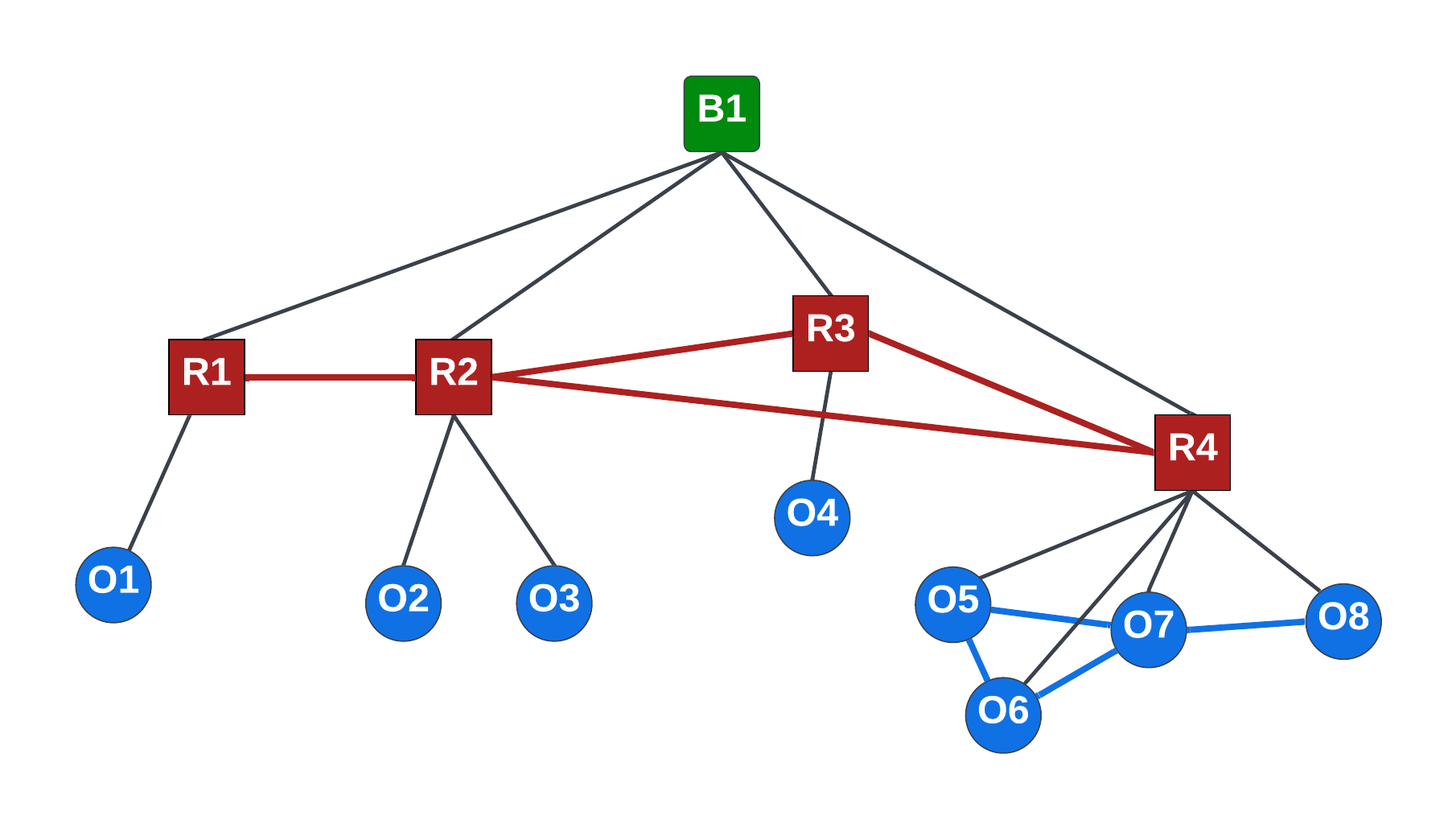}} \\
    \subfloat[\label{fig:tree_decomposition/new_panel_b}]{\includegraphics[trim=8mm 85mm 8mm 9mm,clip,width=0.5\columnwidth]{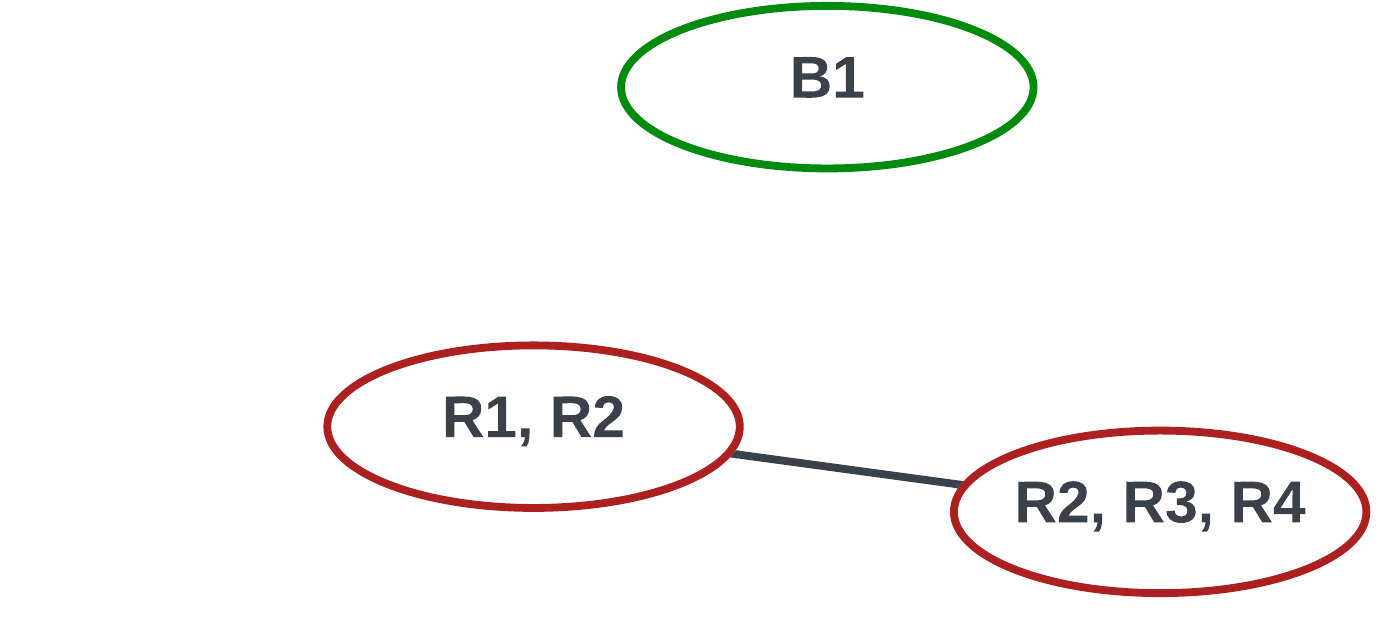}}
    \subfloat[\label{fig:tree_decomposition/new_panel_c}]{\includegraphics[trim=8mm 85mm 8mm 9mm,clip,width=0.5\columnwidth]{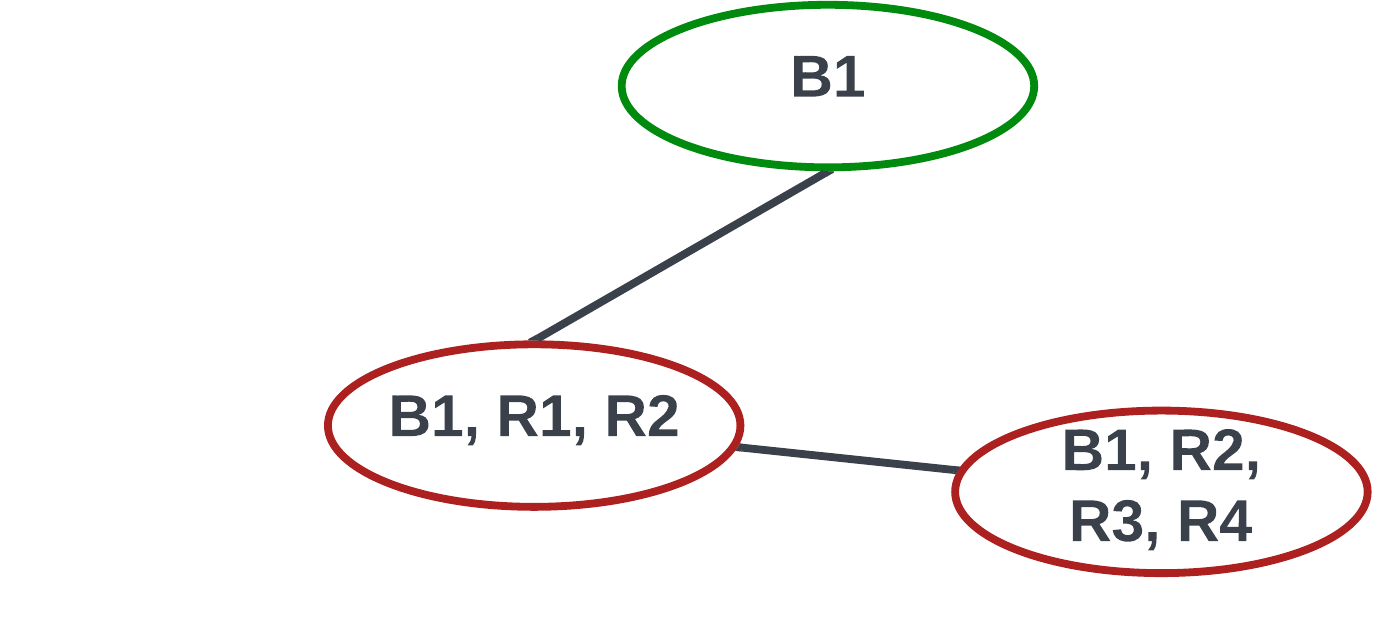}} \\
    \subfloat[\label{fig:tree_decomposition/new_panel_d}]{\includegraphics[trim=8mm 5mm 8mm 9mm,clip,width=0.5\columnwidth]{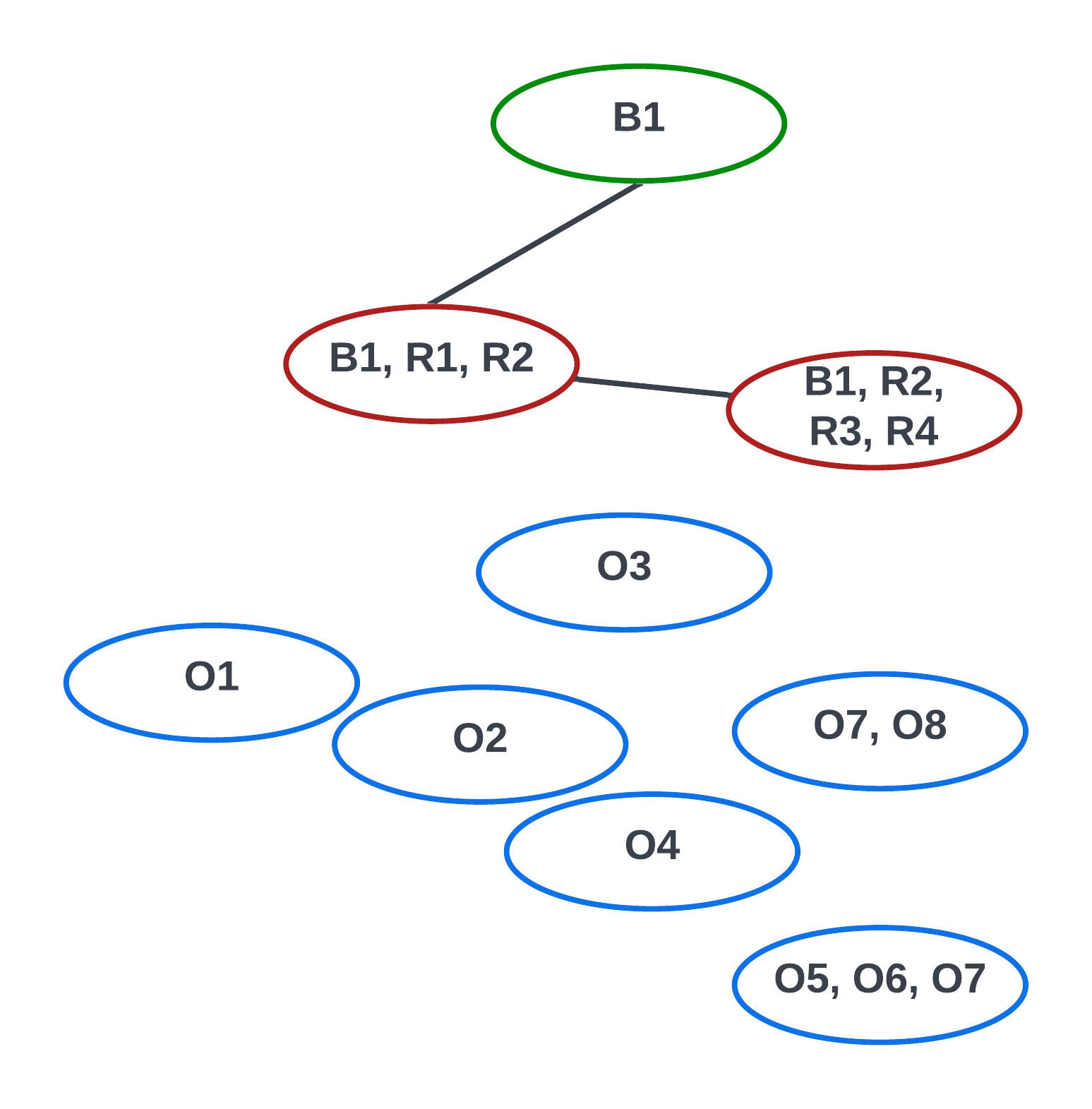}}
    \subfloat[\label{fig:tree_decomposition/new_panel_e}]{\includegraphics[trim=8mm 5mm 8mm 9mm,clip,width=0.5\columnwidth]{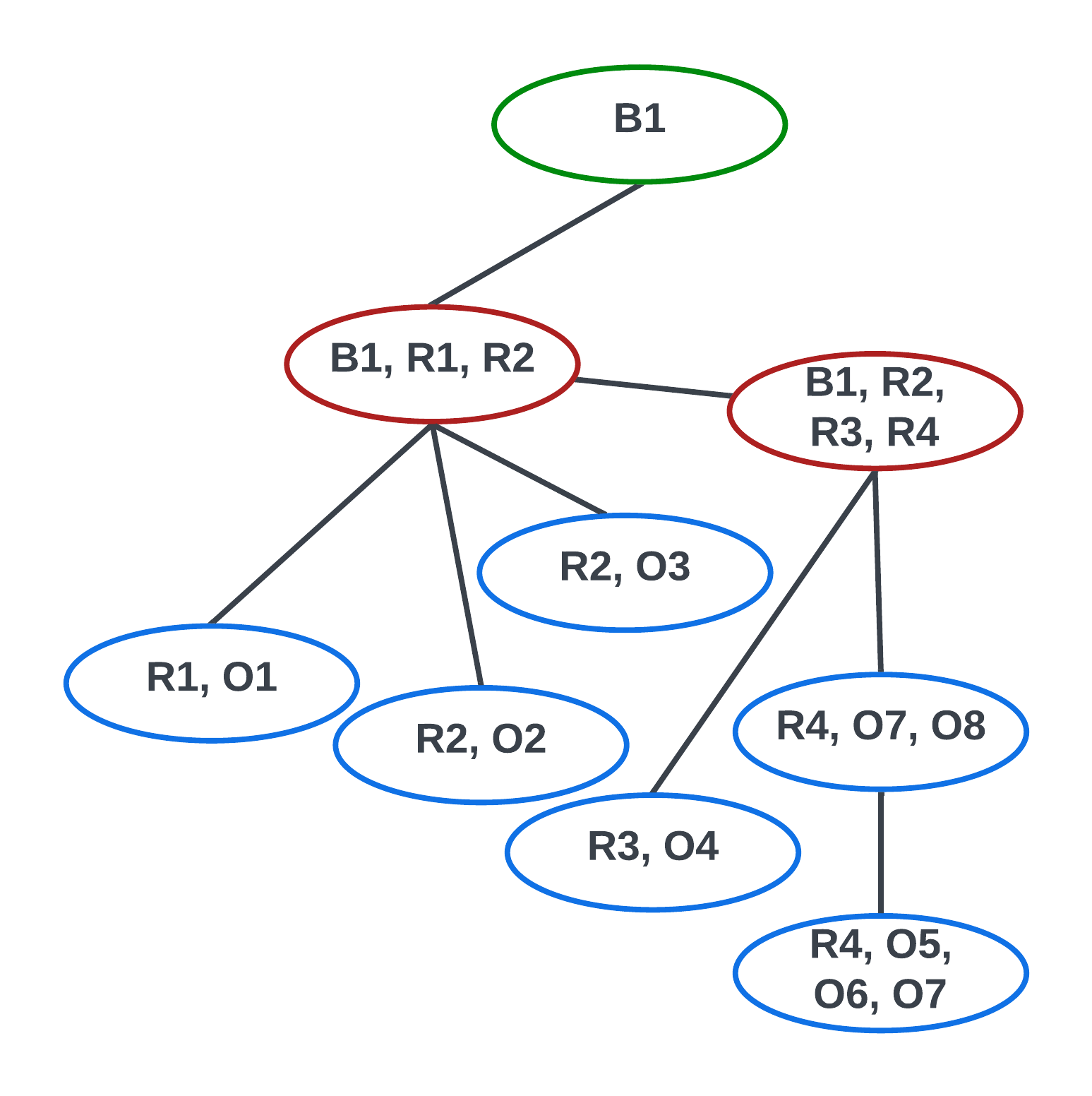}}
    \caption{(a) Hierarchical object-room-building graph with objects O1, O2, $\ldots$ O8,
    rooms R1, R2, R3, R4, and a single building B1.
             (b) Tree decomposition $T$ of the building B1 and associated $T_v$ of the rooms (\ie{} children of B1) as computed in line~\ref{line:initializeT} and~\ref{line:doChildrenTD} of \cref{algo:td-hierarchical}.
             (c) $T$ after adding B1 to every bag of the tree decomposition of the rooms and joining $T_v$ to $T$ (line~\ref{line:addNodeToBags} and line~\ref{line:addEdgeToT} in \cref{algo:td-hierarchical}).
             (d) $T$ and associated $T_v$ for the object nodes.
             (e) Final tree-decomposition of the object-room-building graph formed by the concatenation procedure described in \cref{algo:td-hierarchical}.\label{fig:tree_decomposition} }
     \togglevspace{-5mm}{0mm}
\end{figure}

We prove that, at any iteration $t$, the \treeDecomposition graph $T$ constructed in \cref{algo:td-hierarchical} is a valid \treeDecomposition of the graph $\calG^{e}_t$. We prove this by induction. At $t=1$, the explored graph $\calG^{e}_1$ is nothing but $\calG_\layers$. Note that $T$, after iteration 1 (\ie after the initialization at \cref{line:initializeT}) by definition equals the \treeDecomposition $\calG^{e}_1$. Now we only need to prove that subsequent iterations produce a \treeDecomposition of $\calG^{e}_t$.
Each new iteration (\ie execution of \cref{line:addEdgeToT}) updates $T$ to include $T^{'}_{v}$ and the edge $(b, b')$. The new explored graph is $\calG^{e}_{t+1} = \calG^{e}_{t} \oplus \calG[{C(v) + \{v\}}]$. It suffices to argue that $T \oplus T^{'}_{v} \oplus (b, b')$ (\crefrange{line:doChildrenTD}{line:addEdgeToT}) is a valid \treeDecomposition of $\calG^{e}_{t+1} = \calG^{e}_{t} \oplus \calG[C(v) + \{v\}]$. We argue this by proving that the requirements C1-C5 are satisfied.

For C1, we note that $T \oplus T^{'}_{v} \oplus (b, b')$ is a tree as $T$ and $T^{'}_{v}$ are disjoint trees and $(b, b')$ is an edge connecting the two trees, $T$ and $T^{'}_{v}$. C2 is satisfied because $T$ and $T^{'}_{v}$ are \treeDecompositions of \subgraphs of $\calG$. C3 is satisfied for all edges in $\calG^{e}_{t}$ and in $\calG[C(v)]$. The only edges that remain are those that connect node $v$ to nodes in $C(v)$, namely $\{(v, u)~|~u \in C(v)\}$.
We note that by adding $\{v\}$ to every bag of $T_v$ to obtain $T^{'}_v$ we get that all such edges are also included in the bags of $T^{'}_v$. C4 is satisfied again because $T$ and $T_{v}$ are valid \treeDecompositions and $T^{'}_{v}$ is built from $T_{v}$ by adding $v$ to all bags of $T_{v}$.
As $T_{v}$ is a tree graph disjoint from $T$, we retain property C4.
$\calG^{e}_{t+1} = \calG^{e}_{t} \oplus \calG[C(v) + \{v\}]$ (\ie{} nodes in $\calG^{e}_{t}$ and $C(v)$), have a bag in $T \oplus T^{'}_{v} \oplus (b, b')$.
\end{proof}

\Cref{thm:td-hierarchy} above showed that we can easily build a \treeDecomposition of a hierarchical graph by cleverly ``gluing together'' \treeDecompositions of \subgraphs in the hierarchy. Now we want to use that result to bound the treewidth of a hierarchical graph $\calG$.
Each node in the \treeDecomposition is typically referred to as a \emph{bag}, and intuitively represents a subset of nodes in the original graph.
The treewidth of a \treeDecomposition $T$ is defined as the size of the largest bag minus one (see \cref{sec:treeDecomposition}).
The treewidth of a graph $\calG$ is then defined as the minimum treewidth that can be achieved among all \treeDecompositions of $\Graph$.
Since \cref{algo:td-hierarchical} describes one such \treeDecompositions, \cref{thm:td-hierarchy}
 implies the following upper bound on the treewidth of a hierarchical graph.
\begin{proposition}[Treewidth of Hierarchical Graph]\label{prop:tw-hierarchy}
The treewidth of an $\layers$-layered hierarchical graph $\calG$ is bounded by the treewidths of the \subgraphs $\calG[C(v)]$ spanned by the children $C(v)$ of every node $v$ in $\calG$ and the treewidth of the graph $\calG[\calV_\layers]$ spanned by all layer-$\layers$ nodes:
\begin{equation}
\label{eq:twbound}
\text{tw}\left[\calG\right] \leq \max\left\{\max_{v \in \calV }\{ \text{tw}\left[ \calG[C(v)] \right] + 1\}, \text{tw}\left[\calG[\calV_\layers]\right]\right\}.
\end{equation}
\end{proposition}
Intuitively, the proposition states that the treewidth of a hierarchical graph does not grow with the number of nodes in the graph, but rather with the treewidth (roughly speaking, the complexity) of each layer in the graph.

\subsection{3D Scene Graphs and Indoor Environments}\label{sec:hieRep_sceneGraphs}

While the considerations in the previous sections apply to generic hierarchical representations,
in this section we tailor the discussion to a specific hierarchical representation, namely a \emph{3D scene graph}, and discuss its structure and properties when such representation is used to model indoor environments.

\myParagraph{Choice of Symbols for Indoor Navigation}
We focus on indoor environments and assume the robot is tasked with navigating to certain objects and rooms in a building. Therefore, we include the following groups of symbols that form the \emph{layers} in our hierarchical representation: free-space, obstacles, objects and agents, rooms, and building.
The details of the label spaces that we use for the object and room symbols can be found in \cref{app:labels}.
The only ``agent'' in our setup is the robot mapping the environment.
In terms of relations between symbols, we mostly consider two types of relations: inclusion (\ie the chair \emph{is in} the kitchen) and adjacency or traversability (\ie the chair \emph{is near} the table, the kitchen is adjacent to/reachable from the dining room).

As we mentioned in \cref{sec:hieRep_scalability}, this choice of symbols and relations is dictated by the tasks we envision the robot to perform and therefore it is not universal. For instance, other tasks might require breaking down the free-space symbol into multiple symbols (\eg to distinguish the front from the back of a room), or might require considering other agents (\eg humans moving in the environment as in~\cite{Rosinol21ijrr-Kimera,Rosinol20rss-dynamicSceneGraphs}).
This dependence on the task also justifies the different definitions of 3D scene graph found in the recent literature:
for instance, the original proposal in~\cite{Armeni19iccv-3DsceneGraphs} focuses on visualization and human-machine interaction tasks, rather than robot navigation, hence the corresponding scene graphs do not include the free-space as a symbol; the proposals~\cite{Wu21cvpr-SceneGraphFusion,Kim19tc-3DsceneGraphs} consider smaller-scale tasks and disregard the room and building symbols.
Similarly, the choice of relations is task dependent and can include a much broader set of relations beyond inclusion and adjacency. For instance, relations can describe attributes of an object (\eg ``has color''), material properties (\eg ``is made of''), can be used to compare entities (\eg ``is taller than''), or may encode actions (\eg a person ``carries'' an object, a car ``drives on'' the road). While these other relations are beyond the scope of our paper, we refer the reader to~\cite{Zhu22arxiv-SceneGraphSurvey} for further discussion.

\myParagraph{Choice of Sub-symbolic Representations}
We ground each symbol in compact sub-symbolic representations; intuitively, this reduces to attaching geometric attributes to each symbol observed by the robot.
We ground the ``obstacle'' symbol into a 3D mesh describing the observed surfaces in the environment.
We ground the ``free-space'' symbol using a places \subgraph, which can be understood as a topological map of the environment. Specifically, each place node in the graph of places is associated with an obstacle-free 3D location (more precisely, a sphere of free space described by a centroid and radius), while edges represent traversability between places.\footnote{While we postpone the technical details to \cref{sec:places}, the graph of places can be understood as a sparse approximation of a Generalized Voronoi Diagram (GVD) of the environment, a data structure commonly used in computer graphics and computational geometry to compress shape information~\citep{Tagliasacchi16cgf-3dSkeletonsSurvey}.}
We ground the ``agent'' symbol using the pose-graph describing the robot trajectory~\citep{Cadena16tro-SLAMsurvey}.
We ground each  ``object'', ``room'', and ``building'' symbol using a centroid and a bounding box.
Somewhat redundantly, we also store the mesh vertices corresponding to each object, and the set of places included in each room, which can be also understood as additional groundings for these symbols.
As discussed in the next section, this is mostly a byproduct of our algorithms, rather than a deliberate choice of sub-symbolic representation.

{\bf 3D Scene Graphs for Indoor Environments}
The choices of symbolic and sub-symbolic representations outlined above lead to the
3D scene graph structure visualized in \cref{fig:intro}.
In particular,
\emph{Layer~1} is a metric-semantic 3D mesh.
\emph{Layer~2} is a \subgraph of objects and agents; each object has a semantic label (which identifies the symbol being grounded), a centroid, and a bounding box (providing the grounding);
 each agent is modeled by a pose graph describing its trajectory (in our case the robot itself is the only agent).
\emph{Layer~3} is a \subgraph of \emph{places} (\ie{} a topological map) where each place is an obstacle-free location and an edge between places denotes straight-line traversability.
\emph{Layer~4} is a \subgraph of rooms where each room has a centroid, and edges connect adjacent rooms.
\emph{Layer~5} is a building node connected to all rooms (we assume the robot maps a single building). Edges connect nodes within each layer (\eg{} to model traversability between places or rooms, or connecting nearby objects) or across layers (\eg{} to model that mesh vertices belong to an object, that an object is in a certain room, or that a room belongs to a building).
Note that the 3D scene graph structure forms a hierarchical graph meeting the requirements of \cref{def:hieGraph}.

\myParagraph{Treewidth of Indoor 3D Scene Graphs}
In \cref{sec:hieRep_inference} we concluded that the treewidth of a hierarchical graph is bounded by the treewidth of its layers. In this section, we particularize the treewidth bound  to 3D scene graphs modeling indoor environments.
We first prove bounds on the treewidth of key layers in the 3D scene graph (\cref{lem:tw-room,lem:tw-object}). Then we translate these bounds into a bound
on the treewidth of the object-room-building \subgraph of a 3D scene graph (\cref{cor:tw-orb}); the bound is important in practice since we will need to perform inference over such a \subgraph to infer the room (and potentially building) labels, see \cref{sec:rooms_classification}.

We start by bounding the treewidth of the room layer.
\begin{lemma}[Treewidth of Room Layer]\label{lem:tw-room}
Consider the \subgraph of rooms, including the nodes in the room layer of a 3D scene graph and the corresponding edges.
If each room connects to at most two other rooms that have doors (or more generally passage ways) leading to other rooms, then the room graph has treewidth bounded by two.
\end{lemma}
The insight behind \cref{lem:tw-room} is that the room \subgraph is not very far from a tree (\ie it has low treewidth) as long as there are not many rooms with multiple doors (or passage ways). In particular, the room \subgraph has treewidth equal to 2 if each room has at most two passage ways to other rooms with multiple entries.
Note that the theorem statement allows a room to be connected to an arbitrary number of rooms with a single entry (\ie a corridor leading to multiple rooms that are only accessible via that corridor). \Cref{fig:room-graph-tw-rooms} reports empirical upper-bounds of the treewidth of the room \subgraphs in the 90 scenes of the Matterport3D dataset~\citep{Chang173dv-Matterport3D} (see \cref{sec:experiments} for more details about the experimental setup).
The observed treewidth is at most 2, confirming that the conclusions of \cref{lem:tw-room} hold in several apartment-like environments.

\begin{figure}[t]
    \centering
    \subfloat[\label{fig:room-graph-tw-rooms}]{\includegraphics[width=4.4cm]{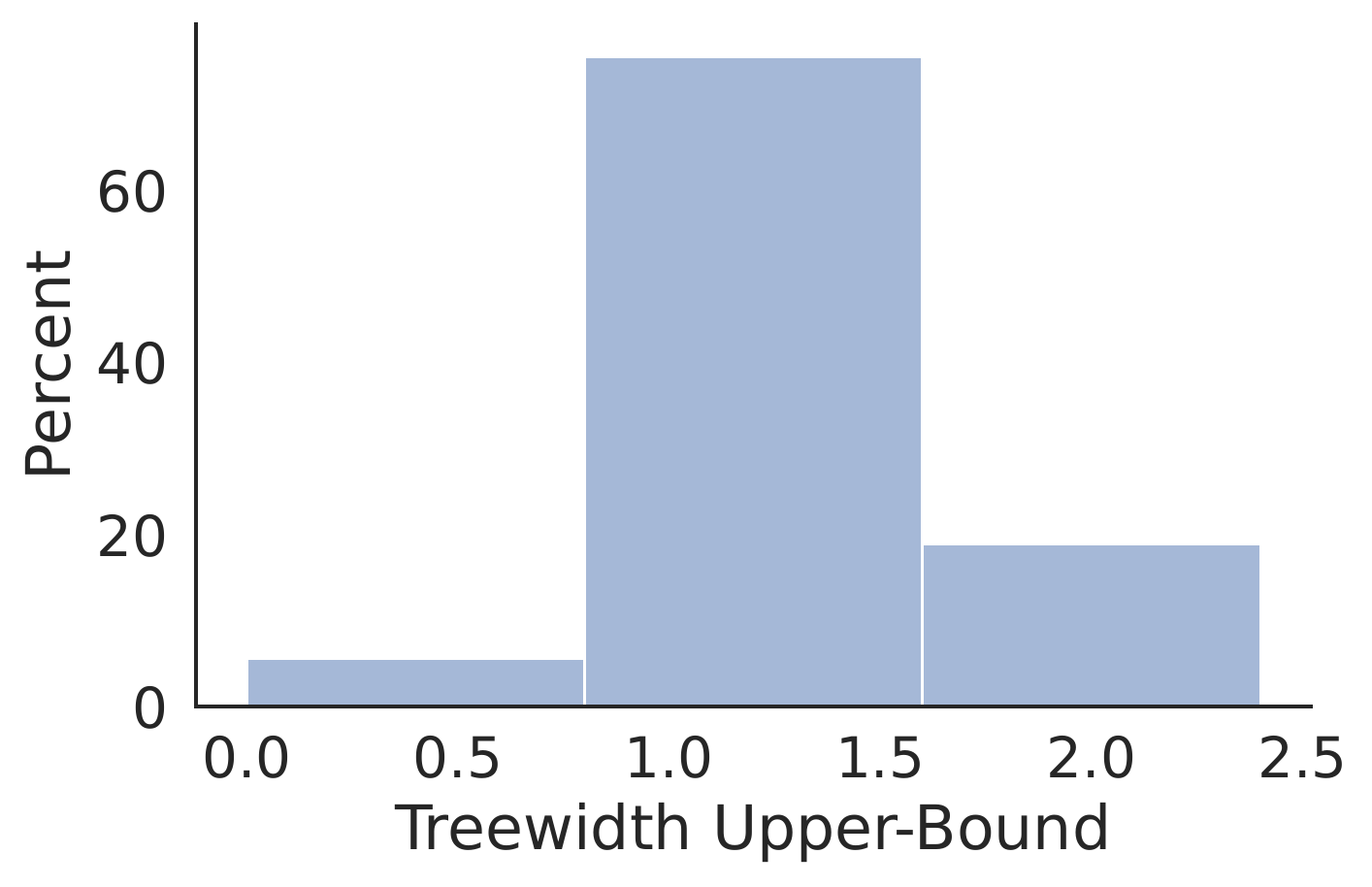}}
    \subfloat[\label{fig:room-graph-tw-object}]{\includegraphics[width=4.4cm]{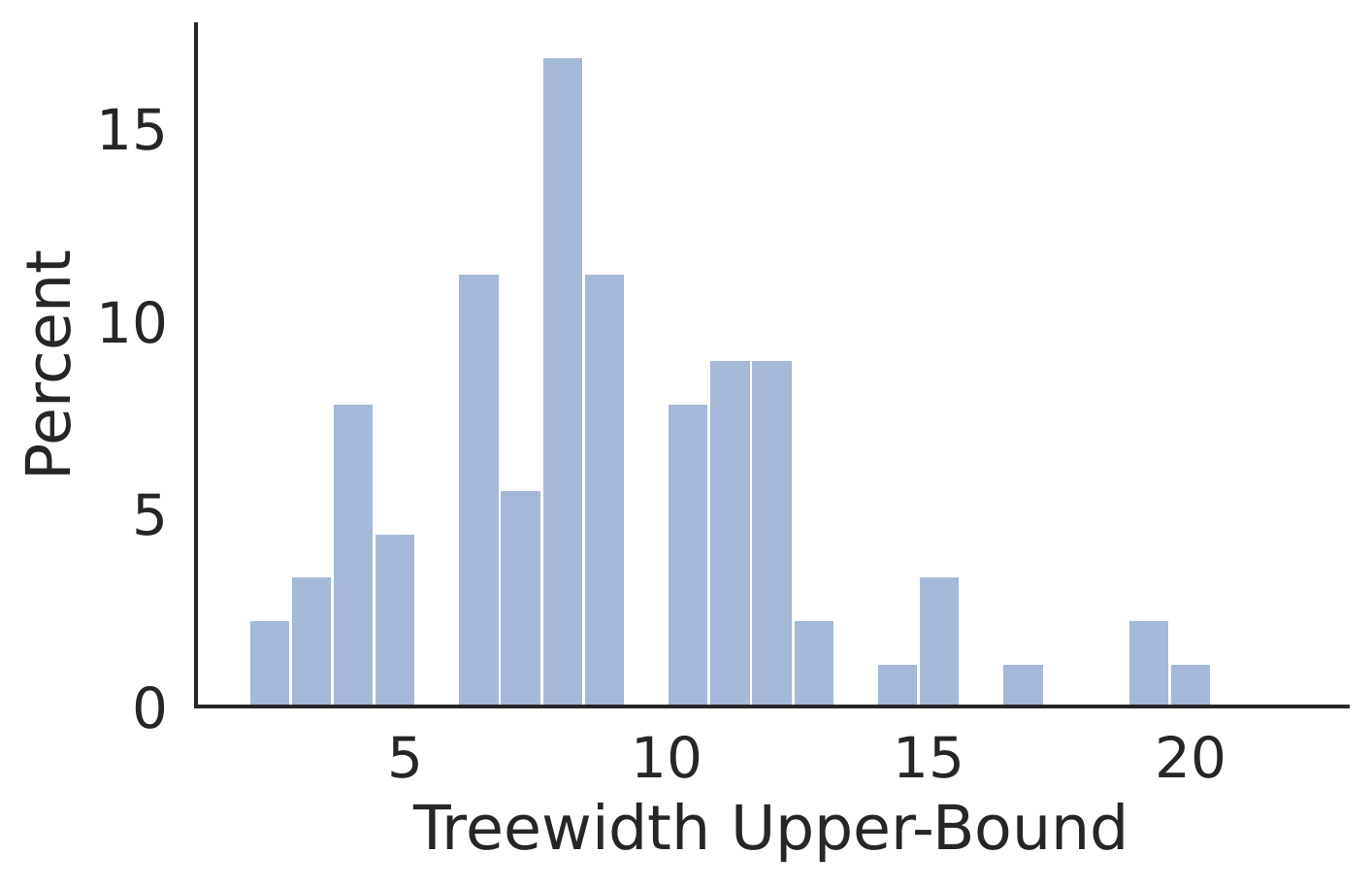}}
    \caption{Histogram of the {treewidth upper-bounds} for (a) the room \subgraph and (b) the object \subgraph of the 3D scene graphs obtained using the 90 scenes of the Matterport3D dataset~\cite{Chang173dv-Matterport3D}. We use the minimum-degree and the minimum-fill-in heuristic to obtain treewidth upper-bounds~\citep{Bodlaender2010j-IC-TreewidthComputationI}. We then compute an upper-bound to be the lowest of the two.
}
\end{figure}

\begin{lemma}[Treewidth of Object Layer]
    \label{lem:tw-object}
Consider the \subgraph of objects, which includes the nodes in the object layer of a 3D scene graph and the corresponding edges.
The treewidth of the object \subgraph is bounded by the maximum number of objects in a room.
\end{lemma}

The result is a simple consequence of the fact that in our 3D scene graph, there is no edge connecting objects in different rooms, therefore, the graph of objects includes disconnected components corresponding to each room, whose treewidth is bounded by the size of that connected component, \ie the number of objects in that room.
\Cref{fig:room-graph-tw-object} reports the treewidth upper-bounds for the object \subgraphs in the Matterport3D dataset.
We observe that the treewidth of the object \subgraphs tends to be larger compared to the room \subgraphs, but
 still remains relatively small (below 20) in all tests.

We can now conclude with a bound on the treewidth of the object-room-building graph in a 3D scene graph.
\begin{proposition}[Treewidth of the Object-Room-Building Graph]\label{cor:tw-orb}
Consider the object-room-building graph of a building, including the nodes in the object, room, and building layers of a 3D scene graph and the corresponding edges.
Assume the treewidth of the room graph is less than the treewidth of the object graph. Then, the treewidth $\text{tw}[\calG]$ of the object-room-building graph  $\calG$ is bounded by
\begin{equation}
\text{tw}[\calG] \leq 1 + N_o,
\end{equation}
where $N_o$ denotes the largest number of objects in a room.
\end{proposition}

\Cref{cor:tw-orb} indicates that the treewidth of the object-room-building graph does not grow with the size of the scene graph, but rather depends on how cluttered each room is. This is in stark contrast with the treewidth of social network graphs~\citep{Maniu19icdt-ExperimentalStudy,Talak21neurips-neuralTree}, and further motivates the proposed hierarchical organization. The treewidth bounds in this section  open the door to tractable inference techniques; in particular, in \cref{sec:rooms_classification}, we show they allow applying novel graph-learning techniques, namely, the neural tree~\citep{Talak21neurips-neuralTree}.

\section{Real-time Incremental \\ 3D Scene Graph Layers Construction}\label{sec:incrementalLayers}

This section describes how to estimate an \emph{odometric} 3D scene graph directly from visual-inertial data as the robot explores an unknown environment.
\cref{sec:LCD-and-SGO} then shows how to correct the scene graph in response to loop closures.

We start by reconstructing a metric-semantic 3D mesh to populate \emph{Layer~1} of the 3D scene graph (\cref{sec:mesh}), and use the mesh to extract the objects in \emph{Layer~2} (\cref{sec:objects}).
We extract the places in~\emph{Layer~3} as a byproduct of the 3D mesh computation by first computing a Generalized Voronoi Diagram (GVD) of the environment, and then approximating the GVD as a sparse graph of places (\cref{sec:places}).
We then populate the rooms  in~\emph{Layer~4} by segmenting their geometry using persistent homology (\cref{sec:rooms_clustering}) and assigning them a semantic label using the neural tree (\cref{sec:rooms_classification}).
In this paper, we assume that the scene we reconstruct consists of a single building, and for~\emph{Layer~5} we instantiate a single building node that we connect to all estimated room nodes.

\subsection{Layers 1: Mesh}\label{sec:mesh}

We build a metric-semantic 3D mesh (\emph{Layer~1} of the 3D scene graph) using Kimera~\citep{Rosinol20icra-Kimera} and Voxblox~\citep{Oleynikova17iros-voxblox}.
In particular, we maintain a metric-semantic voxel-based map within an active window around the robot.
Each voxel of the map contains free-space information and a distribution over possible semantic labels.
We gradually convert this voxel-based map into a 3D mesh using marching cubes, attaching a semantic label to each mesh vertex.
This is now standard practice (\eg{} the same idea was used in~\cite{Whelan12rgbd} without inferring the semantic labels) and the expert reader can safely skip the rest of this subsection.

In more detail, for each keyframe\footnote{Keyframes are selected camera frames chosen by the visual-inertial odometry pipeline, see \cref{sec:objects} and~\cite{Rosinol21ijrr-Kimera} for further details.} we use a 2D semantic segmentation network to obtain a pixel-wise semantic segmentation of the RGB image, and  reconstruct a depth-map using stereo matching (when using a stereo camera) or from the depth channel of the sensor (when using an RGB-D camera).
We then convert the semantic segmentation and depth into a semantically-labeled 3D point cloud and transform it according to the odometric estimate of the robot pose (\cref{sec:objects}).
We use Voxblox~\citep{Oleynikova17iros-voxblox} to integrate the semantically-labeled point cloud into a Truncated Signed Distance Field (\TSDF{}) using ray-casting, and Kimera~\citep{Rosinol21ijrr-Kimera} to perform Bayesian updates over the semantic label of each voxel during ray-casting.
Both Voxblox~\citep{Oleynikova17iros-voxblox} and Kimera~\citep{Rosinol21ijrr-Kimera} operate over an \emph{active window}, \ie{} only reconstruct a voxel-based map
within a user-specified radius $r_a$ around the robot ($r_a = \SI{8}{\meter}$ in our tests)\footnote{The radius of the active window has to be larger than the maximum ray-casting distance to construct the \TSDF{} (\SI{\approx4}{\meter}) and the block resolution of the spatial hashing (\SI{\approx1.5}{\meter}).}  to bound the memory requirements.
Within the active window, we extract the 3D metric-semantic mesh using Voxblox' marching cubes implementation, where each mesh vertex is assigned the most likely semantic label from the corresponding voxel.
We then use spatial hashing~\citep{Niessner2013acm-real} to integrate $\calM_a$, the mesh inside the active window, into the full mesh $\calM_f$, optionally also compressing the mesh to a given resolution.
The use of the active window circumvents the need to maintain a monolithic and memory-hungry voxel-based representation of the entire environment (as done in our original proposal in~\cite{Rosinol20icra-Kimera}), and instead allows to gradually transform the voxel-based map in the active window (moving with the robot) into a lighter-weight 3D mesh.
The full mesh, $\calM_f$, is uncorrected for odometric drift, and we address loop closure detection and optimization in \cref{sec:LCD-and-SGO}.

\subsection{Layers 2: Objects and Agents}
\label{sec:objects}

\myParagraph{Agents} The agent layer consists of the pose graph describing the robot trajectory (we refer the reader to~\cite{Rosinol21ijrr-Kimera} for extensions to multiple agents, including humans).
During exploration, the odometric pose graph is obtained using stereo or RGB-D visual-inertial odometry, which is also available in Kimera~\citep{Rosinol20icra-Kimera,Rosinol21ijrr-Kimera}.
 The poses in the pose graph correspond to the keyframes selected by Kimera, and for each pose we also store visual features and descriptors, which are used for loop closure detection in \cref{sec:LCD-and-SGO}. As usual, edges in the pose graph correspond to odometry measurements between consecutive poses, while we will also add loop closures as described in \cref{sec:LCD-and-SGO}.

\myParagraph{Objects}
 The object layer consists of a graph where each node corresponds to an object (with a semantic label, a centroid, and a bounding box) and edges connect nearby objects.
After extracting the 3D mesh within the active window, we segment objects by performing Euclidean clustering~\citep{Rusu09thesis-SemanticMaps} on $\calV_a$, the vertices of $\mathcal{M}_a$.
In particular, we independently cluster the sets of vertices $\calV_a^i$ with the same semantic label $i$ (where $\calV_a^i = \setdef{v \in \calV_a}{\text{label}(v) = i}$) for all semantic labels of interest.
As in Kimera~\citep{Rosinol21ijrr-Kimera}, each resulting cluster is then used to estimate a centroid and bounding box for each putative object.
After each update, if a newly-detected object overlaps with an existing object node of the same semantic class in the scene graph, we merge them together by adding new mesh vertices to the previous object node; otherwise we add the new object as a new node.
In practice, we consider two objects overlapping if the centroid of one object is contained in the other object's bounding box, a proxy for spatial overlap measures such as Intersection over Union (IoU).\footnote{In previous work~\citep{Rosinol21ijrr-Kimera}, we also detected objects with known shape using 3D registration~\citep{Yang20tro-teaser}; in practice, we see that such approach only works well for large objects since small objects are not well-described by the 3D mesh, whose resolution is implicitly limited by the resolution of the voxel-based map in the active window (\SI{10}{\cm} in our tests). A more promising approach is to directly ground small 3D objects (\eg a pen) from sensor data as in~\cite{Talak23rss-ensemble}.}

\begin{figure}
    \centering
    \subfloat[GVD and 3D Mesh]{\centering
\includegraphics[width=0.48\columnwidth,trim=10cm 0cm 6cm 0cm,clip]{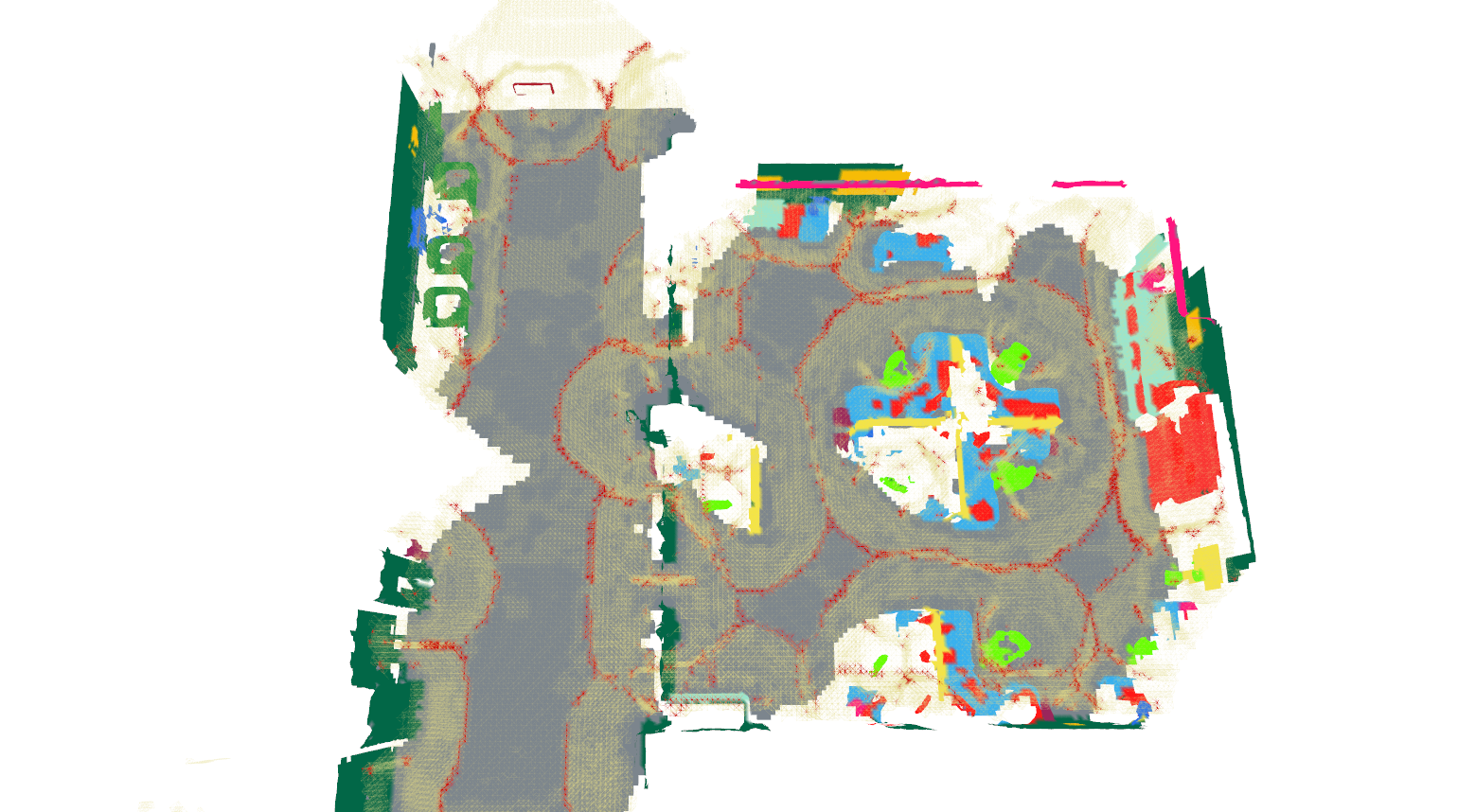}\label{fig:gvd}
    }
    \subfloat[Free-space Representation]{\centering
\includegraphics[width=0.48\columnwidth,trim=10cm 0cm 6cm 0cm,clip]{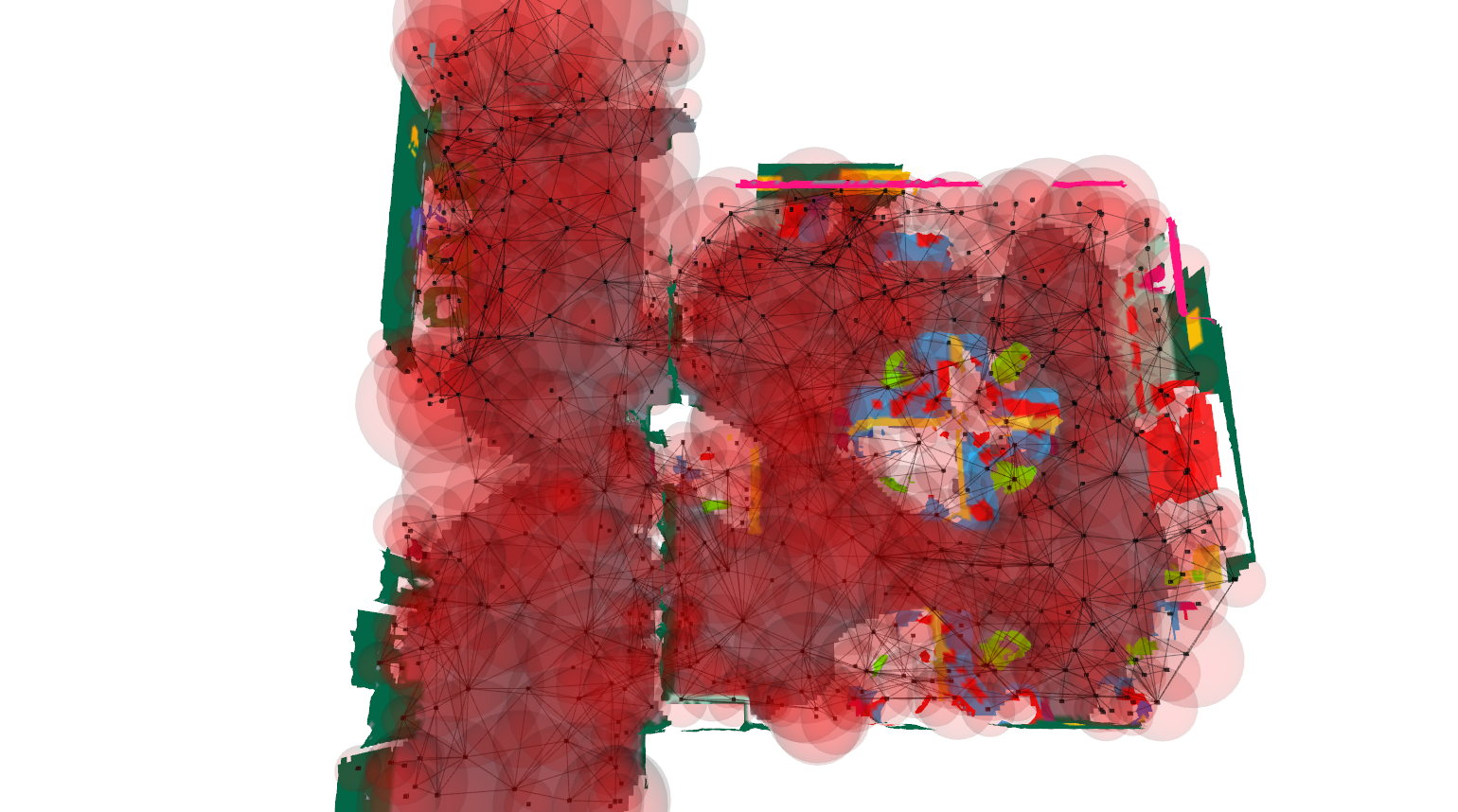}\label{fig:gvd_2}
    }
    \caption{(a) GVD (wireframe) of an environment with GVD voxels having three or more basis points highlighted.
             (b) Sparse places graph (in black) and associated spheres of free-space corresponding shown in red. Note that the union of the spheres roughly approximates the geometry of the free-space.}
     \togglevspace{-5mm}{0mm}
\end{figure}

\subsection{Layers 3: Places}\label{sec:places}

We build the places layer by computing a Generalized Voronoi Diagram (GVD) of the environment, and then sparsifying it into a graph of places.
While the idea of sparsifying the GVD into a graph has appeared in related work~\citep{Rosinol21ijrr-Kimera,Oleynikova18iros-topoMap}, these works extract such representations from a monolithic \ESDF{} of the environment, a process that is computationally expensive and confined to off-line use (\cite{Rosinol21ijrr-Kimera} reports computation times in the order of tens of minutes).
Instead, we combine and adapt the approaches of Voxblox~\citep{Oleynikova17iros-voxblox}, which presents a incremental version of the brushfire algorithm that converts a \TSDF{} to an \ESDF{}, and of \citet{Lau13ras-efficientGridRepresentations}, who present an incremental version of the brushfire algorithm that is capable of constructing a \GVD{} during the construction of the \ESDF{}, but takes as input a 2D occupancy map.
As a result, we show how to obtain a local GVD and a graph of places as a byproduct of the 3D mesh construction.

\myParagraph{From Voxels to the \GVD}
The \GVD{} is a data structure commonly used in computer graphics and computational geometry to compress shape information~\citep{Tagliasacchi16cgf-3dSkeletonsSurvey}.
The GVD is the set of voxels that are equidistant to at least two closest obstacles (``basis points'' or ``parents''), and intuitively forms a skeleton of the environment~\citep{Tagliasacchi16cgf-3dSkeletonsSurvey} (see \cref{fig:gvd}).
Following the approach in~\cite{Lau13ras-efficientGridRepresentations}, we use the fact that
 voxels belonging to the GVD can be easily detected from the wavefronts of the brushfire algorithm used to create and update the \ESDF{} in the active window.
Intuitively, \GVD{} voxels have the property that multiple wavefronts ``meet'' at a given voxel (\ie{} fail to update the signed distance of the voxel), as these voxels are equidistant from multiple obstacles.
The brushfire algorithm (and its incremental variant) are traditionally seeded at voxels containing obstacles (\eg{}~\cite{Lau13ras-efficientGridRepresentations}), but as mentioned previously, we follow the approach of~\cite{Oleynikova17iros-voxblox} and use the \TSDF{} to seed the brushfire wavefronts instead.
In particular, we track all \TSDF{} voxels that correspond to zero-crossings (\ie~voxels containing a surface) when creating $\mathcal{M}_a$ using marching cubes and then use these voxels as a surrogate occupancy grid (\ie{} wavefronts of the brushfire algorithm are treated as if they started from one of these \TSDF{} voxels).
Like~\cite{Oleynikova17iros-voxblox}, we also use the \TSDF{} voxels within the truncation distance to seed the wavefronts of the brushfire algorithm, \ie{} instead of starting the wavefronts directly from occupied voxels, we start them from the outermost voxels still within the truncation distance.
In our implementation, we compute a variant of the \GVD{} called the \emph{$\theta$-simplified medial axis} ($\theta$-SMA) from~\cite{Foskey03jcise-SMA} that filters out less stable and noisier portions of the \GVD{} using a threshold on the minimum angle between basis points.
Each voxel in the GVD is assigned a distance to its basis points, defining a sphere of free-space surrounding the voxel; collectively the \GVD{} provides a compact description of the free-space and its connectivity (\cref{fig:gvd_2}).
Indeed, up to discretization, the \GVD{} provides an \emph{exact} description of the free-space~\citep{Tagliasacchi16cgf-3dSkeletonsSurvey}, meaning that one can reconstruct the full shape of the free-space from its \GVD.

\begin{algorithm}[t]
    \SetArgSty{textrm}
    \DontPrintSemicolon
    \caption{Places Sparsification}\label{algo:places}
    \KwData{GVD graph $\mathcal{G}_a$, previous map of spatial hashes to voxel clusters $C$, spatial resolution $\delta_p$}
    \KwResult{New or updated nodes $\mathcal{N}$ and edges $\mathcal{E}$}
    $\mathcal{N} \gets \emptyset$\;
    $\mathcal{E} \gets \emptyset$\;
    $\mathcal{V}_u \gets \setdef{\text{\emph{voxel}}\ v \in \mathcal{G}_a}{\text{updated}(v) = \text{true}}$\;
    \For{\emph{voxel} $v \in \mathcal{V}_u$}{$h_v \gets \text{hash}(v, \delta_p)$\;\label{line:node_start}
        \textit{found} $\gets$ false\;
        \For{\emph{cluster} $c \in C[h_v]$}{\If{$v$ \emph{neighbors any voxel in cluster} $c$}{$c \gets c \cup \{v\}$\;
                \textit{found} $\gets$ true\;
                \textbf{break}\;
            }
        }

        \If{$\lnot$\textit{found}}{$C[h_v] \gets C[h_v] \cup \{v\}$\;\label{line:node_end}
        }

        \BlankLine{}

        \tcc{add node corresponding to cluster of $v$ to set of returned nodes}
        $\mathcal{N} \gets \mathcal{N} \cup \{\text{node}(v)\}$\;

        \BlankLine{}

        \For{\emph{voxel} $n \in \text{neighbors}(\mathcal{G}_a, v)$}{\label{line:edge_start}$h_n = \text{hash}(n, \delta_p)$\;
            \uIf{$\text{node}(v) \ne \text{node}(n) \land h_v = h_n$}{$\mathcal{N} \gets \mathcal{N} \setminus \{\text{node(n)}\}$\;
                $\mathcal{E} \gets \mathcal{E} \setminus \setdef{e \in \mathcal{E}}{\text{node(n)} \in e}$\;
                Merge clusters for nodes $n$ and $v$\;
            }
            \ElseIf{$\text{node}(v) \ne \text{node}(n)$}{$\mathcal{E} \gets \mathcal{E} \cup \{(\text{node}(v), \text{node}(n))\}$\;
            }\label{line:edge_end}
        }
    }
\end{algorithm}

\myParagraph{From the \GVD{} to the Places Graph}
While the \GVD{} already provides a compressed representation of the free-space, it still typically contains
 a large number of voxels (\eg{} more than $10,000$ in the Office dataset considered in \cref{sec:experiments}).
We could instantiate a graph of places with one node for each \GVD{} voxel and edges between nearby voxels, but such representation would be too large to manipulate efficiently (\eg{} our room segmentation approach in \cref{sec:rooms_clustering} would not scale to operating on the entire \GVD{}).
Previous attempts at sparsifying the \GVD{} (notably~\cite{Oleynikova18iros-topoMap} and our earlier paper~\cite{Hughes22rss-hydra}) used a subset of topological features (edges and vertices in the \GVD{}, which correspond to \GVD{} voxels having more than three and more than four basis points respectively) to form a sparse graph of places.
However, the resulting graph does not capture the same connectivity of free-space as the full \GVD{}, and may lead to graph of places with multiple connected components.
It would also be desirable for the user to balance the level of compression with the quality of the free-space approximation instead of always picking certain \GVD{} voxels.

To resolve these issues, we first form a graph $\mathcal{G}_a$ of the \GVD{} where each \GVD{} voxel corresponds to a node in $\mathcal{G}_a$ and edges connect any two neighboring \GVD{} voxels.
We then use voxel-hashing~\citep{Niessner2013acm-real} to spatially cluster ---{at a user-specified resolution}--- all updated \GVD{} voxels
{with at least $n_b$ basis points} ($n_b=2$ in our tests) inside the active window, ensuring that each cluster forms a connected component in $\mathcal{G}_a$.
This clustering is shown in \cref{algo:places}.
This algorithm takes as input the graph of \GVD{} voxels, $\mathcal{G}_a$, and for every voxel that was updated, computes a spatial hash value~\citep{Niessner2013acm-real}, denoted $\text{hash}(v, \delta_p)$ in the algorithm for a given voxel $v$ and spatial resolution $\delta_p$.
This hash value is the same for voxels within a cube of 3D space with side-lengths of $\delta_p$ and naturally clusters the voxels of the \GVD{} into clusters of the provided spatial resolution.
\crefrange{line:node_start}{line:node_end} incrementally grow clusters of voxels with the same hash value.
However, this results in multiple clusters on the same connected component of $\mathcal{G}_a$.
\crefrange{line:edge_start}{line:edge_end} then search over the neighbors of a given voxel $v$ in $\mathcal{G}_a$ and either combines clusters that share an edge in $\mathcal{G}_a$ and the same spatial hash value, or propose an edge between the clusters if the hash value differs.
Once \cref{algo:places} terminates, we obtain a set of proposed nodes (corresponding to unique clusters of voxels in $\mathcal{G}_a$ that form connected components) and proposed edges based on the connectivity of $\mathcal{G}_a$.

These proposed nodes and edges are assigned information from the voxels that generated them:
each new or updated node is assigned a position and distance to the nearest obstacle from the \GVD{} voxel with the most basis points in the cluster;
each new or updated edge is assigned a distance that is the minimum distance to the nearest obstacle of all the \GVD{} voxels that the edge passes through.
We then associate place nodes with the corresponding mesh-vertices of each basis point using the zero-crossing identified in the \TSDF{} by marching cubes.
We run \cref{algo:places} after every update to the \GVD{} and merge the proposed nodes $\mathcal{N}$ and edges $\mathcal{E}$ into the sparse graph of places.\footnote{Separate book-keeping is done to remove nodes and edges that no longer correspond to any \GVD{} voxels during the brushfire update of the \ESDF{} and \GVD{}, as \GVD{} voxels are removed during this update when a wavefront is able to modify the signed distance of the voxel}

\myParagraph{Inter-layer Edges}
After building the places graph, we add inter-layer edges from each object or agent node to the nearest place node in the active window using nanoflann~\citep{Blanco14-nanoflann}.

\subsection{Layer 4: Rooms}\label{sec:rooms}

We segment the rooms by first geometrically clustering the graph of places into separate rooms (using persistent homology, \cref{sec:rooms_clustering}), and then assigning a semantic label to each room (using the neural tree, \cref{sec:rooms_classification}).
We remark that the construction of the room layer is fundamentally different from the construction of the object layer.
While we can directly observe the objects (\eg{} via a network trained to detect certain objects), we do not have a network trained to detect certain rooms.
Instead, we have to rely on prior knowledge to infer both their geometry and semantics.
For instance, we exploit the fact that rooms are typically separated by small passageways (\eg{} doors), and that their semantics can be inferred from the objects they contain.

\begin{figure}
    \centering
    \includegraphics[width=0.69\columnwidth]{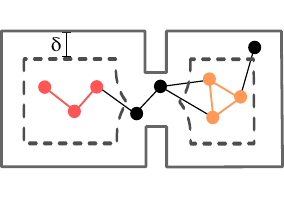}
     \togglevspace{-5mm}{0mm}
    \caption{Connected components (in orange and red) in the \subgraph of places induced by a dilation distance $\delta$ (walls are in gray, dilated walls are dashed, places that disappear after dilation are in black). \label{fig:rooms}}
    \togglevspace{-5mm}{0mm}
\end{figure}

\subsubsection{Layer 4: Room Clustering via Persistent Homology}\label{sec:rooms_clustering}

This section discusses how to geometrically cluster the environment into different rooms.
Many approaches for room segmentation or detection (including~\cite{Rosinol21ijrr-Kimera}) require volumetric or voxel-based representations of the full environments and make assumptions on the environment geometry (\ie{} a single floor or ceiling height) that do not easily extend to arbitrary buildings.
To resolve these issues, we present a novel approach for constructing \emph{Layer 4} of the 3D scene graph
by clustering the nodes in the graph of places $\placeGraph$  (\emph{Layer 3}); this approach does not assume planarity of the environment and circumvents the need to process large voxel-based maps.
Our room segmentation approach consists of two stages: the first stage identifies the number of rooms in the environment (as well as a subset of places associated with each room), and the second stage assigns the remaining places in $\placeGraph$ to rooms via flood-fill.

\myParagraph{Identifying Rooms via Persistent Homology}
We use the key insight that dilating the voxel-based map helps expose rooms in the environment:
if we inflate obstacles, apertures in the environment (\eg{} doors) will gradually close, naturally partitioning the voxel-based map into disconnected components (\ie{} rooms); however, we apply the same insight to the graph of places $\placeGraph$ (rather than the voxel-based map) to enable faster computation and better scalability.
As each node and edge in our place graph $\placeGraph$ stores a distance to its closest obstacle $\dplace{}$ (\cref{sec:places}), dilation operations in the voxel-based map can be directly mapped into topological changes in $\placeGraph$.
More precisely, if we dilate the map by a distance $\delta$, every place node or edge with $\dplace{}$ smaller than $\delta$ will disappear from the graph since it will no longer be in the free space.
We call the dilated graph $\placeGraph^{\delta}$; see \cref{fig:rooms} for an illustration.

Directly specifying a single best dilation distance $\delta^*$ to extract putative rooms is unwieldy: door and other aperture widths vary from scene to scene (think about the small doors in an apartment versus large passage ways in a mansion).
In order to avoid hard-coding a single dilation distance, we use a tool from topology, known as
\emph{persistent homology}~\citep{Ali22arxiv-tdaSurvey,Huber21idsc-persistent}, to automatically compute the best dilation distance for a given graph.\footnote{Persistent homology is the study of the birth and death of homologies across a \emph{filtration} of a simplicial complex; see~\cite{Ali22arxiv-tdaSurvey,Huber21idsc-persistent} for an introduction on the topic.
In this work, we restrict ourselves to graphs (a subset of simplicial complexes) and focus only on 0-homologies (connected components).}
The basic insight is that the most natural choice of rooms is one that is more stable (or \emph{persistent}) across a wide range of dilation distances. To formalize this intuition, the persistent homology literature relies on the notion of \emph{filtration}. A \emph{filtration} of a graph $\placeGraph$ is a set of graphs $\placeGraph^{\delta_i}$ ordered by real-valued parameters $\delta_i$ such that
\begin{equation}
    \emptyset \subseteq \ldots \subseteq \placeGraph^{\delta_{i + 1}} \subseteq \placeGraph^{\delta_i} \subseteq \ldots \subseteq \placeGraph\,.
\end{equation}
In our case, each graph $\placeGraph^{\delta_i}$ is the \subgraph of $\placeGraph$ where nodes and edges with radius smaller than $\delta_i$ are removed;\footnote{Such a filtration is known as the Vietoris-Rips filtration~\citep{Aktas19ans-persistence}.} hence the filtration is a set of graphs corresponding to increasing distances $\delta_i$; see \cref{fig:filtration1,fig:filtration3} for examples of graphs in the filtration. For each graph in the filtration, we compute the 0-homology (the number of connected components), which in our setup corresponds to the number of rooms the graph gets partitioned into;
computationally, the 0-homology can be computed in one shot for all graphs in the filtration by iterating across the edges in the original graph $\placeGraph$ using a union set~\citep{Aktas19ans-persistence}, leading to an extremely efficient algorithm.

The resulting mapping between the dilation distance and the number of connected components is an example of \emph{Betti curve}, or $\beta_0(\delta)$; see \cref{fig:betti} for an example.
Intuitively, for increasing distances $\delta$, the graph first splits into more and more connected components (this is the initial increasing portion of \cref{fig:betti}), and then for large $\delta$ entire components tend to disappear (leading to a decreasing trend in the second part of \cref{fig:betti}), eventually leading to all the nodes disappearing (\ie zero connected components as shown in the right-most part of \cref{fig:betti}).
In practice, we restrict the set of distances to a range $\left[d^{-}, d^{+}\right]$, containing reasonable sizes for openings between rooms ([0.5, 1.2]\SI{}{\meter} in our tests). We also do not count connected components containing
less than a minimum number of vertices ($15$ in our tests), as we are not interested in segmenting overly small regions as rooms.

\begin{figure*}
    \subfloat[Betti Curve\label{fig:betti}]{\centering
        \includegraphics[width=0.49\columnwidth]{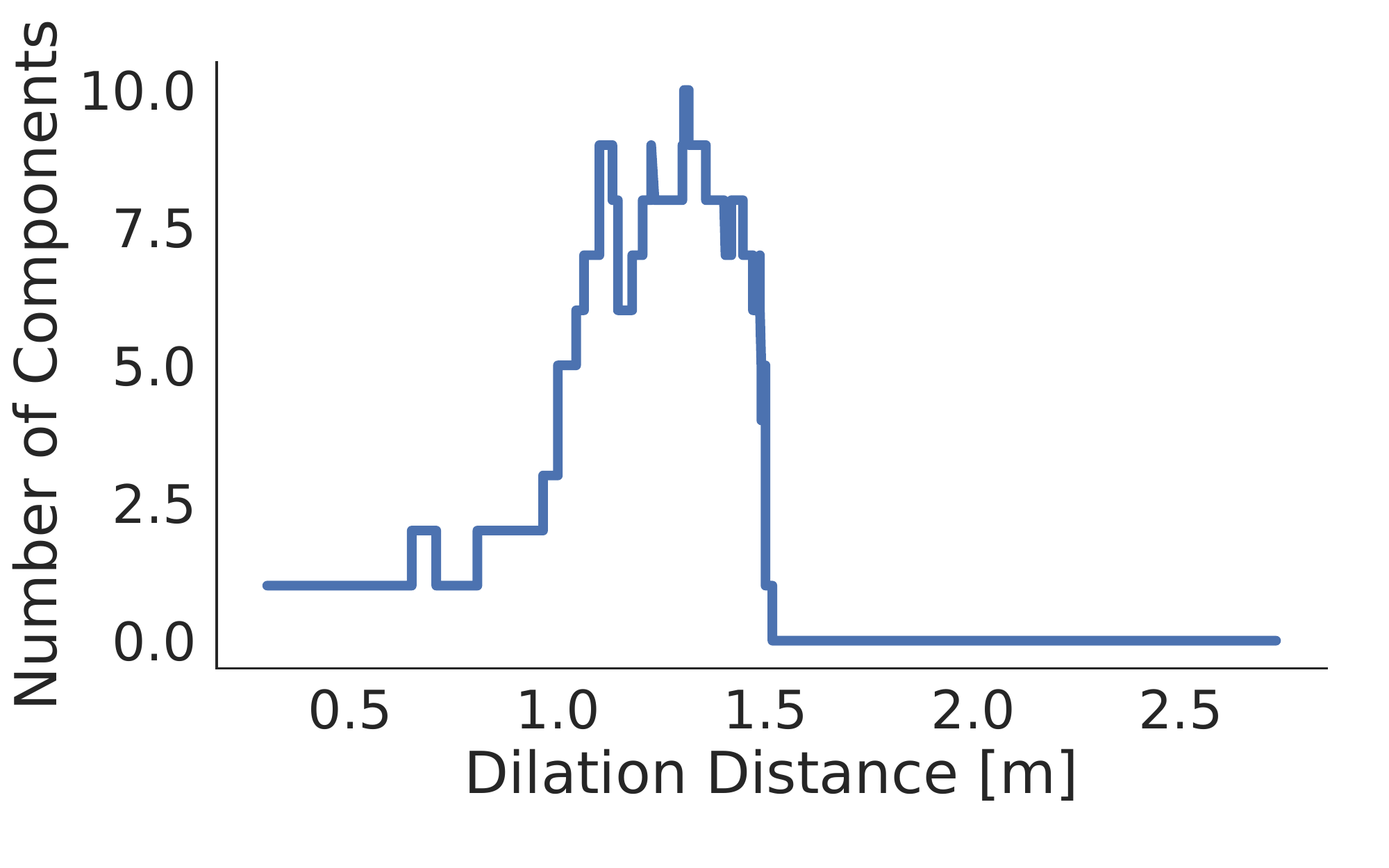}
    }
    \subfloat[$\delta = 0.7$\label{fig:filtration1}]{\centering
        \includegraphics[width=0.50\columnwidth,trim=4cm 1.5cm 4cm 1.5cm,clip]{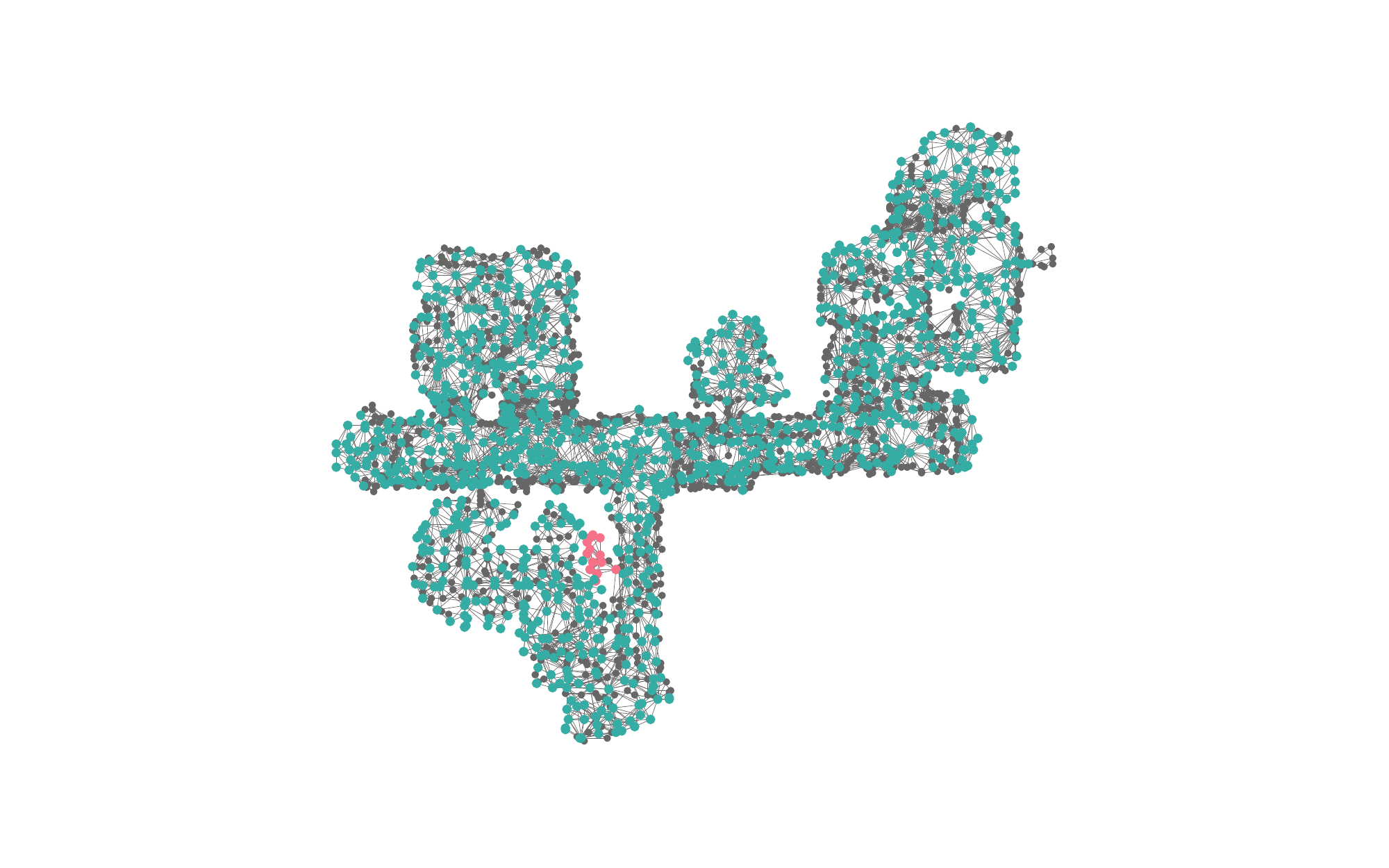}
    }
    \subfloat[$\delta = 1.0$\label{fig:filtration2}]{\centering
        \includegraphics[width=0.50\columnwidth,trim=4cm 1.5cm 4cm 1.5cm,clip]{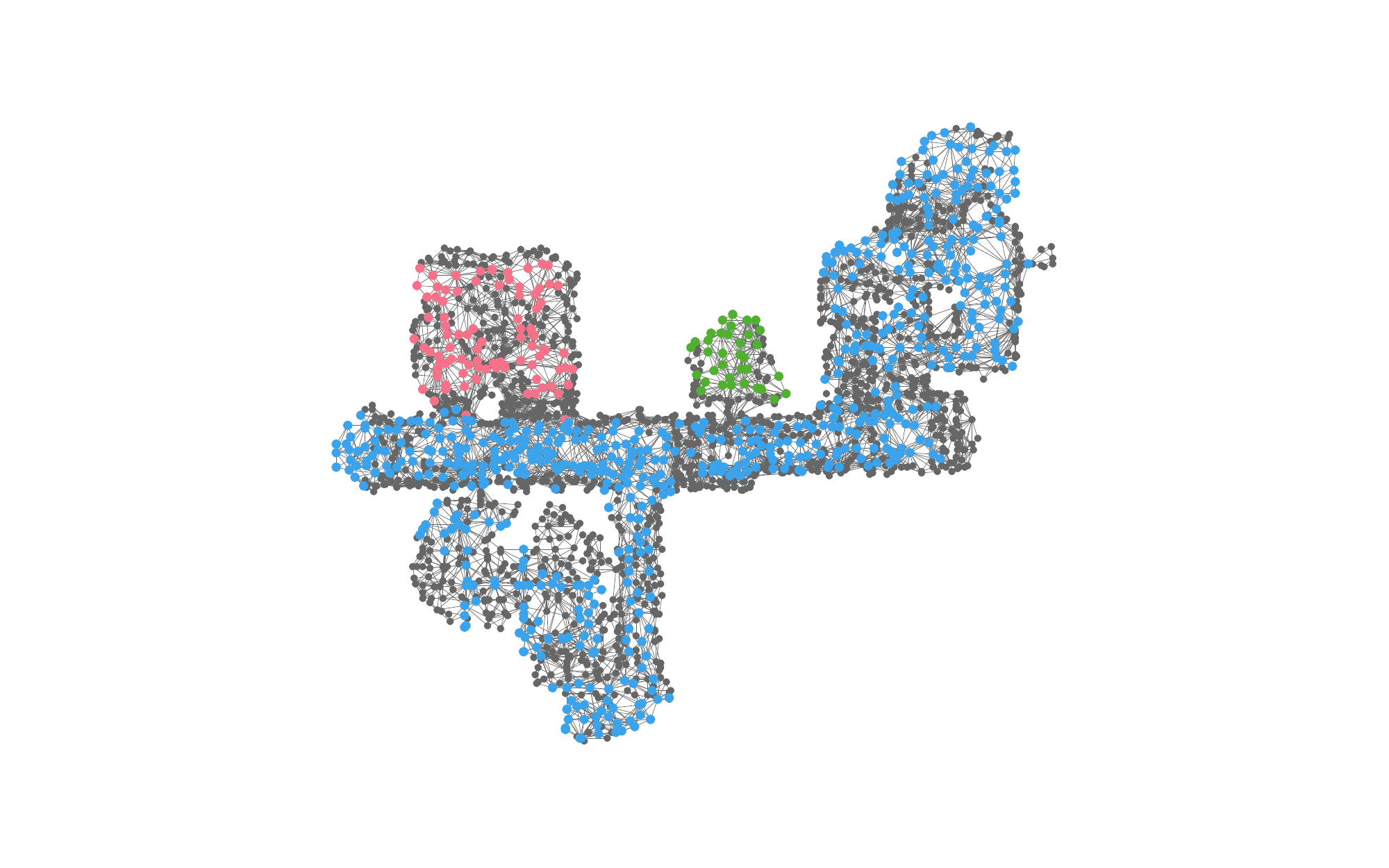}
    }
    \subfloat[$\delta = 1.3$\label{fig:filtration3}]{\centering
        \includegraphics[width=0.50\columnwidth,trim=4cm 1.5cm 4cm 1.5cm,clip]{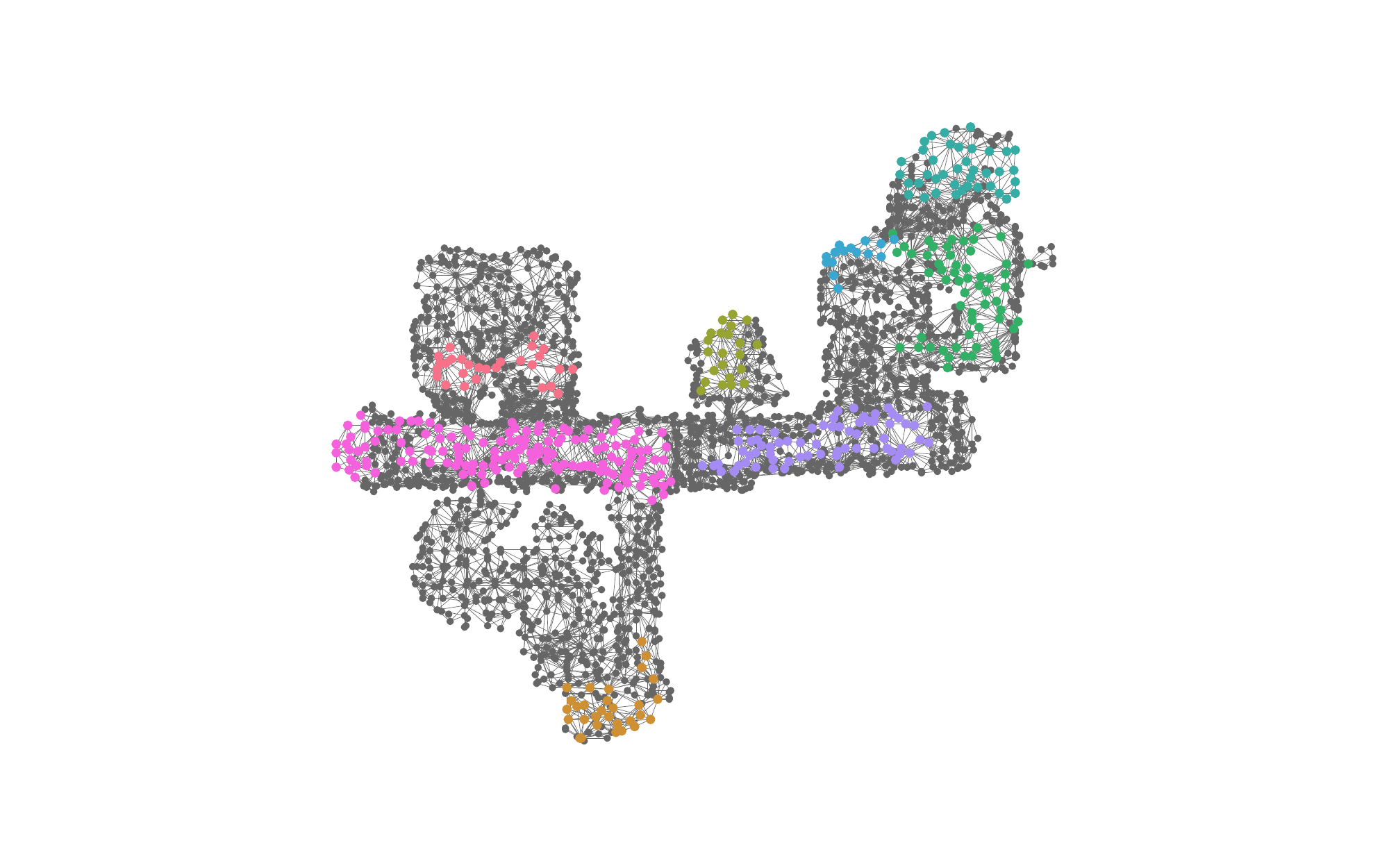}
    }
    \caption{(a) Example of Betti curve; we only consider components with at least 15 nodes.
(b-d) Example of filtration of the graph for various thresholds $\delta$.
Nodes with distances $\dplace{}$ smaller than the threshold $\delta$ are shown in gray, while nodes with distances above $\delta$ are colored by component membership.
}\label{fig:filtration}
    \togglevspace{-5mm}{0mm}
\end{figure*}

Choices of rooms that are more persistent correspond to large ``flat'' horizontal regions of the Betti curve,
where the number of connected components stays the same across a large sub-set of $\delta$ values.
More formally, let us denote with $H$ the set of unique values assumed by the Betti curve,
\ie:
\beq
H = \setdef{i \in \Natural{} }{ i = \beta_0(\delta), \forall \delta \in \left[d^{-}, d^{+}\right]}
\eeq
and ---for some small positive constant $\epsilon$--- denote with $I$ the set of intervals $I_j$ (for each $j \in H$) defined as:
\bea
I_j =
\big\{
(d_\min, d_\max):
\forall \delta \in [d_\min, d_\max], \beta_0(\delta) = j, \nonumber\\
\text{ and }\beta_0(d_\min-\epsilon) \neq j, \beta_0(d_\max+\epsilon) \neq j
\big\}
\eea
where, for each number of connected components $H$, the set $I$ contains the extremes $(d_\min, d_\max)$ of the corresponding flat interval in the Betti curve. Finally, we denote with $L$ the lengths of the intervals in $I$. Then, the most persistent choice for the number of rooms is the one corresponding to the longest interval in $L$.
In our approach, we choose the number of rooms (and associate an initial set of places to each room) by looking at
``sufficiently persistent'' intervals (\ie{} flat regions with length larger than a threshold).
In more detail, we select a subset of $L$, denoted as $\bar{L}$, that only contains the intervals of size
{greater than $\alpha \cdot \max_{l \in L} l$, where $\alpha$ is a user-specified parameter in $\left[0, 1\right]$ that balances between over-segmentation ($\alpha$ closer to 0) and under-segmentation ($\alpha$ closer to 1).}
From these candidates, we assign $\delta^*$ to be $\delta^* = d_\min^*$, where
$(d_\min^*,d_\max^*)$ are the extremes of the interval in $\bar{L}$ attaining the largest number of connected components.
In other words, we choose the number of rooms to be the largest among all flat intervals that are sufficiently long (\ie{} longer than $\alpha \cdot \max_{l \in L} l$). As before, for this choice of $\delta^*$, we use the connected components corresponding to that dilation distance as an initial guess for the room clustering.\footnote{Note that
for $\alpha=1$, the proposed approach simply picks the longest (most persistent) interval. While intuitive, that choice typically ends up picking the very first (left-most interval in \cref{fig:filtration}, with a single connected component)
that persists for small $\delta^*$, which is undesirable in practice. }

\myParagraph{Assigning Remaining Nodes via Flood-Fill}
The connected components produced by the persistent homology do not account for all places in $\placeGraph$, as nodes and edges in the graph disappear depending on the dilation distance.
We assign the remaining place nodes to the putative room via the flood-fill algorithm, where the expansion queue is ordered by the distance of each edge to the nearest obstacle, resulting in node connections with the greatest distance to an obstacle being expanded first.
This ensures that every node in the places graph $\placeGraph$ is assigned to a room, provided that every original connected component of $\placeGraph$ contains at least one room.

\begin{remark}[Novelty, Advantages, and Limitations]
Many classical 2D room segmentation techniques, such as watershed or morphology approaches, are also connected to persistent homology~\citep{Bormann16icra-roomSegmentationSurvey}.
Our room clustering method provides four key advantages.
First, our approach is able to reason over arbitrary 3D environments, instead of a 2D occupancy grid (such as~\citep{Kleiner17iros-roombaRoomSegmentation}). Second, the approach reasons over a sparse set of places and is extremely computationally and memory efficient in practice; this is in stark contrast with approaches processing monolithic voxel-based maps (\eg~\cite{Rosinol21ijrr-Kimera} reported runtimes of tens of minutes). Third, the approach automatically computes the dilation distance used to cluster the graph;
this allows it to work across a variety of environments (in \cref{sec:experiments}, we report experiments in both small apartments and student residences). Fourth, it provides a more solid theoretical framework, compared to the heuristics proposed  in related work~\citep{Hughes22rss-hydra,Bavle22arxiv-SGraphPlus}. As a downside, our room clustering approach, similarly to related work relying on geometric reasoning for room segmentation, may fail to segment rooms without a clear geometric footprint, \eg{} open floor-plans.
\end{remark}

\myParagraph{Intra-Layer and Inter-Layer Edges}
We connect pairs of rooms, say $(i,j)$, whenever a place in room $i$ shares an edge with a place in room $j$.
For each room, we add edges to all places in the room.
We connect all rooms to a single building node, whose position is the mean of all room centroids.

\subsubsection{Layer 4: Room Classification via Neural Tree}\label{sec:rooms_classification}
\newcommand{\Htree}{H-tree\xspace}
\newcommand{\Htrees}{H-trees\xspace}
While in the previous section, we (geometrically) clustered places nodes into different rooms,
the approach described in this section assigns a semantic label (\eg{} kitchen, bedroom) to each room.
The key insight that we are going to use is that objects in a room are typically correlated with the room type
 (\eg{} a room that contains refrigerator, oven, and toaster is likely to be a kitchen, a room with a bathtub and a toilet is likely to be a bathroom).
 Edges connecting rooms may also carry information about the room types (\eg{} a master bedroom is more likely to be connected to a bathroom).
We construct a \subgraph of the 3D scene graph that includes the object and the room layers and the corresponding intra- and inter-layer edges.
Given this graph, and the object semantic labels, centroids, and bounding boxes  (see \cref{sec:objects}), we infer the semantic labels of each room.

While room inference could be attacked via standard techniques for inference in probabilistic graphical models, those techniques typically require handcrafting or estimating expressions for the factors connecting the nodes, inducing a probability distribution over the labels in the graph.
We instead rely on more modern techniques that use \emph{graph neural networks} (GNNs) to learn a suitable neural message passing function between the nodes to infer the missing labels.
In particular, room classification can be thought of as a semi-supervised node classification problem~\citep{Kipf17iclr-gcn, Hamilton17nips-GraphSage}, which has been extensively studied  in machine learning.
We also observe that  our problem has two key features that make it unique.
First, the object-room graph is a \emph{heterogeneous} graph and contains two kinds of nodes, namely objects and rooms, as opposed to large, homogeneous social network graphs (one of the key benchmarks applications in the semi-supervised node classification literature).
Second, the object-room graph is a hierarchical graph (\cref{def:hieGraph}), which gives more structure to the problem (\eg{} \cref{cor:tw-orb}).\footnote{Note that the result in \cref{cor:tw-orb} is general enough to also include the building node and perform building classification (\ie{} classify an indoor environment into an office building, hospital, apartment, etc.).
    Here we tailor the discussion to the object-room graph for a practical reason: we lack a large enough dataset for training and testing a building classification network.
    The dataset in our experiments includes 90 buildings, which are mostly residential.
}
We review a recently proposed GNN architecture, the \emph{neural tree}~\citep{Talak21neurips-neuralTree}, that takes advantage of the hierarchical structure of the graph and leads to (provably and practically) efficient and accurate inference.

\myParagraph{Neural Tree Overview}
While traditional GNNs perform neural message passing on the edges of the given graph $\Graph$ (the object-room graph in our case),
the key idea behind the neural trees architecture is to construct a tree-structured graph from the input graph and perform message passing on the resulting tree instead of the input graph~\citep{Talak21neurips-neuralTree}.
This tree-structured graph, the \emph{\Htree}, is similar to a \treeDecomposition, and is such that every node in it represents either a node or a subset of nodes in the input graph.
Trees are known to be more amenable for message passing (\eg{} the junction tree algorithm enables exact inference for graphs with small treewidth)~\citep{Jordan02book, Koller09book}.
Analogously, the neural tree has been shown to enable strong approximation results~\citep{Talak21neurips-neuralTree} and lead to better classification accuracy in practice (see~\cite{Talak21neurips-neuralTree} and \cref{sec:experiments}).
We briefly review the construction of the \Htree, the choice of message passing, and the resulting performance guarantees, and we refer the reader to~\cite{Talak21neurips-neuralTree} for an in-depth discussion.

\myParagraph{Constructing the H-Tree}
The
neural tree performs message passing on the \Htree, a tree-structured graph constructed from the input graph.
Each node in the \Htree{} corresponds to a \subgraph of the input graph. These \subgraphs are arranged hierarchically in the \Htree{} such that the parent of a node in the \Htree{} always corresponds to a larger \subgraph in the input graph.
The leaf nodes in the \Htree{} correspond to singleton subsets (\ie{} individual nodes) of the input graph.

The first step to construct an \Htree{} is to compute a \treeDecomposition{} $T$ of the object-room graph.
Since the object-room graph is a hierarchical graph, we use \cref{algo:td-hierarchical} to efficiently compute a \treeDecomposition. The bags in such a \treeDecomposition  contain either (C1) only room nodes, (C2) only object nodes, or (C3) object nodes with one room node.
{
To form the \Htree{}, we need to further decompose the leaves of the \treeDecomposition into singleton nodes.
For bags falling in the cases (C1)-(C2), we further decompose the bags using a
tree decomposition of the \subgraphs formed by nodes in the bag, as described in~\cite{Talak21neurips-neuralTree}. For case (C3), we note that the \subgraph is again a hierarchical graph with one room node, hence we again use \cref{algo:td-hierarchical} to compute a \treeDecomposition. We form the \Htree by concatenating these tree-decompositions hierarchically as described in~\cite{Talak21neurips-neuralTree}.
}

\myParagraph{Message Passing and Node Classification}
Message passing on the \Htree{} generates embeddings for all the nodes and important \subgraphs of the input graph.
Any of the existing message passing protocols (\eg{} the ones used in Graph Convolutional Networks (GCN)~\citep{Kipf17iclr-gcn, Henaff15-deep, Defferrard16nips-ChebyNets, Kipf17iclr-gcn, Bronstein17spm-geometricDL}, GraphSAGE~\citep{Hamilton17nips-GraphSage}, or Graph Attention Networks (GAT)~\citep{Velickovic18iclr-GAT, Lee19tkdd-gat-survey, Busbridge2019a-arXiv-Rel-GraphAttentionNetworks}) can be re-purposed to operate on the neural tree.
We provide an ablation of different choices of message passing protocols and node features in \cref{sec:experiments}.
After message passing is complete, the final label for each node is extracted by pooling embeddings from all leaf nodes in the \Htree{} corresponding to the same node in the input graph, as in~\cite{Talak21neurips-neuralTree}.

One important difference between the \Htree{} in~\cite{Talak21neurips-neuralTree}, and the \Htree{} constructed for the object-room graph is the heterogeneity of the latter.
The \Htree{} of a heterogeneous graph will also be heterogeneous, \ie{} the \Htree{} will now contain nodes that correspond to various kinds of \subgraphs in the input object-room graph.
Specifically, the \Htree{} has nodes that correspond to \subgraphs: (i) containing only room nodes, (ii) containing one room node and multiple object nodes, (iii) containing only object nodes, and (iv) leaf nodes which correspond to either an object or a room node.
Accordingly, we treat the neural tree as a heterogeneous graph when performing message passing.
Message passing over heterogeneous graphs can be implemented using off-the-shelf functionalities in the PyTorch geometric library~\citep{Fey19iclrwk-pytorchGeometric}.

\myParagraph{Expressiveness of the Neural Tree and Graph Treewidth}
The following result, borrowed from our previous work~\citep{Talak21neurips-neuralTree}, establishes a connection between the expressive power of the neural tree and the treewidth of the corresponding graph.

\begin{theorem}[Neural Tree Expressiveness, Theorem 7 and Corollary 8 in~\cite{Talak21neurips-neuralTree}]\label{thm:approx}
Call $\calF(\Graph, N)$ the space of functions that can be produced by applying the neural tree architecture with $N $ parameters to the graph $\Graph$.
Let $f:[0,1]^{n} \rightarrow [0, 1]$ be a function compatible with a graph $\Graph$ with $n$ nodes, \ie
a function that can be written as $f(\MX) = \textstyle\sum_{C \in \CliqueSetOf{\Graph}}\theta_{C}(\vxx_C)$,
 where $\CliqueSetOf{\Graph}$ denotes the collection of all maximal cliques in $\Graph$ and $\theta_{C}$ is some function that maps features associated to nodes in a clique $C$ to a real number.
Let each clique function $\theta_{c}$ in $f$ be $1$-Lipschitz and be bounded to $[0, 1]$.
Then, for any $\epsilon > 0$, there exists a function $g \!\in\! \calF(\Graph, N)$ such that $|| f \!-\! g||_{\infty} \!<\! \epsilon$, while the number of parameters $N$ is bounded by
\begin{equation}
\label{eq:param_approx}
N = \textstyle\calO\left( n \times (\treewidth{\JTH{\Graph}} + 1)^{2\treewidth{\JTH{\Graph}} + 3}\times \epsilon^{- (\treewidth{\JTH{\Graph}} + 1)}\right),
\end{equation}
where $\treewidth{\JTH{\Graph}}$ denotes the treewidth of the tree-decomposition of $\Graph$,
computed according to \cref{algo:td-hierarchical}.
\end{theorem}

While we refer the reader to~\cite{Talak21neurips-neuralTree} for a more extensive discussion, the intuition is that graph-compatible functions can model (the logarithm of) any probability distribution over the given graph.
Hence, \cref{thm:approx} essentially states that the neural tree can learn any (sufficiently well-behaved) graphical model over $\calG$, with a number of parameters that scales exponentially in the graph treewidth, and only linearly in the number of nodes in the graph.
Therefore, for graphs with small treewidth (as the ones of \cref{cor:tw-orb}),
we can approximate arbitrary relations between the nodes without requiring too many parameters.
Furthermore, \cref{cor:tw-orb} and \cref{prop:tw-hierarchy} ensure that we can compute the \treeDecomposition{} (and hence the \Htree) efficiently in practice.
Beyond these theoretical results, in~\cref{sec:experiments} we show that the use of the neural tree leads to improved accuracy in practice.

\section{Persistent Representations: Detecting and Enforcing Loop Closures in 3D Scene Graphs}\label{sec:LCD-and-SGO}

\begin{figure}[t]
    \centering
    \includegraphics[width=0.95\columnwidth,trim=0 5mm 0 0,clip]{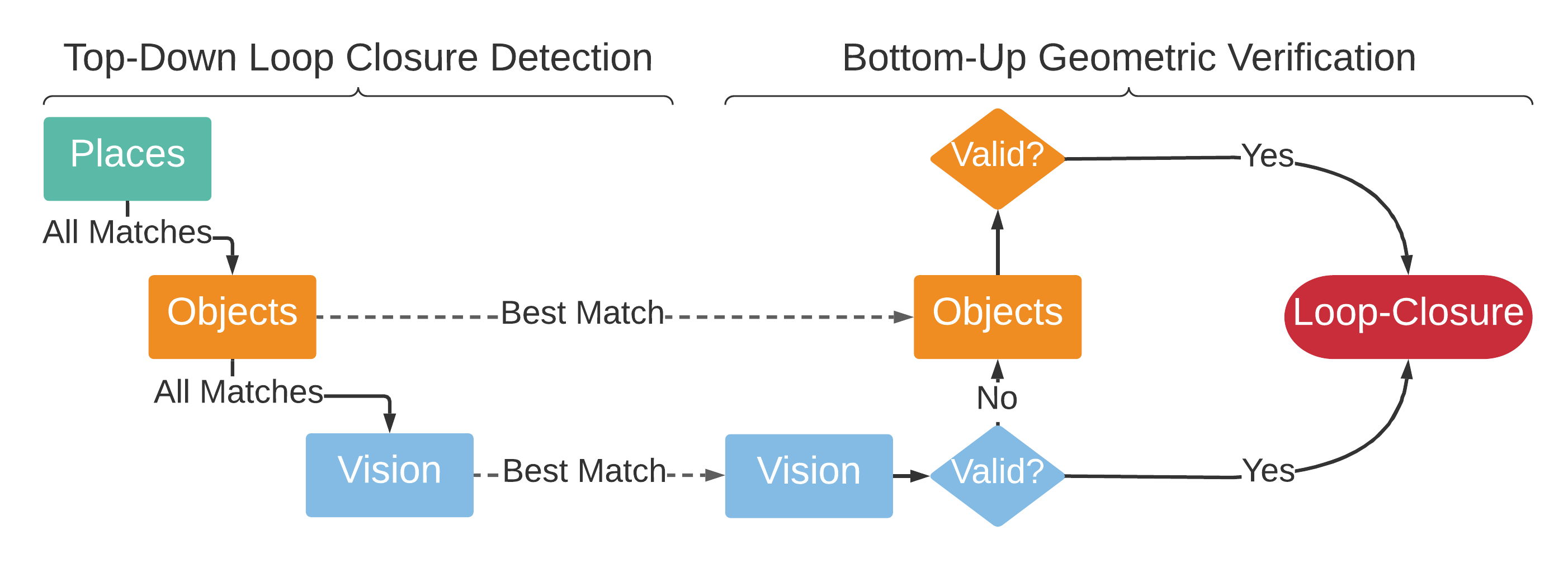}
    \caption{Loop closure detection (left) and geometric verification (right).
             To find a match, we ``descend'' the 3D scene graph layers, comparing descriptors.
             We then ``ascend'' the 3D scene graph layers, attempting registration.}\label{fig:lcd}
    \togglevspace{-6mm}{0mm}
\end{figure}

The previous section discussed how to estimate the layers of an ``odometric'' 3D scene graph as the robot explores an unknown environment.
In this section, we discuss how to use the 3D scene graph to \emph{detect} loop closures (\cref{sec:LCD}),
and how to \emph{correct} the entire 3D scene graph in response to putative loop closures (\cref{sec:SGO}).

\subsection{Loop Closure Detection and Geometric Verification}\label{sec:LCD}

We augment visual loop closure detection and geometric verification by using information across multiple layers in the 3D scene graph.
Standard approaches for visual place recognition rely on visual features (\eg{} SIFT, SURF, ORB) and fast retrieval methods (\eg{} bag of words~\citep{Galvez12tro-dbow}) to detect loop closures.
Advantageously, the 3D scene graph not only contains visual features (included in each node of the agent layer), but also additional information about the semantics of the environment (described by the object layer) and the geometry and topology of the environment (described by the places layer).
In the following we discuss how to use this additional 3D information to develop better descriptors for loop closure detection and geometric verification.

\subsubsection{Top-Down Loop Closure Detection}
As mentioned in \cref{sec:objects}, the agent layer stores visual features for each keyframe pose along the robot trajectory.
We refer to each such poses as \emph{agent nodes}.
Loop closure detection then aims at finding a past agent node that matches (\ie{} observes the same portion of the scene seen by) the latest agent node, which corresponds to the current robot pose.

\myParagraph{Top-Down Loop Closure Detection Overview}
For each agent node, we construct a hierarchy of descriptors describing statistics of the node's surroundings, from low-level appearance to semantics and geometry.
At the lowest level, our hierarchical descriptors include standard DBoW2 appearance descriptors~\citep{Galvez12tro-dbow}.
We augment the appearance descriptor with an object-based descriptor and a place-based descriptor computed from the objects and places in a \subgraph surrounding the agent node.
We provide details about two choices of descriptors (hand-crafted and learning-based) below.
To detect loop closures, we compare the hierarchical descriptor of the current (query) node with all the past agent node hierarchical descriptors, searching for a match.
When comparing descriptors, we walk down the hierarchy of descriptors (from places, to objects, to appearance descriptors).
In particular, we first compare the places descriptor and ---if the descriptor distance is below a threshold--- we move on to comparing object descriptors and then appearance descriptors.
If any of the appearance or object descriptor comparisons return a strong enough match (\ie{} if two distance between two descriptors is below a threshold), we perform geometric verification; see \cref{fig:lcd} for a summary.

\myParagraph{Hand-crafted Scene Graph Descriptors}
Our top-down loop closure detection relies on having descriptors (\ie{} vector embeddings) of the \subgraphs of objects and places around each agent node.
In the conference version~\citep{Hughes22rss-hydra} of this paper, we proposed  hand-crafted descriptors.
In particular, for the objects, we use the histogram of the semantic labels of the object nodes in the \subgraph as an object-level descriptor.
For the places, we use the histogram of the distances associated to each place node in the \subgraph as a place-level descriptor.
As shown in~\cite{Hughes22rss-hydra} and confirmed in \cref{sec:experiments}, the resulting hierarchical descriptors already lead to improved loop closure detection performance over traditional appearance-based loop closures.
However, these descriptors fail to capture relevant information about objects and places, \eg{} their spatial layout and connectivity.
In the following, we describe learning-based descriptors that use graph neural networks to automatically find a suitable embedding for the object and place \subgraphs; these are observed to further improve loop closure detection performance in some cases; see \cref{sec:experiments}.

\myParagraph{Learning-based Scene Graph Descriptors}
Given a \subgraph of objects and places around the agent node,
we learn fixed-size embeddings using a Graph Neural Network (GNN).
 At a high level, we learn such embeddings from scene graph datasets, such that the Euclidean distance between descriptors is smaller if the corresponding agent nodes are spatially close.

In more detail, we learn separate embeddings for the \subgraph of objects and the \subgraph of places.
For every object layer \subgraph, we encode the bounding-box size and semantic label of each object as node features in a GNN\@.
For every places layer \subgraph, we encode the distance of the place node to the nearest obstacle and the number of basis points of the node as node features.
Rather than including absolute node positions in the respective node features, we assign a weight to each edge $(i, j)$ between nodes $i$ and $j$ as $w_{ij} = e^{-\norm{x_i - x_j}}$,
where $x_i$ and $x_j$ are the positions of nodes $i$ and $j$.
This results in a weight in the range $[0, 1]$, where the closer two nodes are, the higher the edge weight.
Associating intra-node distances to edges (rather than using absolute positions as node features) makes the resulting embedding pose-invariant; this is  due to the fact that the node positions only enter the network in terms of their distance, which is invariant to rigid transformations.
Our GNN model architecture follows the graph embedding architecture presented in~\cite{Li19icml-GraphMatching}, which consists of multi-layer perceptrons as encoders for node and edge features, message passing layers, and a graph-level multi-level perception to aggregate node embeddings into the final graph embedding.
We use triplet loss~\citep{Li19icml-GraphMatching} to train the models and defer the details of constructing triplets and other model and training parameters to the experiments in \cref{sec:experiments}.

\subsubsection{Bottom-up Geometric Verification}
After we have a putative loop closure between our query and match agent nodes (say $i$ and $j$), we attempt to compute a relative pose between the two by performing bottom-up geometric verification.
Whenever we have a match at a given layer (\eg{} between appearance descriptors at the agent layer, or between object descriptors at the object layer), we attempt to register frames $i$ and $j$.
For registering visual features we use standard RANSAC-based geometric verification as in~\cite{Rosinol20icra-Kimera}.
If that fails, we attempt registering objects using TEASER++~\citep{Yang20tro-teaser}, discarding loop closures that also fail object registration.
This bottom-up approach has the advantage that putative matches that fail appearance-based geometric verification (\eg{} due to viewpoint or illumination changes) can successfully lead to valid loop closures during the object-based geometric verification.
\Cref{sec:experiments} shows the proposes hierarchical descriptors improve the quality and quantity of detected loop closures.

\subsection{3D Scene Graph Optimization}\label{sec:SGO}

This section describes a framework to correct the entire 3D scene graph in response to putative loop closures.
Assume we use the algorithms in \cref{sec:incrementalLayers} to build an ``odometric'' 3D scene graph, which drifts over time as it is built from the odometric trajectory of the robot --- we refer to this as the \emph{frontend (or odometric) 3D scene graph}.
Then, our goals here are (i) to optimize all layers in the 3D scene graph in a consistent manner
while enforcing the detected loop closures (\cref{sec:LCD}),
and (ii) to post-processes the results to remove redundant \subgraphs corresponding to the robot visiting the same location multiple times. The resulting 3D scene graph is what we call a \emph{backend (or optimized) 3D scene graph}, and we
refer to the module producing such a graph as the \emph{scene graph backend}.
Below, we describe the two main processes implemented by the scene graph backend: a 3D scene graph optimization (which simultaneously corrects all layers of the scene graph by optimizing a sparse subset of variables), and an interpolation and reconciliation step (which recovers the dense geometry and removes redundant variables); see \cref{fig:deformation}.

\begin{figure}
    \centering
    \subfloat[Loop closure detection]{\centering
        \includegraphics[trim={80, 50, 80, 40}, clip, width=0.47\columnwidth]{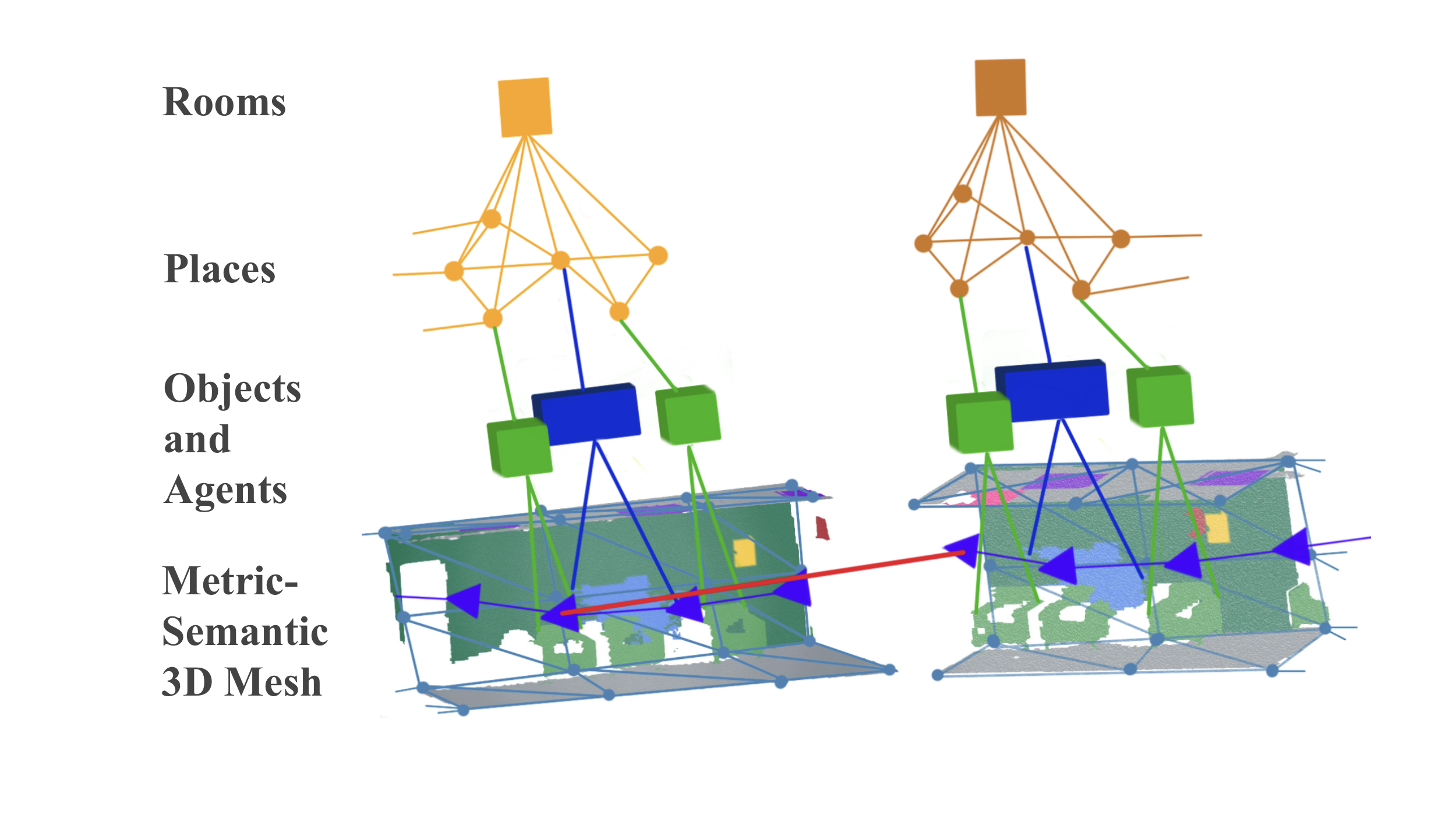}
    }
    \subfloat[Optimization]{\centering
        \includegraphics[trim={10, 60, 60, 50}, clip, width=0.47\columnwidth]{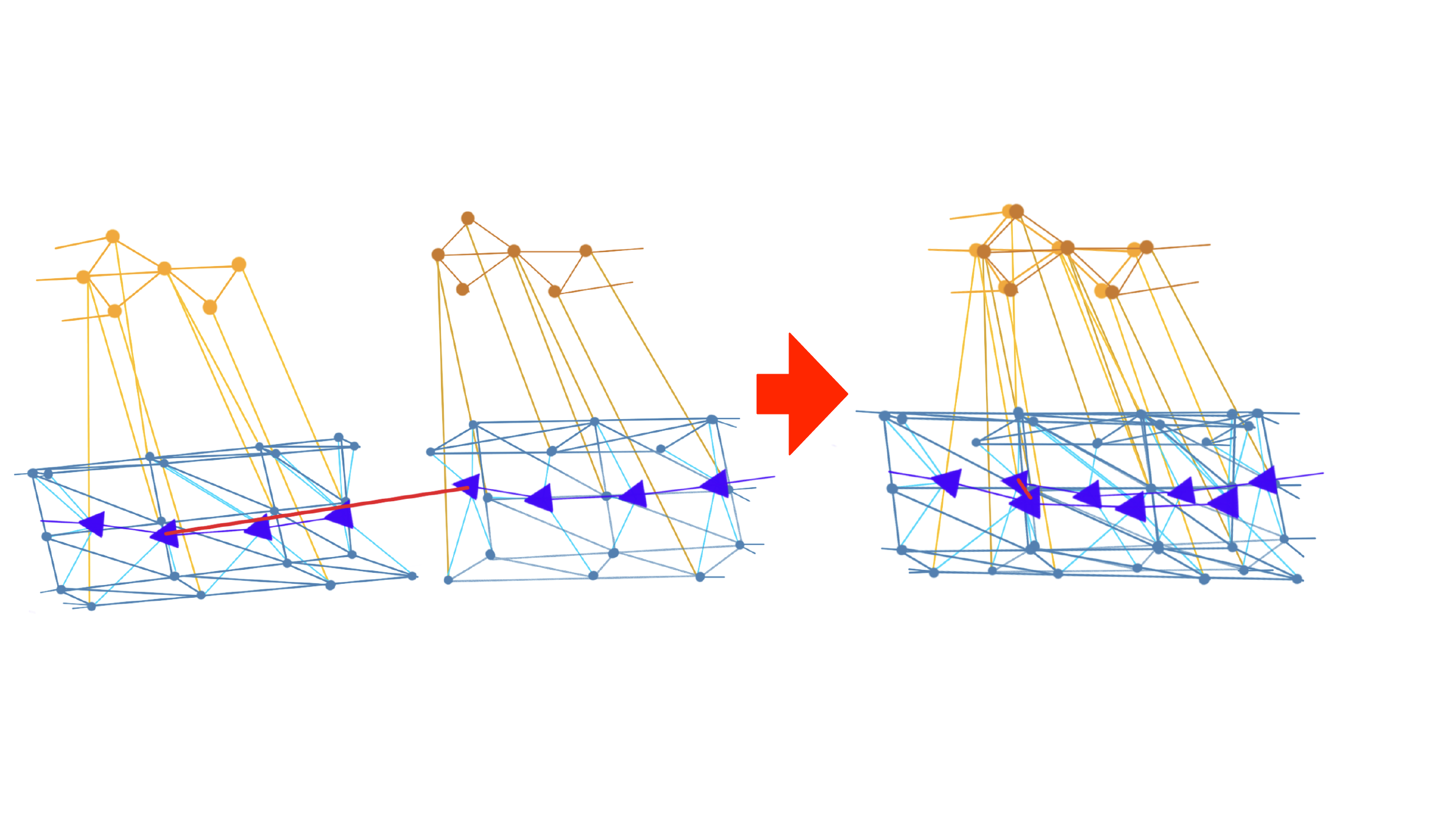}
    }\\
    \subfloat[Interpolation]{\centering
        \includegraphics[trim={210, 50, 150, 80}, clip, width=0.47\columnwidth]{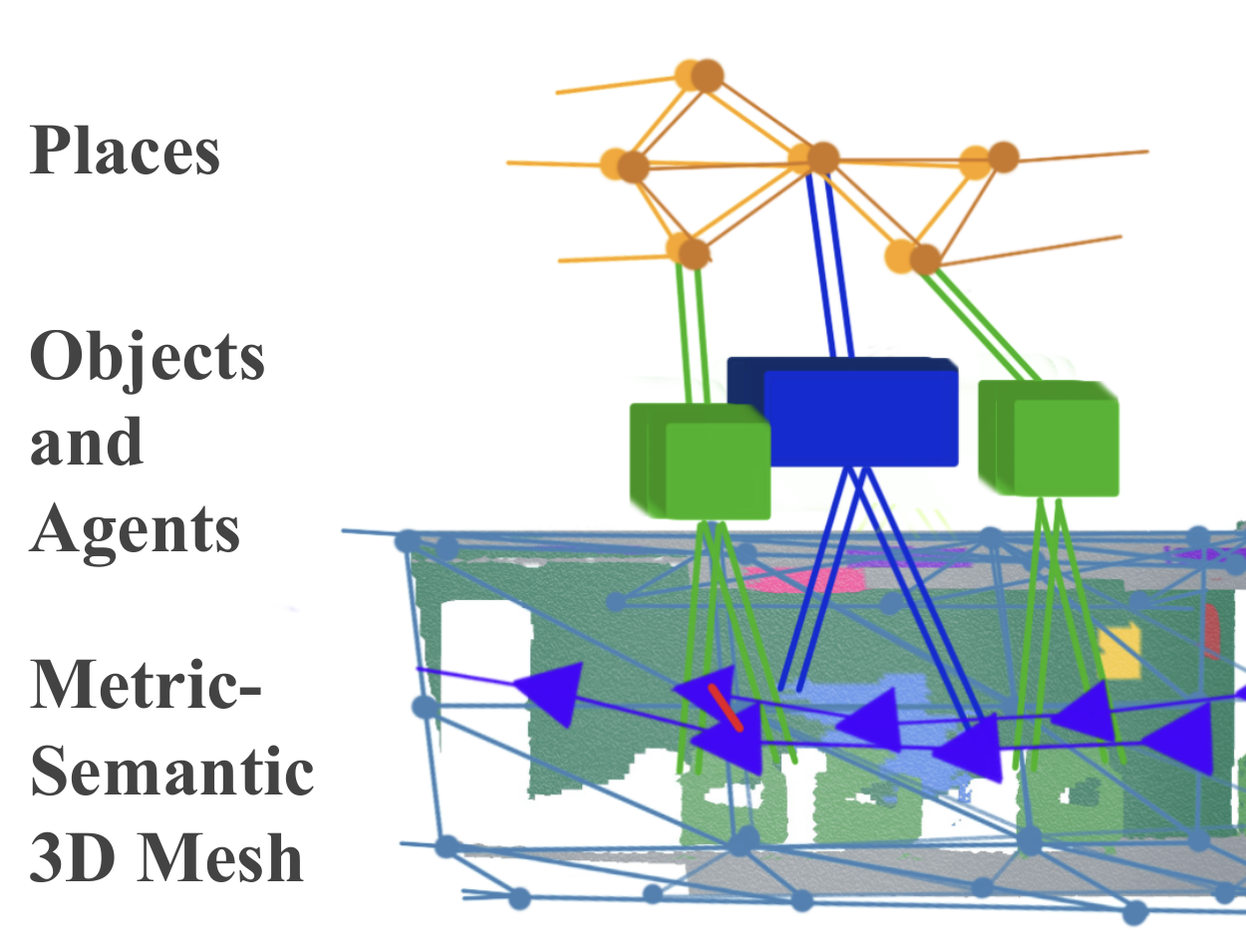}
    }
    \subfloat[Reconciliation]{\centering
        \includegraphics[trim={160, 50, 150, 35}, clip, width=0.47\columnwidth]{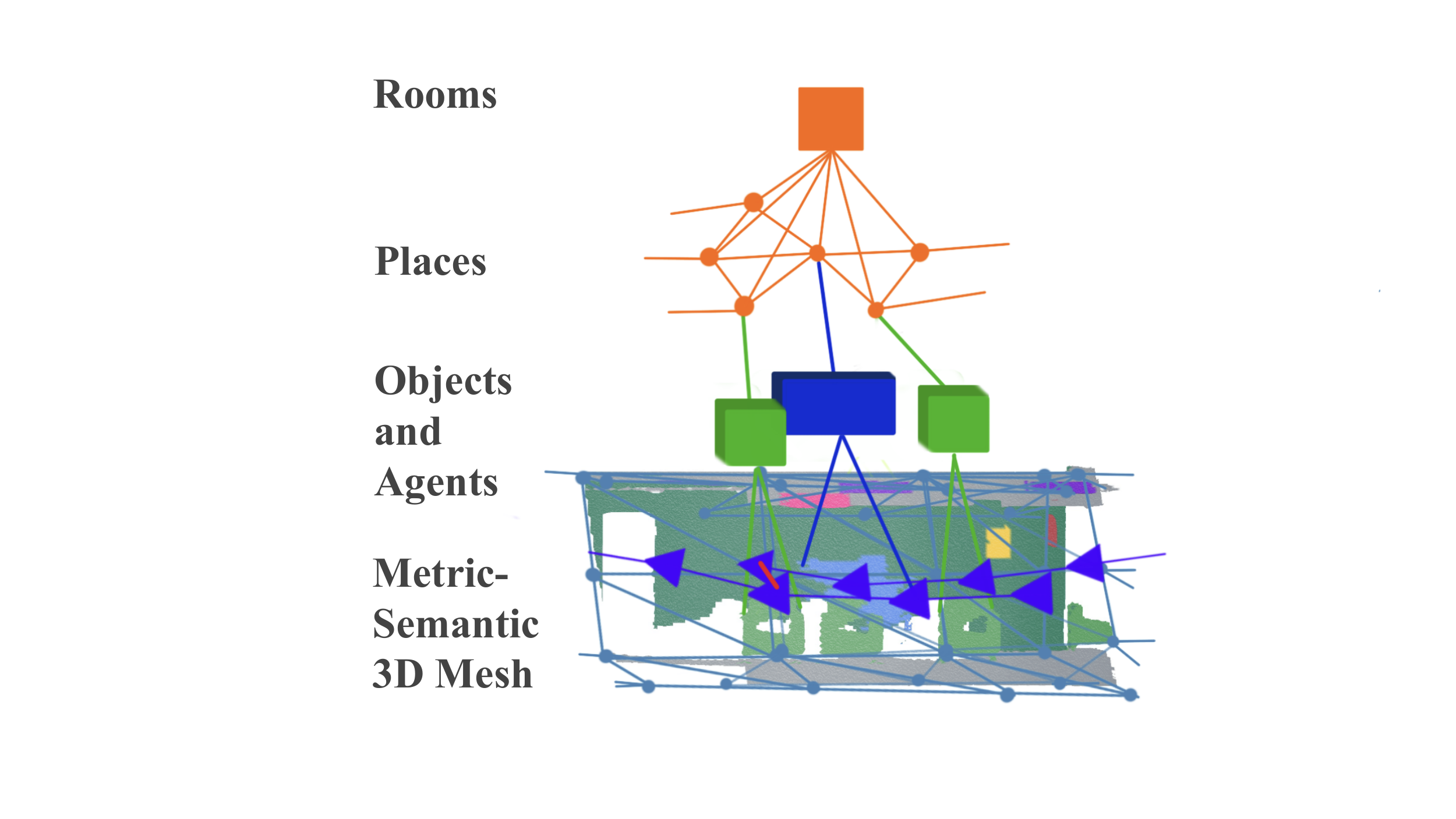}
    }
    \caption{Loop closure detection and optimization:
    (a) after a loop closure is detected,
    (b) we extract and optimize a \subgraph of the 3D scene graph ---the \emph{deformation graph}--- that includes the agent poses, the places, and a subset of the mesh vertices.
    (c) We then reconstruct the rest of the graph via interpolation as in~\cite{Sumner07siggraph-embeddedDeformation},
    and (d) reconcile overlapping nodes.
    }\label{fig:deformation}
    \togglevspace{-7mm}{0mm}
\end{figure}

\myParagraph{3D Scene Graph Optimization}
We propose an approach to simultaneously deform the 3D scene graph layers using an \emph{embedded deformation graph}~\citep{Sumner07siggraph-embeddedDeformation}. This approach generalizes the \emph{pose graph and mesh optimization approach} in~\cite{Rosinol21ijrr-Kimera} as it also includes the graph of places in the optimization.
At a high-level, the backend optimizes a sparse graph (the {embedded deformation graph}) built by downsampling the nodes in the 3D scene graph, and then reconstructs
 the other nodes in the scene graph via interpolation as in~\cite{Sumner07siggraph-embeddedDeformation}.

Specifically, we form the deformation graph as the \subgraph of the 3D scene graph that includes
(i) the agent layer, consisting of a pose graph that includes both odometry and loop closures edges,
(ii) the \emph{3D mesh control points}, \ie{} uniformly subsampled vertices of the 3D mesh (obtained using the same spatial hashing process described in \cref{sec:mesh}), with edges connecting control points closer together than a distance ($\SI{2.5}{\meter}$ in our implementation);
(iii) a minimum spanning tree of the places layer.\footnote{The choice of using the minimum spanning tree of places is motivated by computational reasons: the use of the spanning tree increases sparsity of the resulting deformation graph, enabling faster optimization.}
By construction, these three layers form a connected \subgraph (recall the presence of the inter-layer edges discussed in \cref{sec:incrementalLayers}).

\begin{figure*}[!t]
    \centering
    \includegraphics[width=0.99\textwidth]{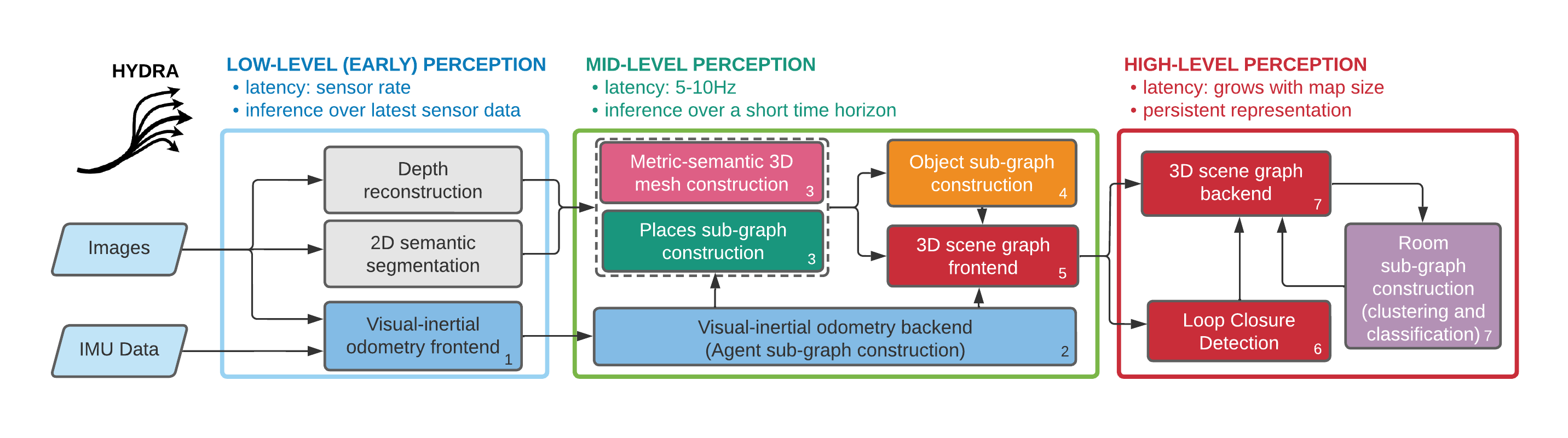}
    \togglevspace{-5mm}{0mm}
    \caption{\name's functional blocks.
             We conceptualize three different functional block groupings: low-level perception, mid-level perception, and high-level perception in order of increasing latency.
             Each functional block is labeled with a number that identifies the ``logical'' thread that the module belongs to.}\label{fig:architecture}
    \togglevspace{-5mm}{0mm}
\end{figure*}

The embedded deformation graph approach associates a local frame (\ie{} a pose) to each node in the deformation graph and then solves an optimization problem to adjust the local frames in a way that minimizes deformations associated to each edge (including loop closures). Let us call $\calT_a$, $\calT_m$, $\calT_p$ the set of poses associated with the agent layer, the mesh control points, and the places. The poses in $\calT_a$ are initially set to be the odometric poses of the robot;
the poses in $\calT_m$ are initially set to have identity rotation
and translation equal to the mesh control points' positions;
similarly, the poses in $\calT_p$ are initially set to have identity rotation, and translation equal to the position of the corresponding places.
Each edge represents a relative pose or a relative position measurement between pair of poses.
In particular, the set of edges is $\calE = \calE_{aa} \cup \calE_{mm} \cup \calE_{pp} \cup \calE_{am} \cup \calE_{ap} \cup \calE_{mp}$, where each subset denotes intra-layer edges (\eg $\calE_{aa}$ contains the edges within the agent layer),
or inter-layer edges (\eg $\calE_{am}$ contains the edges between a robot pose and a mesh control point).\footnote{The set $\calE_{aa}$ contains all the odometric and loop closure measurements relating poses in the robot trajectory.
$\calE_{mm}$ contains all the relative positions between pairs of nearby mesh control points (expressed in the local frame attached to one of the control points).
Similarly, $\calE_{pp}$, $\calE_{am}$, $\calE_{ap}$ and $\calE_{mp}$ contain the relative positions between pairs of places,  between robot poses and mesh control points visible from that pose, between robot poses and places near that pose,
and between a place and the corresponding basis points in the mesh, respectively.
  }
Intuitively, the proposed 3D scene graph optimization finds the set of poses that minimizes the mismatch with respect to these relative measurements: a small mismatch corresponds to small deformations of the local geometry of the scene graph and encourages small errors for the loop closure edges.

We can find an optimal configuration for the poses  $\calT = \calT_a \cup \calT_m \cup \calT_p$ in the deformation graph by solving the following optimization problem:
\beq
\label{eq:pgo}
\calT^\star = \argmin_{\MT_1, \MT_2, \ldots\in\calT} \sum_{(i,j) \in \calE} \normsq{  \MT_i\inv \MT_j - \ME_{ij} }{\MOmega_{ij}}
\eeq
where $\MT_i \in \SEthree$ and $\MT_j \in \SEthree$ are pairs of 3D poses in $\calT$, $\ME_{ij}$ is the relative measurement associated to each edge $(i,j) \in \calE$ (written as a 3D pose), and for a matrix $\MM$ we use the notation $\normsq{ \MM }{\MOmega} \triangleq \trace{\MM \Omega\MM\tran}$.\footnote{When the relative measurement involves only a translation, the rotation component of $\ME_{ij}$ is conventionally set to the identity --- a suitable choice of the information matrix $\MOmega_{ij}$ for those measurements ensures that such rotation component is disregarded by the optimization.}
The $4\times 4$ positive semidefinite matrix $\MOmega_{ij}$ is chosen as the inverse of the odometry (or loop closure) covariances for the edges in $\calE_{aa}$, while it is set to $\diag{[ 0 \; 0 \; 0 \; \omega_t]}$ for the other relative position measurements, where the zeros cancel out the rotation component of the residual error $\MT_i\inv \MT_j - \ME_{ij}$, and $\omega_t$ is a user-specified parameter that controls how much deformation we want to allow for each edge during the optimization.\footnote{The interested reader might find a step-by-step derivation of~\eqref{eq:pgo} (but restricted to the agent layer and the mesh control points) in~\cite{Rosinol21ijrr-Kimera}. }

In hindsight, 3D scene graph optimization transforms a subset of the 3D scene graph into a \emph{factor graph}~\citep{Cadena16tro-SLAMsurvey}, where edge potentials need to be minimized.
The expert reader might also realize that~\eqref{eq:pgo} is mathematically equivalent to standard pose graph optimization in SLAM~\citep{Cadena16tro-SLAMsurvey}, which enables the use of established off-the-shelf solvers.
In particular, we solve~\eqref{eq:pgo} using the Graduated Non-Convexity (GNC) solver in GTSAM~\citep{Antonante21tro-outlierRobustEstimation}, which is also able to reject incorrect loop closures as outliers.

\myParagraph{Interpolation and Reconciliation}
Once the optimization terminates, the agent and place nodes are updated with their new (optimized) positions and the full mesh is interpolated back from its control points according to the deformation graph approach in~\cite{Sumner07siggraph-embeddedDeformation,Rosinol21ijrr-Kimera}.
After the 3D scene graph optimization and the interpolation step, certain portions of the scene graph ---corresponding to
areas revisited multiple times by the robot--- contain redundant information.
To avoid this redundancy, we merge overlapping nodes.
For places nodes, we merge nodes within a distance threshold (\SI{0.4}{\meter} in our implementation).
For object nodes we merge nodes if the corresponding objects have the same semantic label and if one of nodes is contained inside the bounding box of the other node.
After this process is complete, we recompute the object centroids and bounding boxes from the position of the corresponding vertices in the optimized mesh.
Finally, we recompute the rooms from the graph of places using the approach in \cref{sec:rooms}.

\section{Thinking Fast and Slow: \\ the \name{} Architecture}\label{sec:hydra}

We integrate the algorithms described in this paper into a highly parallelized \emph{spatial perception system}, named \emph{\name}.
\name{} involves a combination of processes that run at sensor rate (\eg{} feature tracking for visual-inertial odometry), at sub-second rate (\eg{} mesh and place reconstruction, object bounding box computation), and at slower rates (\eg{} the scene graph optimization, whose complexity depends on the map size).
Therefore these processes have to be organized such that slow-but-infrequent computation (\eg{} scene graph optimization) does not get in the way of faster processes.

We visualize \name{} in \cref{fig:architecture}.
Each block in the figure denotes an algorithmic module, matching the discussion in the previous sections.
\name{} starts with fast \emph{early} perception processes (\cref{fig:architecture}, left), which perform low-level perception tasks such as feature detection and tracking (required for visual-inertial odometry, and executed at frame-rate), 2D semantic segmentation, and stereo-depth reconstruction (at keyframe rate).
The result of early perception processes are passed to mid-level perception processes (\cref{fig:architecture}, center).
These include algorithms that incrementally construct (an odometric version of) the agent layer (\eg{} the visual-inertial odometry backend), the mesh and places layers, and the object layer.
Mid-level perception also includes the \emph{scene graph frontend}, which is a module that collects the result of the other modules into an ``unoptimized'' scene graph.
Finally, the high-level perception processes perform loop closure detection, execute scene graph backend optimization, and extract rooms (including both room clustering and classification).\footnote{While room detection is fast enough to be executed at keyframe rate, it still operates on the entire graph, hence it is more suitable as a slow high-level perception process.
} This results in a globally consistent, persistent 3D scene graph.

\name{} runs in real-time on a multi-core CPU\@.
The only module that relies on GPU computing is the
2D semantic segmentation, which uses a standard off-the-shelf deep network.
The neural tree (\cref{sec:rooms_classification}) and the GNN-based loop closure detection (\cref{sec:LCD})
 can be optionally executed on a GPU, but the forward pass is relatively fast even on a CPU (see \cref{sec:runtime}).
The fact that most modules run on CPU has the advantage of (i) leaving the GPU to learning-oriented components, and (ii) being compatible with the power limitations imposed by current mobile robots.
In the next section, we will report real-time results with \name{} running on a mobile robot, a Unitree A1 quadruped, with onboard sensing and computation (an NVIDIA Xavier embedded computer).

\section{Experiments}\label{sec:experiments}

The experiments in this section
(i) qualitatively and quantitatively compare the 3D scene graph produced by \name{} to another state-of-the-art 3D scene graph construction method, SceneGraphFusion~\citep{Wu21cvpr-SceneGraphFusion},
(ii) examine the performance of \name{} in comparison to batch offline methods, \ie{} Kimera~\citep{Rosinol21ijrr-Kimera},
(iii) validate design choices for learned components in our method via ablation studies,
and (iv) present a runtime analysis of \name{}.
We also document our experimental setup, including training details for both the GNN-based loop closure descriptors and the neural-tree room classification, and datasets used.
Our implementation of \name{} is available at \hydraURL.

\subsection{Datasets}

We use four primary datasets for training and evaluation: two simulated datasets (Matterport3d~\citep{Chang173dv-Matterport3D} and uHumans2~\citep{Rosinol21ijrr-Kimera}) and two real-world datasets (SidPac and Simmons).
In addition, we use the Stanford3D dataset~\citep{Armeni19iccv-3DsceneGraphs}
to motivate neural tree design choices with respect to our initial proposal in~\cite{Talak21neurips-neuralTree}.

\myParagraph{Matterport3d}
We utilize the Matterport3D (MP3D) dataset~\citep{Chang173dv-Matterport3D}, an RGB-D dataset consisting of 90 reconstructions of indoor building-scale scenes.
We use the Habitat Simulator~\citep{Savva19iccv-habitat} to traverse the scenes from the MP3D dataset and render color imagery, depth, and ground-truth 2D semantic segmentation.
We generate two different 3D scene graphs datasets  for the 90 MP3D scenes by running \name{} on pre-generated trajectories; one for training descriptors and one for room classification.
For training the GNN-based descriptors for loop closure detection (GNN-LCD), we generate a single trajectory for each scene through the navigable scene area such that we get coverage of the entire scene, resulting in 90 scene graphs.
For training the room classification approaches, we generate 5 trajectories for each scene by randomly sampling navigable positions until a total path length of at least \SI{100}{\meter} is reached, resulting in 450 trajectories.
When running \name{} on these 450 trajectories, we save intermediate scene graphs every 100 timesteps (resulting in roughly 15 scene graphs per trajectory), giving us 6810 total scene graphs.

\myParagraph{uHumans2}
The uH2 dataset is a Unity-based simulated dataset~\citep{Rosinol21ijrr-Kimera} that includes four scenes: a small apartment, an office, a subway station, and an outdoor neighborhood.
For the purposes of this paper, we only use the apartment and office scenes.
The dataset provides visual-inertial data, ground-truth depth, and 2D semantic segmentation.
The dataset also provides ground truth trajectories that we use for benchmarking.

\myParagraph{SidPac}
The SidPac dataset is a real dataset collected in a graduate student residence using a visual-inertial hand-held device.
We used a Kinect Azure camera as the primary collection device, providing color and depth imagery, with an Intel RealSense T265 rigidly attached to the Kinect to provide an external odometry source.
The dataset consists of two separate recordings, both of which are used in our previous paper~\cite{Hughes22rss-hydra}.
We only use the first recording for the purposes of this paper.
This first recording covers two floors of the building (Floors 1 \& 3), where we walked through a common room, a music room, and a recreation room on the first floor of the graduate residence, went up a stairwell, through a long corridor as well as a student apartment on the third floor, then finally down another stairwell to revisit the music room and the common room, ending where we started.
These scenes are particularly challenging given the scale of the scenes (average traversal of around \SI{400}{\meter}), the prevalence of glass and strong sunlight in regions of the scenes (causing partial depth estimates from the Kinect), and feature-poor regions in hallways.
We obtain a proxy for the ground-truth trajectory for the Floor 1 \& 3 scene via a hand-tuned pose graph optimization with additional height priors, to reduce drift and qualitatively match the building floor plans.

\begin{figure}
    \subfloat[Clearpath Jackal]{\centering
        \includegraphics[width=0.48\columnwidth]{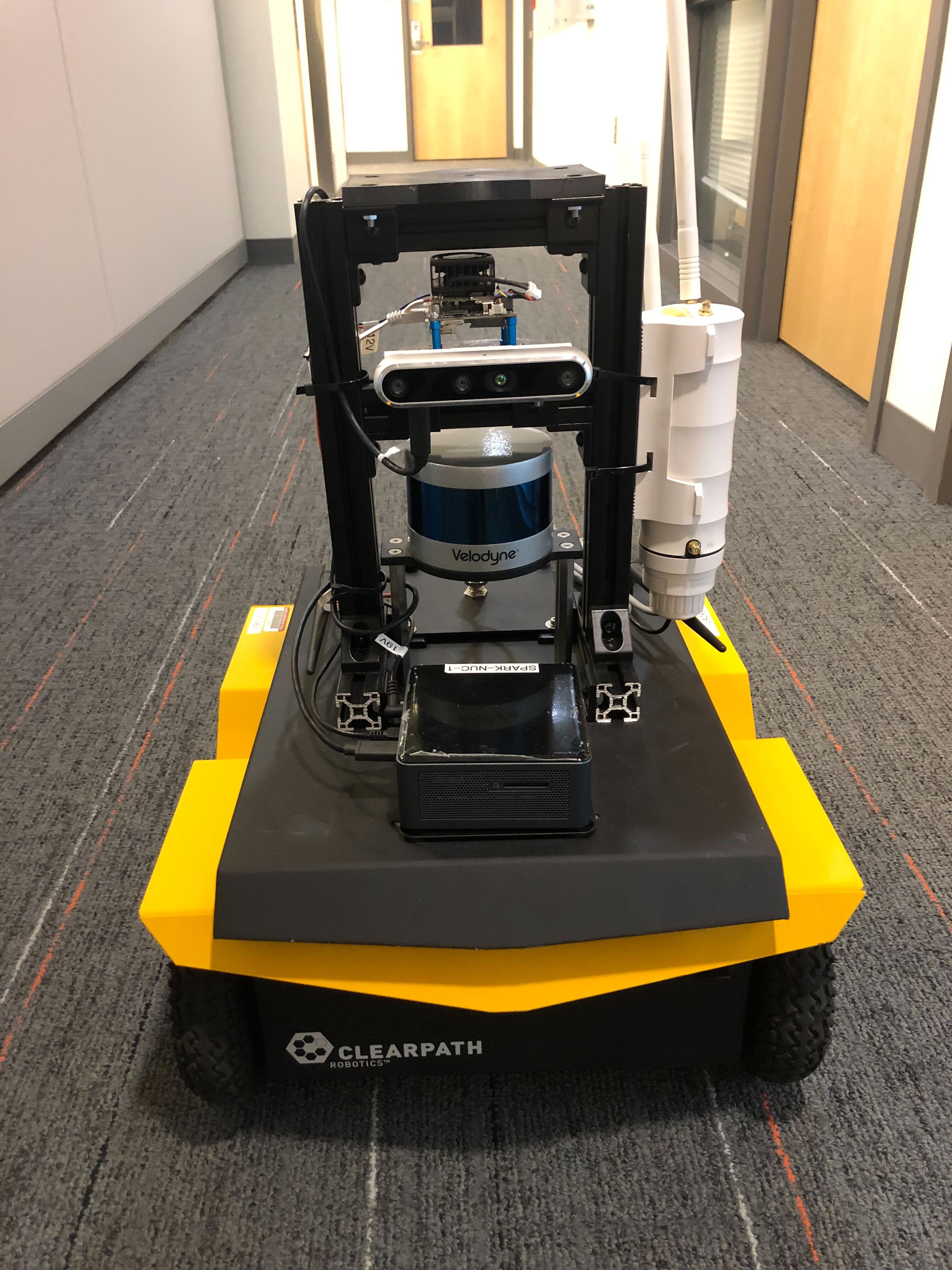}
    }
    \subfloat[Unitree A1]{\centering
        \includegraphics[width=0.48\columnwidth]{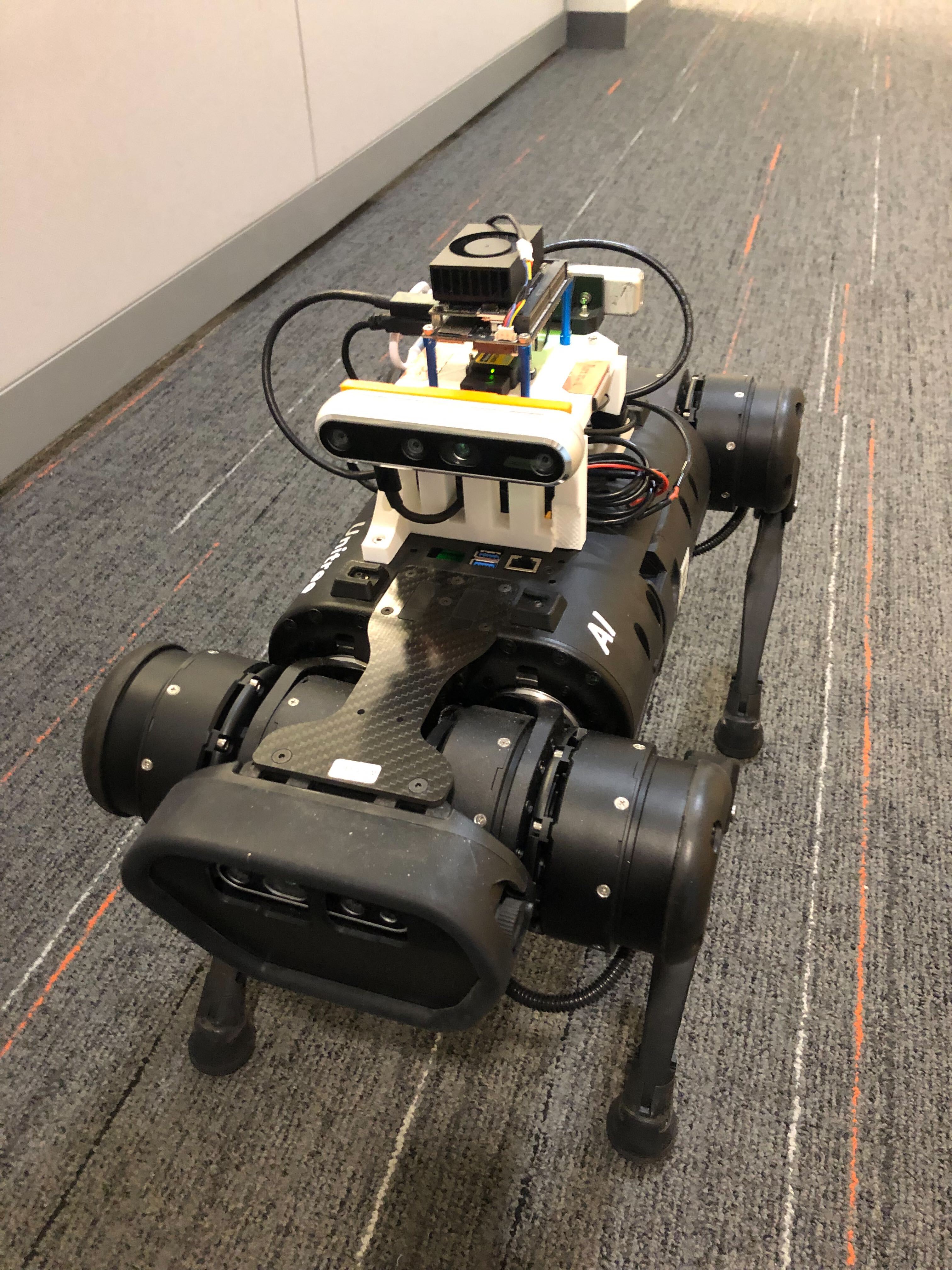}
    }
    \caption{(a) The Clearpath Jackal platform used to record one of the Simmons sequences.
             The Jackal is equipped with both a RealSense D455 camera and a Velodyne LIDAR\@.
             (b) The Unitree A1 platform used to record the second Simmons sequence.
             The A1 is equipped with a RealSense D455 camera and a Microstrain IMU\@.
             Both platforms have an Nvidia Xavier NX embedded computer onboard.}\label{fig:rboots}
    \togglevspace{-5mm}{0mm}
\end{figure}

\myParagraph{Simmons}
The Simmons dataset is a real dataset collected on a single floor of an undergraduate student residence
with a Clearpath Jackal rover and a Unitree A1 quadruped robot (\cref{fig:rboots}).
The Clearpath Jackal rover uses the RealSense D455 camera as the primary collection device to provide color and depth imagery, but the rover is also equipped with
a Velodyne to perform LIDAR-based odometry~\citep{Reinke22ral-LOCUS2} as an external odometry source.
The A1 also uses the RealSense D455 camera as the primary collection device, but is not equipped with a Velodyne;
instead, it is equipped with an industrial grade Microstrain IMU to improve the performance of the visual-inertial odometry~\citep{Rosinol20icra-Kimera}.
The dataset consists of two recordings, one recorded on the Jackal, and one recorded using the A1.
The Jackal recording covers the rooms scattered throughout half of a single floor of the building,
where the Jackal traverses a distance of around \SI{500}{\meter} through mostly student bedrooms,
but also a lounge area, a kitchen, and a laundry room;
the rooms are all joined by a long hallway that spans the full floor.
The A1 sequence takes place on one end of the floor and maps 4 bedrooms, a small lounge, and a section of the hallway that connects all the rooms.
The dataset is challenging due to the visual and structural similarity across student rooms that have similar layouts and furnishing.
We obtain a proxy for the ground-truth trajectory for the Jackal using LIDAR-based SLAM, by running LOCUS~\citep{Reinke22ral-LOCUS2} with flat ground assumption
and LAMP~\citep{Chang22ral-LAMP2} for loop closure corrections.
The ground-truth trajectory of the A1 is obtained by registering individual visual keyframes with the visual keyframes in the Jackal sequence,
and then corrected using the proxy ground-truth of the Jackal sequence.

\myParagraph{Stanford3D}
We use the 35 human-verified scene graphs from the Stanford 3D Scene Graph (Stanford3d) dataset~\citep{Armeni19iccv-3DsceneGraphs} to compare the neural tree against standard graph neural networks for node classification
and to assess new design choices against our initial proposal in~\cite{Talak21neurips-neuralTree}.
These scene graphs represent individual residential units, and each consists of building, room, and object nodes with inter-layer connectivity.
We use the same pre-processed graphs as in~\cite{Talak21neurips-neuralTree} where the single-type building nodes (residential) are removed and additional 4920 intra-layer object edges are added to connect nearby objects in the same room.
This results in 482 room-object graphs, each containing one room and at least one object per room.
The full dataset has 482 room nodes with 15 semantic labels, and 2338 objects with 35 labels.

\subsection{Experimental Setup}

In this section, we first discuss the implementation of \name{}, including our choice of networks for the 2D semantic segmentation.
Then we provide details regarding the training of our learning-based approaches: the GNN-LCD descriptors and the neural tree for room classification.

\begin{figure}
    \subfloat[Original Image]{\centering
        \includegraphics[width=0.48\columnwidth]{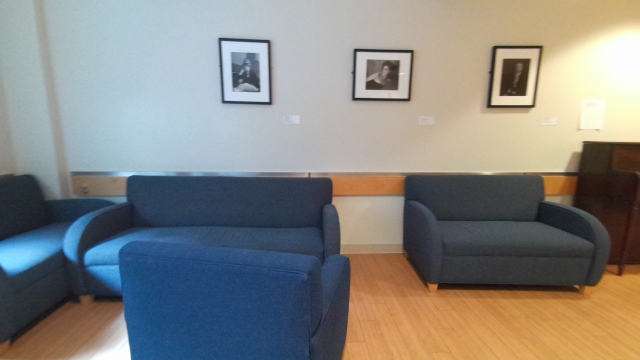}
    }
    \subfloat[OneFormer]{\centering
        \includegraphics[width=0.48\columnwidth]{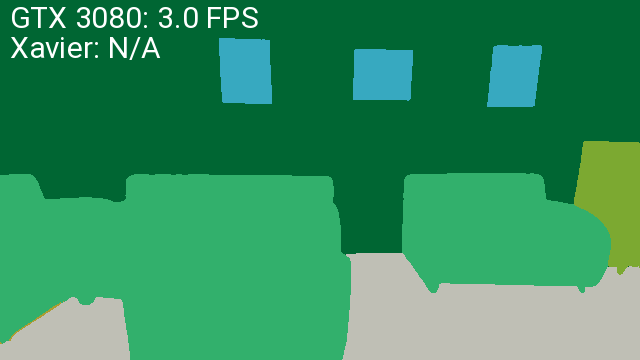}
    } \\
    \subfloat[HRNet]{\centering
        \includegraphics[width=0.48\columnwidth]{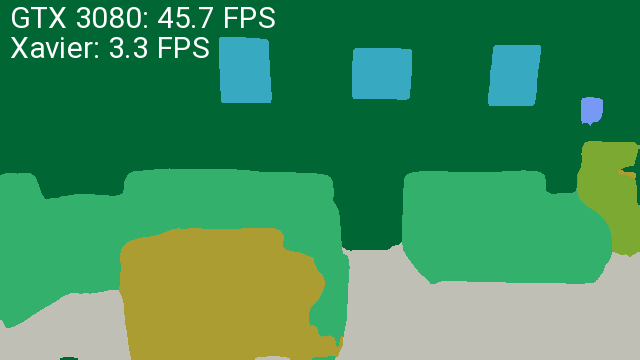}
    }
    \subfloat[Mobilenet]{\centering
        \includegraphics[width=0.48\columnwidth]{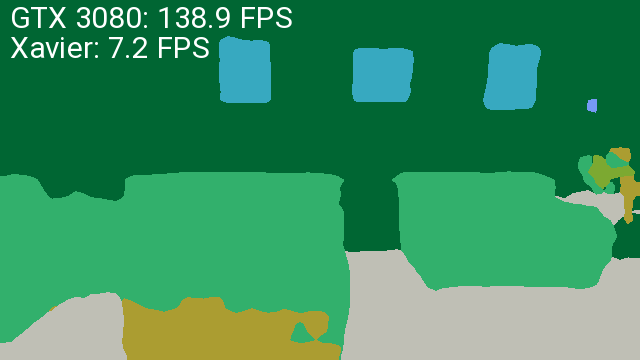}
    }
    \caption{(a) A sample RGB image from SidPac Floor 1-3.
(b-d)~2D semantic segmentation from OneFormer~\citep{Jain23cvpr-Oneformer}, HRNet~\citep{Wang21pami-hrnet}, and MobileNetV2~\citep{Sandler18cvpr-MobileNetV2}.
    The framerate of each approach is overlaid on the images for both the GPU in the workstation used for evaluation (an Nvidia GTX 3080), as well as for a less powerful embedded computer (the Nvidia Xavier NX\@). OneFormer was not tested on the Xavier NX.}\label{fig:semantics}
    \togglevspace{0mm}{0mm}
\end{figure}

\myParagraph{\name}
To provide 2D semantic segmentation for the real-world datasets, we compare three different models, all of which use the ADE20k~\citep{Zhou17cvpr-ade20k} label space.
We use HRNet~\citep{Wang21pami-hrnet} as a ``nominal'' semantic segmentation source and MobileNetV2~\citep{Sandler18cvpr-MobileNetV2} as a light-weight source of semantics for use on a robot.
For both HRNet and MobileNetV2, we use the pre-trained model from the MIT Scene Parsing challenge~\citep{Zhou17cvpr-ade20k} to export a deployable model for our inference toolchain (ONNX and TensorRT).
Additionally, we use a state-of-the-art model, OneFormer~\citep{Jain23cvpr-Oneformer}, to provide
more-accurate-but-slower 2D semantic segmentation.
A comparison of the accuracy and frame-rate of these models is shown in \cref{fig:semantics}.

For both simulated and real datasets we use Kimera-VIO~\citep{Rosinol20icra-Kimera} for visual-inertial odometry.
For SidPac, we fuse the Kimera-VIO estimates with the output of the RealSense T265 to improve the quality of the odometric trajectory.
For Simmons, we fuse  the Kimera-VIO estimates with the odometry output of~\cite{Reinke22ral-LOCUS2} to improve the quality of the odometric trajectory for the sequence recorded by the Jackal platform.

All the remaining blocks in \cref{fig:architecture} are implemented in C++, following the approach described in this paper.
In the experiments we primarily use a workstation with an AMD Ryzen9 3960X with 24 cores and two Nvidia GTX3080s.
We also report timing results on an embedded computer (Nvidia Xavier NX) at the end of this section, and demonstrate the full \name{} system running on the same embedded computer onboard the A1 robot; \toggleformat{video of the experiment is available.\footnote{\videoURL}}{see the video attachment.}

\myParagraph{GNN-LCD Training}
We use the MP3D scene graphs to train loop closure descriptors for the object and place \subgraphs.
We generate a \subgraph around every agent pose which consists of all nodes and edges within a provided radius of the parent place node that the agent pose is nearest to.
We adaptively determine the radius for each \subgraph; each \subgraph has a minimum radius of \SI{3}{\meter} and is grown until either a minimum number of nodes (10) or a maximum radius of \SI{5}{\meter} is reached.
Each object and place \subgraph contains the node and edges features as described in \cref{sec:LCD}.
The semantic label of each object node is encoded either using word2vec~\citep{Mikolov13-wordRepresentation} or a one-hot encoding.
We explore the impact of this label encoding in \cref{sec:exp_results}.

We use triplet loss to train our model.
To construct the triplets, we have an anchor (a candidate \subgraph), and need to find positive and negative examples to compute the loss.
A good proxy for how similar two \subgraphs are is looking at the spatial overlap of the two sets of nodes of the \subgraphs.
We compute this overlap between the \subgraphs by using IoU over the bounding boxes that encompass the positions of the nodes of each \subgraph.
If the overlap between two \subgraphs is at least 40 percent, we consider them to be similar (and candidates for positive examples), otherwise they are considered negative examples for that particular anchor.
Sub-graphs from different scenes are always considered negative examples.
To train our GNN-LCD descriptors, we use online triplet mining with batch-all triplet loss~\citep{Schroff15cvpr-Facenet}, where we construct valid triplets for a batch of \subgraph input embeddings and average loss on the triplets that produce a positive loss.

The message-passing architecture that we selected is GCNConv~\citep{Kipf17iclr-gcn}.
While more expressive or performant architectures exist, few are compatible with the computational framework used for inference in Hydra (\ie{} ONNX).
We split the dataset by scene; we use 70\% of the original 90 scene graphs to train on, 20\% of the scene graphs for validation, and the last 10\% for testing.
Our learning rate is \num{5e-4}, and we train the object models for 50 epochs and place models for 80 epochs, saving the model when the average validation error is at a minimum.
Each model produces a descriptor of dimension 64.
Other key model parameters are reported in \cref{app:lcd}.

\myParagraph{Neural Tree Training}
We train the neural tree and related baselines on two datasets: Stanford3d and MP3D.

For the 482 object-room scene graphs in the Stanford3D dataset, we train the neural tree and GNN baselines for the same semi-supervised node classification task examined in~\citep{Talak21neurips-neuralTree}, where the architecture has to label a subset of room and object nodes.
The goal of this comparison is to understand the impact of some design choices (\eg heterogeneous graphs, edge features) with respect to our original proposal in~\citep{Talak21neurips-neuralTree} and related work.
For this comparison, we implement the neural tree and baseline approaches with four different message passing functions: GCN~\citep{Kipf17iclr-gcn}, GraphSAGE~\citep{Hamilton17nips-GraphSage}, GAT~\citep{Velickovic18iclr-GAT}, and GIN~\citep{Xu19iclr-gin}.
We consider both homogeneous and heterogeneous graphs; for all nodes, we use their centroid and bounding box size as node features.
In some of our comparisons, we also examine the use of relative node positions as edge features, and discard the centroid from the node features.
For the GNN baselines, we construct heterogeneous graphs consisting of two node types: rooms and objects.
For the neural tree, we obtain graphs with four node types: room cliques, room-object cliques, room leaves, and object leaves; see \cref{sec:rooms_classification}.
There are few message passing functions that are compatible with both heterogeneous graphs and edge features; therefore, we compare all heterogeneous approaches using only GAT~\citep{Velickovic18iclr-GAT}.
For all tests on Stanford3d, we report average test accuracy over 100 runs.
For each run, we randomly generate a 70\%, 10\%, 20\% split across all nodes (\ie{} objects and rooms) for training, validation, and testing, respectively.

For the 6180 scene graphs in the MP3D dataset, we test the neural tree and baselines for room classification on object-room graphs.
The goal of this experiment is to understand the impact of the node features and the connectivity between rooms on the accuracy of room classification.
We only consider heterogeneous graphs for this dataset, and as such only use GAT~\citep{Velickovic18iclr-GAT} for message passing.
The heterogeneous node types are the same as for Stanford3D.
We use the bounding box size as the base feature for each node, and use relative positions between centroids as edge features.
We also evaluate the impact of using semantic labels as additional features for the object nodes using word2vec~\citep{Mikolov13-wordRepresentation}.
The MP3D dataset contains scene graphs with partially explored rooms.
We discard any room nodes where the IoU between the 2D footprint of the places within the room and the 2D footprint of the ground truth room is less than a specified ratio (60\% in our tests).
Further details on the construction and pre-processing of the dataset are provided in \cref{sec:gnn_training}.
We predict a set of 25 room labels which are provided in \cref{app:labels}.
For training, we use the scene graphs from the official 61 training scenes of the MP3D dataset. For the remaining 29 scenes, we use graphs from two trajectories of the five total trajectories for validation and the other three trajectories for testing.
For use with \name{}, we select the best-performing heterogeneous neural tree architecture;
this architecture uses bounding box size and word2vec embeddings of the semantic labels of the object nodes as node features, as well as relative positions between nodes as edge features.

All training and testing is done using a single Nvidia A10G Tensor Core GPU\@.
For both datasets, we use cross entropy loss between predicted and ground-truth labels during training, and save the models with the highest validation accuracy for testing.
All models are implemented using PyTorch 1.12.1 and PyTorch Geometric 2.2.0~\citep{Fey19iclrwk-pytorchGeometric}.
We base our implementation on our previous open-source version of the neural tree~\citep{Talak21neurips-neuralTree}.
We provide additional training details, including model implementation, timing, and hyper-parameter tuning in \cref{sec:gnn_training}.
     \subsection{Results and Ablation Study}\label{sec:exp_results}

We begin this section with a comparison between \name{} and \sgf{}~\citep{Wu21cvpr-SceneGraphFusion} (\cref{sec:sgf}).
We then analyze the accuracy and provide an ablation of the modules in \name{} (\cref{sec:ablation}), and show an example of the quality of scene graph that \name{} is able to produce while running onboard the Unitree A1 (\cref{sec:robot_operation}).
Finally,  we report a breakdown of the runtime  of our system (\cref{sec:runtime}).

\begin{figure*}
    \centering
    \subfloat[\HydraBest]{\centering
        \includegraphics[width=0.49\textwidth,trim=7cm 2cm 10cm 1cm,clip]{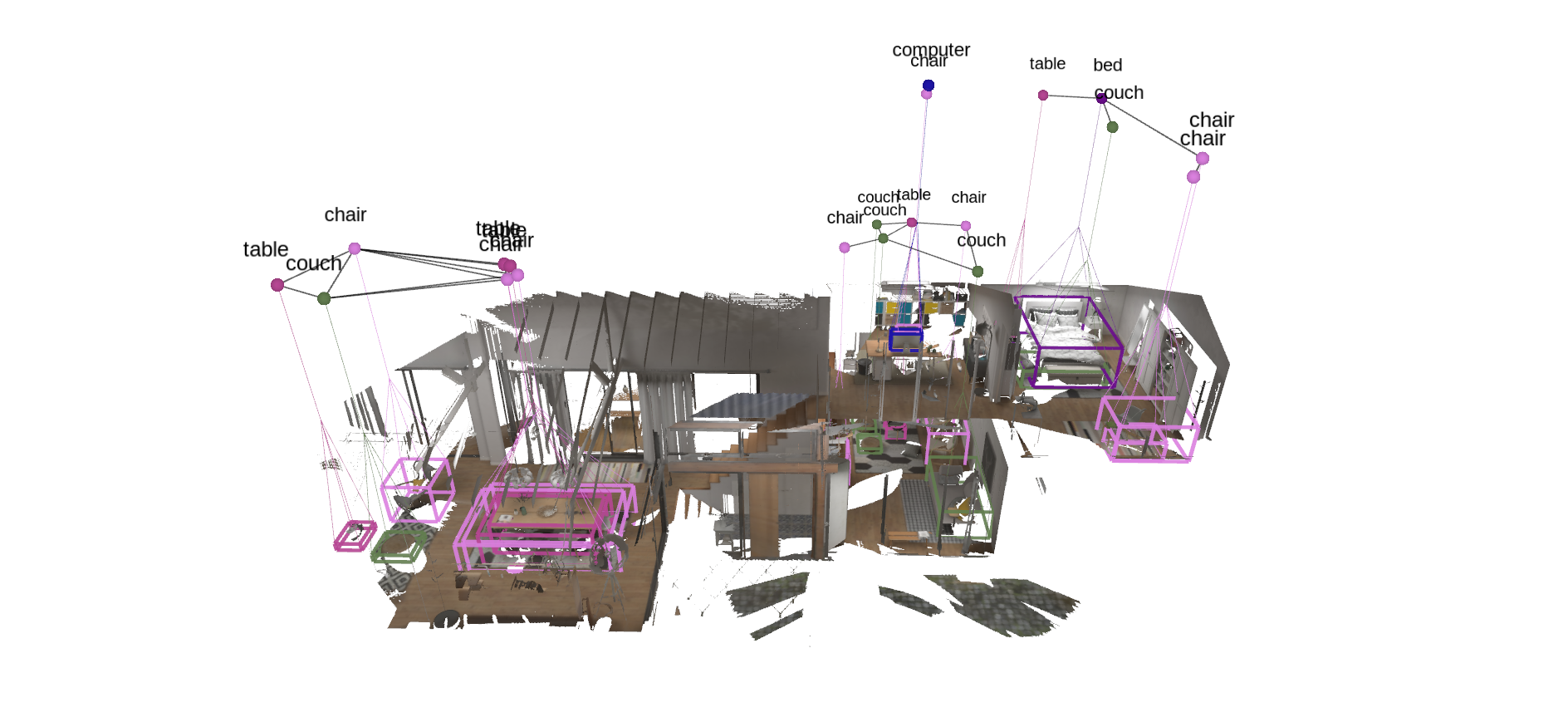}
    }
    \subfloat[\sgf]{\centering
        \includegraphics[width=0.49\textwidth,trim=7cm 2cm 10cm 1cm,clip]{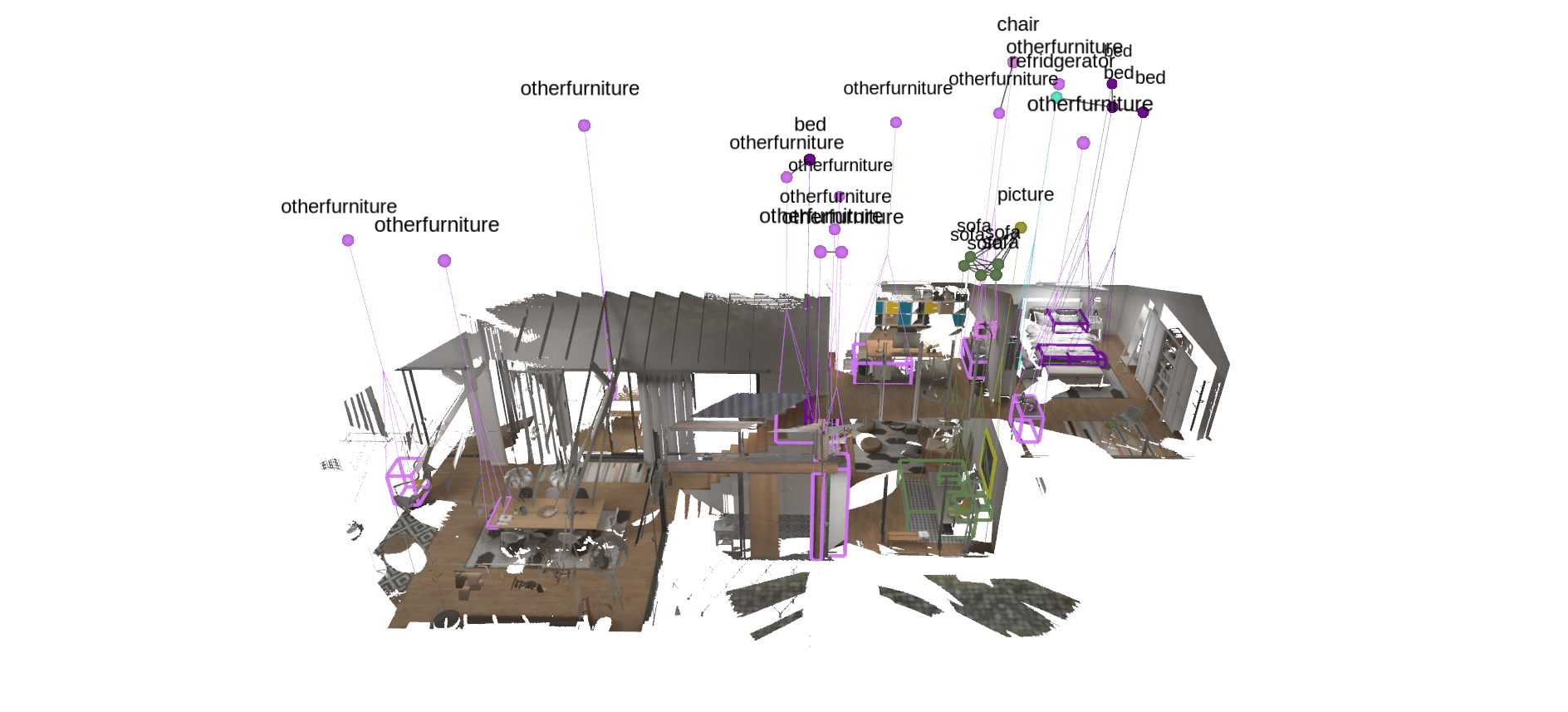}
    }
    \caption{(a) Objects produced by \name{} for the uHumans2 apartment scene.
             (b) 3D scene graph produced by \sgf{} for the uHumans2 apartment scene.
             Note that \sgf{} directly infers geometric object relationships, and edges are drawn between objects for each detected relationship.}\label{fig:sota}
    \togglevspace{-5mm}{0mm}
\end{figure*}

\setlength{\tabcolsep}{\toggleformat{4pt}{3pt}}
\begin{table}[t!]
    \toggleformat{}{\footnotesize}
    \centering
    \begin{tabular}{c cc cc}
        \toprule
        & \multicolumn{2}{c}{uH2 Apartment} & \multicolumn{2}{c}{uH2 Office} \\
        \cmidrule(l{2pt}r{2pt}){2-3} \cmidrule(l{2pt}r{2pt}){4-5}
        & \% Correct & \% Found & \% Correct & \% Found \\
        \midrule
        \sgf{}          & $25.0$ & $17.9$ & $35.4$ & $36.0$ \\
\HydraBest{}    & $68.4$ & $53.6$ & $71.6$ & $60.4$ \\
        \HydraGT{}      & $\mathbf{95.5}$ & $\mathbf{82.1}$ & $\mathbf{86.5}$ & $\mathbf{83.3}$ \\
        \bottomrule
    \end{tabular}
    \caption{Object accuracy for \name and \sgf. Best results in {\bf bold}.}\label{tab:sota}
    \togglevspace{-5mm}{0mm}
\end{table}

\subsubsection{Comparison against \sgf{}}\label{sec:sgf}

This section shows that \name{} produces better object maps compared to \sgf{}~\citep{Wu21cvpr-SceneGraphFusion}, while also creating additional layers in the scene graph (\sgf{} does not estimate rooms or places).
We compare our system with \sgf{}~\citep{Wu21cvpr-SceneGraphFusion} for both the uHumans2 apartment and office scene.
For this comparison,
the only learned component used by our system is the off-the-shelf 2D semantic segmentation network, hence for a fair comparison, we do not retrain the object label prediction GNN used by \sgf{}~\citep{Wu21cvpr-SceneGraphFusion}.
Examples of the produced objects by \name{} and \sgf{} are shown in \cref{fig:sota}.
Note that \sgf{} tends to over-segment larger objects (\eg{} the sofa in the lower right corner of \cref{fig:sota} or the bed in the top right corner of the same figure), while \name{} tends to under-segment nearby objects (such as the dining table and chairs in the bottom left of \cref{fig:sota}).
For the purposes of a fair comparison, we use OneFormer~\citep{Jain23cvpr-Oneformer} to provide semantics for \name{}, and use ground-truth robot poses as input to both systems.
A quantitative comparison is reported in \cref{tab:sota}, which also includes results for Hydra using ground-truth semantics as a upper bound on performance.
We report two metrics.
The first, \emph{Percent Correct}, is the percent of estimated objects that are within some distance threshold (\SI{0.5}{\meter}) of a ground-truth object (as provided by the simulator for uHumans2).
The second, \emph{Percent Found}, is the percent of ground-truth objects that are within some distance threshold (also \SI{0.5}{\meter}) of a ground-truth object (as provided by the simulator for uHumans2).
As the label space of \sgf{} for objects does not line up well with the ground-truth object label space, and as \sgf{} does a poor job predicting object labels for the two scenes used for comparison (see \cref{fig:sota}), we do not take the semantic labels into account when computing the metrics (as opposed to the stricter metrics used in \cref{sec:ablation} for examining object accuracy, which do take the semantic labels of the object into account).

\Cref{tab:sota} shows that \name{} largely outperforms \sgf{} in terms of both \emph{Percent Correct} and \emph{Percent Found}, even after disregarding the incorrect semantic labels for the objects produced by \sgf{}.
In hindsight, \name{} can benefit from more powerful 2D segmentation networks and estimate better objects, while \sgf{} directly attempts to extract semantics from the 3D reconstruction, which is arguably a harder task and does not benefit from the large datasets available for 2D image segmentation.
The last row in the table  ---\name{} (GT)--- shows that \name{}'s accuracy can further improve with a better 2D semantic segmentation.
While \sgf{} is less competitive in terms of object understanding, we remark that
it predicts a richer set of object relationships for each edge between objects~\citep{Wu21cvpr-SceneGraphFusion} (\eg{} \semantic{standing on}, \semantic{attached to}) compared to \name{}, which might be also useful for certain applications.

\subsubsection{Accuracy Evaluation and Ablation}\label{sec:ablation}

Here, we examine the accuracy of \name{}, running in real-time, as benchmarked against the accuracy attained by constructing a 3D scene graph in an offline manner (\ie{} by running Kimera~\citep{Rosinol21ijrr-Kimera}).
We also present a detailed quantitative analysis to justify key design choices in \name{}. To do this, we break down this quantitative analysis across the objects, places, and rooms layers.
We examine the impact of (i) the choice of loop closure detection and 2D semantic segmentation on the accuracy of each layer of the 3D scene graph, and (ii) the impact of model architectures and training parameters on the accuracy of the room classification network.
As such, we first present the accuracy of the predicted object and place layers by \name{}.
We then move to looking at room classification accuracy, including two ablation studies on the neural tree and an evaluation of our room clustering approach based on persistent homology.
Finally, we examine the quality of predicted loop closures when using the proposed hierarchical descriptors.

When benchmarking the accuracy of the layers of \name{}, we consider five different configurations for odometry estimation and loop closure detection.
The first configuration (``\igx'') uses ground-truth poses to incrementally construct the scene graph and disregards loop closures.
The second, third, and fourth configurations (``\ivl'', ``\ivd'', and ``\ivn'', respectively) use visual-inertial odometry (VIO) for odometry estimation and use vision-based loop closures (\ivl), the proposed hierarchical handcrafted descriptors for scene graph loop closures (\ivd), or the proposed hierarchical learned descriptors for scene graph loop closures (\ivn).
The last configuration, \ivx, uses visual-inertial odometry without loop closures.

\begin{figure*}
    \centering
    \subfloat[uH2 Apartment]{\centering
        \includegraphics[width=0.49\textwidth]{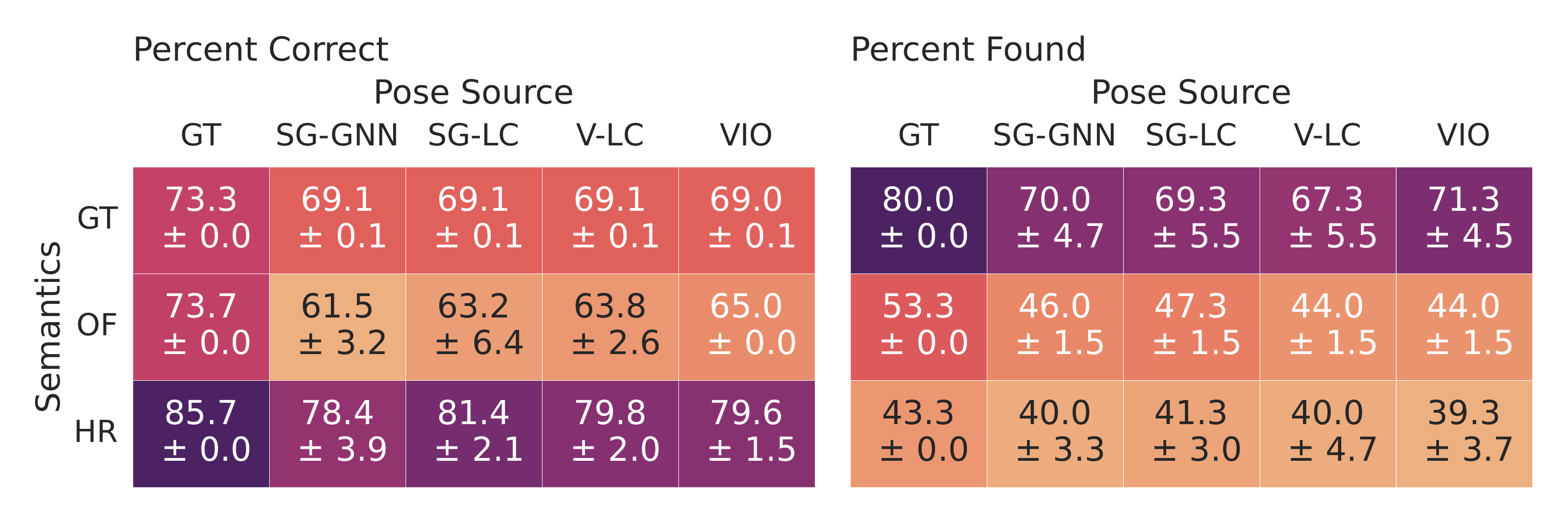}
    }
    \subfloat[uH2 Office]{\centering
        \includegraphics[width=0.49\textwidth]{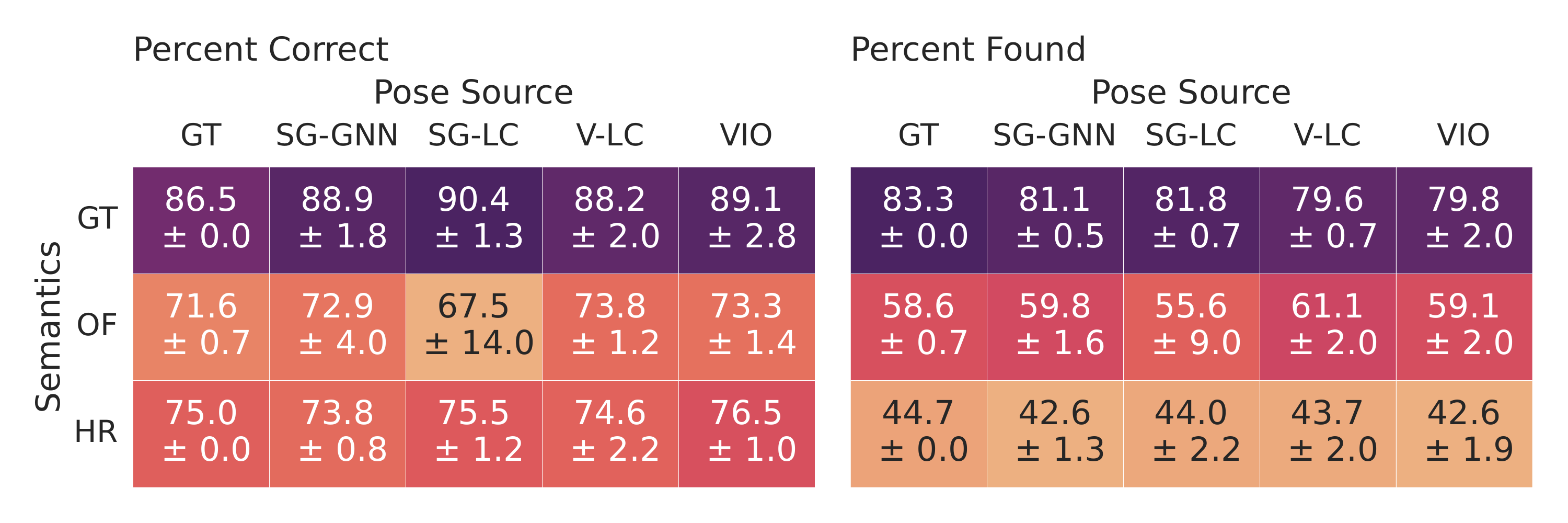}
    } \\
    \subfloat[SidPac]{\centering
        \includegraphics[width=0.49\textwidth]{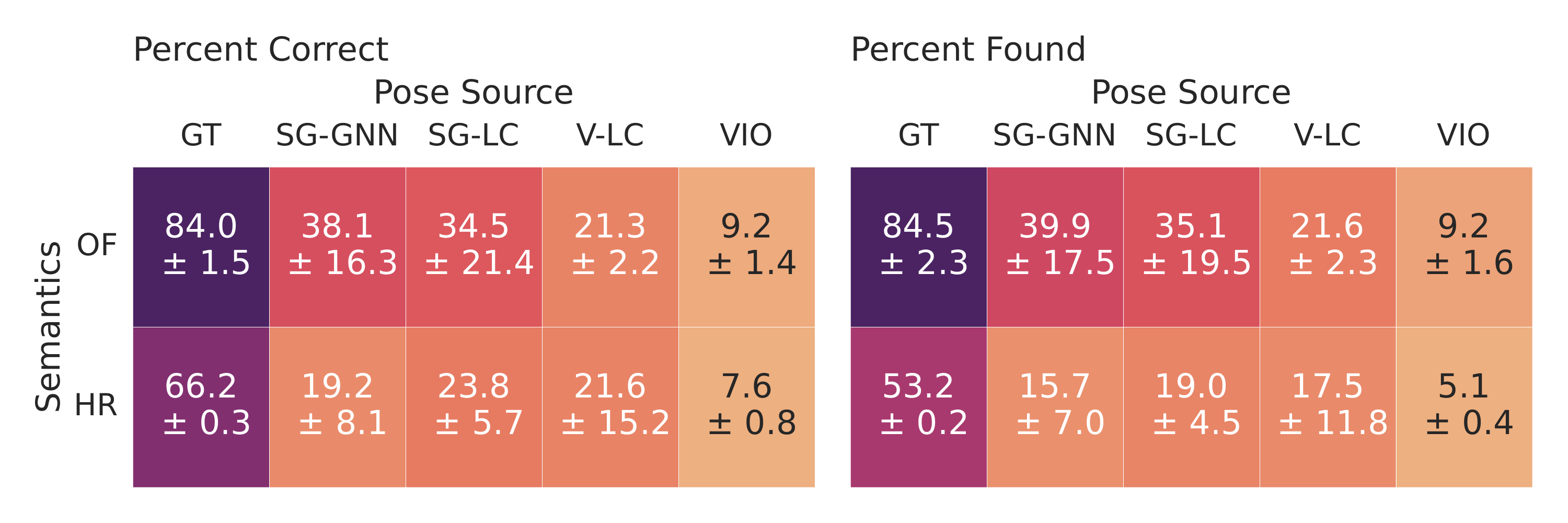}
    }
    \subfloat[Simmons Jackal]{\centering
        \includegraphics[width=0.49\textwidth]{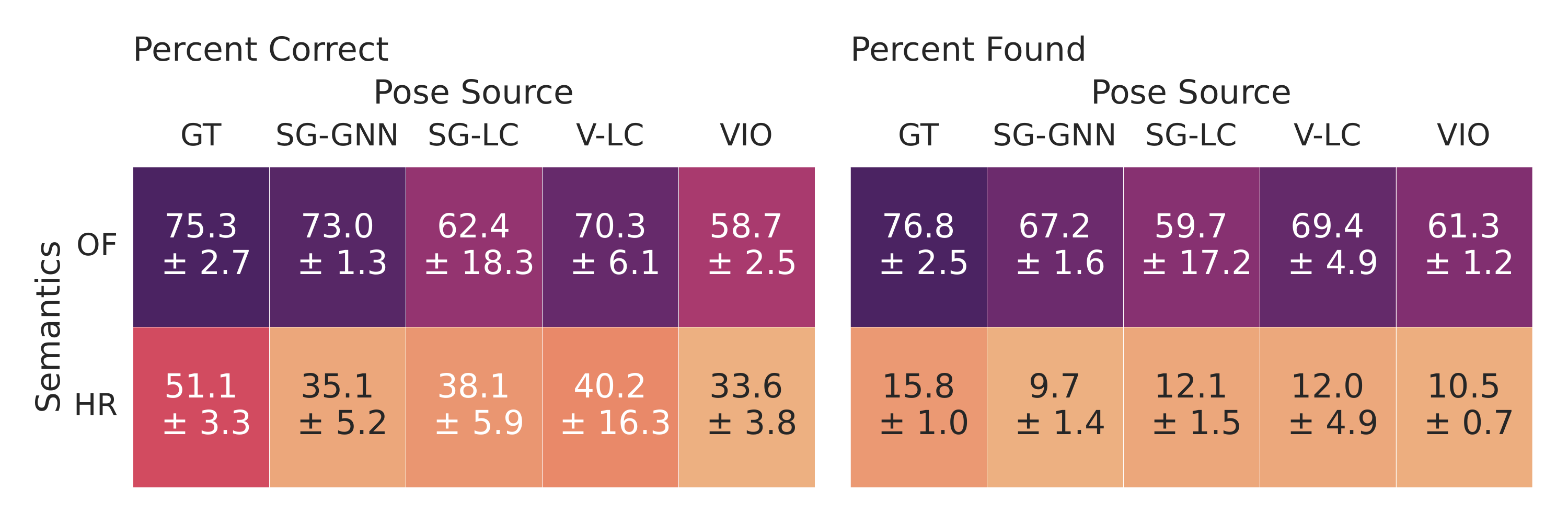}
    } \\
    \subfloat[Simmons A1]{\centering
        \raisebox{0.55cm}{\includegraphics[width=0.49\textwidth]{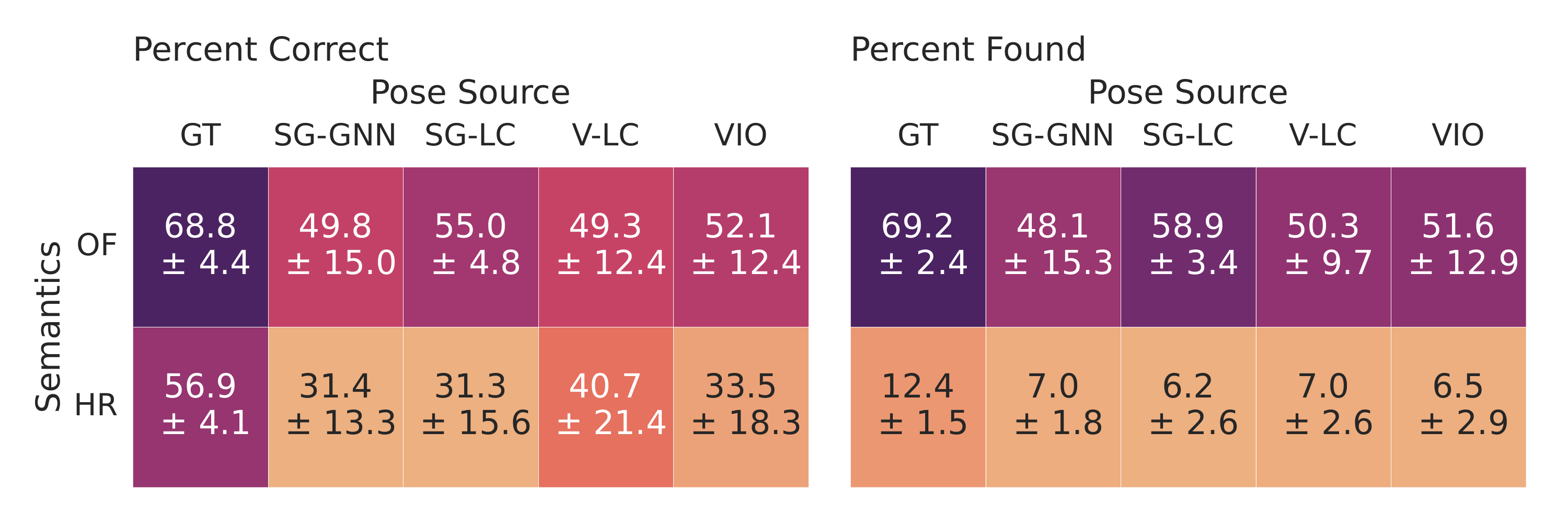}}
    }
    \subfloat[Place Metrics]{\label{fig:place_metrics}\centering
        \includegraphics[width=0.49\textwidth,trim=0cm 0cm 0cm 2.5cm,clip]{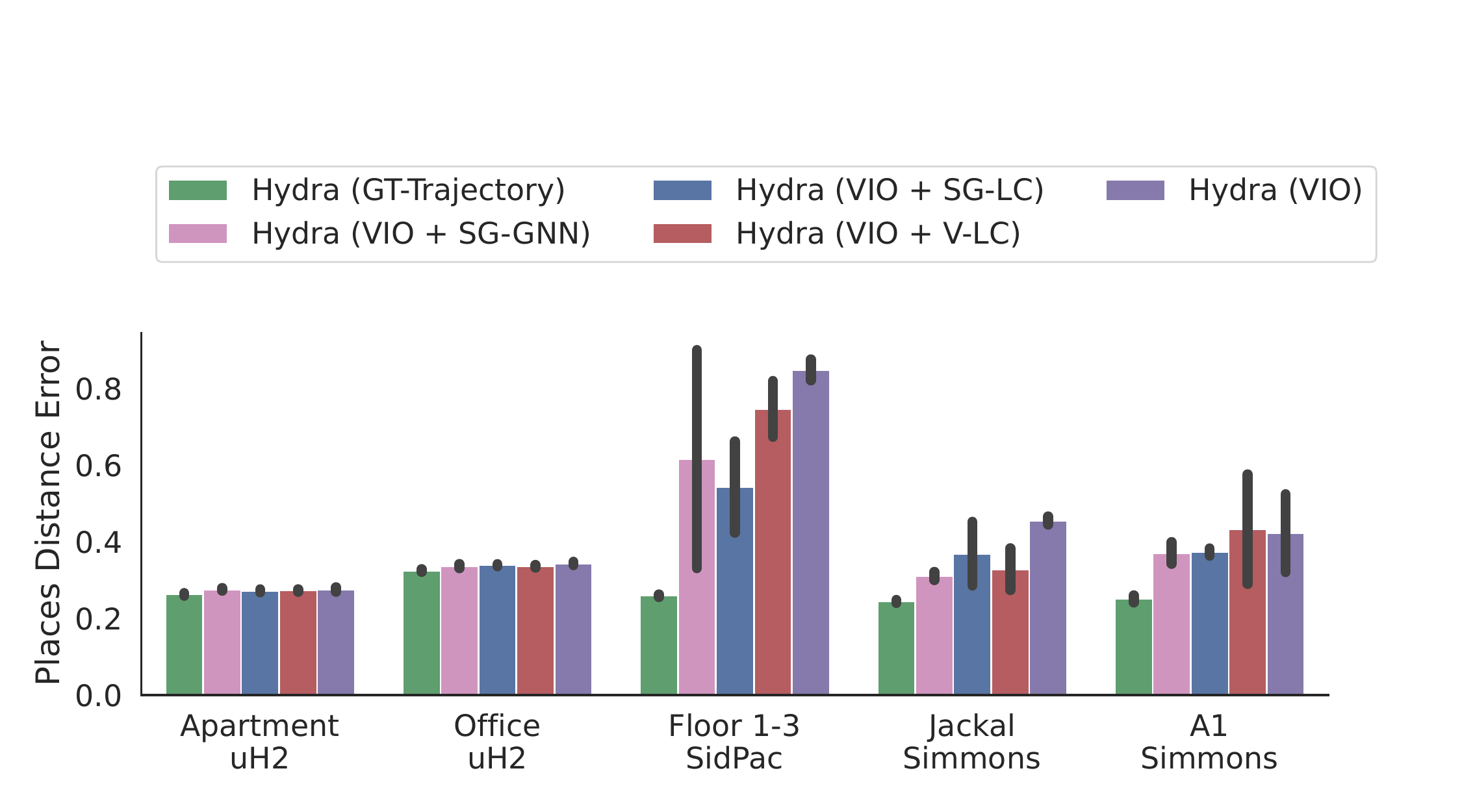}
    }
    \caption{(a-e) Object accuracy metrics for all scenes across different 2D semantic segmentation sources and different LCD and pose source configurations.
The metrics are averaged over 5 trials for each pair of configurations (\ie{} there were 5 trials for every combination of 2D semantics and a pose source).
For semantic segmentation sources, ``OF'' is Oneformer, ``HR'' is HRNet and ``GT'' is the ground truth semantic segmentation provided by the simulator. Each cell reports the mean and standard deviation. Values are colored from highest (dark purple, best performance) to lowest (light orange, worst performance) by mean.
(f) Place layer accuracy for all scenes across different configurations for \name.
    Each bar shows the average distance error for a given scene and configuration over 5 trials; standard deviation across the 5 trials is shown as a black error bar.}\label{fig:object_metrics}
    \togglevspace{-5mm}{0mm}
\end{figure*}

\myParagraph{Accuracy Evaluation: Objects}
\Cref{fig:object_metrics} evaluates the object layer of \name{} across the five  configurations of \name{}, as well as across various 2D semantic segmentation sources:  ground-truth semantics (when available), OneFormer, and HRNet.
For this evaluation, we benchmark against the ground-truth objects from the underlying Unity simulation for uHumans2.
As the recording for the uHumans2 scenes only explores a portion of the scene, we only compare against ground-truth objects that have a high enough number of mesh vertices inside their bounding boxes (40 vertices).
For the real datasets, where we do not have ground-truth objects, we consider the object layer of the batch scene graph constructed from ground-truth poses and OneFormer-based 2D semantic segmentation using Kimera~\citep{Rosinol21ijrr-Kimera} as the ground-truth objects for the purposes of evaluation.
For the object layer, we report two metrics:
the percentage of objects in the ground-truth scene graph that have an estimated object with the correct semantic label within a specified radius (``\percFound'') and the percentage of objects in the estimated scene graph that have a ground-truth object with the correct semantic label  within a specified radius (``\percCorrect'').\footnote{As distance thresholds, we use \SI{0.5}{\meter} for the uHumans2 Apartment and Office scene, and \SI{1}{\meter} for SidPac, the Simmons Jackal scene, and the Simmons A1 scene.
    These thresholds were chosen to roughly correspond to the mean Absolute Trajectory Error (ATE) for each scene, in order to normalize the metrics between simulated and real scenes.
}

We note some important trends in \cref{fig:object_metrics}.
First, when using \igx, \name{} produces objects that are reasonably close to the ground-truth (80--100\% found and correct objects for the Office scene), or close to the objects produced by offline approaches (70--80\% found and correct objects for SidPac and Simmons).
This demonstrates that ---given the trajectory--- the real-time scene graph from \name{} is comparable to the batch and offline approaches at the state of the art.
Second, \ivl{}, \ivd{}, and \ivn{} maintain reasonable levels of accuracy for the objects and attain comparable performance in small to medium-sized scenes (\ie{} the uHumans2 Apartment, Office, and Simmons A1) for the same choice of
semantic segmentation method.
In these scenes, the drift is small and the loop closure strategy does not radically impact performance (differences are within standard deviation).
However, in larger scenes (\eg{} SidPac) scene graph loop closures are more important and both \ivn{} and \ivd{} largely outperform \ivl{} and \ivx{} in terms of object accuracy.

Finally, as expected, the quality of the semantic segmentation impacts the accuracy; this is much more pronounced for ``\percFound{}'' than ``\percCorrect{}''.
Note that the two metrics are somewhat coupled; a lower ``\percFound{}'' can lead to a higher ``\percCorrect{}''; this is one reason for the inversion in performance between OneFormer and HRNet for the uHumans2 scenes.
Additionally, note that OneFormer identifies objects that exist in the ground-truth scene that are not exported by the simulator (this is another factor responsible for the lower ``\percCorrect{}'' of OneFormer).
We also note a slight correlation between reduced semantic segmentation performance and reduced accuracy for \ivn{} and \ivd{}, such as for the Simmons Jackal scene.
The low number of objects identified by HRNet in this scene contributes to fewer scene-graph based loop closures, and as such, worse performance.

\myParagraph{Accuracy Evaluation: Places}
\Cref{fig:place_metrics} evaluates the place layer by comparing the five configurations for \name.
We use either the ground-truth semantics (when available) or OneFormer (when ground-truth semantics is not available), as the places are not influenced by the semantic segmentation quality.
For the places evaluation, we construct a ``ground-truth'' \GVD{} from a 3D reconstruction of the entire scene using the ground-truth trajectory of the scene.
Using this, we measure the mean distance of an estimated place node to the nearest voxel in the ground-truth GVD, which we call ``\positionError''.
\Cref{fig:place_metrics} reports this metric, along with standard deviation across 5 trials shown as black error bars.

We note some important trends in \cref{fig:place_metrics}.
Similar to the case of objects, the place positions errors are almost identical for all configurations of \name{} for the uHumans2 scenes.
These scenes have very low drift, and are relatively small, so it is expected that a higher quality trajectory makes less of an impact.
For the larger scenes, we see a range of results.
However, we see for SidPac, \ivn{} performs worse than \ivd{}, though both outperform \ivl{} and \ivx{}.
Interestingly, we note that \ivd{} performs slightly worse than \ivl{} for the Simmons Jackal scene, while \ivn{} does the same as \ivl{} (but with lower standard deviation, as indicated by the black confidence bars in the plot).
Note that the Simmons Jackal scene has multiple uniform rooms with very similar objects and layouts.
In general, \ivx{} is outperformed by methods that incorporate loop closures, and the ability to correct for loop closures is important to maintaining an accurate 3D scene graph.

\myParagraph{Neural Tree Ablation 1: Node Classification}
We first replicate the semi-supervised node classification experiment described in~\cite{Talak21neurips-neuralTree}
and compare  different message passing architectures on the original graphs (\ie{} standard GNNs) and on the \Htree{} (\ie{} the proposed neural tree).
To construct the \Htree{} graphs, we apply the proposed hierarchical tree decomposition algorithm (\cref{algo:td-hierarchical}), which concatenates the tree decomposition of each layer.
The results are shown in \cref{tab:mp_acc_ablation}.
With the proposed tree decomposition approach, the neural tree achieves an advantage between 1.63\% to 11.72\% over standard GNN models, depending on the type of message passing architecture.
In comparison to the results in~\cite{Talak21neurips-neuralTree}, the results in \cref{tab:mp_acc_ablation} are generated using a later version of the PyTorch Geometric library, which supports heterogeneous GNN operators.
Also note that the tree decomposition used in \cref{tab:mp_acc_ablation} differs slightly from the tree decomposition algorithm used in~\cite{Talak21neurips-neuralTree}.
We provide additional results in \cref{sec:stanford3d_extra} that examine the impact of these changes.

\setlength{\tabcolsep}{5pt}
\begin{table}[t!]
    \toggleformat{}{\footnotesize}
    \centering
    \begin{tabular}{l cc}
        \toprule
        Message Passing & Original & H-tree \\
        \midrule
        GCN       & $42.91 \pm  2.01\%$ & $\mathbf{54.63 \pm 2.19\%}$ \\
        GraphSAGE & $56.97 \pm  2.02\%$ & $\mathbf{58.60 \pm 2.13\%}$ \\
        GAT       & $45.06 \pm  2.32\%$ & $\mathbf{53.71 \pm 2.10\%}$ \\
        GIN       & $48.03 \pm  2.21\%$ & $\mathbf{55.00 \pm 2.68\%}$ \\
        \bottomrule
    \end{tabular}
    \caption{Stanford3d: Node classification accuracy for different message passing architectures. Best results in {\bf bold}.}\label{tab:mp_acc_ablation}
    \togglevspace{0mm}{0mm}
\end{table}

\setlength{\tabcolsep}{\toggleformat{5pt}{4pt}}
\begin{table}[t!]
    \toggleformat{}{\footnotesize}
    \centering
    \begin{tabular}{ll cc}
        \toprule
        \multicolumn{2}{c}{Graph Types} & Original & H-tree \\
        \midrule
        \multirow{2}{*}{Homogeneous} & absolute pos. & $45.06 \pm	2.32\%$ & $\mathbf{53.71 \pm 2.10 \%}$ \\
                                     & relative pos. & $46.05 \pm 1.98 \%$ & $\mathbf{61.16 \pm 2.03 \%}$ \\
        \midrule
        \multirow{2}{*}{Heterogeneous} & absolute pos. & $\mathbf{46.56 \pm 2.42 \%}$ & $45.30 \pm 2.57 \%$ \\
                                       & relative pos. & $45.79 \pm 2.04 \%$ & $\mathbf{48.16 \pm 2.21 \%}$\\
        \bottomrule
    \end{tabular}
    \caption{Stanford3d: Node classification accuracy for different graph types and position encodings. Best results in {\bf bold}.}\label{tab:graph_acc_ablation}
    \togglevspace{0mm}{0mm}
\end{table}

\setlength{\tabcolsep}{\toggleformat{5pt}{3.5pt}}
\begin{table}[t!]
    \toggleformat{}{\footnotesize}
    \centering
    \begin{tabular}{ll cc}
        \toprule
        \multicolumn{2}{c}{Graph Types} & Original & H-Tree \\
        \midrule
        \multirow{2}{*}{w/o word2vec} & w/o room edges    & $\mathbf{41.28 \pm 0.48 \%}$ & $38.63 \pm 0.25 \%$ \\
                                      & with room edges   & $41.20 \pm 1.08 \%$ & $\mathbf{43.27 \pm 0.59 \%}$ \\
        \midrule
        \multirow{2}{*}{with word2vec} & w/o room edges  & $\mathbf{56.74 \pm 0.85 \%}$ & $55.84 \pm 0.37 \%$ \\
                                       & with room edges & $56.02 \pm 0.84 \%$ & $\mathbf{57.67 \pm 0.57 \%}$ \\
        \bottomrule
    \end{tabular}
    \caption{MP3D: Room classification accuracy for different graph types and node features. Best results in {\bf bold}.}\label{tab:room_acc_ablation}
    \togglevspace{0mm}{0mm}
\end{table}

We note that the absolute position of a node centroid is not invariant to translation, and that invariance is important  to generalize to new scenes regardless of the choice of coordinate frames.
We therefore examine using the relative positions of node centroids as edge features, instead of including them directly as node features.
For the \Htree, which contains clique nodes that are comprised of multiple objects or rooms, we use the mean room centroid as the centroid of the clique when computing relative positions.
In addition, we also investigate the impact of using heterogeneous graphs (which can accommodate different node and edge types, as the ones arising in the 3D scene graph) against standard homogeneous graphs.
For this comparison, we only use the GAT message passing function since it is the only one in \cref{tab:mp_acc_ablation} that can both handle heterogeneous message passing and incorporate edge features.

We present the results of this ablation study
in \cref{tab:graph_acc_ablation}.
Comparing the two position encodings, the neural tree models achieve significantly higher accuracy when using relative positions as edge features: 7.45\% on homogeneous graphs and 2.86\% on heterogeneous graphs. Choosing features that are translation invariant (\ie{} relative positions) has a clear advantage when using neural tree models.
The standard GNNs show near-identical performance between the two position encodings.
The heterogeneous graph structure has no significant impact on standard GNNs,
but it degrades performance of the neural trees.
The exact mechanism behind this decrease in performance in unclear; we posit that the heterogeneous neural tree variants have more parameters than the other architectures, and that the Standford3D dataset may not have enough training data for these architectures.
We do not observe a significant performance decrease between the homogeneous and heterogeneous neural tree variants for the MP3D dataset.

\myParagraph{Neural Tree Ablation 2: Room Classification}
We compare the neural tree against standard GNNs on a room classification task using the MP3D dataset,  predicting room semantic labels for the object-room graphs extracted from \name{}.
Compared to Stanford3d, this dataset contains both semantic labels of the object nodes and room layer connectivity from \name{}, in addition to the geometric features (\ie{} centroid position and bounding box size) of each node.
Therefore, we study the effect of these two additional pieces of information.
We use pre-trained word2vec vectors to represent object semantic labels and concatenate them with the geometric feature vectors.
As described previously, we filter out partially explored rooms, using a threshold of 60\% IoU.
As before, we use GAT as the message passing function when training, as all graphs are heterogeneous and have edge features.
The results are shown in \cref{tab:room_acc_ablation}.
Both approaches show a significant performance improvement using the semantic labels of the objects, ranging from 14\% to 17\%.
The neural tree also shows substantial improvement when incorporating room layer edges, while the standard GNNs do not.
The best-performing neural tree model is obtained when using semantic labels for the objects and accounting for room connectivity; this is the model we use in the rest of this paper.

\myParagraph{Accuracy Evaluation: Rooms}
We evaluate the accuracy of the room segmentation in Hydra, by first evaluating the quality of the geometric room clustering described in~\cref{sec:rooms_clustering} and then testing the room classification from~\cref{sec:rooms_classification} for different choices of 2D segmentation network.

\Cref{fig:room_metrics} evaluates the room clustering performance, using the precision and recall metrics defined in~\cite{Bormann16icra-roomSegmentationSurvey} (here we compute precision and recall over 3D voxels instead of 2D pixels).
More formally, these metrics are:
\begin{equation}\label{eq:room_pr}
\begin{split}
\text{\emph{Precision}} &= \frac{1}{|R_e|}\sum_{r_e \in R_e} \max_{r_g \in R_g} \frac{|r_g \cap r_e|}{|r_e|} \\
\text{\emph{Recall}} &= \frac{1}{|R_g|}\sum_{r_g \in R_g} \max_{r_e \in R_e} \frac{|r_e \cap r_g|}{|r_g|}
\end{split}
\end{equation}
where $R_e$ is the set of estimated rooms, $R_g$ is the set of ground-truth rooms, and $|\cdot|$ returns the cardinality of a set;
here, each room $r_e$ (or $r_g$) is defined as a set of free-space voxels.
We hand-label the {ground-truth} rooms $R_g$ from the ground-truth reconstruction of the environment.
In particular, we manually define sets of bounding boxes for each room and then identify the set of free-space voxels for $r_g$ as all free-space voxels that fall within any of the defined bounding boxes for $r_g$.
For the estimated rooms $R_e$, we derive the set of free-space voxels $r_e$ from the places comprising each estimated room.
In~eq.~\eqref{eq:room_pr}, \emph{Precision} then measures the maximum overlap in voxels with a ground-truth room for every estimated room, and \emph{Recall} measures the maximum overlap in voxels with an estimated room for every ground-truth room.
Intuitively, low precision corresponds to under-segmentation, \ie{} fewer and larger room estimates, and low recall corresponds to over-segmentation, \ie{} more and smaller room estimates.
For benchmarking purposes, we also include the approach in~\cite{Rosinol21ijrr-Kimera} (\emph{Kimera}) as a baseline for evaluation.

\Cref{fig:room_metrics} shows that \name{} generally outperforms Kimera~\citep{Rosinol21ijrr-Kimera} when given the ground-truth trajectory, and that \name{} is slightly more robust to multi-floor environments.
This is expected, as Kimera performs room segmentation by taking a 2D slice of the voxel-based map of the environment, which does not generalize to multi-floor scenes.
For both the split-level Apartment scene and the multi-floor SidPac scene, we achieve higher precision as compared to Kimera.
These differences stem from the difficulty of setting an appropriate height to attempt to segment rooms at for Kimera (this is the height at which Kimera takes a 2D slice of the \ESDF).
Finally, it is worth noting that our room segmentation approach is able to mostly maintain the same levels of precision and recall for \ivn{}, \ivd{} and \ivl{}. In some cases, our approach outperforms Kimera, despite the drift inherent in \ivn{}, \ivd{} and \ivl{} (Kimera uses ground-truth poses).

Finally, we report room classification accuracy for uHumans2, SidPac, and Simmons across semantic segmentation sources in \cref{tab:room_acc_hydra} using the proposed neural tree. We manually assign a ground-truth room category to each hand labeled ground-truth room, and compute the percentage of estimated room labels that match their corresponding ground-truth category.
Label correspondences are computed by picking the ground-truth room that contains the most place nodes that comprise each estimated room; an estimated room is assigned a correspondence to \semantic{unknown} if too few place nodes (less than 10) fall inside the best ground-truth room.

We note two interesting trends.
First, the richness of the object label space impacts the predicted room label accuracy, as shown by the relatively poor performance of the GT column as compared to either HRNet or OneFormer for the uHumans2 scenes ($27$\% versus $40$\% or $45\%$): the GT semantics available in the simulator has a smaller number of semantic classes for the objects, hindering performance. This is consistent with the observations in~\cite{Chen22arxiv-LLM2}.
Second, scenes such as the uHumans2 office or Simmons that are out of distribution ---compared to the MP3D dataset we use for training--- perform poorly.
An example of a resulting 3D scene graph for SidPac with room category labels is shown in \cref{fig:intro}.

\begin{figure}[t]
    \centering
    \includegraphics[width=0.95\columnwidth,trim=0 7mm 0 0,clip]{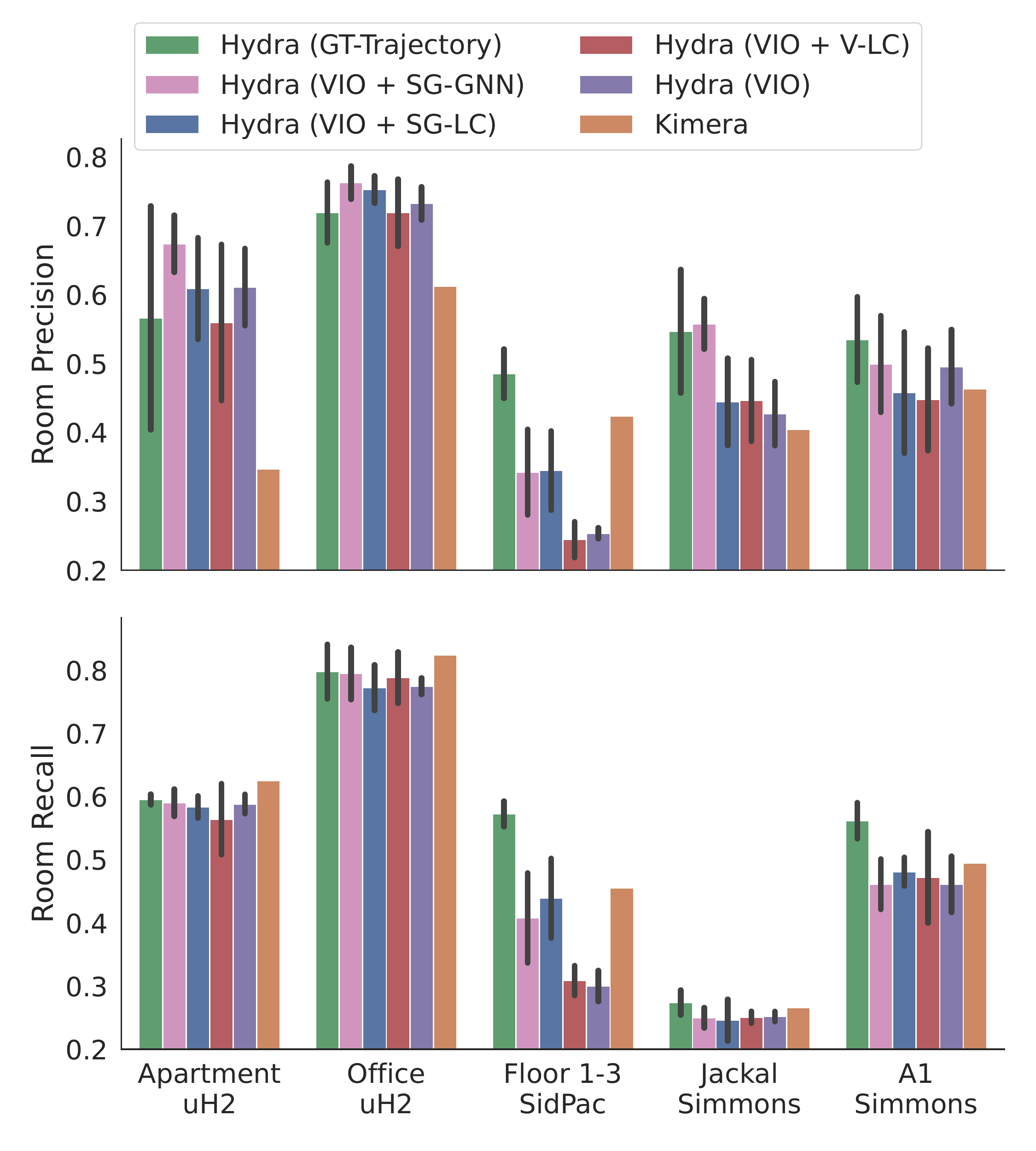}
    \caption{Room metrics for all scenes across different LCD pose source configurations. Each bar shows the average precision or recall over 5 trials; standard deviation across the 5 trials is shown as a black error bar}\label{fig:room_metrics}
    \togglevspace{-1mm}{0mm}
\end{figure}

\setlength{\tabcolsep}{\toggleformat{5pt}{4pt}}
\begin{table}[t!]
    \toggleformat{}{\footnotesize}
    \centering
    \begin{tabular}{l ccc}
        \toprule
         & GT & HRNet & OneFormer \\
        \midrule
       uHumans2 Apartment & $ 26.7 \pm   3.7$\% & $ 38.0 \pm  21.7$\% & $ 45.0 \pm  11.2$\% \\
         uHumans2 Office & $ 27.6 \pm   7.5$\% & $ 28.4 \pm   6.9$\% & $ 27.0 \pm  10.1$\% \\
       SidPac Floor 1--3 &                 N/A & $ 46.2 \pm  11.9$\% & $ 47.7 \pm  12.5$\% \\
          Simmons Jackal &                 N/A & $ 32.3 \pm  16.6$\% & $ 15.3 \pm   6.5$\% \\
              Simmons A1 &                 N/A & $ 29.0 \pm  24.4$\% & $ 38.0 \pm  28.0$\% \\
        \bottomrule
    \end{tabular}
    \caption{Room classification accuracy for \name{}.}\label{tab:room_acc_hydra}
    \togglevspace{0mm}{0mm}
\end{table}

\begin{figure}
    \centering
    \includegraphics[trim=0mm 0 0mm 0,clip,width=\columnwidth]{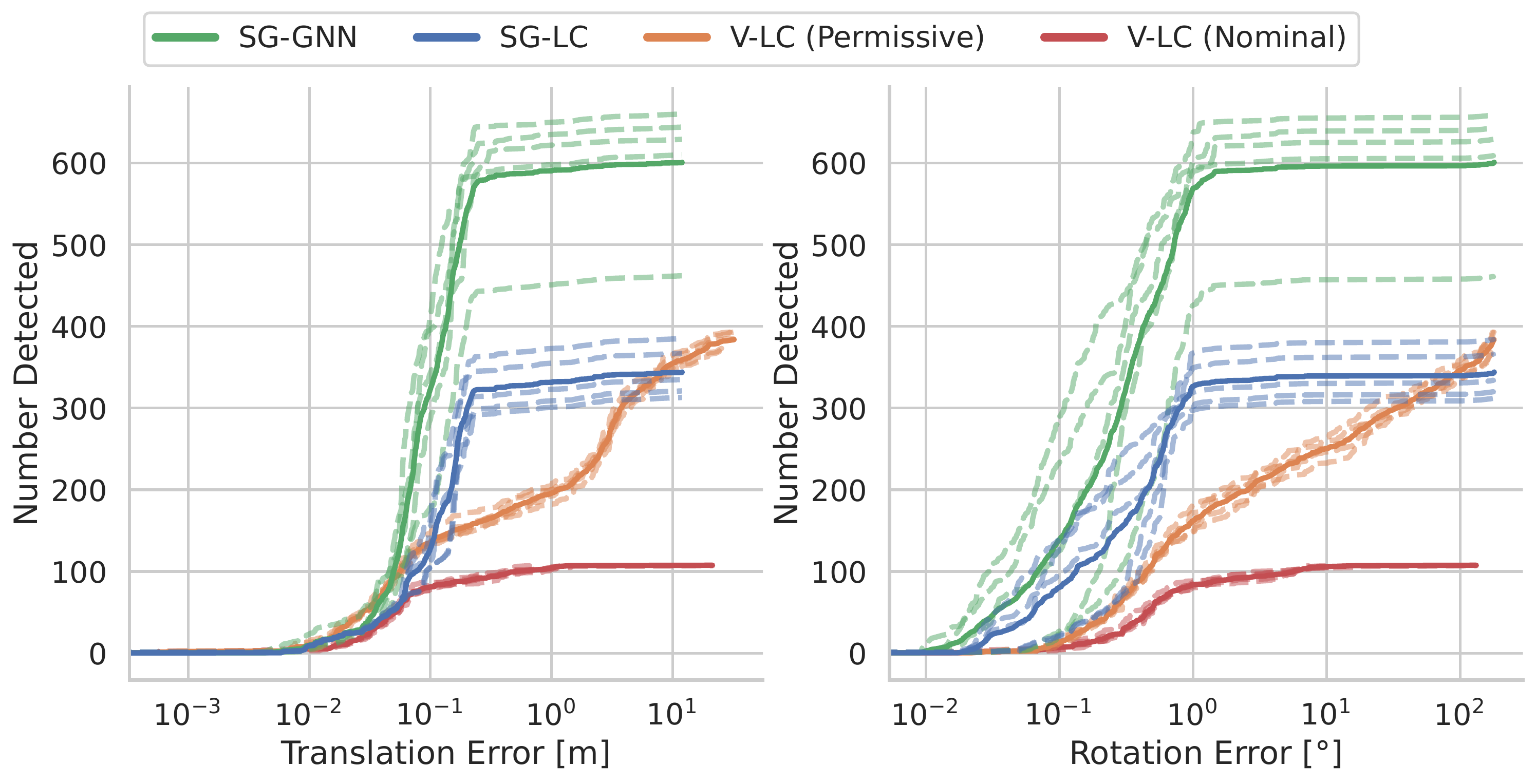}
    \caption{Number of detected loop closures versus error of the estimated loop closure pose for four different loop closure detection configurations.
             Five individual trials and a trend-line are shown for each configuration.}\label{fig:lcd_experiment}
     \togglevspace{0mm}{0mm}
\end{figure}

\myParagraph{Loop Closure Ablation}
Finally, we take a closer look at the quality of the loop closures candidates proposed by our hierarchical loop closure detection approach, and compare it against traditional vision-based approaches on the Office scene.
In particular, we compare our approach against a vision-based loop closure detection that uses DBoW2 for place recognition and ORB feature matching, as described in~\cite{Rosinol21ijrr-Kimera}.

\Cref{fig:lcd_experiment} shows the number of detected loop closures against the error of the registered solution (\ie{} the relative pose between query and match computed by the geometric verification) for four different loop closure configurations:
(i) ``SG-LC'': the proposed handcrafted scene graph loop closure detection,
(ii) ``SG-GNN'': the proposed learned scene graph loop closure detection,
(iii) ``V-LC (Nominal)'': a traditional vision-based loop closure detection with nominal parameters,
and (iv) ``V-LC (Permissive)'': a vision-based loop closure detection with more permissive parameters (\ie{} a decreased score threshold and less restrictive geometric verification settings).
We report key parameters used in this evaluation in \cref{app:lcd}.
As expected, making the vision-based detection parameters more permissive leads to more but lower-quality loop closures.
On the other hand, the scene graph loop closure approach produces approximately twice as many loop closures within \SI{10}{\centi\meter} of translation error and \SI{1}{\degree} of rotation error as the permissive vision-based approach.
The proposed approach produces quantitatively and quantitatively better loop closures compared to both baselines.
Notably, SG-GNN also outperforms both vision baselines and the SG-LC\@.
We present further discussion and a breakdown of additional statistics of the proposed loop closures of each method in \cref{app:lcd-stats}.

\setlength{\tabcolsep}{\toggleformat{5pt}{4pt}}
\begin{table}[t!]
    \toggleformat{}{\footnotesize}
    \centering
    \begin{tabular}{ccl cc}
        \toprule
        & & & \multicolumn{2}{c}{p@10} \\
        \cmidrule(l{2pt}r{2pt}){4-5}
        & & & IoU $ = 0.4$ & IoU = $0.6$ \\
        \midrule
        \multirow{6}{*}{Objects}
            & \multirow{3}{*}{MP3D}& Handcrafted & $\mathbf{70.9}$ & $\mathbf{58.6}$ \\
& & Learned+OneHot & $65.5$ & $56.0$ \\
                                    & & Learned+Word2Vec & $60.3$ & $50.9$ \\
            \cmidrule(l{2pt}r{2pt}){2-5}
            & \multirow{3}{*}{uH2 Office}& Handcrafted & $46.5$ & $\mathbf{31.5}$ \\
& & Learned+OneHot & $\mathbf{47.4}$ & $31.4$ \\
                                    & & Learned+Word2Vec & $37.6$ & $26.5$ \\
        \midrule
        \multirow{4}{*}{Places}
            & \multirow{2}{*}{MP3D}& Handcrafted & $\mathbf{76.4}$ & $\mathbf{58.0}$ \\
& & Learned & $59.5$ & $47.4$ \\
            \cmidrule(l{2pt}r{2pt}){2-5}
            & \multirow{2}{*}{uH2 Office}& Handcrafted & $\mathbf{68.6}$ & $\mathbf{41.4}$ \\
& & Learned & $54.1$ & $34.9$ \\
        \bottomrule
    \end{tabular}
    \caption{p@k results for loop closure detection. Best results in {\bf bold}.}\label{tab:lcd_precision}
    \togglevspace{-5mm}{0mm}
\end{table}

To further examine the relative performance of SG-LC and SG-GNN, we evaluate the top-k precision of both approaches on both the test split of the MP3D dataset used to train the descriptors, and the uHumans2 office scene.
For this metric, we compute the percent of the $k$-highest-scored\footnote{We map distances between descriptors to a $[0, 1]$ range, where $1$ corresponds to a match, and $0$ corresponds to a non-match
} descriptors for each query descriptor that are valid matches; two descriptors are determined to match if the bounding box of their corresponding \subgraphs have a IoU above a specified threshold.
We include two configurations of SG-GNN for object \subgraphs in this analysis; one that uses a one-hot encoding and one that uses word2vec embeddings to represent the semantic labels of the object nodes.
We report this metric for both SG-LC (handcrafted) and SG-GNN (learned) in \cref{tab:lcd_precision} for two IoU thresholds: 0.4 and 0.6.

We note some interesting trends in \cref{tab:lcd_precision}.
First, the one-hot encoding outperforms the word2vec encoding for both datasets, and appears to transfer better between datasets (\ie{} showing 5\% better performance for the MP3D dataset, but 10\% better performance for uHumans2).
Additionally, the original handcrafted descriptors maintain good performance compared to the learned descriptors, and only the learned object descriptors appear competitive to the handcrafted descriptors in terms of precision.
We believe that the high performance of the handcrafted descriptors is due to the semantic and geometric diversity among the scenes of the MP3D dataset.
Previous experiments (using the \ivn{} configuration of \name{}) imply that the learned descriptors offer improved performance in environments with a more uniform object distribution (\ie{} the Simmons Jackal scene).

\subsubsection{Onboard Operation on a Robot}\label{sec:robot_operation}

\begin{figure}
    \centering
    \includegraphics[trim={0cm, 1cm, 0cm, 2cm}, clip, width=\columnwidth]{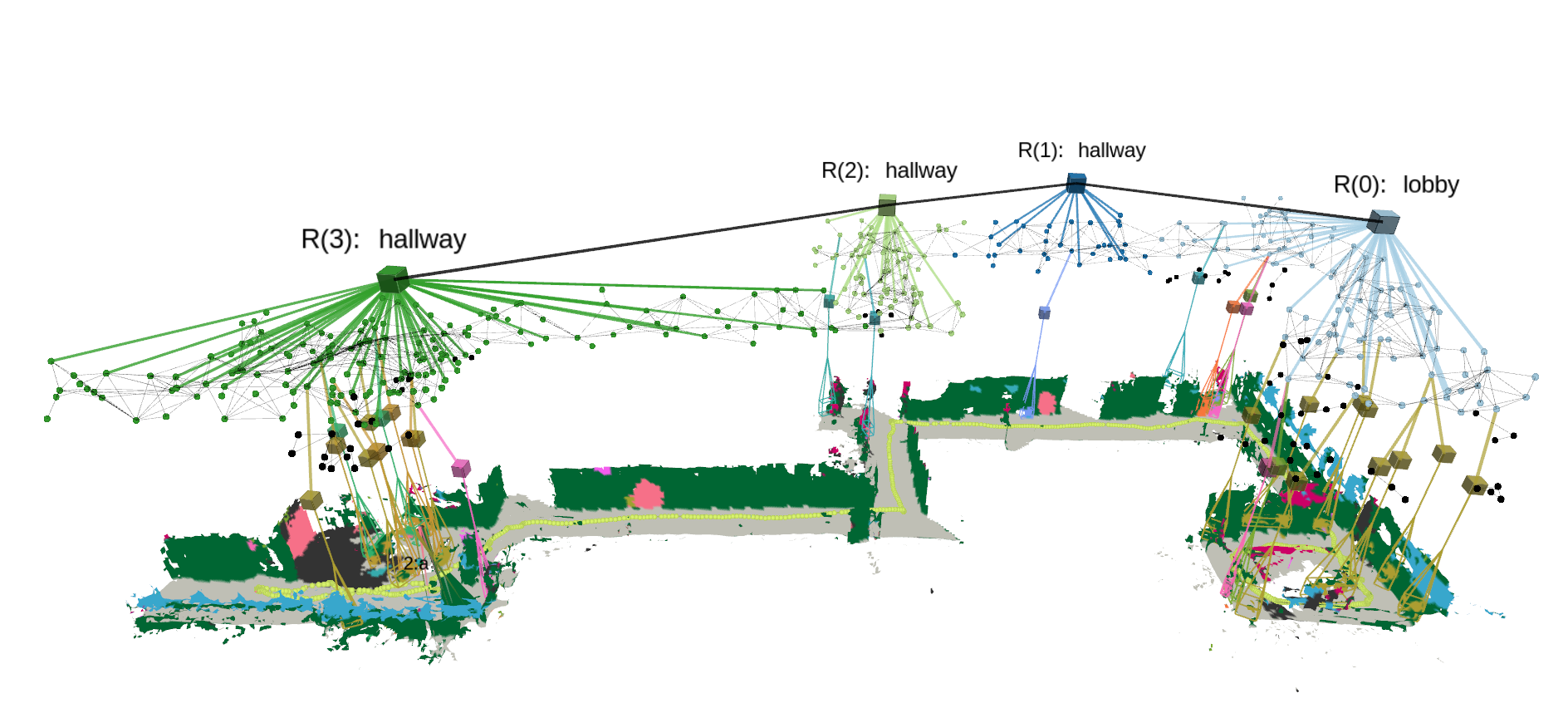}
    \caption{Intermediate 3D scene graph created by \name{} onboard the A1 quadruped while exploring building 31 on the MIT campus.
    Estimated room labels are shown.
    \toggleformat{}{The entire mapping sequence is shown in the video supplement.}}\label{fig:a1_qualitative}
    \togglevspace{-3mm}{0mm}
\end{figure}

This section shows that \name{} is capable of running in real-time and is deployable to a robot.
We show this by performing a qualitative experiment, running \name{} online on the Unitree A1.
We run \name{} and the chosen 2D semantic segmentation network (MobilenetV2~\citep{Sandler18cvpr-MobileNetV2}) on the Nvidia Xavier NX mounted on the back of the A1.
Additionally, we run our room classification approach in the loop with \name{} on the CPU of the same Nvidia Xavier.
In this test, to circumvent the computational cost of running Kimera-VIO, we use an Intel RealSense T265 to provide odometry estimates to \name{}.
As a result, we use our proposed scene-graph loop closure detection method without the appearance-based descriptors (which would rely on Kimera-VIO for computation of the visual descriptors and vision-based geometric verification).
To maintain real-time performance, we configure the frontend of \name{} to run every fifth keyframe (instead of every keyframe), and limit the reconstruction range to \SI{3.5}{\meter} (instead of the nominal \SI{4.5}{\meter}).
Note that this results in a nominal update rate of \SI{1}{\hertz} for \name{}, though we still perform dense 3D metric-semantic reconstruction at keyframe rate (\SI{5}{\hertz}). A breakdown of the runtime of \name{}'s modules during this experiment is available in \cref{sec:runtime}.

For this experiment, we partially explore a floor of building 31 on the MIT campus consisting of a group of cubicles, a conference room, two impromptu lounge areas, all of which are connected by a hallway.
An intermediate scene graph produced by \name{} while running the experiment is shown in \cref{fig:a1_qualitative}, and \toggleformat{video is available of the experiment.\footnote{\videoURL}}{the experiment itself is shown in the video supplement.}
\name{} estimates four rooms for the scene graph; of these four rooms, room 3 is centered over one of the two lounges, and room 0 covers both the conference room (located in the lower right corner of \cref{fig:a1_qualitative}) and a portion of the hallway.
Qualitatively, \name{} over-segments the scene, but the produced rooms and labels are still somewhat consistent with the underlying room structure and labels.
In this instance, only room 0 has an incorrectly estimated label; however the room categories that \name{} estimates over the course of the experiment are not as consistent.
This is likely due to the poor quality of the 2D semantics from MobilenetV2, and general lack of useful object categories for inferring room labels.

\begin{figure}
    \centering
    \includegraphics[trim={0mm 9mm 0mm 0}, clip, width=0.95\columnwidth]{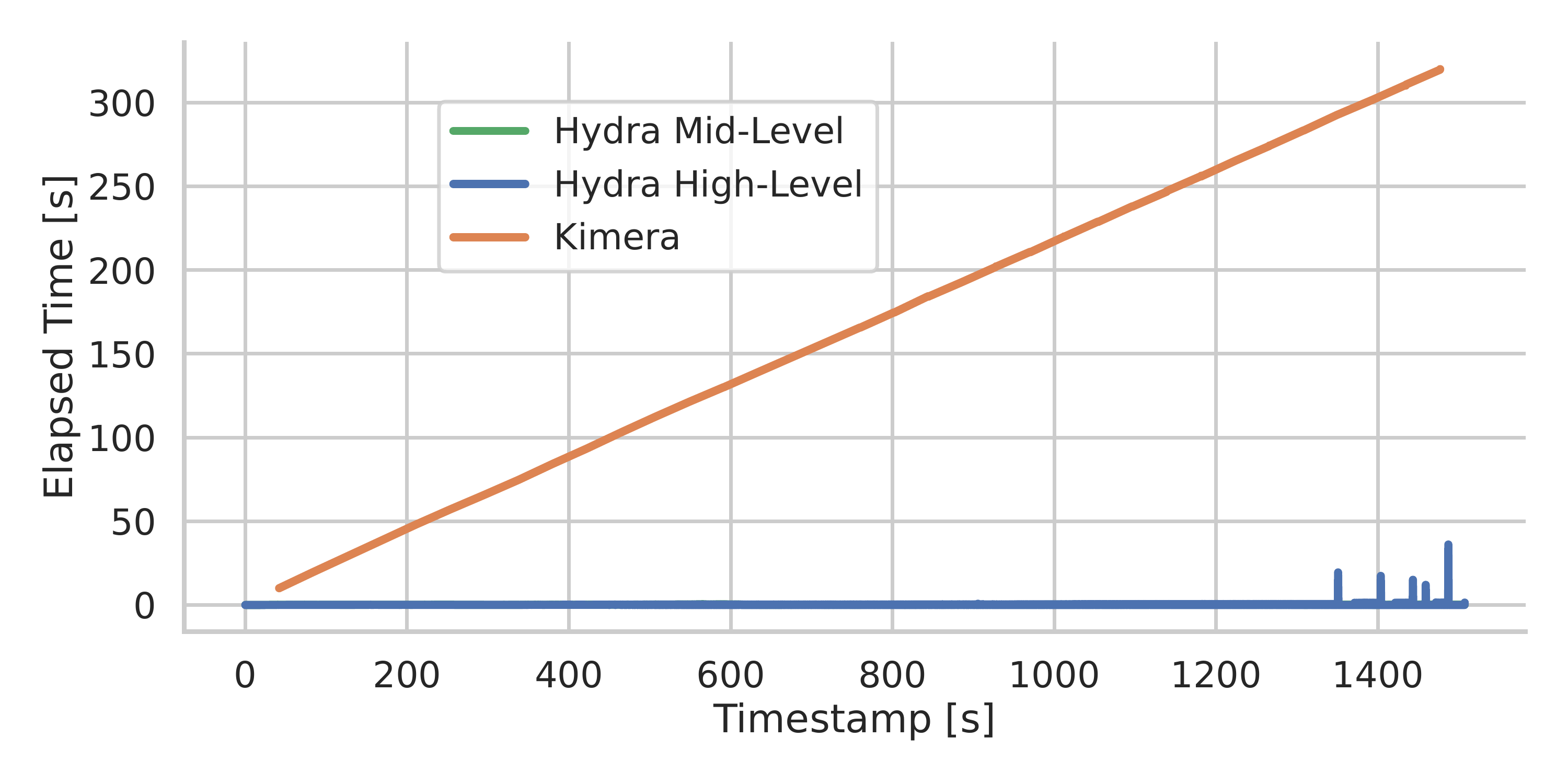}
    \caption{Runtime required for scene graph construction vs.\ timestamp for the SidPac Floor 1--3 dataset for a batch approach (Kimera) and for the proposed incremental approach (\name).
             For timing of the low-level processes in \name, we refer the reader to the analysis in~\cite{Rosinol21ijrr-Kimera}, as we also rely on Kimera-VIO\@.}\label{fig:timing}
     \togglevspace{0mm}{0mm}
\end{figure}

\setlength{\tabcolsep}{5pt}
\begin{table}[t!]
    \toggleformat{}{\footnotesize}
    \centering
    \begin{tabular}{l c c c}
        \toprule
        & Objects [ms] & Places [ms] & Rooms [ms] \\
        \midrule
        uH2 Apartment     & $52.8 \pm 14.9$ & $11.3 \pm 6.1$ & $1.8 \pm 0.9$ \\
        uH2 Office        & $34.0 \pm 26.1$ & $12.5 \pm 6.9$ & $14.6 \pm 11.9$ \\
        SidPac Floor 1--3 & $57.4 \pm 55.7$ & $15.7 \pm 9.2$ & $5.9 \pm 9.0$ \\
        Simmons Jackal    & $63.9 \pm 45.8$ & $19.6 \pm 13.1$ & $6.6 \pm 9.0$ \\
        Simmons A1        & $71.1 \pm 50.5$ & $18.6 \pm 11.7$ & $1.4 \pm 1.0$ \\
        \bottomrule
    \end{tabular}
    \caption{\name: timing breakdown.}\label{tab:component_timing}
    \togglevspace{-4mm}{0mm}
\end{table}

\subsubsection{Runtime Evaluation}\label{sec:runtime}

\Cref{fig:timing} reports the runtime of \name{} versus the batch approach in~\cite{Rosinol21ijrr-Kimera}.
This plot shows that the runtime of the batch approach increases over time and takes more than five minutes to generate the entire scene graph for moderate scene sizes;
as we mentioned, most processes in the batch approach~\citep{Rosinol21ijrr-Kimera} entail processing the entire \ESDF{} (\eg{} place extraction and room detection), inducing a linear increase in the runtime as the \ESDF{} grows.
On the other hand, our scene graph frontend (\emph{\name{} Mid-Level} in \cref{fig:timing}) has a fixed computation cost.
In \cref{fig:timing}, a slight upward trend is observable for \emph{\name{} High-Level}, driven by room detection and scene graph optimization computation costs, though remaining much lower than batch processing.
Noticeable spikes in the runtime for \emph{\name{} High-Level} (\eg{} at 1400 seconds) correspond to the execution of the 3D scene graph optimization when new loop closures are added.

When reconstructing the scene shown in \cref{fig:timing} (SidPac Floor1--3), \name{} uses a maximum of \mb{7.2} for the \TSDF{}, \mb{19.1} for the storage of the semantic labels, and \mb{47.8} for the storage of the dense \GVD{} inside the active window, for a total of \mb{74.1}.
Kimera  instead uses \mb{79.2} for the \TSDF{}, \mb{211} for the storage of semantic labels, and \mb{132} for the storage of the \ESDF{} when reconstructing the entire environment, for a total of \mb{422} of memory.
Note that the memory storage requirement of \name{} for the active window is a fifth of the memory required for Kimera, and that the memory usage of Kimera grows with the size of the scene.

\Cref{tab:component_timing} reports timing breakdown for the incremental creation of each layer across scenes for a single trial.
The object layer runtime is determined by the number of mesh vertices in the active window and the number of possible object semantic classes.
The room layer runtime is determined by the number of places (a combination of how complicated and large the scene is);
this is why the Office has the largest computation cost for the rooms despite being smaller than the SidPac scenes.
Note that our target keyframe rate implies a limit of \SI{200}{\milli\second} for any of these processes. While the mean and standard deviation of all layers are well below this rate, we do see that the real-time rate is exceeded in some cases for extracting the objects.
This specifically occurs for larger scenes (\ie{} SidPac and Simmons Jackal), due to the presence of many objects.
At the same time, we remark that \name{} is architected in such a way that temporarily exceeding the real-time threshold does not preclude online performance and real-time operation is restored shortly after these delays occur.

While the timing results in \cref{tab:component_timing} are obtained with a relatively powerful workstation, here we restate that \name{} can run in real-time on embedded computers commonly used in robotics applications.
Towards this goal, we report the timing statistics from the online experiment running \name{} onboard the Unitree A1 as shown in \cref{fig:a1_qualitative}.
\name{} processes the objects in \SI[parse-numbers=false]{83.9 \pm{} 65}{\milli\second}, places in \SI[parse-numbers=false]{114.8 \pm{} 103}{\milli\second}, and rooms in \SI[parse-numbers=false]{34.7 \pm{} 37.6}{\milli\second}.
While these numbers imply that \name{} can run faster than the \SI{1}{\hertz} target rate on the Xavier, note that the \SI{1}{\hertz} limit is chosen to not fully max out the computational resources of the Xavier.
Additionally, there are other modules of \name{} and external processes that limit (sometimes significantly) the computational resources of the Xavier (\eg{} the reconstruction of the \TSDF{} and \GVD{}).
While there is still margin to optimize computation (see conclusions), these initial results stress the practicality and real-time capability of \name{} in building 3D scene graphs.

\section{Related Work}\label{sec:relatedWork}

We provide a broad literature review touching on abstractions and symbolic representations (\cref{sec:rw-abstractions}),
 metric-semantic and hierarchical map representations and algorithms to build them from sensor data (\cref{sec:rw-mapRepresentations}), and
 loop closure detection and optimization~(\cref{sec:rw-lcd}).

\subsection{Need for Abstractions and Symbolic Representations}\label{sec:rw-abstractions}

\myParagraph{State and Action Abstractions}
The need to abstract sensor data into higher-level representations has been studied in the context of planning and decision-making.
\citet{Konidaris19cobs-Necessityabstraction} points out the necessity of state and action abstraction for efficient task and motion planning problems.
\citet{Konidaris18jair-SkillsSymbols} extract task-specific state abstractions.
\citet{James20icml-PortableRepresentations} show how task-independent state abstractions can also be learned.
\citet{James22iclr-AutonomousLearning} show how to autonomously learn object-centric representations of a continuous and high-dimensional environment, and argue that such a representation enables efficient planning for long-horizon tasks.
\citet{Berg22icra-UsingLanguage} propose a hierarchical representation for planning in large outdoor environments. The hierarchy contains three levels: landmarks (\eg forests, buildings, streets), neighborhoods, and cities.
Several related works also discuss action abstractions, \eg how to group a sequence of actions into macro-actions  (usually referred to as \emph{options}).
\citet{Jinnai19icml-OptionPlanning}, for instance, provide a polynomial-time approximation algorithm for learning options for a Markov decision process.

\myParagraph{Ontologies and Knowledge Graphs}
A fundamental requirement for an embodied agent to build and communicate a meaningful representation of an environment is to use a common vocabulary describing concepts of interest\footnote{The term ``concept'' may be ambiguous without context as concepts may include a description of a task, a thought process, or simply the labels for a classifier. This ambiguity is discussed in more detail in~\cite{Smith04book-BeyondConcepts}.} and their relations. Such a vocabulary may be represented as an \emph{ontology} or \emph{the Kimera-VIO estimatesknowledge graph}.
The precise definition ofthe Kimera-VIO estimates an ontology varies between communities~\citep{Guarino09hoo-WhatIsAnOntology, Jepsen09itpm-WhatIsOntology}.
A well accepted definition was proposed by \citet{Gruber95ijhcs-DesignOfOntologies} as {\it an explicit specification of conceptualization} and later extended to require the conceptualization must be a shared view~\citep{Borst97diss-constructOntologies} and written as a formal language~\citep{Studer98dke-KnowledgeEngineering}.
In practice, the definition used is often not as strict.
For example, the definition provided by W3C as part of the Web Ontology Language (OWL)~\citep{Mcguinness04w3c-OWL} describes an ontology as {\it a set of terms in a vocabulary and their inter-relationships}.
Generally, an embodied agent is committed to a conceptualization regardless if the commitment is explicit (e.g., represented as knowledge graph) or implicit (e.g., represented as a set of labels for a classifier)~\citep{Genesereth12book-LogicalFoundations}.

The need to create a standard ontology was identified by~\cite{Schlenoff12iros-CORA}, which resulted in the emergence of several knowledge processing frameworks focused on robotics applications~\citep{Lemaignan10iros-ORO, Tenorth13ijrr-KnowRob, Beetz18icra-KnowRob2, Diab19sensors-PMK}.
A significant effort has been made in the creation of common-sense ontologies and knowledge graphs~\citep{Miller95acm-WordNet, Lenat95cacm-OpenCyc, Niles01fois-SUMO, Auer07sw-DBpedia, Suchanek08sw-YAGO, Bollacker08icmd-Freebase, Carlson10aaai-NELL, Vrandecic14-Wikidata, Speer17aaai-ConceptNet}.
In recent years, there has been a surge in applying these ontologies and knowledge graphs to problems such as 2D scene graph generation~\citep{Zareian20eccv-BridgingKnowledgeSGG, Guo21iccv-GeneralToSpecificSGG, Amodeo22access-OGSGG, Chen23wacv-UnbiasingSGGOntology}, image classification~\citep{Movshovitz15cvpr-OntoStoreClass, Porello15nci-IntegratingOntologies}, visual question answering~\citep{Aditya18cviu-SceneDescriptionGraphReasoning, Garderes20emnlp-ConceptBert, Marino21cvpr-KRISP, Zheng21pr-EmbeddingDesignVQA, Ding22cvpr-MuKEA}, task planning~\cite{Chen20icra-IncompleteLangInstr, Daruna21icra-OneShotTask, Tuli22jair-ToolTango}, and representation learning~\citep{Hao19kdd-RepLearnOnto, Kwak22ral-GraspKnowGraph}, to name a few.

\myParagraph{Scene Grammars and Compositional Models}
Compositionality, which is the ability of humans to perceive reality as a hierarchy of parts, has been considered fundamental to human cognition~\citep{Geman02qam-CompositionSystems, Lake17bbs-Buildingmachines}.
\citet{Geman02qam-CompositionSystems} propose a mathematical formulation of \emph{compositional models}, by recursively grouping or composing constituents and assigning probability factors to composition.
\citet{Zhu11jmiv-RecursiveCompositional} represent and infer visual patters as a recursive compositional model, which is nothing but a tree-structured graphical model where the leaf nodes represent the visual patters, and the higher-layers represent complex compositions.
\citet{Zhu06book-StochasticGrammar, Zhu21book-ComputerVision} discuss how to model objects, images, and scenes using such a hierarchical tree structure, called \emph{stochastic grammar}, and advocate it to be a general framework for visual representation.
Inspired by these insights, recent works such as~\cite{Izatt20icra-sceneGrammarSim, Qi18cvpr-HumanCentricIndoor, Chua18th-ProbabilisticScene} propose probabilistic generative models to capture hierarchical relationship between entities in a scene.

Recent works have tended to use such a hierarchical structure along with deep neural networks to provide better learning models.
\citet{Wang22pami-HierarchicalHuman} show that using a hierarchical, compositional model of the human body results in better human pose estimation. The model consists of body parts (\eg shoulder, elbow, arm) organized as a hierarchy (\eg shoulder, elbow are children of arm). A bottom-up/top-down inference strategy is proposed that is able to correct ambiguities in perceiving the lowest-level parts.
\citet{Niemeyer21cvpr-GIRAFFERepresenting} show that modeling a 3D scene as one composed of objects and background leads to more controllable and accurate image synthesis.
\citet{Mo20cvpr-StructEditLearning} model object shape as a hierarchical assembly of individual parts, and the object shape is then generated by transforming the shape parameter of each object part.
\citet{Ichien21css-VisualAnalogy} show that compositional model significantly outperform deep learning models on the task of analogical reasoning.
\citet{Yuan23arxiv-CompositionalScene} survey work on compositional scene representation.
While early work by \citet{Fodor88cognition-Connectionismcognitive} argued  that neural-network-based models are not compositional in nature, recent works have suggested otherwise.
The works~\citep{Mhaskar1996j-NeurComput-NNApproxAnalyticFun, Mhaskar2016j-AA-DeepVsShallowNN, Poggio2017j-IJAC-WhyWhenDNN} show that deep convolutional networks avoid the curse of dimensionality in approximating a class of compositional functions.
\citet{Webb22arxiv-EmergentAnalogical} show that large language models such as GPT-3 show compositional generalizability as an emergent property. Such compositionality, however, is not yet evident in generative models trained on multi-object scenes~\citep{Xie22icml-COATMeasuring}.

\subsection{Metric-semantic and Hierarchical Representations \\ for Scene Understanding and Mapping}\label{sec:rw-mapRepresentations}

\myParagraph{2D Scene Graphs in Computer Vision}
2D scene graphs are a popular model for image understanding that describes the content of an image in terms of objects (typically grounded by bounding boxes in the image) and their attributes and relations~\citep{Krishna16arxiv-VisualGenome, Johnson15cvpr}.
The estimation of 2D scene graphs (either from a single image or a sequence of images) has been well studied in the computer vision community and is surveyed in~\cite{Zhu22arxiv-SceneGraphSurvey}.
The seminal works~\citep{Krishna16arxiv-VisualGenome, Johnson15cvpr} advocated the use of 2D scene graphs to perform cognitive tasks such as image search, image captioning, and answering questions. These tasks, unlike object detection, require the model to reason about  object attributes and object-to-object relations, and therefore, enable better image understanding.
2D scene graphs have been successfully used in
image retrieval~\citep{Johnson15cvpr}, caption generation~\citep{Karpathy15cvpr-caption, Anderson16eccv-sceneGraphDescription}, visual question answering~\citep{Ren15corr-QA, Krishna16arxiv-VisualGenome}, and relationship detection~\citep{Lu16eccv-visualRelations}.
GNNs are a popular tool for joint object labels and/or relationship inference on scene graphs~\citep{Xu17cvpr-sceneGraph, Li17iccv-sceneGraphGeneration, Yang18eccv-sceneGraph, Zellers18cvpr-sceneGraph}.
\citet{Chang23pami-sceneGraphSurvey} provide a comprehensive survey on the various methods that have been proposed to infer a 2D scene graph from an image.

Despite their popularity, 2D scene graphs have several limitations. First of all, they are designed to ground concepts in the image space, hence they are not suitable to model large-scale scenes (see discussion in \cref{sec:symbolGroundingAndHieRep}). Second, many of the annotated object-to-object relationships in~\cite{Krishna16arxiv-VisualGenome} like ``behind'', ``next to'', ``near'', ``above'', ``under'' are harder to assess from a 2D image due to the lack of depth information, and are more easily inferred in 3D.
Finally, 2D representations such as 2D scene graphs are not invariant to viewpoint changes (\ie viewing the same 3D scene from a different viewing angle may result in a different 2D scene graph), as observed in~\cite{Armeni19iccv-3DsceneGraphs}.

\myParagraph{Flat Metric-Semantic Representations}
The last few years have seen a surge of interest towards \emph{metric-semantic mapping}, simultaneously triggered by the maturity of traditional 3D reconstruction and SLAM techniques, and by the novel opportunities for semantic understanding afforded by deep learning.
The literature has focused on both object-based maps~\citep{Salas-Moreno13cvpr,Dong17cvpr-XVIO,Mo19iros-orcVIO,Nicholson18ral-quadricSLAM,Bowman17icra,Ok21icra-home} and dense maps, including volumetric models~\citep{McCormac17icra-semanticFusion,Grinvald19ral-voxbloxpp,Narita19iros-metricSemantic}, point clouds~\citep{Behley19iccv-semanticKitti,Tateno15iros-metricSemantic,Lianos18eccv-VSO}, and 3D meshes~\citep{Rosinol20icra-Kimera,Rosu19ijcv-semanticMesh}.
Some approaches combine objects and dense map models~\citep{Li16iros-metricSemantic,McCormac183dv-fusion++,Xu19icra-midFusion,Schmid21arxiv-panoptic}.
These approaches are not concerned with estimating higher-level semantics (\eg{} rooms) and typically return dense models that might not be directly amenable for navigation~\citep{Oleynikova18iros-topoMap}.

\myParagraph{Building Parsing} A somewhat parallel research line investigates how to \emph{parse the layout of a building} from 2D or 3D data.
A large body of work focuses on parsing 2D maps~\citep{Bormann16icra-roomSegmentationSurvey}, including rule-based~\citep{Kleiner17iros-roombaRoomSegmentation} and learning-based methods~\citep{Liu18eccv-floorNet}.
\citet{Friedman07ijcai-voronoiRF} compute a Voronoi graph from a 2D occupancy grid, which is then labeled using a conditional random field.
Recent work focuses on 3D data. \citet{Liu18eccv-floorNet} and \citet{Stekovic21arxiv-monteFloor} project 3D point clouds to 2D maps, which however is not directly applicable to multi-story buildings.
\citet{Furukawa09iccv} reconstruct floor plans from images using multi-view stereo combined with a Manhattan World assumption.
\citet{Lukierski17icra-floorPlan} use dense stereo from an omni-directional camera to fit cuboids to objects and rooms.
\citet{Zheng20pami-buildingFusion} detect rooms by performing region growing on a 3D metric-semantic model.

\myParagraph{Hierarchical Representations and 3D Scene Graphs}
Hierarchical maps have been pervasive in robotics since its inception~\citep{Kuipers00ai,Kuipers78cs,Chatila85,Thrun02a}.
Early work focuses on 2D maps and investigates the use of hierarchical maps to resolve the divide between metric and topological representations~\citep{Ruiz-Sarmiento17kbs-multiversalMaps,Galindo05iros-multiHierarchicalMaps,Zender08ras-spatialRepresentations,Choset01tra,Beeson10ijrr-factoringSSH}.
These works preceded the ``deep learning revolution'' and could not leverage the rich semantics currently accessible via deep neural networks.

More recently, \emph{3D scene graphs} have been proposed as expressive hierarchical models for 3D environments.
\citet{Armeni19iccv-3DsceneGraphs} model the environment as a graph including low-level geometry (\ie{} a metric-semantic mesh), objects, rooms, and camera locations.
\citet{Rosinol21ijrr-Kimera,Rosinol20rss-dynamicSceneGraphs} augment the model with a topological map of places (modeling traversability), as well as a layer describing dynamic entities in the environment.
The approaches in~\cite{Armeni19iccv-3DsceneGraphs,Rosinol21ijrr-Kimera,Rosinol20rss-dynamicSceneGraphs} are designed for offline use.
Other papers focus on reconstructing a graph of objects and their relations~\citep{Kim19tc-3DsceneGraphs,Wald20cvpr-semanticSceneGraphs,Wu21cvpr-SceneGraphFusion}.
\citet{Wu21cvpr-SceneGraphFusion} predict objects and relations in real-time using a graph neural network.
\citet{Izatt21tr-sceneGraphs} parse objects and relations into a scene grammar model using mixed-integer programming.
\citet{Gothoskar21arxiv-3dp3} use an MCMC approach.
\citet{Gay18accv-VisualGraphsFromMotion} use a quadric representation to estimate 3D ellipsoids for objects in the scene given input 2D bounding boxes, and use an RNN to infer relationships between detected objects.

\subsection{Maintaining Persistent Representations}\label{sec:rw-lcd}

\myParagraph{Loop Closures Detection}
Established approaches for visual loop closure detection in robotics trace back to  place recognition and image retrieval techniques in computer vision;
these approaches are broadly adopted in SLAM pipelines but are known to suffer from appearance and viewpoint changes~\citep{Lowry16tro-surveyPlaceRecognition}.
Alternative approaches investigate place recognition using image sequences~\citep{Schubert21rss-fastICM,Milford12icra,Garg21ral-seqnet} or deep learning~\citep{Arandjelovic16cvpr-netvlad}.
More related to our proposal is the set of papers leveraging semantic information for loop closure detection.
\citet{Gawel18ral-xview} perform object-graph-based loop closure detection using random-walk descriptors built from 2D images.
\citet{Liu19icra-globalLocalizationObjects} use similar object-based descriptors but built from a 3D reconstruction.
\citet{Lin21ral-topologyAwareObjectLocalization} adopt random-walk object-based descriptors and then compute loop closure poses via object registration.
\citet{Qin21jvcir-semantic} propose an object-based approach based on \subgraph similarity matching.
\citet{Zheng20pami-buildingFusion} propose a room-level loop closure detector.
None of these approaches are hierarchical in nature.

\myParagraph{Loop Closures Correction}
After a loop closure is detected, the map needs to be corrected accordingly.
While this process is easy in sparse (\eg{} landmark-based) representations~\citep{Cadena16tro-SLAMsurvey}, it is non-trivial to perform in real-time when using dense representations.
\citet{Stuckler14jvcir} and \citet{Whelan16ijrr-elasticFusion} optimize a map of \emph{surfels}, to circumvent the need to correct structured representations (\eg{} meshes or voxels).
\citet{Dai17tog-bundlefusion} propose reintegrating a volumetric map after each loop closure.
\citet{Reijgwart20ral-voxgraph} correct drift in volumetric representations by breaking the map into submaps that can be rigidly re-aligned after loop closures.
\citet{Whelan15ijrr} propose a 2-step optimization that first corrects the robot trajectory and then deforms the map (represented as a point cloud or a mesh) using a deformation graph approach~\citep{Sumner07siggraph-embeddedDeformation}.
\citet{Rosinol21ijrr-Kimera} unify the two steps into a single pose graph and mesh optimization.
None of these works are concerned with simultaneously correcting multiple layers in a hierarchical representation.

\section{Conclusions}\label{sec:conclusions}

This paper argues that large-scale spatial perception for robotics requires hierarchical representations.
In particular, we show that hierarchical representations scale better in terms of memory and are more suitable for efficient inference.
Our second contribution is to introduce algorithms to build a hierarchical representation of an indoor environment, namely a \emph{3D scene graph},
 in real-time as the robot explores the environment. Our algorithms combine 3D geometry (\eg to cluster the free space into a graph of places), topology (to cluster the places into rooms), and
geometric deep learning (\eg to classify the type of rooms the robot is moving across).
Our third contribution is to discuss loop closure detection and correction in 3D scene graphs. We introduce
 (handcrafted and learning-based) hierarchical descriptors for loop closure detection, and develop a unified optimization framework to  correct drift in the 3D scene graph in response to loop closures.
 We integrate our algorithmic contributions into a heavily parallelized system, named \emph{\name}, and show it can build accurate
 3D scene graphs in real-time across a variety of photo-realistic simulations and real datasets.

\myParagraph{Limitations}
While we believe the proposed contributions constitute a substantial step towards high-level scene understanding and spatial perception for robotics, our current proposal has several limitations.
First, the \subgraph of places captures the free-space in 3D, which is directly usable for a drone to navigate. However,
traversability for ground robots is also influenced by other aspects (\eg terrain type, steepness). These aspects are currently disregarded in the construction of the GVD and the places \subgraph, but we believe they are particularly important for outdoor extensions of 3D scene graphs.
Second, our approach for room segmentation, which first clusters the rooms geometrically, and then labels each room, mainly applies to rooms with clear geometric boundaries (\eg it would not work in an open floor-plan). We believe this limitation is surmountable (at the expense of using extra training data) by replacing the 2-stage approach with a single learning-based method (\eg a neural tree or a standard graph neural network) that can directly classify places into rooms.
Third, for computational reasons, we restricted the inference in the neural tree to operate at the level of objects, rooms, and buildings. However, it would be desirable for high-level semantic concepts (\eg room and object labels) to propagate information downward towards the mesh geometry. While a more compact representation of the low-level geometry~\citep{Czarnowski20ral-deepFactors,Sucar203dv-nodeSLAM} might facilitate this process, it remains unclear how to fully integrate top-down reasoning in the construction of the 3D scene graph, which is currently a bottom-up process.

\myParagraph{Future Work}
This work opens many avenues for current and future investigation.
First of all, while this paper mostly focused on inclusion and adjacency relations, it would be interesting to label nodes and edges of the 3D scene graph with a richer set of relations and affordances, for instance building on~\cite{Wu21cvpr-SceneGraphFusion}.
Second, the connections between our scene graph optimization approach and pose graph optimization offer opportunities to improve the efficiency of the optimization by leveraging recent advances in pose graph sparsification as well as novel solvers.
Third, it would be interesting to replace the sub-symbolic representation (currently, a 3D mesh and a graph of places) with neural models, including neural implicit representations~\citep{Park19cvpr-deepSDF} or neural radiance fields~\citep{Rosinol22arxiv-nerfSLAM}, which can more easily incorporate priors and better support shape completion~\citep{Kong23arxiv-vmap}.
 Fourth, the current set of detectable objects is fairly limited and restricted by the use of pre-trained 2D segmentation networks. However, we have noticed in \cref{sec:experiments} and in~\cite{Chen22arxiv-LLM2} that a larger object vocabulary leads to better room classification; therefore, it would be interesting to investigate novel techniques that leverage language models for open-set segmentation, \eg~\cite{Ha22icrl-openWorld,Jatavallabhula23arxiv-conceptFusion}.
Fifth, it would be desirable the extend the framework in this paper to arbitrary (\ie mixed indoor-outdoor environments).
Finally, the implications of using 3D scene graphs for prediction, planning, and decision-making are still mostly unexplored (see~\cite{Ravichandran22icra-RLwithSceneGraphs,Agia22corl-taskography,Chang22arxiv-DLite} for early examples).

\section*{Disclaimer}

Research was sponsored by the United States Air Force Research Laboratory and the United States Air Force Artificial Intelligence Accelerator and was accomplished under Cooperative Agreement Number FA8750-19-2-1000.
The views and conclusions contained in this document are those of the authors and should not be interpreted as representing the official policies, either expressed or implied, of the United States Air Force or the U.S. Government.
The U.S. Government is authorized to reproduce and distribute reprints for Government purposes notwithstanding any copyright notation herein.

\toggleformat{}{

\begin{funding}
    This work was partially funded by the AIA CRA FA8750-19-2-1000,
    ARL DCIST CRA W911NF-17-2-0181,
    ONR RAIDER N00014-18-1-2828,
    MIT Lincoln Laboratory's Autonomy al Fresco program,
    Lockheed Martin Corporation’s Neural Prediction in 3D Dynamic Scene
    Graphs program, and by Carlone's Amazon Research Award.
\end{funding}

\begin{dci}
    The authors declare that there is no conflict of interest.
\end{dci}
}

{\smaller
\toggleformat{\bibliographystyle{IEEEtranN}}{\bibliographystyle{SageH}}

\begin{thebibliography}{206}
\providecommand{\natexlab}[1]{#1}
\providecommand{\url}[1]{#1}
\csname url@samestyle\endcsname
\providecommand{\newblock}{\relax}
\providecommand{\bibinfo}[2]{#2}
\providecommand{\BIBentrySTDinterwordspacing}{\spaceskip=0pt\relax}
\providecommand{\BIBentryALTinterwordstretchfactor}{4}
\providecommand{\BIBentryALTinterwordspacing}{\spaceskip=\fontdimen2\font plus
\BIBentryALTinterwordstretchfactor\fontdimen3\font minus
  \fontdimen4\font\relax}
\providecommand{\BIBforeignlanguage}[2]{{\expandafter\ifx\csname l@#1\endcsname\relax
\typeout{** WARNING: IEEEtranN.bst: No hyphenation pattern has been}\typeout{** loaded for the language `#1'. Using the pattern for}\typeout{** the default language instead.}\else
\language=\csname l@#1\endcsname
\fi
#2}}
\providecommand{\BIBdecl}{\relax}
\BIBdecl

\bibitem[Davison(2018)]{Davison18-futuremapping}
A.~J. Davison, ``{FutureMapping}: The computational structure of spatial {AI}
  systems,'' \emph{arXiv preprint arXiv:1803.11288}, 2018.

\bibitem[Rosinol et~al.(2021)Rosinol, Violette, Abate, Hughes, Chang, Shi,
  Gupta, and Carlone]{Rosinol21ijrr-Kimera}
A.~Rosinol, A.~Violette, M.~Abate, N.~Hughes, Y.~Chang, J.~Shi, A.~Gupta, and
  L.~Carlone, ``Kimera: from {SLAM} to spatial perception with {3D} dynamic
  scene graphs,'' \emph{Intl. J. of Robotics Research}, vol.~40, no. 12--14,
  pp. 1510--1546, 2021, arXiv preprint: 2101.06894,
  \linkToPdf{https://arxiv.org/pdf/2101.06894.pdf}.

\bibitem[Armeni et~al.(2019)Armeni, He, Gwak, Zamir, Fischer, Malik, and
  Savarese]{Armeni19iccv-3DsceneGraphs}
I.~Armeni, Z.~He, J.~Gwak, A.~Zamir, M.~Fischer, J.~Malik, and S.~Savarese,
  ``{3D} scene graph: A structure for unified semantics, {3D} space, and
  camera,'' in \emph{Intl. Conf. on Computer Vision (ICCV)}, 2019, pp.
  5664--5673.

\bibitem[Rosinol et~al.(2020{\natexlab{a}})Rosinol, Gupta, Abate, Shi, and
  Carlone]{Rosinol20rss-dynamicSceneGraphs}
\BIBentryALTinterwordspacing
A.~Rosinol, A.~Gupta, M.~Abate, J.~Shi, and L.~Carlone, ``{3D} dynamic scene
  graphs: Actionable spatial perception with places, objects, and humans,'' in
  \emph{Robotics: Science and Systems (RSS)}, 2020,
  \linkToPdf{https://arxiv.org/pdf/2002.06289.pdf},
  \linkToMedia{http://news.mit.edu/2020/robots-spatial-perception-0715},
  \linkToVideo{https://www.youtube.com/watch?v=SWbofjhyPzI&feature=youtu.be}.
  [Online]. Available:
  \url{http://news.mit.edu/2020/robots-spatial-perception-0715}
\BIBentrySTDinterwordspacing

\bibitem[Kim et~al.(2019)Kim, Park, Song, and Kim]{Kim19tc-3DsceneGraphs}
U.~Kim, J.~Park, T.~Song, and J.~Kim, ``{3-D} scene graph: A sparse and
  semantic representation of physical environments for intelligent agents,''
  \emph{IEEE Trans. Cybern.}, vol.~PP, pp. 1--13, Aug. 2019.

\bibitem[Wald et~al.(2020)Wald, Dhamo, Navab, and
  Tombari]{Wald20cvpr-semanticSceneGraphs}
J.~Wald, H.~Dhamo, N.~Navab, and F.~Tombari, ``Learning {3D} semantic scene
  graphs from {3D} indoor reconstructions,'' in \emph{Proceedings of the
  IEEE/CVF Conference on Computer Vision and Pattern Recognition}, 2020, pp.
  3961--3970.

\bibitem[Wu et~al.(2021)Wu, Wald, Tateno, Navab, and
  Tombari]{Wu21cvpr-SceneGraphFusion}
S.~Wu, J.~Wald, K.~Tateno, N.~Navab, and F.~Tombari, ``{SceneGraphFusion}:
  Incremental {3D} scene graph prediction from {RGB-D} sequences,'' in
  \emph{IEEE Conf. on Computer Vision and Pattern Recognition (CVPR)}, 2021.

\bibitem[Izatt and Tedrake(2021)]{Izatt21tr-sceneGraphs}
G.~Izatt and R.~Tedrake, ``Scene understanding and distribution modeling with
  mixed-integer scene parsing,'' in \emph{technical report (under review)},
  2021.

\bibitem[Gothoskar et~al.(2021)Gothoskar, Cusumano-Towner, Zinberg,
  Ghavamizadeh, Pollok, Garrett, Tenenbaum, Gutfreund, and
  Mansinghka]{Gothoskar21arxiv-3dp3}
N.~Gothoskar, M.~Cusumano-Towner, B.~Zinberg, M.~Ghavamizadeh, F.~Pollok,
  A.~Garrett, J.~Tenenbaum, D.~Gutfreund, and V.~Mansinghka, ``{3DP3}: {3D}
  scene perception via probabilistic programming,'' in \emph{ArXiv preprint:
  2111.00312}, 2021.

\bibitem[Oleynikova et~al.(2017)Oleynikova, Taylor, Fehr, Siegwart, and
  Nieto]{Oleynikova17iros-voxblox}
H.~Oleynikova, Z.~Taylor, M.~Fehr, R.~Siegwart, and J.~Nieto, ``Voxblox:
  Incremental 3d euclidean signed distance fields for on-board mav planning,''
  in \emph{IEEE/RSJ Intl. Conf. on Intelligent Robots and Systems
  (IROS)}.\hskip 1em plus 0.5em minus 0.4em\relax IEEE, 2017, pp. 1366--1373.

\bibitem[Hughes et~al.(2022)Hughes, Chang, and Carlone]{Hughes22rss-hydra}
N.~Hughes, Y.~Chang, and L.~Carlone, ``{Hydra:} a real-time spatial perception
  engine for {3D} scene graph construction and optimization,'' in
  \emph{Robotics: Science and Systems (RSS)}, 2022,
  \linkToPdf{https://arxiv.org/pdf/2201.13360.pdf}.

\bibitem[Bavle et~al.(2022{\natexlab{a}})Bavle, Sanchez-Lopez, Shaheer, Civera,
  and Voos]{Bavle22ral-SGraph}
H.~Bavle, J.~L. Sanchez-Lopez, M.~Shaheer, J.~Civera, and H.~Voos,
  ``Situational graphs for robot navigation in structured indoor
  environments,'' \emph{{IEEE} Robotics and Automation Letters}, vol.~7, no.~4,
  pp. 9107--9114, 2022.

\bibitem[Bavle et~al.(2022{\natexlab{b}})Bavle, Sanchez-Lopez, Shaheer, Civera,
  and Voos]{Bavle22arxiv-SGraphPlus}
------, ``S-graphs+: Real-time localization and mapping leveraging hierarchical
  representations,'' \emph{arXiv preprint arXiv:2212.11770}, 2022.

\bibitem[{Grinvald} et~al.(2019){Grinvald}, {Furrer}, {Novkovic}, {Chung},
  {Cadena}, {Siegwart}, and {Nieto}]{Grinvald19ral-voxbloxpp}
M.~{Grinvald}, F.~{Furrer}, T.~{Novkovic}, J.~J. {Chung}, C.~{Cadena},
  R.~{Siegwart}, and J.~{Nieto}, ``{Volumetric Instance-Aware Semantic Mapping
  and 3D Object Discovery},'' \emph{{IEEE} Robotics and Automation Letters},
  vol.~4, no.~3, pp. 3037--3044, 2019.

\bibitem[Huber(2021)]{Huber21idsc-persistent}
S.~Huber, ``Persistent homology in data science,'' in \emph{Data
  Science--Analytics and Applications: Proceedings of the 3rd International
  Data Science Conference--iDSC2020}.\hskip 1em plus 0.5em minus 0.4em\relax
  Springer, 2021, pp. 81--88.

\bibitem[Talak et~al.(2021)Talak, Hu, Peng, and
  Carlone]{Talak21neurips-neuralTree}
R.~Talak, S.~Hu, L.~Peng, and L.~Carlone, ``Neural trees for learning on
  graphs,'' in \emph{Conf. on Neural Information Processing Systems (NeurIPS)},
  2021, \linkToPdf{https://arxiv.org/pdf/2105.07264.pdf}.

\bibitem[Sumner et~al.(2007)Sumner, Schmid, and
  Pauly]{Sumner07siggraph-embeddedDeformation}
R.~Sumner, J.~Schmid, and M.~Pauly, ``Embedded deformation for shape
  manipulation,'' \emph{ACM SIGGRAPH 2007 papers on - SIGGRAPH '07}, 2007.

\bibitem[Harnad(1990)]{Harnard90physica-symbolGrounding}
S.~Harnad, ``The symbol grounding problem,'' \emph{Physica D: Nonlinear
  Phenomena}, vol.~42, no.~1, pp. 335--346, 1990.

\bibitem[Garcia-Garcia et~al.(2017)Garcia-Garcia, Orts-Escolano, Oprea,
  Villena-Martinez, and Garc{\'i}a-Rodr{\'i}guez]{GarciaGarcia17arxiv}
A.~Garcia-Garcia, S.~Orts-Escolano, S.~Oprea, V.~Villena-Martinez, and
  J.~Garc{\'i}a-Rodr{\'i}guez, ``A review on deep learning techniques applied
  to semantic segmentation,'' \emph{ArXiv Preprint: 1704.06857}, 2017.

\bibitem[Zeng et~al.(2013)Zeng, Zhao, Zheng, and Liu]{Zeng2013graphical-octree}
M.~Zeng, F.~Zhao, J.~Zheng, and X.~Liu, ``Octree-based fusion for realtime 3d
  reconstruction,'' \emph{Graphical Models}, vol.~75, no.~3, pp. 126--136,
  2013.

\bibitem[Park et~al.(2019)Park, Florence, Straub, Newcombe, and
  Lovegrove]{Park19cvpr-deepSDF}
J.~Park, P.~Florence, J.~Straub, R.~Newcombe, and S.~Lovegrove, ``{DeepSDF}:
  Learning continuous signed distance functions for shape representation,'' in
  \emph{IEEE Conf. on Computer Vision and Pattern Recognition (CVPR)}.\hskip
  1em plus 0.5em minus 0.4em\relax IEEE, 2019.

\bibitem[Bodlaender(2006)]{Bodlaender2006c-Chapter-TreewidthComputation}
H.~L. Bodlaender, ``Treewidth: Characterizations, applications, and
  computations,'' in \emph{Graph-Theoretic Concepts in Computer Science}.\hskip
  1em plus 0.5em minus 0.4em\relax Springer Berlin Heidelberg, 2006, pp. 1--14.

\bibitem[Dechter and Mateescu(2007)]{Dechter07ai}
R.~Dechter and R.~Mateescu, ``{AND/OR} search spaces for graphical models,''
  \emph{Artificial Intelligence}, vol. 171, no. 2-3, pp. 73--106, 2007.

\bibitem[Feder and Vardi(1993)]{Feder93stoc}
T.~Feder and M.~Vardi, ``Monotone monadic snp and constraint satisfaction,'' in
  \emph{ACM Symp. on Theory of Computing (STOC)}.\hskip 1em plus 0.5em minus
  0.4em\relax New York, NY, USA: ACM Press, 1993, pp. 612--622.

\bibitem[Grohe et~al.(2020)Grohe, Neuen, Schweitzer, and
  Wiebking]{Grohe18talgo-GraphIsomorphism-BndTreeWidth}
M.~Grohe, D.~Neuen, P.~Schweitzer, and D.~Wiebking, ``An improved isomorphism
  test for bounded-tree-width graphs,'' \emph{ACM Trans. Algorithms}, vol.~16,
  no.~3, Jun. 2020.

\bibitem[Chandrasekaran et~al.(2008)Chandrasekaran, Srebro, and
  Harsha]{Venkat12uai}
V.~Chandrasekaran, N.~Srebro, and P.~Harsha, ``Complexity of inference in
  graphical models,'' in \emph{Conf. on Uncertainty in Artificial Intelligence
  (UAI)}, 2008, p. 70–78.

\bibitem[Cooper(1990)]{Cooper90ai}
G.~Cooper, ``{The computational complexity of probabilistic inference using
  Bayesian belief networks},'' \emph{Artificial Intelligence}, vol.~42, no.
  2-3, pp. 393--405, 1990.

\bibitem[Zhu et~al.(2022)Zhu, Zhang, Jiang, Dang, Hou, Shen, Feng, Zhao, Miao,
  Shah, et~al.]{Zhu22arxiv-SceneGraphSurvey}
G.~Zhu, L.~Zhang, Y.~Jiang, Y.~Dang, H.~Hou, P.~Shen, M.~Feng, X.~Zhao,
  Q.~Miao, S.~A.~A. Shah \emph{et~al.}, ``Scene graph generation: A
  comprehensive survey,'' \emph{arXiv preprint arXiv:2201.00443}, 2022.

\bibitem[Tagliasacchi et~al.(2016)Tagliasacchi, Delame, Spagnuolo, Amenta, and
  Telea]{Tagliasacchi16cgf-3dSkeletonsSurvey}
A.~Tagliasacchi, T.~Delame, M.~Spagnuolo, N.~Amenta, and A.~Telea, ``3d
  skeletons: A state-of-the-art report,'' in \emph{Computer Graphics Forum},
  vol.~35, no.~2.\hskip 1em plus 0.5em minus 0.4em\relax Wiley Online Library,
  2016, pp. 573--597.

\bibitem[Cadena et~al.(2016)Cadena, Carlone, Carrillo, Latif, Scaramuzza,
  Neira, Reid, and Leonard]{Cadena16tro-SLAMsurvey}
C.~Cadena, L.~Carlone, H.~Carrillo, Y.~Latif, D.~Scaramuzza, J.~Neira, I.~Reid,
  and J.~Leonard, ``Past, present, and future of simultaneous localization and
  mapping: Toward the robust-perception age,'' \emph{{IEEE} Trans. Robotics},
  vol.~32, no.~6, pp. 1309--1332, 2016, arxiv preprint: 1606.05830,
  \linkToPdf{https://arxiv.org/abs/1606.05830}.

\bibitem[Chang et~al.(2017)Chang, Dai, Funkhouser, Halber, Niessner, Savva,
  Song, Zeng, and Zhang]{Chang173dv-Matterport3D}
A.~Chang, A.~Dai, T.~Funkhouser, M.~Halber, M.~Niessner, M.~Savva, S.~Song,
  A.~Zeng, and Y.~Zhang, ``Matterport3d: Learning from rgb-d data in indoor
  environments,'' \emph{International Conference on 3D Vision (3DV)}, 2017.

\bibitem[Bodlaender and Koster(2010)]{Bodlaender2010j-IC-TreewidthComputationI}
H.~L. Bodlaender and A.~M. Koster, ``Treewidth computations {I}: Upper
  bounds,'' \emph{Information and Computation}, vol. 208, no.~3, pp. 259 --
  275, 2010.

\bibitem[Maniu et~al.(2019)Maniu, Senellart, and
  Jog]{Maniu19icdt-ExperimentalStudy}
S.~Maniu, P.~Senellart, and S.~Jog, ``{An Experimental Study of the Treewidth
  of Real-World Graph Data},'' in \emph{Intl. Conf. Database Theory}, 2019, pp.
  12:1--12:18.

\bibitem[Rosinol et~al.(2020{\natexlab{b}})Rosinol, Abate, Chang, and
  Carlone]{Rosinol20icra-Kimera}
A.~Rosinol, M.~Abate, Y.~Chang, and L.~Carlone, ``Kimera: an open-source
  library for real-time metric-semantic localization and mapping,'' in
  \emph{IEEE Intl. Conf. on Robotics and Automation (ICRA)}, 2020, arXiv
  preprint: 1910.02490,
  \linkToVideo{https://www.youtube.com/watch?v=-5XxXRABXJs},
  \linkToCode{https://github.com/MIT-SPARK/Kimera},
  \linkToPdf{https://arxiv.org/pdf/1910.02490.pdf}.

\bibitem[Whelan et~al.(2012)Whelan, McDonald, Kaess, Fallon, Johannsson, and
  Leonard]{Whelan12rgbd}
T.~Whelan, J.~B. McDonald, M.~Kaess, M.~F. Fallon, H.~Johannsson, and J.~J.
  Leonard, ``Kintinuous: Spatially extended {Kinect-Fusion},'' in \emph{RSS
  Workshop on RGB-D: Advanced Reasoning with Depth Cameras}, Sydney, Australia,
  July 2012.

\bibitem[Nie{\ss}ner et~al.(2013)Nie{\ss}ner, Zollh{\"o}fer, Izadi, and
  Stamminger]{Niessner2013acm-real}
M.~Nie{\ss}ner, M.~Zollh{\"o}fer, S.~Izadi, and M.~Stamminger, ``Real-time 3d
  reconstruction at scale using voxel hashing,'' \emph{ACM Transactions on
  Graphics (ToG)}, vol.~32, no.~6, p. 169, 2013.

\bibitem[Rusu(2009)]{Rusu09thesis-SemanticMaps}
R.~B. Rusu, ``Semantic {3D} object maps for everyday manipulation in human
  living environments,'' Ph.D. dissertation, Computer Science department,
  Technische Universitaet Muenchen, Germany, October 2009.

\bibitem[Yang et~al.(2020)Yang, Shi, and Carlone]{Yang20tro-teaser}
H.~Yang, J.~Shi, and L.~Carlone, ``{TEASER: Fast and Certifiable Point Cloud
  Registration},'' \emph{{IEEE} Trans. Robotics}, vol.~37, no.~2, pp. 314--333,
  2020, extended arXiv version 2001.07715
  \linkToPdf{https://arxiv.org/pdf/2001.07715.pdf}.

\bibitem[Talak et~al.(2023)Talak, Shi, Maggio, and
  Carlone]{Talak23rss-ensemble}
R.~Talak, J.~Shi, D.~Maggio, and L.~Carlone, ``A correct-and-certify approach
  to self-supervise object pose estimators via ensemble self-training,''
  \emph{Robotics: Science and Systems (RSS)}, 2023,
  \linkToPdf{https://arxiv.org/pdf/2302.06019.pdf}.

\bibitem[Oleynikova et~al.(2018)Oleynikova, Taylor, Siegwart, and
  Nieto]{Oleynikova18iros-topoMap}
H.~Oleynikova, Z.~Taylor, R.~Siegwart, and J.~Nieto, ``Sparse {3D} topological
  graphs for micro-aerial vehicle planning,'' in \emph{IEEE/RSJ Intl. Conf. on
  Intelligent Robots and Systems (IROS)}, 2018.

\bibitem[Lau et~al.(2013)Lau, Sprunk, and
  Burgard]{Lau13ras-efficientGridRepresentations}
\BIBentryALTinterwordspacing
B.~Lau, C.~Sprunk, and W.~Burgard, ``Efficient grid-based spatial
  representations for robot navigation in dynamic environments,'' \emph{Robot.
  Auton. Syst.}, vol.~61, no.~10, p. 1116–1130, Oct. 2013. [Online].
  Available: \url{https://doi.org/10.1016/j.robot.2012.08.010}
\BIBentrySTDinterwordspacing

\bibitem[Foskey et~al.(2003)Foskey, Lin, and Manocha]{Foskey03jcise-SMA}
M.~Foskey, M.~C. Lin, and D.~Manocha, ``Efficient computation of a simplified
  medial axis,'' \emph{J. Comput. Inf. Sci. Eng.}, vol.~3, no.~4, pp. 274--284,
  2003.

\bibitem[Blanco and Rai(2014)]{Blanco14-nanoflann}
J.~L. Blanco and P.~K. Rai, ``nanoflann: a {C}++ header-only fork of {FLANN}, a
  library for nearest neighbor ({NN}) with kd-trees,''
  \url{https://github.com/jlblancoc/nanoflann}, 2014.

\bibitem[Ali et~al.(2022)Ali, Asaad, Jimenez, Nanda, Paluzo-Hidalgo, and
  Soriano-Trigueros]{Ali22arxiv-tdaSurvey}
D.~Ali, A.~Asaad, M.-J. Jimenez, V.~Nanda, E.~Paluzo-Hidalgo, and
  M.~Soriano-Trigueros, ``A survey of vectorization methods in topological data
  analysis,'' \emph{arXiv preprint arXiv:2212.09703}, 2022.

\bibitem[Aktas et~al.(2019)Aktas, Akbas, and Fatmaoui]{Aktas19ans-persistence}
M.~E. Aktas, E.~Akbas, and A.~E. Fatmaoui, ``Persistence homology of networks:
  methods and applications,'' \emph{Applied Network Science}, vol.~4, no.~1,
  pp. 1--28, 2019.

\bibitem[Bormann et~al.(2016)Bormann, Jordan, Li, Hampp, and
  H\:{a}gele]{Bormann16icra-roomSegmentationSurvey}
R.~Bormann, F.~Jordan, W.~Li, J.~Hampp, and M.~H\:{a}gele, ``Room segmentation:
  Survey, implementation, and analysis,'' in \emph{2016 IEEE International
  Conference on Robotics and Automation (ICRA)}, 2016, pp. 1019--1026.

\bibitem[Kleiner et~al.(2017)Kleiner, Baravalle, Kolling, Pilotti, and
  Munich]{Kleiner17iros-roombaRoomSegmentation}
A.~Kleiner, R.~Baravalle, A.~Kolling, P.~Pilotti, and M.~Munich, ``A solution
  to room-by-room coverage for autonomous cleaning robots,'' in \emph{2017
  IEEE/RSJ International Conference on Intelligent Robots and Systems (IROS)},
  2017, pp. 5346--5352.

\bibitem[Kipf and Welling(2017)]{Kipf17iclr-gcn}
T.~Kipf and M.~Welling, ``Semi-supervised classification with graph
  convolutional networks,'' in \emph{Intl. Conf. on Learning Representations
  (ICLR)}, Apr. 2017.

\bibitem[Hamilton. et~al.(2017)Hamilton., Ying, and
  Leskovec]{Hamilton17nips-GraphSage}
W.~L. Hamilton., R.~Ying, and J.~Leskovec, ``Inductive representation learning
  on large graphs,'' in \emph{Advances in Neural Information Processing Systems
  (NIPS)}, Dec. 2017, p. 1025–1035.

\bibitem[Jordan(2002)]{Jordan02book}
M.~Jordan, ``An introduction to probabilistic graphical models,'' November
  2002, unpublished Lecture Notes.

\bibitem[Koller and Friedman(2009)]{Koller09book}
D.~Koller and N.~Friedman, \emph{Probabilistic Graphical Models: Principles and
  Techniques}.\hskip 1em plus 0.5em minus 0.4em\relax The MIT Press, 2009.

\bibitem[Henaff et~al.(2015)Henaff, Bruna, and LeCun]{Henaff15-deep}
M.~Henaff, J.~Bruna, and Y.~LeCun, ``Deep convolutional networks on
  graph-structured data,'' \emph{arXiv preprint arXiv:1506.05163}, Jun. 2015.

\bibitem[Defferrard et~al.(2016)Defferrard, Bresson, and
  Vandergheynst]{Defferrard16nips-ChebyNets}
M.~Defferrard, X.~Bresson, and P.~Vandergheynst, ``Convolutional neural
  networks on graphs with fast localized spectral filtering,'' in
  \emph{Advances in Neural Information Processing Systems (NIPS)}, vol.~29,
  Dec. 2016, pp. 3844--3852.

\bibitem[Bronstein et~al.(2017)Bronstein, Bruna, LeCun, Szlam, and
  Vandergheynst]{Bronstein17spm-geometricDL}
M.~M. Bronstein, J.~Bruna, Y.~LeCun, A.~Szlam, and P.~Vandergheynst,
  ``Geometric deep learning: going beyond euclidean data,'' \emph{{IEEE} Signal
  Process. Mag.}, vol.~34, no.~4, pp. 18--42, 2017.

\bibitem[Veli\v{c}kovi\'{c} et~al.(2018)Veli\v{c}kovi\'{c}, Cucurull, Casanova,
  Romero, Li\'{o}, and Bengio]{Velickovic18iclr-GAT}
P.~Veli\v{c}kovi\'{c}, G.~Cucurull, A.~Casanova, A.~Romero, P.~Li\'{o}, and
  Y.~Bengio, ``Graph attention networks,'' in \emph{Intl. Conf. on Learning
  Representations (ICLR)}, May 2018.

\bibitem[Lee et~al.(2019)Lee, Rossi, Kim, Ahmed, and Koh]{Lee19tkdd-gat-survey}
J.~Lee, R.~Rossi, S.~Kim, N.~Ahmed, and E.~Koh, ``Attention models in graphs: A
  survey,'' \emph{{ACM} Trans. Knowl. Discov. Data}, vol.~13, no.~6, Nov. 2019.

\bibitem[Busbridge et~al.(2019)Busbridge, Sherburn, Cavallo, and
  Hammerla]{Busbridge2019a-arXiv-Rel-GraphAttentionNetworks}
D.~Busbridge, D.~Sherburn, P.~Cavallo, and N.~Y. Hammerla, ``Relational graph
  attention networks,'' \emph{arXiv preprint arXiv:1904.05811}, Apr. 2019.

\bibitem[Fey and Lenssen(2019)]{Fey19iclrwk-pytorchGeometric}
M.~Fey and J.~E. Lenssen, ``Fast graph representation learning with {PyTorch
  Geometric},'' in \emph{{Intl. Conf. on Learning Representations (ICLR)}
  Workshop on Representation Learning on Graphs and Manifolds}, 2019.

\bibitem[G\'alvez-L\'opez and Tard\'os(2012)]{Galvez12tro-dbow}
D.~G\'alvez-L\'opez and J.~D. Tard\'os, ``Bags of binary words for fast place
  recognition in image sequences,'' \emph{IEEE Transactions on Robotics},
  vol.~28, no.~5, pp. 1188--1197, October 2012.

\bibitem[Li et~al.(2019)Li, Gu, Dullien, Vinyals, and
  Kohli]{Li19icml-GraphMatching}
Y.~Li, C.~Gu, T.~Dullien, O.~Vinyals, and P.~Kohli, ``Graph matching networks
  for learning the similarity of graph structured objects,'' in \emph{Intl.
  Conf. on Machine Learning (ICML)}.\hskip 1em plus 0.5em minus 0.4em\relax
  PMLR, 2019, pp. 3835--3845.

\bibitem[Antonante et~al.(2021)Antonante, Tzoumas, Yang, and
  Carlone]{Antonante21tro-outlierRobustEstimation}
P.~Antonante, V.~Tzoumas, H.~Yang, and L.~Carlone, ``Outlier-robust estimation:
  Hardness, minimally tuned algorithms, and applications,'' \emph{{IEEE} Trans.
  Robotics}, vol.~38, no.~1, pp. 281--301, 2021,
  \linkToPdf{https://arxiv.org/pdf/2007.15109.pdf}.

\bibitem[Savva et~al.(2019)Savva, Kadian, Maksymets, Zhao, Wijmans, Jain,
  Straub, Liu, Koltun, Malik, Parikh, and Batra]{Savva19iccv-habitat}
M.~Savva, A.~Kadian, O.~Maksymets, Y.~Zhao, E.~Wijmans, B.~Jain, J.~Straub,
  J.~Liu, V.~Koltun, J.~Malik, D.~Parikh, and D.~Batra, ``Habitat: {A}
  {P}latform for {E}mbodied {AI} {R}esearch,'' in \emph{Proceedings of the
  IEEE/CVF International Conference on Computer Vision (ICCV)}, 2019.

\bibitem[Reinke et~al.(2022)Reinke, Palieri, Morrell, Chang, Ebadi, Carlone,
  and Agha-mohammadi]{Reinke22ral-LOCUS2}
A.~Reinke, M.~Palieri, B.~Morrell, Y.~Chang, K.~Ebadi, L.~Carlone, and
  A.~Agha-mohammadi, ``{LOCUS 2.0}: Robust and computationally efficient lidar
  odometry for real-time underground {3D} mapping,'' vol.~7, no.~4, pp.
  9043--9050, 2022, \linkToPdf{https://arxiv.org/pdf/2205.11784.pdf}.

\bibitem[Chang et~al.(2022{\natexlab{a}})Chang, Ebadi, Denniston, Ginting,
  Rosinol, Reinke, Palieri, Shi, A, Morrell, Agha-mohammadi, and
  Carlone]{Chang22ral-LAMP2}
Y.~Chang, K.~Ebadi, C.~Denniston, M.~F. Ginting, A.~Rosinol, A.~Reinke,
  M.~Palieri, J.~Shi, C.~A, B.~Morrell, A.~Agha-mohammadi, and L.~Carlone,
  ``{LAMP 2.0}: A robust multi-robot {SLAM} system for operation in challenging
  large-scale underground environments,'' \emph{{IEEE} Robotics and Automation
  Letters ({RA-L})}, vol.~7, no.~4, pp. 9175--9182, 2022,
  \linkToPdf{https://arxiv.org/pdf/2205.13135.pdf}.

\bibitem[Jain et~al.(2023)Jain, Li, Chiu, Hassani, Orlov, and
  Shi]{Jain23cvpr-Oneformer}
J.~Jain, J.~Li, M.~Chiu, A.~Hassani, N.~Orlov, and H.~Shi, ``{OneFormer: One
  Transformer to Rule Universal Image Segmentation},'' 2023.

\bibitem[Wang et~al.(2021)Wang, Sun, Cheng, Jiang, Deng, Zhao, Liu, Mu, Tan,
  Wang, Liu, and Xiao]{Wang21pami-hrnet}
J.~Wang, K.~Sun, T.~Cheng, B.~Jiang, C.~Deng, Y.~Zhao, D.~Liu, Y.~Mu, M.~Tan,
  X.~Wang, W.~Liu, and B.~Xiao, ``Deep high-resolution representation learning
  for visual recognition,'' \emph{IEEE Transactions on Pattern Analysis and
  Machine Intelligence}, vol.~43, no.~10, pp. 3349--3364, 2021.

\bibitem[Sandler et~al.(2018)Sandler, Howard, Zhu, Zhmoginov, and
  Chen]{Sandler18cvpr-MobileNetV2}
M.~Sandler, A.~Howard, M.~Zhu, A.~Zhmoginov, and L.-C. Chen, ``Mobilenetv2:
  Inverted residuals and linear bottlenecks,'' in \emph{IEEE Conf. on Computer
  Vision and Pattern Recognition (CVPR)}, 2018, pp. 4510--4520.

\bibitem[Zhou et~al.(2017)Zhou, Zhao, Puig, Fidler, Barriuso, and
  Torralba]{Zhou17cvpr-ade20k}
B.~Zhou, H.~Zhao, X.~Puig, S.~Fidler, A.~Barriuso, and A.~Torralba, ``Scene
  parsing through ade20k dataset,'' in \emph{Proceedings of the IEEE Conference
  on Computer Vision and Pattern Recognition}, 2017.

\bibitem[Mikolov et~al.(2013)Mikolov, Chen, Corrado, and
  Dean]{Mikolov13-wordRepresentation}
T.~Mikolov, K.~Chen, G.~Corrado, and J.~Dean, ``Efficient estimation of word
  representations in vector space,'' \emph{arXiv preprints arXiv:1301.3781},
  2013.

\bibitem[Schroff et~al.(2015)Schroff, Kalenichenko, and
  Philbin]{Schroff15cvpr-Facenet}
F.~Schroff, D.~Kalenichenko, and J.~Philbin, ``Facenet: A unified embedding for
  face recognition and clustering,'' in \emph{IEEE Conf. on Computer Vision and
  Pattern Recognition (CVPR)}, 2015, pp. 815--823.

\bibitem[Xu et~al.(2019{\natexlab{a}})Xu, Hu, Leskovec, and
  Jegelka]{Xu19iclr-gin}
K.~Xu, W.~Hu, J.~Leskovec, and S.~Jegelka, ``How powerful are graph neural
  networks?'' in \emph{Intl. Conf. on Learning Representations (ICLR)}, May
  2019.

\bibitem[Chen et~al.(2022)Chen, Hu, Talak, and Carlone]{Chen22arxiv-LLM2}
W.~Chen, S.~Hu, R.~Talak, and L.~Carlone, ``Leveraging large language models
  for robot {3D} scene understanding,'' \emph{arXiv preprint: 2209.05629},
  2022, \linkToPdf{https://arxiv.org/pdf/2209.05629.pdf}.

\bibitem[Konidaris(2019)]{Konidaris19cobs-Necessityabstraction}
G.~Konidaris, ``On the necessity of abstraction,'' \emph{Current Opinion in
  Behavioral Sciences}, vol.~29, pp. 1--7, Oct. 2019.

\bibitem[Konidaris et~al.(2018)Konidaris, Kaelbling, and
  {Lozano-Perez}]{Konidaris18jair-SkillsSymbols}
G.~Konidaris, L.~P. Kaelbling, and T.~{Lozano-Perez}, ``From {{Skills}} to
  {{Symbols}}: {{Learning Symbolic Representations}} for {{Abstract High-Level
  Planning}},'' \emph{J. of Artificial Intelligence Research}, vol.~61, pp.
  215--289, Jan. 2018.

\bibitem[James et~al.(2020)James, Rosman, and
  Konidaris]{James20icml-PortableRepresentations}
S.~James, B.~Rosman, and G.~Konidaris, ``Learning portable representations for
  high-level planning,'' in \emph{Intl. Conf. on Machine Learning
  (ICML)}.\hskip 1em plus 0.5em minus 0.4em\relax PMLR, 2020, pp. 4682--4691.

\bibitem[James et~al.(2022)James, Rosman, and
  Konidaris]{James22iclr-AutonomousLearning}
------, ``Autonomous {{Learning}} of {{Object-Centric Abstractions}} for
  {{High-Level Planning}},'' in \emph{Intl. Conf. on Learning Representations
  (ICLR)}, May 2022.

\bibitem[Berg et~al.(2022)Berg, Konidaris, and
  Tellex]{Berg22icra-UsingLanguage}
M.~Berg, G.~Konidaris, and S.~Tellex, ``Using {{Language}} to {{Generate State
  Abstractions}} for {{Long-Range Planning}} in {{Outdoor Environments}},'' in
  \emph{IEEE Intl. Conf. on Robotics and Automation (ICRA)}, May 2022, pp.
  1888--1895.

\bibitem[Jinnai et~al.(2019)Jinnai, Abel, Hershkowitz, Littman, and
  Konidaris]{Jinnai19icml-OptionPlanning}
Y.~Jinnai, D.~Abel, D.~Hershkowitz, M.~Littman, and G.~Konidaris, ``Finding
  options that minimize planning time,'' in \emph{Intl. Conf. on Machine
  Learning (ICML)}.\hskip 1em plus 0.5em minus 0.4em\relax PMLR, 2019, pp.
  3120--3129.

\bibitem[Smith(2004)]{Smith04book-BeyondConcepts}
B.~Smith, \emph{Beyond Concepts: Ontology as Reality Representation}, 2004.

\bibitem[Guarino et~al.(2009)Guarino, Oberle, and
  Staab]{Guarino09hoo-WhatIsAnOntology}
N.~Guarino, D.~Oberle, and S.~Staab, ``What is an ontology?'' \emph{Handbook on
  ontologies}, pp. 1--17, 2009.

\bibitem[Jepsen(2009)]{Jepsen09itpm-WhatIsOntology}
T.~Jepsen, ``Just what is an ontology, anyway?'' \emph{IT Professional
  Magazine}, vol.~11, no.~5, p.~22, 2009.

\bibitem[Gruber(1995)]{Gruber95ijhcs-DesignOfOntologies}
T.~R. Gruber, ``Toward principles for the design of ontologies used for
  knowledge sharing?'' \emph{International journal of human-computer studies},
  vol.~43, no. 5-6, pp. 907--928, 1995.

\bibitem[Borst(1997)]{Borst97diss-constructOntologies}
W.~N. Borst, ``Construction of engineering ontologies for knowledge sharing and
  reuse,'' 1997, phD Dissertation.

\bibitem[Studer et~al.(1998)Studer, Benjamins, and
  Fensel]{Studer98dke-KnowledgeEngineering}
R.~Studer, R.~Benjamins, and D.~Fensel, ``Knowledge engineering: Principles and
  methods,'' \emph{Data \& Knowledge Engineering}, vol.~25, no. 1-2, pp.
  161--197, 1998.

\bibitem[McGuinness and Van~Harmelen(2004)]{Mcguinness04w3c-OWL}
D.~McGuinness and F.~Van~Harmelen, ``{OWL} web ontology language overview,''
  \emph{W3C recommendation}, 2004.

\bibitem[Genesereth and Nilsson(2012)]{Genesereth12book-LogicalFoundations}
M.~R. Genesereth and N.~J. Nilsson, \emph{Logical Foundations of Artificial
  Intelligence}.\hskip 1em plus 0.5em minus 0.4em\relax Morgan Kaufmann, 2012.

\bibitem[Schlenoff et~al.(2012)Schlenoff, Prestes, Madhavan, Goncalves, Li,
  Balakirsky, Kramer, and Miguelanez]{Schlenoff12iros-CORA}
C.~Schlenoff, E.~Prestes, R.~Madhavan, P.~Goncalves, H.~Li, S.~Balakirsky,
  T.~Kramer, and E.~Miguelanez, ``An {IEEE} standard ontology for robotics and
  automation,'' in \emph{IEEE/RSJ Intl. Conf. on Intelligent Robots and Systems
  (IROS)}.\hskip 1em plus 0.5em minus 0.4em\relax IEEE, 2012, pp. 1337--1342.

\bibitem[Lemaignan et~al.(2010)Lemaignan, Ros, M{\"o}senlechner, Alami, and
  Beetz]{Lemaignan10iros-ORO}
S.~Lemaignan, R.~Ros, L.~M{\"o}senlechner, R.~Alami, and M.~Beetz, ``{ORO}, a
  knowledge management platform for cognitive architectures in robotics,'' in
  \emph{IEEE/RSJ Intl. Conf. on Intelligent Robots and Systems (IROS)}.\hskip
  1em plus 0.5em minus 0.4em\relax IEEE, 2010, pp. 3548--3553.

\bibitem[Tenorth and Beetz(2013)]{Tenorth13ijrr-KnowRob}
M.~Tenorth and M.~Beetz, ``Knowrob: A knowledge processing infrastructure for
  cognition-enabled robots,'' \emph{Intl. J. of Robotics Research}, vol.~32,
  no.~5, pp. 566--590, 2013.

\bibitem[Beetz et~al.(2018)Beetz, Be{\ss}ler, Haidu, Pomarlan, Bozcuo{\u{g}}lu,
  and Bartels]{Beetz18icra-KnowRob2}
M.~Beetz, D.~Be{\ss}ler, A.~Haidu, M.~Pomarlan, A.~K. Bozcuo{\u{g}}lu, and
  G.~Bartels, ``{KnowRob 2.0}—a 2nd generation knowledge processing framework
  for cognition-enabled robotic agents,'' in \emph{IEEE Intl. Conf. on Robotics
  and Automation (ICRA)}.\hskip 1em plus 0.5em minus 0.4em\relax IEEE, 2018,
  pp. 512--519.

\bibitem[Diab et~al.(2019)Diab, Akbari, Ud~Din, and Rosell]{Diab19sensors-PMK}
M.~Diab, A.~Akbari, M.~Ud~Din, and J.~Rosell, ``{PMK}—a knowledge processing
  framework for autonomous robotics perception and manipulation,''
  \emph{Sensors}, vol.~19, no.~5, p. 1166, 2019.

\bibitem[Miller(1995)]{Miller95acm-WordNet}
\BIBentryALTinterwordspacing
G.~A. Miller, ``Wordnet: A lexical database for english,'' \emph{Commun. ACM},
  vol.~38, no.~11, p. 39–41, Nov. 1995. [Online]. Available:
  \url{https://doi.org/10.1145/219717.219748}
\BIBentrySTDinterwordspacing

\bibitem[Lenat(1995)]{Lenat95cacm-OpenCyc}
D.~B. Lenat, ``Cyc: A large-scale investment in knowledge infrastructure,''
  \emph{Communications of the ACM}, vol.~38, no.~11, pp. 33--38, 1995.

\bibitem[Niles and Pease(2001)]{Niles01fois-SUMO}
I.~Niles and A.~Pease, ``Towards a standard upper ontology,'' in
  \emph{Proceedings of the International Conference on Formal Ontology in
  Information Systems}, 2001, pp. 2--9.

\bibitem[Auer et~al.(2007)Auer, Bizer, Kobilarov, Lehmann, Cyganiak, and
  Ives]{Auer07sw-DBpedia}
S.~Auer, C.~Bizer, G.~Kobilarov, J.~Lehmann, R.~Cyganiak, and Z.~Ives,
  ``{DBpedia}: A nucleus for a web of open data,'' in \emph{Semantic
  Web}.\hskip 1em plus 0.5em minus 0.4em\relax Springer, 2007, pp. 722--735.

\bibitem[Suchanek et~al.(2008)Suchanek, Kasneci, and Weikum]{Suchanek08sw-YAGO}
F.~M. Suchanek, G.~Kasneci, and G.~Weikum, ``{YAGO}: A large ontology from
  wikipedia and {WordNet},'' \emph{Journal of Web Semantics}, vol.~6, no.~3,
  pp. 203--217, 2008.

\bibitem[Bollacker et~al.(2008)Bollacker, Evans, Paritosh, Sturge, and
  Taylor]{Bollacker08icmd-Freebase}
K.~Bollacker, C.~Evans, P.~Paritosh, T.~Sturge, and J.~Taylor, ``Freebase: a
  collaboratively created graph database for structuring human knowledge,'' in
  \emph{Proceedings of the ACM SIGMOD International Conference on Management of
  Data}, 2008, pp. 1247--1250.

\bibitem[Carlson et~al.(2010)Carlson, Betteridge, Kisiel, Settles, Hruschka,
  and Mitchell]{Carlson10aaai-NELL}
A.~Carlson, J.~Betteridge, B.~Kisiel, B.~Settles, E.~Hruschka, and T.~Mitchell,
  ``Toward an architecture for never-ending language learning,'' in
  \emph{Proceedings of the AAAI Conference on Artificial Intelligence},
  vol.~24, no.~1, 2010, pp. 1306--1313.

\bibitem[Vrande\v{c}i\'{c} and Kr\"{o}tzsch(2014)]{Vrandecic14-Wikidata}
\BIBentryALTinterwordspacing
D.~Vrande\v{c}i\'{c} and M.~Kr\"{o}tzsch, ``Wikidata: A free collaborative
  knowledgebase,'' \emph{Commun. ACM}, vol.~57, no.~10, p. 78–85, sep 2014.
  [Online]. Available: \url{https://doi.org/10.1145/2629489}
\BIBentrySTDinterwordspacing

\bibitem[Speer et~al.(2017)Speer, Chin, and Havasi]{Speer17aaai-ConceptNet}
R.~Speer, J.~Chin, and C.~Havasi, ``{ConceptNet 5.5:} an open multilingual
  graph of general knowledge,'' in \emph{Proceedings of the Thirty-First AAAI
  Conference on Artificial Intelligence}, ser. AAAI’17.\hskip 1em plus 0.5em
  minus 0.4em\relax AAAI Press, 2017, p. 4444–4451.

\bibitem[Zareian et~al.(2020)Zareian, Karaman, and
  Chang]{Zareian20eccv-BridgingKnowledgeSGG}
A.~Zareian, S.~Karaman, and S.-F. Chang, ``Bridging knowledge graphs to
  generate scene graphs,'' in \emph{European Conf. on Computer Vision
  (ECCV)}.\hskip 1em plus 0.5em minus 0.4em\relax Springer, 2020, pp. 606--623.

\bibitem[Guo et~al.(2021)Guo, Gao, Wang, Hu, Xu, Lu, Shen, and
  Song]{Guo21iccv-GeneralToSpecificSGG}
Y.~Guo, L.~Gao, X.~Wang, Y.~Hu, X.~Xu, X.~Lu, H.~T. Shen, and J.~Song, ``From
  general to specific: Informative scene graph generation via balance
  adjustment,'' in \emph{Intl. Conf. on Computer Vision (ICCV)}, 2021, pp.
  16\,383--16\,392.

\bibitem[Amodeo et~al.(2022)Amodeo, Caballero, D{\'\i}az-Rodr{\'\i}guez, and
  Merino]{Amodeo22access-OGSGG}
F.~Amodeo, F.~Caballero, N.~D{\'\i}az-Rodr{\'\i}guez, and L.~Merino, ``Og-sgg:
  Ontology-guided scene graph generation—a case study in transfer learning
  for telepresence robotics,'' \emph{IEEE Access}, vol.~10, pp.
  132\,564--132\,583, 2022.

\bibitem[Chen et~al.(2023)Chen, Rezayi, and
  Li]{Chen23wacv-UnbiasingSGGOntology}
Z.~Chen, S.~Rezayi, and S.~Li, ``More knowledge, less bias: Unbiasing scene
  graph generation with explicit ontological adjustment,'' in \emph{Proceedings
  of the IEEE/CVF Winter Conference on Applications of Computer Vision}, 2023,
  pp. 4023--4032.

\bibitem[Movshovitz-Attias et~al.(2015)Movshovitz-Attias, Yu, Stumpe, Shet,
  Arnoud, and Yatziv]{Movshovitz15cvpr-OntoStoreClass}
Y.~Movshovitz-Attias, Q.~Yu, M.~C. Stumpe, V.~Shet, S.~Arnoud, and L.~Yatziv,
  ``Ontological supervision for fine grained classification of street view
  storefronts,'' in \emph{IEEE Conf. on Computer Vision and Pattern Recognition
  (CVPR)}, 2015, pp. 1693--1702.

\bibitem[Porello et~al.(2015)Porello, Cristani, and
  Ferrario]{Porello15nci-IntegratingOntologies}
D.~Porello, M.~Cristani, and R.~Ferrario, ``Integrating {{Ontologies}} and
  {{Computer Vision}} for {{Classification}} of {{Objects}} in {{Images}},'' in
  \emph{Workshop on Neural Cognitive Integration}, 2015, p.~15.

\bibitem[Aditya et~al.(2018)Aditya, Yang, Baral, Aloimonos, and
  Ferm{\"u}ller]{Aditya18cviu-SceneDescriptionGraphReasoning}
S.~Aditya, Y.~Yang, C.~Baral, Y.~Aloimonos, and C.~Ferm{\"u}ller, ``Image
  understanding using vision and reasoning through scene description graph,''
  \emph{Computer Vision and Image Understanding}, vol. 173, pp. 33--45, 2018.

\bibitem[Gard{\`e}res et~al.(2020)Gard{\`e}res, Ziaeefard, Abeloos, and
  Lecue]{Garderes20emnlp-ConceptBert}
F.~Gard{\`e}res, M.~Ziaeefard, B.~Abeloos, and F.~Lecue, ``{ConceptBert}:
  Concept-aware representation for visual question answering,'' in
  \emph{Conference on Empirical Methods in Natural Language Processing}, 2020,
  pp. 489--498.

\bibitem[Marino et~al.(2021)Marino, Chen, Parikh, Gupta, and
  Rohrbach]{Marino21cvpr-KRISP}
K.~Marino, X.~Chen, D.~Parikh, A.~Gupta, and M.~Rohrbach, ``{KRISP}:
  Integrating implicit and symbolic knowledge for open-domain knowledge-based
  {VQA},'' in \emph{IEEE Conf. on Computer Vision and Pattern Recognition
  (CVPR)}, 2021, pp. 14\,111--14\,121.

\bibitem[Zheng et~al.(2021)Zheng, Yin, Chen, Ma, Liu, and
  Yang]{Zheng21pr-EmbeddingDesignVQA}
W.~Zheng, L.~Yin, X.~Chen, Z.~Ma, S.~Liu, and B.~Yang, ``Knowledge base graph
  embedding module design for visual question answering model,'' \emph{Pattern
  Recognition}, vol. 120, p. 108153, 2021.

\bibitem[Ding et~al.(2022)Ding, Yu, Liu, Hu, Cui, and Wu]{Ding22cvpr-MuKEA}
Y.~Ding, J.~Yu, B.~Liu, Y.~Hu, M.~Cui, and Q.~Wu, ``{MuKEA}: Multimodal
  knowledge extraction and accumulation for knowledge-based visual question
  answering,'' in \emph{IEEE Conf. on Computer Vision and Pattern Recognition
  (CVPR)}, 2022, pp. 5089--5098.

\bibitem[Chen et~al.(2020)Chen, Tan, Kuntz, Bansal, and
  Alterovitz]{Chen20icra-IncompleteLangInstr}
H.~Chen, H.~Tan, A.~Kuntz, M.~Bansal, and R.~Alterovitz, ``Enabling robots to
  understand incomplete natural language instructions using commonsense
  reasoning,'' in \emph{IEEE Intl. Conf. on Robotics and Automation
  (ICRA)}.\hskip 1em plus 0.5em minus 0.4em\relax IEEE, 2020, pp. 1963--1969.

\bibitem[Daruna et~al.(2021)Daruna, Nair, Liu, and
  Chernova]{Daruna21icra-OneShotTask}
A.~Daruna, L.~Nair, W.~Liu, and S.~Chernova, ``Towards robust one-shot task
  execution using knowledge graph embeddings,'' in \emph{IEEE Intl. Conf. on
  Robotics and Automation (ICRA)}.\hskip 1em plus 0.5em minus 0.4em\relax IEEE,
  2021, pp. 11\,118--11\,124.

\bibitem[Tuli et~al.(2022)Tuli, Bansal, Paul, et~al.]{Tuli22jair-ToolTango}
S.~Tuli, R.~Bansal, R.~Paul \emph{et~al.}, ``{ToolTango}: Common sense
  generalization in predicting sequential tool interactions for robot plan
  synthesis,'' \emph{The Journal of Artificial Intelligence Research}, vol.~75,
  pp. 1595--1631, 2022.

\bibitem[Hao et~al.(2019)Hao, Chen, Yu, Sun, and Wang]{Hao19kdd-RepLearnOnto}
J.~Hao, M.~Chen, W.~Yu, Y.~Sun, and W.~Wang, ``Universal representation
  learning of knowledge bases by jointly embedding instances and ontological
  concepts,'' in \emph{ACM SIGKDD International Conference on Knowledge
  Discovery and Data Mining}, 2019, pp. 1709--1719.

\bibitem[Kwak et~al.(2022)Kwak, Lee, Whang, and Jo]{Kwak22ral-GraspKnowGraph}
J.~H. Kwak, J.~Lee, J.~J. Whang, and S.~Jo, ``Semantic grasping via a knowledge
  graph of robotic manipulation: A graph representation learning approach,''
  \emph{{IEEE} Robotics and Automation Letters}, vol.~7, no.~4, pp. 9397--9404,
  2022.

\bibitem[Geman et~al.(2002)Geman, Potter, and
  Chi]{Geman02qam-CompositionSystems}
S.~Geman, D.~F. Potter, and Z.~Chi, ``Composition {{Systems}},''
  \emph{Quarterly of Applied Mathematics}, vol.~60, no.~4, pp. 707--736, 2002.

\bibitem[Lake et~al.(2017)Lake, Ullman, Tenenbaum, and
  Gershman]{Lake17bbs-Buildingmachines}
B.~M. Lake, T.~D. Ullman, J.~B. Tenenbaum, and S.~J. Gershman, ``Building
  machines that learn and think like people,'' \emph{Behavioral and Brain
  Sciences}, vol.~40, 2017.

\bibitem[Zhu et~al.(2011)Zhu, Chen, and
  Yuille]{Zhu11jmiv-RecursiveCompositional}
L.~L. Zhu, Y.~Chen, and A.~Yuille, ``Recursive {{Compositional Models}} for
  {{Vision}}: {{Description}} and {{Review}} of {{Recent Work}},''
  \emph{Journal of Mathematical Imaging and Vision}, vol.~41, no.~1, pp.
  122--146, Sep. 2011.

\bibitem[Zhu and Mumford(2006)]{Zhu06book-StochasticGrammar}
S.-C. Zhu and D.~Mumford, ``A {{Stochastic Grammar}} of {{Images}},''
  \emph{Foundations and Trends in Computer Graphics and Vision}, vol.~2, no.~4,
  pp. 259--362, 2006.

\bibitem[Zhu and Huang(2021)]{Zhu21book-ComputerVision}
S.-C. Zhu and S.~Huang, \emph{Computer {{Vision}}: {{Stochastic Grammars}} for
  {{Parsing Objects}}, {{Scenes}}, and {{Events}}}.\hskip 1em plus 0.5em minus
  0.4em\relax Springer, 2021.

\bibitem[Izatt and Tedrake(2020)]{Izatt20icra-sceneGrammarSim}
G.~Izatt and R.~Tedrake, ``Generative modeling of environments with scene
  grammars and variational inference,'' in \emph{2020 IEEE International
  Conference on Robotics and Automation (ICRA)}, 2020, pp. 6891--6897.

\bibitem[Qi et~al.(2018)Qi, Zhu, Huang, Jiang, and
  Zhu]{Qi18cvpr-HumanCentricIndoor}
S.~Qi, Y.~Zhu, S.~Huang, C.~Jiang, and S.-C. Zhu, ``Human-{{Centric Indoor
  Scene Synthesis Using Stochastic Grammar}},'' in \emph{IEEE Conf. on Computer
  Vision and Pattern Recognition (CVPR)}, Jun. 2018, pp. 5899--5908.

\bibitem[Chua(2018)]{Chua18th-ProbabilisticScene}
J.~Chua, ``Probabilistic {{Scene Grammars}}: {{A General-Purpose Framework For
  Scene Understanding}},'' \emph{Brown University Thesis}, pp. 1--146, 2018.

\bibitem[Wang et~al.(2022)Wang, Zhou, Qi, Shen, and
  Zhu]{Wang22pami-HierarchicalHuman}
W.~Wang, T.~Zhou, S.~Qi, J.~Shen, and S.-C. Zhu, ``Hierarchical {{Human
  Semantic Parsing With Comprehensive Part-Relation Modeling}},'' \emph{{IEEE}
  Trans. Pattern Anal. Machine Intell.}, vol.~44, no.~7, pp. 3508--3522, Jul.
  2022.

\bibitem[Niemeyer and Geiger(2021)]{Niemeyer21cvpr-GIRAFFERepresenting}
M.~Niemeyer and A.~Geiger, ``{{GIRAFFE}}: {{Representing Scenes}} as
  {{Compositional Generative Neural Feature Fields}},'' in \emph{IEEE Conf. on
  Computer Vision and Pattern Recognition (CVPR)}, Jun. 2021, pp.
  11\,448--11\,459.

\bibitem[Mo et~al.(2020)Mo, Guerrero, Yi, Su, Wonka, Mitra, and
  Guibas]{Mo20cvpr-StructEditLearning}
K.~Mo, P.~Guerrero, L.~Yi, H.~Su, P.~Wonka, N.~J. Mitra, and L.~J. Guibas,
  ``{{StructEdit}}: {{Learning Structural Shape Variations}},'' in \emph{IEEE
  Conf. on Computer Vision and Pattern Recognition (CVPR)}, 2020, pp.
  8859--8868.

\bibitem[Ichien et~al.(2021)Ichien, Liu, Fu, Holyoak, Yuille, and
  Lu]{Ichien21css-VisualAnalogy}
N.~Ichien, Q.~Liu, S.~Fu, K.~Holyoak, A.~Yuille, and H.~Lu, ``Visual
  {{Analogy}}: {{Deep Learning Versus Compositional Models}},'' \emph{Annual
  Meeting of the Cognitive Science Society}, vol.~43, 2021.

\bibitem[Yuan et~al.(2023)Yuan, Chen, Li, and
  Xue]{Yuan23arxiv-CompositionalScene}
J.~Yuan, T.~Chen, B.~Li, and X.~Xue, ``Compositional {{Scene Representation
  Learning}} via {{Reconstruction}}: {{A Survey}},'' Feb. 2023.

\bibitem[Fodor and Pylyshyn(1988)]{Fodor88cognition-Connectionismcognitive}
J.~A. Fodor and Z.~W. Pylyshyn, ``Connectionism and cognitive architecture:
  {{A}} critical analysis,'' \emph{Cognition}, vol.~28, no.~1, pp. 3--71, Mar.
  1988.

\bibitem[Mhaskar(1996)]{Mhaskar1996j-NeurComput-NNApproxAnalyticFun}
H.~N. Mhaskar, ``Neural networks for optimal approximation of smooth and
  analytic functions,'' \emph{Neural Computation}, vol.~8, no.~1, pp. 164--177,
  1996.

\bibitem[Mhaskar and Poggio(2016)]{Mhaskar2016j-AA-DeepVsShallowNN}
H.~N. Mhaskar and T.~Poggio, ``Deep vs. shallow networks: An approximation
  theory perspective,'' \emph{Analysis and Applications}, vol.~14, no.~06, pp.
  829--848, 2016.

\bibitem[Poggio et~al.(2017)Poggio, Mhaskar, Rosasco, Miranda, and
  Liao]{Poggio2017j-IJAC-WhyWhenDNN}
T.~Poggio, H.~Mhaskar, L.~Rosasco, B.~Miranda, and Q.~Liao, ``Why and when can
  deep - but not shallow - networks avoid the curse of dimensionality: A
  review,'' \emph{Int. J. Autom. Comput.}, vol.~14, pp. 503--519, Mar. 2017.

\bibitem[Webb et~al.(2022)Webb, Holyoak, and
  Lu]{Webb22arxiv-EmergentAnalogical}
T.~Webb, K.~J. Holyoak, and H.~Lu, ``Emergent {{Analogical Reasoning}} in
  {{Large Language Models}},'' Dec. 2022.

\bibitem[Xie et~al.(2022)Xie, Morcos, Zhu, and
  Vedantam]{Xie22icml-COATMeasuring}
S.~Xie, A.~S. Morcos, S.-C. Zhu, and R.~Vedantam, ``{{COAT}}: {{Measuring
  Object Compositionality}} in {{Emergent Representations}},'' in \emph{Intl.
  Conf. on Machine Learning (ICML)}, Jun. 2022, pp. 24\,388--24\,413.

\bibitem[Krishna et~al.(2016)Krishna, Zhu, Groth, Johnson, Hata, Kravitz, Chen,
  Kalantidis, Li, Shamma, Bernstein, and Fei-Fei]{Krishna16arxiv-VisualGenome}
R.~Krishna, Y.~Zhu, O.~Groth, J.~Johnson, K.~Hata, J.~Kravitz, S.~Chen,
  Y.~Kalantidis, L.~Li, D.~Shamma, M.~Bernstein, and L.~Fei-Fei, ``{Visual
  Genome}: Connecting language and vision using crowdsourced dense image
  annotations,'' \emph{arXiv preprints arXiv:1602.07332}, 2016.

\bibitem[Johnson et~al.(2015)Johnson, Krishna, Stark, Li, Shamma, Bernstein,
  and Fei-Fei]{Johnson15cvpr}
J.~Johnson, R.~Krishna, M.~Stark, L.~Li, D.~Shamma, M.~Bernstein, and
  L.~Fei-Fei, ``Image retrieval using scene graphs,'' in \emph{IEEE Conf. on
  Computer Vision and Pattern Recognition (CVPR)}, 2015, pp. 3668--3678.

\bibitem[Karpathy and Fei-Fei(2015)]{Karpathy15cvpr-caption}
A.~Karpathy and L.~Fei-Fei, ``Deep visual-semantic alignments for generating
  image descriptions,'' in \emph{IEEE Conf. on Computer Vision and Pattern
  Recognition (CVPR)}, 2015.

\bibitem[Anderson et~al.(2016)Anderson, Fernando, Johnson, and
  Gould]{Anderson16eccv-sceneGraphDescription}
P.~Anderson, B.~Fernando, M.~Johnson, and S.~Gould, ``Spice: Semantic
  propositional image caption evaluation,'' in \emph{European Conf. on Computer
  Vision (ECCV)}, 2016, pp. 382--398.

\bibitem[Ren et~al.(2015)Ren, Kiros, and Zemel]{Ren15corr-QA}
M.~Ren, R.~Kiros, and R.~S. Zemel, ``Image question answering: {A} visual
  semantic embedding model and a new dataset,'' \emph{arXiv preprints
  arXiv:1505.02074}, 2015.

\bibitem[Lu et~al.(2016)Lu, Krishna, Bernstein, and
  Li]{Lu16eccv-visualRelations}
C.~Lu, R.~Krishna, M.~Bernstein, and F.-F. Li, ``Visual relationship detection
  with language priors,'' in \emph{European Conference on Computer Vision},
  2016, pp. 852--869.

\bibitem[Xu et~al.(2017)Xu, Zhu, Choy, and Fei-Fei]{Xu17cvpr-sceneGraph}
D.~Xu, Y.~Zhu, C.~B. Choy, and L.~Fei-Fei, ``Scene graph generation by
  iterative message passing,'' in \emph{IEEE Conf. on Computer Vision and
  Pattern Recognition (CVPR)}, 2017, pp. 3097--3106.

\bibitem[Li et~al.(2017)Li, Ouyang, Zhou, Wang, and
  Wang]{Li17iccv-sceneGraphGeneration}
Y.~Li, W.~Ouyang, B.~Zhou, K.~Wang, and X.~Wang, ``Scene graph generation from
  objects, phrases and region captions,'' in \emph{International Conference on
  Computer Vision (ICCV)}, 2017.

\bibitem[Yang et~al.(2018)Yang, Lu, Lee, Batra, and
  Parikh]{Yang18eccv-sceneGraph}
J.~Yang, J.~Lu, S.~Lee, D.~Batra, and D.~Parikh, ``Graph {R-CNN} for scene
  graph generation,'' in \emph{European Conf. on Computer Vision (ECCV)}, 2018.

\bibitem[Zellers et~al.(2017)Zellers, Yatskar, Thomson, and
  Choi]{Zellers18cvpr-sceneGraph}
R.~Zellers, M.~Yatskar, S.~Thomson, and Y.~Choi, ``Neural motifs: Scene graph
  parsing with global context,'' in \emph{IEEE Conf. on Computer Vision and
  Pattern Recognition (CVPR)}, 2017.

\bibitem[Chang et~al.(2023)Chang, Ren, Xu, Li, Chen, and
  Hauptmann]{Chang23pami-sceneGraphSurvey}
X.~Chang, P.~Ren, P.~Xu, Z.~Li, X.~Chen, and A.~Hauptmann, ``A comprehensive
  survey of scene graphs: Generation and application,'' \emph{{IEEE} Trans.
  Pattern Anal. Machine Intell.}, vol.~45, no.~1, pp. 1--26, 2023.

\bibitem[Salas-Moreno et~al.(2013)Salas-Moreno, Newcombe, Strasdat, Kelly, and
  Davison]{Salas-Moreno13cvpr}
R.~F. Salas-Moreno, R.~A. Newcombe, H.~Strasdat, P.~H.~J. Kelly, and A.~J.
  Davison, ``{SLAM++}: Simultaneous localisation and mapping at the level of
  objects,'' in \emph{IEEE Conf. on Computer Vision and Pattern Recognition
  (CVPR)}, 2013.

\bibitem[Dong et~al.(2017)Dong, Fei, and Soatto]{Dong17cvpr-XVIO}
J.~Dong, X.~Fei, and S.~Soatto, ``{Visual-Inertial-Semantic} scene
  representation for {3D} object detection,'' in \emph{IEEE Conf. on Computer
  Vision and Pattern Recognition (CVPR)}, 2017.

\bibitem[Shan et~al.(2020)Shan, Feng, and Atanasov]{Mo19iros-orcVIO}
M.~Shan, Q.~Feng, and N.~Atanasov, ``Object residual constrained
  visual-inertial odometry,'' in \emph{IEEE/RSJ Intl. Conf. on Intelligent
  Robots and Systems (IROS)}, 2020, pp. 5104--5111.

\bibitem[Nicholson et~al.(2018)Nicholson, Milford, and
  S{\"u}nderhauf]{Nicholson18ral-quadricSLAM}
L.~Nicholson, M.~Milford, and N.~S{\"u}nderhauf, ``{QuadricSLAM}: Dual quadrics
  from object detections as landmarks in object-oriented {SLAM},'' \emph{{IEEE}
  Robotics and Automation Letters}, vol.~4, pp. 1--8, 2018.

\bibitem[Bowman et~al.(2017)Bowman, Atanasov, Daniilidis, and
  Pappas]{Bowman17icra}
S.~Bowman, N.~Atanasov, K.~Daniilidis, and G.~Pappas, ``Probabilistic data
  association for semantic {SLAM},'' in \emph{IEEE Intl. Conf. on Robotics and
  Automation (ICRA)}, 2017, pp. 1722--1729.

\bibitem[Ok et~al.(2021)Ok, Liu, and Roy]{Ok21icra-home}
K.~Ok, K.~Liu, and N.~Roy, ``Hierarchical object map estimation for efficient
  and robust navigation,'' in \emph{2021 IEEE International Conference on
  Robotics and Automation (ICRA)}, 2021, pp. 1132--1139.

\bibitem[McCormac et~al.(2017)McCormac, Handa, Davison, and
  Leutenegger]{McCormac17icra-semanticFusion}
J.~McCormac, A.~Handa, A.~J. Davison, and S.~Leutenegger, ``{SemanticFusion:
  Dense 3D Semantic Mapping with Convolutional Neural Networks},'' in
  \emph{IEEE Intl. Conf. on Robotics and Automation (ICRA)}, 2017.

\bibitem[Narita et~al.(2019)Narita, Seno, Ishikawa, and
  Kaji]{Narita19iros-metricSemantic}
G.~Narita, T.~Seno, T.~Ishikawa, and Y.~Kaji, ``Panopticfusion: Online
  volumetric semantic mapping at the level of stuff and things,'' in
  \emph{IEEE/RSJ Intl. Conf. on Intelligent Robots and Systems (IROS)}, 2019.

\bibitem[Behley et~al.(2019)Behley, Garbade, Milioto, Quenzel, Behnke,
  Stachniss, and Gall]{Behley19iccv-semanticKitti}
J.~Behley, M.~Garbade, A.~Milioto, J.~Quenzel, S.~Behnke, C.~Stachniss, and
  J.~Gall, ``{SemanticKITTI: A Dataset for Semantic Scene Understanding of
  LiDAR Sequences},'' in \emph{Intl. Conf. on Computer Vision (ICCV)}, 2019.

\bibitem[{Tateno} et~al.(2015){Tateno}, {Tombari}, and
  {Navab}]{Tateno15iros-metricSemantic}
K.~{Tateno}, F.~{Tombari}, and N.~{Navab}, ``Real-time and scalable incremental
  segmentation on dense {SLAM},'' in \emph{IEEE/RSJ Intl. Conf. on Intelligent
  Robots and Systems (IROS)}, 2015, pp. 4465--4472.

\bibitem[Lianos et~al.(2018)Lianos, Sch{\"o}nberger, Pollefeys, and
  Sattler]{Lianos18eccv-VSO}
K.~Lianos, J.~Sch{\"o}nberger, M.~Pollefeys, and T.~Sattler, ``Vso: Visual
  semantic odometry,'' in \emph{European Conf. on Computer Vision (ECCV)},
  2018, pp. 246--263.

\bibitem[Rosu et~al.(2019)Rosu, Quenzel, and Behnke]{Rosu19ijcv-semanticMesh}
R.~Rosu, J.~Quenzel, and S.~Behnke, ``Semi-supervised semantic mapping through
  label propagation with semantic texture meshes,'' \emph{Intl. J. of Computer
  Vision}, 06 2019.

\bibitem[{Li} et~al.(2016){Li}, {Xiao}, {Tateno}, {Tombari}, {Navab}, and
  {Hager}]{Li16iros-metricSemantic}
C.~{Li}, H.~{Xiao}, K.~{Tateno}, F.~{Tombari}, N.~{Navab}, and G.~D. {Hager},
  ``Incremental scene understanding on dense {SLAM},'' in \emph{IEEE/RSJ Intl.
  Conf. on Intelligent Robots and Systems (IROS)}, 2016, pp. 574--581.

\bibitem[McCormac et~al.(2018)McCormac, Clark, Bloesch, Davison, and
  Leutenegger]{McCormac183dv-fusion++}
J.~McCormac, R.~Clark, M.~Bloesch, A.~Davison, and S.~Leutenegger, ``Fusion++:
  Volumetric object-level {SLAM},'' in \emph{Intl. Conf. on 3D Vision (3DV)},
  2018, pp. 32--41.

\bibitem[Xu et~al.(2019{\natexlab{b}})Xu, Li, Tzoumanikas, Bloesch, Davison,
  and Leutenegger]{Xu19icra-midFusion}
B.~Xu, W.~Li, D.~Tzoumanikas, M.~Bloesch, A.~Davison, and S.~Leutenegger,
  ``{MID-Fusion}: Octree-based object-level multi-instance dynamic {SLAM},''
  2019, pp. 5231--5237.

\bibitem[Schmid et~al.(2021)Schmid, Delmerico, Sch{\"o}nberger, Nieto,
  Pollefeys, Siegwart, and Cadena]{Schmid21arxiv-panoptic}
L.~Schmid, J.~Delmerico, J.~Sch{\"o}nberger, J.~Nieto, M.~Pollefeys,
  R.~Siegwart, and C.~Cadena, ``Panoptic multi-tsdfs: a flexible representation
  for online multi-resolution volumetric mapping and long-term dynamic scene
  consistency,'' \emph{arXiv preprint arXiv:2109.10165}, 2021.

\bibitem[Liu et~al.(2018)Liu, Wu, and Furukawa]{Liu18eccv-floorNet}
C.~Liu, J.~Wu, and Y.~Furukawa, ``{FloorNet}: A unified framework for floorplan
  reconstruction from 3d scans,'' in \emph{Proceedings of the European
  Conference on Computer Vision (ECCV)}, September 2018.

\bibitem[Friedman et~al.(2007)Friedman, Pasula, and
  Fox]{Friedman07ijcai-voronoiRF}
S.~Friedman, H.~Pasula, and D.~Fox, ``Voronoi random fields: Extracting the
  topological structure of indoor environments via place labeling,'' in
  \emph{Intl. Joint Conf. on AI (IJCAI)}.\hskip 1em plus 0.5em minus
  0.4em\relax San Francisco, CA, USA: Morgan Kaufmann Publishers Inc., 2007, p.
  2109–2114.

\bibitem[Stekovic et~al.(2021)Stekovic, Rad, Fraundorfer, and
  Lepetit]{Stekovic21arxiv-monteFloor}
S.~Stekovic, M.~Rad, F.~Fraundorfer, and V.~Lepetit, ``{MonteFloor}: Extending
  {MCTS} for reconstructing accurate large-scale floor plans,'' 2021.

\bibitem[Furukawa et~al.(2009)Furukawa, Curless, Seitz, and
  Szeliski]{Furukawa09iccv}
Y.~Furukawa, B.~Curless, S.~M. Seitz, and R.~Szeliski, ``Reconstructing
  building interiors from images,'' in \emph{Intl. Conf. on Computer Vision
  (ICCV)}, 2009.

\bibitem[{Lukierski} et~al.(2017){Lukierski}, {Leutenegger}, and
  {Davison}]{Lukierski17icra-floorPlan}
R.~{Lukierski}, S.~{Leutenegger}, and A.~J. {Davison}, ``Room layout estimation
  from rapid omnidirectional exploration,'' in \emph{IEEE Intl. Conf. on
  Robotics and Automation (ICRA)}, 2017, pp. 6315--6322.

\bibitem[Zheng et~al.(2020)Zheng, Zhang, Han, Xu, and
  Fang]{Zheng20pami-buildingFusion}
T.~Zheng, G.~Zhang, L.~Han, L.~Xu, and L.~Fang, ``Building fusion:
  Semantic-aware structural building-scale 3d reconstruction,'' \emph{IEEE
  Transactions on Pattern Analysis and Machine Intelligence}, pp. 1--1, 2020.

\bibitem[Kuipers(2000)]{Kuipers00ai}
B.~Kuipers, ``The {S}patial {S}emantic {H}ierarchy,'' \emph{Artificial
  {I}ntelligence}, vol. 119, pp. 191--233, 2000.

\bibitem[Kuipers(1978)]{Kuipers78cs}
------, ``Modeling spatial knowledge,'' \emph{Cognitive Science}, vol.~2, pp.
  129--153, 1978.

\bibitem[Chatila and Laumond(1985)]{Chatila85}
R.~Chatila and J.-P. Laumond, ``Position referencing and consistent world
  modeling for mobile robots,'' in \emph{IEEE Intl. Conf. on Robotics and
  Automation (ICRA)}, 1985, pp. 138--145.

\bibitem[Thrun(2003)]{Thrun02a}
S.~Thrun, ``Robotic mapping: a survey,'' in \emph{Exploring artificial
  intelligence in the new millennium}.\hskip 1em plus 0.5em minus 0.4em\relax
  Morgan Kaufmann, Inc., 2003, pp. 1--35.

\bibitem[Ruiz-Sarmiento et~al.(2017)Ruiz-Sarmiento, Galindo, and
  Gonzalez-Jimenez]{Ruiz-Sarmiento17kbs-multiversalMaps}
J.-R. Ruiz-Sarmiento, C.~Galindo, and J.~Gonzalez-Jimenez, ``Building
  multiversal semantic maps for mobile robot operation,'' \emph{Knowledge-Based
  Systems}, vol. 119, pp. 257--272, 2017.

\bibitem[Galindo et~al.(2005)Galindo, Saffiotti, Coradeschi, Buschka,
  Fern\'andez-Madrigal, and Gonz\'alez]{Galindo05iros-multiHierarchicalMaps}
C.~Galindo, A.~Saffiotti, S.~Coradeschi, P.~Buschka, J.~Fern\'andez-Madrigal,
  and J.~Gonz\'alez, ``Multi-hierarchical semantic maps for mobile robotics,''
  in \emph{IEEE/RSJ Intl. Conf. on Intelligent Robots and Systems (IROS)},
  2005, pp. 3492--3497.

\bibitem[Zender et~al.(2008)Zender, Mozos, Jensfelt, Kruijff, and
  Burgard]{Zender08ras-spatialRepresentations}
H.~Zender, O.~M. Mozos, P.~Jensfelt, G.-J. Kruijff, and W.~Burgard,
  ``Conceptual spatial representations for indoor mobile robots,''
  \emph{Robotics and Autonomous Systems}, vol.~56, no.~6, pp. 493--502, 2008,
  from Sensors to Human Spatial Concepts.

\bibitem[Choset and Nagatani(2001)]{Choset01tra}
H.~Choset and K.~Nagatani, ``Topological simultaneous localization and mapping
  ({SLAM}): toward exact localization without explicit localization,''
  \emph{{IEEE} Trans. Robot. Automat.}, vol.~17, no.~2, pp. 125 -- 137, April
  2001.

\bibitem[Beeson et~al.(2010)Beeson, Modayil, and
  Kuipers]{Beeson10ijrr-factoringSSH}
P.~Beeson, J.~Modayil, and B.~Kuipers, ``Factoring the mapping problem: Mobile
  robot map-building in the hybrid spatial semantic hierarchy,'' \emph{Intl. J.
  of Robotics Research}, vol.~29, no.~4, pp. 428--459, 2010.

\bibitem[Gay et~al.(2018)Gay, Stuart, and
  Del~Bue]{Gay18accv-VisualGraphsFromMotion}
P.~Gay, J.~Stuart, and A.~Del~Bue, ``Visual graphs from motion ({VGfM}): Scene
  understanding with object geometry reasoning,'' in \emph{Asian Conf. on
  Computer Vision (ACCV)}.\hskip 1em plus 0.5em minus 0.4em\relax Springer
  International Publishing, 2018, pp. 330--346.

\bibitem[Lowry et~al.(2016)Lowry, S{\"u}nderhauf, Newman, Leonard, Cox, Corke,
  and Milford]{Lowry16tro-surveyPlaceRecognition}
S.~Lowry, N.~S{\"u}nderhauf, P.~Newman, J.~Leonard, D.~Cox, P.~Corke, and
  M.~Milford, ``Visual place recognition: A survey,'' \emph{{IEEE} Trans.
  Robotics}, vol.~32, no.~1, pp. 1--19, 2016.

\bibitem[Schubert et~al.(2021)Schubert, Neubert, and
  Protzel]{Schubert21rss-fastICM}
S.~Schubert, P.~Neubert, and P.~Protzel, ``Fast and memory efficient graph
  optimization via {ICM} for visual place recognition,'' in \emph{Proc. of
  Robotics: Science and Systems (RSS)}, 2021.

\bibitem[Milford and Wyeth(2012)]{Milford12icra}
M.~Milford and G.~Wyeth, ``Seqslam: Visual route-based navigation for sunny
  summer days and stormy winter nights,'' in \emph{IEEE Intl. Conf. on Robotics
  and Automation (ICRA)}, may 2012, pp. 1643 --1649.

\bibitem[Garg and Milford(2021)]{Garg21ral-seqnet}
S.~Garg and M.~Milford, ``Seqnet: Learning descriptors for sequence-based
  hierarchical place recognition,'' \emph{IEEE Robotics and Automation
  Letters}, vol.~6, no.~3, pp. 4305--4312, 2021.

\bibitem[{Arandjelovic} et~al.(2016){Arandjelovic}, {Gronat}, {Torii},
  {Pajdla}, and {Sivic}]{Arandjelovic16cvpr-netvlad}
R.~{Arandjelovic}, P.~{Gronat}, A.~{Torii}, T.~{Pajdla}, and J.~{Sivic},
  ``{NetVLAD}: {CNN} architecture for weakly supervised place recognition,'' in
  \emph{IEEE Conf. on Computer Vision and Pattern Recognition (CVPR)}, 2016,
  pp. 5297--5307.

\bibitem[Gawel et~al.(2018)Gawel, Don, Siegwart, Nieto, and
  Cadena]{Gawel18ral-xview}
A.~Gawel, C.~D. Don, R.~Siegwart, J.~Nieto, and C.~Cadena, ``X-view:
  Graph-based semantic multi-view localization,'' \emph{IEEE Robotics and
  Automation Letters}, vol.~3, no.~3, pp. 1687--1694, 2018.

\bibitem[Liu et~al.(2019)Liu, Petillot, Lane, and
  Wang]{Liu19icra-globalLocalizationObjects}
Y.~Liu, Y.~Petillot, D.~Lane, and S.~Wang, ``Global {Localization} with
  {Object}-{Level} {Semantics} and {Topology},'' in \emph{2019 {International}
  {Conference} on {Robotics} and {Automation} ({ICRA})}, May 2019, pp.
  4909--4915, iSSN: 2577-087X.

\bibitem[Lin et~al.(2021)Lin, Wang, Xu, Zhao, and
  Chen]{Lin21ral-topologyAwareObjectLocalization}
S.~Lin, J.~Wang, M.~Xu, H.~Zhao, and Z.~Chen, ``Topology aware object-level
  semantic mapping towards more robust loop closure,'' \emph{IEEE Robotics and
  Automation Letters}, vol.~6, no.~4, pp. 7041--7048, 2021.

\bibitem[Qin et~al.(2021)Qin, Zhang, Liu, and Lv]{Qin21jvcir-semantic}
C.~Qin, Y.~Zhang, Y.~Liu, and G.~Lv, ``Semantic loop closure detection based on
  graph matching in multi-objects scenes,'' \emph{Journal of Visual
  Communication and Image Representation}, vol.~76, p. 103072, 2021.

\bibitem[St\"{u}ckler and Behnke(2014)]{Stuckler14jvcir}
J.~St\"{u}ckler and S.~Behnke, ``Multi-resolution surfel maps for efficient
  dense 3d modeling and tracking,'' \emph{J. Vis. Comun. Image Represent.},
  vol.~25, no.~1, pp. 137--147, 2014.

\bibitem[Whelan et~al.(2016)Whelan, Salas-Moreno, Glocker, Davison, and
  Leutenegger]{Whelan16ijrr-elasticFusion}
T.~Whelan, R.~Salas-Moreno, B.~Glocker, A.~J. Davison, and S.~Leutenegger,
  ``{ElasticFusion: Real-Time Dense SLAM and Light Source Estimation},'' 2016.

\bibitem[Dai et~al.(2017)Dai, Nie{\ss}ner, Zollh{\"o}fer, Izadi, and
  Theobalt]{Dai17tog-bundlefusion}
A.~Dai, M.~Nie{\ss}ner, M.~Zollh{\"o}fer, S.~Izadi, and C.~Theobalt,
  ``Bundlefusion: Real-time globally consistent 3d reconstruction using
  on-the-fly surface reintegration,'' \emph{ACM Transactions on Graphics
  (ToG)}, vol.~36, no.~4, p.~1, 2017.

\bibitem[{Reijgwart} et~al.(2020){Reijgwart}, {Millane}, {Oleynikova},
  {Siegwart}, {Cadena}, and {Nieto}]{Reijgwart20ral-voxgraph}
V.~{Reijgwart}, A.~{Millane}, H.~{Oleynikova}, R.~{Siegwart}, C.~{Cadena}, and
  J.~{Nieto}, ``Voxgraph: Globally consistent, volumetric mapping using signed
  distance function submaps,'' \emph{{IEEE} Robotics and Automation Letters},
  2020.

\bibitem[Whelan et~al.(2015)Whelan, Kaess, Johannsson, Fallon, Leonard, and
  McDonald]{Whelan15ijrr}
T.~Whelan, M.~Kaess, H.~Johannsson, M.~Fallon, J.~Leonard, and J.~McDonald,
  ``Real-time large-scale dense {RGB-D SLAM} with volumetric fusion,''
  \emph{Intl. J. of Robotics Research}, vol.~34, no. 4--5, pp. 598--626, 2015.

\bibitem[Czarnowski et~al.(2020)Czarnowski, Laidlow, Clark, and
  Davison]{Czarnowski20ral-deepFactors}
J.~Czarnowski, T.~Laidlow, R.~Clark, and A.~Davison, ``{DeepFactors}: Real-time
  probabilistic dense monocular {SLAM},'' \emph{{IEEE} Robotics and Automation
  Letters}, vol.~5, no.~2, pp. 721--728, 2020.

\bibitem[Sucar et~al.(2020)Sucar, Wada, and Davison]{Sucar203dv-nodeSLAM}
E.~Sucar, K.~Wada, and A.~Davison, ``{NodeSLAM}: Neural object descriptors for
  multi-view shape reconstruction,'' in \emph{2020 Intl. Conference on 3D
  Vision (3DV)}, 2020, pp. 949--958.

\bibitem[Rosinol et~al.(2022)Rosinol, Leonard, and
  Carlone]{Rosinol22arxiv-nerfSLAM}
A.~Rosinol, J.~Leonard, and L.~Carlone, ``{NeRF-SLAM}: Real-time dense
  monocular {SLAM} with neural radiance fields,'' \emph{arXiv preprint:
  2210.13641}, 2022, \linkToPdf{https://arxiv.org/pdf/2210.13641.pdf}.

\bibitem[Kong et~al.(2023)Kong, Liu, Taher, and Davison]{Kong23arxiv-vmap}
X.~Kong, S.~Liu, M.~Taher, and A.~Davison, ``{vMAP}: Vectorised object mapping
  for neural field {SLAM},'' \emph{ArXiv, preprint: 2302.01838}, 2023.

\bibitem[Ha and Song(2022)]{Ha22icrl-openWorld}
H.~Ha and S.~Song, ``Semantic abstraction: Open-world 3d scene understanding
  from 2d vision-language models,'' 2022.

\bibitem[Jatavallabhula et~al.(2023)Jatavallabhula, Kuwajerwala, Gu, Omama,
  Chen, Li, Iyer, Saryazdi, Keetha, Tewari, Tenenbaum, de~Melo, , Krishna,
  Paull, Shkurti, and Torralba]{Jatavallabhula23arxiv-conceptFusion}
K.~M. Jatavallabhula, A.~Kuwajerwala, Q.~Gu, M.~Omama, T.~Chen, S.~Li, G.~Iyer,
  S.~Saryazdi, N.~Keetha, A.~Tewari, J.~Tenenbaum, C.~M. de~Melo, , M.~Krishna,
  L.~Paull, F.~Shkurti, and A.~Torralba, ``{ConceptFusion}: Open-set multimodal
  {3D} mapping,'' \emph{arXiv: 2302.07241}, 2023.

\bibitem[Ravichandran et~al.(2022)Ravichandran, Peng, Hughes, Griffith, and
  Carlone]{Ravichandran22icra-RLwithSceneGraphs}
Z.~Ravichandran, L.~Peng, N.~Hughes, J.~Griffith, and L.~Carlone,
  ``Hierarchical representations and explicit memory: Learning effective
  navigation policies on {3D} scene graphs using graph neural networks,'' in
  \emph{IEEE Intl. Conf. on Robotics and Automation (ICRA)}, 2022,
  \linkToPdf{https://arxiv.org/pdf/2108.01176.pdf}.

\bibitem[Agia et~al.(2022)Agia, Jatavallabhula, Khodeir, Miksik, Vineet,
  Mukadam, Paull, and Shkurti]{Agia22corl-taskography}
C.~Agia, K.~M. Jatavallabhula, M.~Khodeir, O.~Miksik, V.~Vineet, M.~Mukadam,
  L.~Paull, and F.~Shkurti, ``Taskography: {{Evaluating}} robot task planning
  over large {{3D}} scene graphs,'' in \emph{Conference on Robot Learning
  (CoRL)}.\hskip 1em plus 0.5em minus 0.4em\relax {PMLR}, Jan. 2022, pp.
  46--58.

\bibitem[Chang et~al.(2022{\natexlab{b}})Chang, Ballotta, and
  Carlone]{Chang22arxiv-DLite}
Y.~Chang, L.~Ballotta, and L.~Carlone, ``{D-Lite}: Navigation-oriented
  compression of {3D} scene graphs under communication constraints,''
  \emph{arXiv preprint: 2209.06111}, 2022,
  \linkToPdf{https://arxiv.org/pdf/2209.06111.pdf},\linkToVideo{https://youtu.be/nKYXU5VC6A8}.

\bibitem[Becker and Geiger(1996)]{Becker1996c-UAI-FastApproxOptJT}
A.~Becker and D.~Geiger, ``A sufficiently fast algorithm for finding close to
  optimal junction trees,'' in \emph{Conf. on Uncertainty in Artificial
  Intelligence (UAI)}, Feb. 1996, pp. 81--89.

\bibitem[Thomas and Green(2009)]{Thomas2009j-JCGS-EnumerateJT}
A.~Thomas and P.~J. Green, ``Enumerating the junction trees of a decomposable
  graph,'' \emph{J. Comput Graph Stat.}, vol.~18, no.~4, pp. 930--940, Dec.
  2009.

\bibitem[Bodlaender and
  Koster(2011)]{Bodlaender2010j-IC-TreewidthComputationII}
H.~L. Bodlaender and A.~M. Koster, ``Treewidth computations {II}: Lower
  bounds,'' \emph{Information and Computation}, vol. 209, no.~7, pp. 1103 --
  1119, 2011.

\bibitem[Jensen and Jensen(1994)]{Jensen1994c-UAI-OptimalJT}
F.~V. Jensen and F.~Jensen, ``Optimal junction trees,'' in \emph{Conf. on
  Uncertainty in Artificial Intelligence (UAI)}, Jul. 1994, p. 360–366.

\bibitem[Bodlaender(1988)]{Bodlaender88alp-Dynamicprogramming}
H.~L. Bodlaender, ``Dynamic programming on graphs with bounded treewidth,'' in
  \emph{Automata, {{Languages}} and {{Programming}}}, 1988, pp. 105--118.

\end{thebibliography}

}

\appendix
\section{Appendix}

\subsection{Preliminaries on Tree Decomposition and Treewidth}\label{sec:treeDecomposition}

For a graph $\Graph$ with vertices $\calV(\calG)$ and edges $\calE(\calG)$, a \emph{\treeDecomposition} is a tuple $(\Tree, \bag)$ where $\Tree$ is a tree graph and $\bag = \{ B_{\tau} \}_{\tau \in \Nodes(\Tree)}$ is a family of \emph{bags}, where $B_{\tau} \subset \Nodes(\Graph)$ for every tree node $\tau \in \Nodes(\Tree)$, such that the tuple $(\Tree, \bag)$ satisfies the following two properties:

(1) \emph{Connectedness Property:} for every graph node $v \in \Nodes(\Graph)$,
the \subgraph of $\Tree$ induced by tree nodes $\tau$ whose bag contains node $v$, is connected, \ie
\begin{equation}
    \Tree_v \triangleq \Tree\left[\{ \tau \in \Nodes(\Tree)~|~v \in B_{\tau}\}\right]
\end{equation}
is a connected \subgraph of $\Tree$ for every $v \in \Nodes(\Graph)$.

(2) \emph{Covering Property:} for every edge $\{u, v\} \in \Edges(\Graph)$ there exists a node $\tau \in \Nodes(\Tree)$ such that $u, v \in B_{\tau}$.

The simplest \treeDecomposition of any graph $\Graph$ is a tree with a single node, whose bag contains all the nodes in $\Graph$. However, in practical applications,
it is desirable to obtain decompositions where the size of the largest bag is small.
This is captured by the notion of \emph{treewidth}.
The treewidth of a \treeDecomposition $(\Tree, \bag)$ is defined as the  size of the largest bag minus one:
\begin{equation}
\label{eq:def_treewidth}
\treewidth{(\Tree, \bag)} \triangleq \max_{\tau \in \Nodes(\Tree)} |B_{\tau}| - 1.
\end{equation}
The treewidth of a graph $\Graph$ is defined as the minimum treewidth that can be achieved among all \treeDecompositions of $\Graph$. While
finding a \treeDecomposition with minimum treewidth is NP-hard, many algorithms exist that generate \treeDecompositions with small (\ie close to the minimum attainable for that graph) treewidth~\citep{Becker1996c-UAI-FastApproxOptJT, Bodlaender2006c-Chapter-TreewidthComputation, Thomas2009j-JCGS-EnumerateJT, Bodlaender2010j-IC-TreewidthComputationI, Bodlaender2010j-IC-TreewidthComputationII}.

One of the most popular \treeDecompositions is the \emph{junction tree decomposition}, which was introduced in~\cite{Jensen1994c-UAI-OptimalJT}.
In it, the graph $\Graph$ is first triangulated. Triangulation is done by adding a chord between any two nodes in every cycle of length $4$ or more. This eliminates all the cycles of length $4$ or more in the graph $\Graph$ to produce a chordal graph $\Graph_{c}$.
The collection of bags $\bag = \{ B_{\tau} \}_{\tau}$ in the junction tree is chosen as the set of all maximal cliques in the chordal graph $\Graph_{c}$.
Then, an \emph{intersection graph} $\calI$ on $\bag$ is built, which has a node for every bag in $\bag$ and an edge between two bags $B_{\tau}$ and $B_{\mu}$ if they have a non-empty intersection, \ie $|B_{\tau}\cap B_{\mu}| \geq 1$. The weight of every link $\{\tau, \mu\}$ in the intersection graph $\calI$ is set to $|B_{\tau}\cap B_{\mu}|$. Finally, the desired junction tree is obtained by extracting a maximum-weight spanning tree on the weighted intersection graph $\calI$. It is know that this extracted tree $\Tree$, with the bag $\calB$, is a valid \treeDecomposition of $\Graph$ that satisfies the connectedness and covering property.

 \subsection{Proof of \cref{prop:tw-hierarchy}}\label{app:prop:tw-hierarchy}

The proposition provides an upper-bound on the treewidth of a hierarchical graph.
We can deduce this bound from \cref{algo:td-hierarchical} by noting that:
(i) all the bags in the \treeDecomposition of an $\layers$-layered hierarchical graph are nothing but the bags from the \treeDecomposition of either the layer-$\layers$ graph $\calG[\calV_\layers]$ or the graphs $\calG[C(v)]$, for $v \in \calV$;
and (ii) the size of the bags of all \treeDecompositions of $C(v)$ are increased by one in \cref{line:addNodeToBags} of \cref{algo:td-hierarchical}. Since the treewidth of a \treeDecompositions  is defined as the size of its maximum bag minus one, and since the treewidth of the graph is the minimum attainable by any possible \treeDecompositions, the bound in eq.~\eqref{eq:twbound} follows.

 \subsection{Proof of \cref{lem:tw-room}}\label{app:lem:tw-room}

Let $\calG_R$ be the room graph, as described in the statement of the theorem. We first note that the nodes in $\calG_R$ that have degree one are those rooms that connect to only one other room. If we were to remove all the degree one rooms, in $\calG_R$, and form another graph $\calG^{'}_R$, then \emph{(i)~the treewidth $\calG_R$ and $\calG^{'}_R$ would be the same}.
We can see this by noting that one can form a \treeDecomposition{} for $\calG_R$ by appending the \treeDecomposition{} of $\calG^{'}_R$ with additional bags of size $2$, representing the degree-one room node $v$ in $\calG_R \backslash \calG^{'}_R$ and its connecting room in $\calG^{'}_R$.
Now we note that \emph{(ii)~the treewidth of $\calG^{'}_R$ is always bounded by two under the theorem's assumptions}.
This is because all nodes in $\calG^{'}_R$ have degree less than or equal to $2$ (by assumption). And, hence, $\calG^{'}_R$ will never have a complete graph of four nodes as a \subgraph.
This is enough to conclude that the graph $\calG^{'}_R$ will have a treewidth less than or equal to $2$~\citep{Bodlaender88alp-Dynamicprogramming}.

\subsection{Object and Room Labels}\label{app:labels}

\myParagraph{Objects}
For any configuration of \name{} that uses a 2D semantic segmentation network as input, we use the following set of object labels, which is derived from the mpcat40 label space~\citep{Chang173dv-Matterport3D}: \semantic{chair}, \semantic{table}, \semantic{picture}, \semantic{cabinet}, \semantic{cushion}, \semantic{sofa}, \semantic{bed}, \semantic{curtain}, \semantic{chest of drawers}, \semantic{plant}, \semantic{sink}, \semantic{toilet}, \semantic{stool}, \semantic{towel}, \semantic{mirror}, \semantic{tv monitor}, \semantic{bathtub}, \semantic{fireplace}, \semantic{lighting}, \semantic{shelving}, \semantic{blinds}, \semantic{seating}, \semantic{board panel}, \semantic{furniture}, \semantic{appliances}, \semantic{clothes}, \semantic{objects}.
In addition, the following labels are detected, but treated as structure: \semantic{wall}, \semantic{floor}, \semantic{door}, \semantic{window}, \semantic{stairs}, \semantic{ceiling}, \semantic{column}, \semantic{counter}, \semantic{railing}.
We manually construct a mapping from the ADE20k label space~\citep{Zhou17cvpr-ade20k} that the 2D semantic segmentation networks provide output in to this subset of the mpcat40 label space~\citep{Chang173dv-Matterport3D}.
The ground truth segmentation for the uHumans2 dataset has a smaller label space.
For the apartment, we use the following set of object labels: \semantic{chair}, \semantic{couch}, \semantic{computer}, \semantic{lamp}, \semantic{bed}, \semantic{table}, \semantic{trashcan}.
For the office, we use the following set of object labels: \semantic{chair}, \semantic{bench}, \semantic{couch}, \semantic{objects}, \semantic{painting}, \semantic{plant}, \semantic{trashcan}.
Both scenes also have similar structure labels that are detected (\eg{} \semantic{wall}, \semantic{floor}, \semantic{ceiling}).

\myParagraph{Rooms}
We predict a subset of the room category labels available in MP3D~\citep{Chang173dv-Matterport3D}, as we group several room category synonyms together.
First, MP3D includes several synonyms for ``outdoor'' regions in the provided ground-truth room labels (\ie{} \semantic{porch}, \semantic{balcony}, \semantic{outdoor}).
As we only focus on indoor applications, we group these outdoor labels with the original \semantic{unknown} label from MP3D, and ignore this label for both computing the training loss and label accuracy.
Additionally, we group \semantic{toilet} and \semantic{bathroom} together into a single label (\semantic{bathroom}), and group \semantic{other room} and \semantic{junk} with \semantic{unknown}.
This gives us the following room label space: \semantic{bedroom}, \semantic{closet}, \semantic{dining room}, \semantic{lobby}, \semantic{family room}, \semantic{garage}, \semantic{hallway}, \semantic{library}, \semantic{laundry room}, \semantic{kitchen}, \semantic{living room}, \semantic{conference room}, \semantic{lounge}, \semantic{office}, \semantic{game room}, \semantic{stairwell}, \semantic{utility room}, \semantic{theater}, \semantic{gym}, \semantic{balcony}, \semantic{bar}, \semantic{classroom}, \semantic{dining booth}, \semantic{spa}, \semantic{bathroom}.
Rooms that do not meet the criteria for inference are given the label \semantic{unknown}.

\subsection{Graph Neural Network Training}\label{sec:gnn_training}
We provide more details on (i) the pre-processing performed on the 3D scene graphs in the MP3D dataset, (ii) the implementation of the examined architectures, (iii) the training and testing runtime on the Stanford3d and MP3D datasets, (iv) the approach we used for hyper-parameter tuning.

\myParagraph{MP3D Graph Processing}
Here we describe additional pre-processing steps that we take to better condition the room classification training on the MP3D dataset.
For the room classification ablation in~\cref{sec:exp_results},
we use the ground-truth room geometry in MP3D to re-segment the places layer, and add intra-layer edges between the re-segmented rooms as described in \cref{sec:rooms_clustering}.
This allows us to disentangle the room classification accuracy from potential errors in the room clustering.
We then extract object-room graphs by connecting room and object nodes through places nodes in between.
We discard all room nodes that both have no adjacent rooms and contain two or fewer objects.

\myParagraph{Implementation}
We implement all graph learning architectures in PyTorch Geometric~\citep{Fey19iclrwk-pytorchGeometric}, which supports all four types of message passing functions we test on, as well as a heterogeneous data structure for learning on heterogeneous graphs.
We use the ReLU activation function and dropout in between message passing operations unless otherwise specified.
For semi-supervised learning on homogeneous graphs, we send the final node hidden states ---after performing message passing--- through a linear layer to project embeddings to room and object label spaces separately, and then through softmax layers for the final prediction.
For the neural tree, we add an additional mean pooling layer to combine leaf node hidden states after all message passing iterations on the \Htree.
For semi-supervised learning on heterogeneous graphs and room classification on both graph types, we match the hidden dimension of the last message passing iteration to the number of labels and hence remove the final linear layers before softmax.

\setlength{\tabcolsep}{5pt}
\begin{table}[t!]
    \toggleformat{}{\footnotesize}
    \centering
    \begin{tabular}{l cc cc}
    	\toprule
    	\multirow{2}{*}{Message Passing}  & \multicolumn{2}{c}{Original} & \multicolumn{2}{c}{\Htree} \\
    	& Training & Testing & Training  & Testing \\
    	\midrule
    	GCN         & 0.087s & 0.021s & 0.135s & 0.035s \\
    	GraphSAGE  	& 0.083s & 0.020s & 0.127s & 0.033s	\\
    	GAT       	& 0.092s & 0.022s & 0.169s & 0.042s	\\
    	GIN  		& 0.088s & 0.020s & 0.132s & 0.033s	\\
    	\bottomrule
    \end{tabular}
    \caption{Stanford3d: training and testing runtimes.}\label{tab:timing_homogeneous_s3d}
    \togglevspace{0mm}{0mm}
\end{table}

\setlength{\tabcolsep}{4pt}
\begin{table}[t!]
    \toggleformat{}{\footnotesize}
	\centering
	\begin{tabular}{llcccc}
		\toprule
		\multicolumn{2}{c}{\multirow{2}{*}{Graph Types}}  & \multicolumn{2}{c}{Original} & \multicolumn{2}{c}{\Htree} \\
		& & Training & Testing & Training  & Testing \\
		\midrule
		\multirow{2}{*}{Homogeneous} & absolute pos. & 0.092s & 0.022s & 0.169s & 0.042s \\
		& relative pos. & 0.100s & 0.023s & 0.181s & 0.045s \\
		\midrule
		\multirow{2}{*}{Heterogeneous} & absolute pos. & 0.294s & 0.079s & 1.023s & 0.275s \\
		& relative pos. & 0.329s & 0.083s & 1.142s & 0.298s \\
		\bottomrule
	\end{tabular}
	\caption{Stanford3d, GAT: training and testing runtimes for different graph types.}\label{tab:timing_gat_s3d}
    \togglevspace{0mm}{0mm}
\end{table}

\setlength{\tabcolsep}{\toggleformat{3.8pt}{3.25pt}}
\begin{table}[t!]
    \toggleformat{}{\footnotesize}
	\centering
	\begin{tabular}{llcccc}
		\toprule
		\multicolumn{2}{c}{\multirow{2}{*}{Graph Types}}  & \multicolumn{2}{c}{Original} & \multicolumn{2}{c}{\Htree} \\
		& & Training & Testing & Training  & Testing \\
		\midrule
		\multirow{2}{*}{w/o word2vec} & w/o room edges & 1.915s & 0.216s & 6.106s & 0.704s \\
		& with room edges & 1.960s & 0.222s & 6.498s & 0.750s \\
		\midrule
		\multirow{2}{*}{with word2vec} & w/o room edges & 2.020s & 0.261s & 6.637s & 0.906s \\
		& with room edges & 2.075s & 0.238s & 7.065s & 0.875s \\
		\bottomrule
	\end{tabular}
	\caption{MP3D, GAT: training and testing runtimes for different graph types.}\label{tab:timing_gat_mp3d}
    \togglevspace{0mm}{0mm}
\end{table}

\myParagraph{Runtime}
Here we report the time required for pre-processing the datasets, as well as training and testing our models.
We save all object-room scene graphs with at least one room and one object node for both Stanford3d and MP3D datasets.
The total \Htree construction time is \SI{2.8}{\second} for 482 Stanford3d scene graphs with a maximum of \SI{0.087}{\second} for a single graph.
For the MP3D dataset
the total \Htree construction time is \SI{516}{\second} for a total of 6114 graphs, with a maximum of \SI{0.84}{\second} for a graph.
\Cref{tab:timing_homogeneous_s3d} reports the (per epoch) train and test time for the standard GNN architectures ---GCN, GraphSAGE, GAT, GIN--- and the corresponding neural trees on the homogeneous Stanford3d scene graphs with the training batch size set to 128.
\Cref{tab:timing_gat_s3d,tab:timing_gat_mp3d} report timing for the GAT architecture with various input scene graph types.
We observe that the neural tree takes 1.5 to 2 times as long compared to the corresponding standard GNN for homogeneous graph inputs.
This is expected as the neural tree often uses more message passing iterations, and the \Htrees{} are larger than the original graphs.
For heterogeneous graphs, the neural tree takes around 3.5 times as long, as heterogeneous \Htrees are both larger than the corresponding original graphs and contain more node and edge types.

\setlength{\tabcolsep}{\toggleformat{3.8pt}{3.25pt}}
\begin{table}[t!]
    \toggleformat{}{\footnotesize}
    \centering
    \begin{tabular}{l l llll}
        \toprule
        \multicolumn{2}{c}{Message Passing}    & iter. & hidden dim. & learning rate & regularization \\
        \midrule
        \multirow{2}{*}{GCN}  & Original & 3       & $64$   &0.01       & 0.0   \\
           & \Htree{}     & 4  & $128$       &0.01        & 0.0   \\
        \midrule
        \multirow{2}{*}{GraphSAGE}  & Original   & 3       & $128$  & 0.005 & 0.001\\
           & \Htree{}     & 4  & $128$       & 0.005      & 0.001 \\
        \midrule
        \multirow{2}{*}{GAT} & Original & 2   & $128$   & 0.001         & 0.0001    \\
           & \Htree{}    & 4  & $128$    &0.0005      & 0.0001    \\
        \midrule
        \multirow{2}{*}{GIN} & Original  & 3   & $64$      & 0.005      & 0.001 \\
           & \Htree{}     & 4    & $128$  & 0.005     & 0.001 \\
        \bottomrule
    \end{tabular}
    \caption{Stanford3d: tuned hyper-parameters.}\label{tab:param_homogeneous_s3d}
    \togglevspace{0mm}{0mm}
\end{table}

\setlength{\tabcolsep}{5pt}
\begin{table}[t!]
    \toggleformat{}{\footnotesize}
    \centering
    \begin{tabular}{l l lll}
        \toprule
        \multicolumn{2}{c}{Graph Types}    & heads & learning rate & dropout \\
        \midrule
        \multirow{2}{*}{w/o word2vec}  & Original & 5   & 0.001    & 0.2        \\
                                       & \Htree   & 5   & 0.001    & 0.2        \\
        \midrule
        \multirow{2}{*}{with word2vec}  & Original  & 3  & 0.002       & 0.4    \\
                                        & \Htree    & 3  & 0.002       & 0.4        \\
        \bottomrule
    \end{tabular}
    \caption{MP3D: tuned hyper-parameters.}\label{tab:param_mp3d}
    \togglevspace{0mm}{0mm}
\end{table}

\myParagraph{Hyper-Parameter Tuning}
For replicating the experiments in~\cite{Talak21neurips-neuralTree} on Stanford3d, we tune hyper-parameters in the same order as~\cite{Talak21neurips-neuralTree}, by searching over the following sets:
\begin{itemize}
	\item Message passing iterations: [1, 2, 3, 4, 5, 6]
	\item Message passing hidden size: [16, 32, 64, 128, 256]
	\item Learning rate: [0.0005, 0.001, 0.005, 0.01]
	\item Dropout probability: [0.25, 0.5, 0.75]
	\item $L_2$ regularization strength: [0, 0.0001, 0.001, 0.01]
\end{itemize}
For both the standard GNN and the neural tree, the choice of the number of message passing iterations, message passing hidden dimensions, and learning rate has a significantly higher impact on accuracy than the other two hyper-parameters.
Therefore, these three hyper-parameters are first tuned using a grid search while keeping dropout and $L_2$ regularization to the lowest values.
With these three hyper-parameters fixed, the dropout and $L_2$ regularization are then tuned via another grid search.
A dropout ratio of 0.25 turns out to be the optimal choice in all cases.
The other tuned hyper-parameters ---used in the experiment in \cref{tab:mp_acc_ablation}--- are reported in \cref{tab:param_homogeneous_s3d}.
We find the best hyper-parameters to be consistent with those reported in~\cite{Talak21neurips-neuralTree}.
Any changes within the search space does not result in more than 0.5 standard deviations of improvement.

Therefore, we use the same sets of hyper parameters as in~\cite{Talak21neurips-neuralTree} and report them in \cref{tab:param_homogeneous_s3d}.
Apart from the four listed hyper-parameters, some architectures (GAT, GraphSAGE, GIN) have architecture-specific design choices and hyper-parameters.
In the case of GAT, we use 6 attention heads with concatenated output and ELU activation function (instead of ReLU) to be consistent with the original paper. For GraphSAGE, we use the GraphSAGE-mean from the original paper, which does mean-pooling after each convolution operation. In the case of GIN, we use the more general GIN-$\epsilon$ for better performance.
For the additional ablation study in \cref{tab:graph_acc_ablation}, we keep the same hyper parameters, though we set the hidden dimension to 64 for training on heterogeneous graphs.
We use the Adam optimizer and run optimization for 1000 epochs to achieve reasonable convergence during training.

We use a similar hyper-parameter tuning procedure for experiments on MP3D. The parameter range we explored are:
\begin{itemize}
	\item Message passing iterations: [2, 3, 4]
	\item Message passing hidden size: [16, 32, 64, 128]
	\item Number of attention heads: [1, 3, 5]
	\item Learning rate: [0.0005, 0.001, 0.002]
	\item Dropout probability: [0, 0.2, 0.4]
	\item $L_2$ regularization strength: [0, 0.0001, 0.001]
\end{itemize}
For MP3D, all parameters except $L_2$ regularization have significant impact on accuracy. Therefore, we perform grid search over these hyper parameters first, and then tune $L_2$ regularization.
For all setups with more than 1 attention head, we observe averaging over the output performs consistently better than concatenation and therefore use the average when using multi-heads attention.
For both approaches, we tune separate sets of hyper parameters for configurations with and without word2vec object semantic features.
We reuse these hyper parameters for configurations without room edges.
Following this approach, we find that 3 message passing iterations with a hidden dimension of 64 (for the first two iterations) and a $L_2$ regularization constant of 0.001 work best for all setups.
The rest of hyper parameters are summarized in \cref{tab:param_mp3d}.
Training is run for 800 epochs of SGD using the Adam optimizer.

\setlength{\tabcolsep}{\toggleformat{5pt}{4pt}}
\begin{table}[t!]
    \toggleformat{}{\footnotesize}
    \centering
    \begin{tabular}{l l cc}
        \toprule
        {Message Passing} & {PyG version} & Original & H-tree \\
        \midrule
        \multirow{2}{*}{GCN} &  $1.7.0^*$  & $40.88 \pm 2.28 \%$ & $\mathbf{50.63 \pm 2.25 \%}$ \\
            & 2.2.0 & $42.91 \pm    2.01\%$ & $\mathbf{51.45 \pm    2.00 \%}$ \\
        \midrule
        \multirow{2}{*}{GraphSAGE} &  $1.7.0^*$  & $59.54 \pm 1.35 \%$ & $\mathbf{63.57 \pm 1.54 \%}$ \\
            & 2.2.0 & {$ 56.97 \pm  2.02 \% $} & {$\mathbf{59.39 \pm    2.10 \%}$} \\
        \midrule
        \multirow{2}{*}{GAT} &  $1.7.0^*$  & $46.56 \pm 2.21 \%$ & $\mathbf{62.16 \pm 2.03 \%}$ \\
            & 2.2.0 & $45.06 \pm    2.32\%$ & {$\mathbf{53.86 \pm   2.06\%}$ }\\
        \midrule
        \multirow{2}{*}{GIN} &  $1.7.0^*$  & $49.25 \pm 1.15 \%$ & $\mathbf{63.53 \pm 1.38 \%}$ \\
            & 2.2.0 & $48.03 \pm    2.21\%$ & {$\mathbf{59.43 \pm   2.12 \%}$} \\
        \bottomrule
    \end{tabular}
    \caption{Stanford3d: Node classification accuracy using different versions of PyTorch Geometric and the \treeDecomposition from~\cite{Talak21neurips-neuralTree}. PyTorch Geometric version 1.7.0 (starred in the table) corresponds to the version used by \citet{Talak21neurips-neuralTree}.
    Best results in {\bf bold}.}\label{tab:mp_acc_ablation_extra}
    \togglevspace{0mm}{0mm}
\end{table}

\subsection{Additional Experiments on Semi-supervised Node Classification}\label{sec:stanford3d_extra}
The semi-supervised node classification experiment in \cref{tab:mp_acc_ablation} shows a significant accuracy change compared to the same experiment in~\cite{Talak21neurips-neuralTree}.
Our implementation differs from~\cite{Talak21neurips-neuralTree} in two ways: the PyTorch Geometric version used and the \Htree{} decomposition method.
Here, we provide additional experimental results with respect to these changes.

We use PyTorch Geometric 2.2.0 for the results in \cref{tab:mp_acc_ablation}, which has support for heterogeneous GNN operators.
The previous results from~\cite{Talak21neurips-neuralTree} used PyTorch Geometric 1.7.0, which was the latest version at the time of publication.
We keep the tree decomposition the same as in~\cite{Talak21neurips-neuralTree} and re-run the neural tree experiments using PyTorch Geometric 2.2.0.
\Cref{tab:mp_acc_ablation_extra} compares the results trained using different versions.
While the advantage of using the neural tree is retained, the table shows
 a significant accuracy penalty in moving from PyTorch Geometric 1.7.0 to 2.2.0, especially for the neural tree.
The only exception is GCN, where there is a slight improvement for both networks.
Using GraphSAGE, GAT, and GIN message passing architecture, standard GNNs trained on original graphs show 1 to 3\% decrease in accuracy moving to PyTorch Geometric 2.2.0.
For the neural tree models, there is about 4\% accuracy drop using GraphSAGE and GIN, and a more significant 8.3\% drop using GAT.
As a side note, using the proposed \treeDecomposition, we are able to get comparable accuracy using GAT on the \Htree{}  with relative position feature representations, compared to the approach in~\cite{Talak21neurips-neuralTree} (see \cref{tab:graph_acc_ablation}).

The neural tree results in \cref{tab:mp_acc_ablation_extra} computed using PyTorch Geometric 2.2.0 and those in \cref{tab:mp_acc_ablation} also differ in the \Htree{} decomposition method.
The former computes a tree decomposition of the entire graph, while the later breaks down hierarchical graphs into \subgraphs and then combines the decomposed trees according to~\cref{algo:td-hierarchical}.
These results are within 0.5 standard deviations when we use GraphSAGE or GAT.
The proposed tree decomposition gains a 3.19\% advantage with GCN, but a 4.43\% disadvantage with GIN.
This could be due to the difference in the \Htree{} construction, or to the PyTorch Geometric updates which make training certain message passing architectures much more sensitive.

Overall, regardless of PyTorch Geometric versions and \Htree{} decomposition methods, the neural tree maintains a clear advantage over standard GNNs on all message passing architectures.

\subsection{Loop Closure Ablation Parameters}\label{app:lcd}

\Cref{tab:params-V} reports key parameters used for ``V-LC (Permissive)'' and ``V-LC (Nominal)'' in  the loop closure ablation study in \cref{fig:lcd_experiment}.
For the meaning of these parameters we refer the reader to the open-source Kimera implementation from~\cite{Rosinol21ijrr-Kimera}.

\setlength{\tabcolsep}{\toggleformat{4pt}{3.5pt}}
\begin{table}[ht]
    \toggleformat{}{\footnotesize}
    \centering
    \begin{tabular}{ccc}
        \toprule
        Parameter & V-LC (Permissive) & V-LC (Nominal) \\
        \midrule
        L1 Score Threshold & 0.05 & 0.4 \\
        5pt RANSAC Inlier Threshold & \num{1e-5} & \num{1e-5} \\
        Lowe Matching Ratio & 0.9 & 0.9 \\
        \bottomrule
    \end{tabular}
    \caption{Visual loop closure parameters.}\label{tab:params-V}
    \togglevspace{0mm}{0mm}
\end{table}

\Cref{tab:params-SG} reports key parameters used for ``SG-LC'' in  the ablation study in \cref{fig:lcd_experiment}.
SG-LC does not use NSS (Normalized Similarity Scoring) to filter out matches.
As such, the matching threshold for the agent visual descriptors as shown in \cref{tab:params-V} was chosen to produce similar numbers of matches at the visual level as ``V-LC (nominal)''.
For the meaning of each parameter (and details on other parameters), we refer the reader to our open-source implementation at \hydraURL{}.

\begin{table}[ht]
    \toggleformat{}{\footnotesize}
    \centering
    \begin{tabular}{ccc}
        \toprule
        Parameter & SG-LC \\
        \midrule
        Agent L1 Match Threshold & 0.01 \\
        Object L1 Match Threshold & 0.3 \\
        Places L1 Match Threshold & 0.5 \\
        5pt RANSAC Inlier Threshold & \num{1e-6} \\
        Lowe Matching Ratio & 0.8 \\
        Object L1 Registration Threshold & 0.8 \\
        Object Minimum Inliers & 5 \\
        TEASER Noise Bound [m] & 0.1 \\
        \bottomrule
    \end{tabular}
    \caption{Scene graph loop closure parameters.}\label{tab:params-SG}
    \togglevspace{0mm}{0mm}
\end{table}

Finally, \cref{tab:params-SG-GNN} reports key model parameters used for the GNN architecture of ``SG-GNN'' in  the ablation study in \cref{fig:lcd_experiment}.
In this table, we denote a multi-layer perceptron (MLP) as \mlp{i, h_1, h_2, \ldots, h_n, o}, where $i$ is the input feature size, $o$ is the last (\ie{} output) layer size, and $h_1, h_2 \ldots, h_n$ are the hidden layer sizes.
All intermediate layers are connected via a ReLU nonlinearity activation function.
In our architecture, the input node and edge features are passed through node and edge MLP encoders and then fed through multiple iterations of message passing.
After message passing, the resulting node embedding are aggregated through another MLP (Graph Aggregation in \cref{tab:params-SG-GNN}), and then the final aggregated result is passed through another MLP (Graph Transform in \cref{tab:params-SG-GNN}).
As we follow the architecture from~\cite{Li19icml-GraphMatching}, we refer the reader to that paper for more  details.

\setlength{\tabcolsep}{\toggleformat{4pt}{3.0pt}}
\begin{table}[ht]
    \toggleformat{}{\footnotesize}
    \centering
    \begin{tabular}{ccc}
        \toprule
        Parameter & Places & Objects \\
        \midrule
        Message Passing Iterations & 3 & 3 \\
        Message Passing Channels & 256 & 128 \\
        Edge Encoder & \mlp{1, 3, 3, 1} & \mlp{1, 3, 3, 1} \\
        Node Encoder & \mlp{2, 3, 3} & \mlp{43, 128, 64} \\
        Graph Aggregation & \mlp{128, 128} & \mlp{128, 128} \\
        Graph Transform & \mlp{128, 96, 64, 64} & \mlp{128, 96, 64, 64} \\
        \bottomrule
    \end{tabular}
    \caption{Model Parameters for SG-GNN.}\label{tab:params-SG-GNN}
    \togglevspace{0mm}{0mm}
\end{table}

\begin{figure}[h!]
    \centering
    \subfloat[]{\centering
        \includegraphics[width=0.45\textwidth]{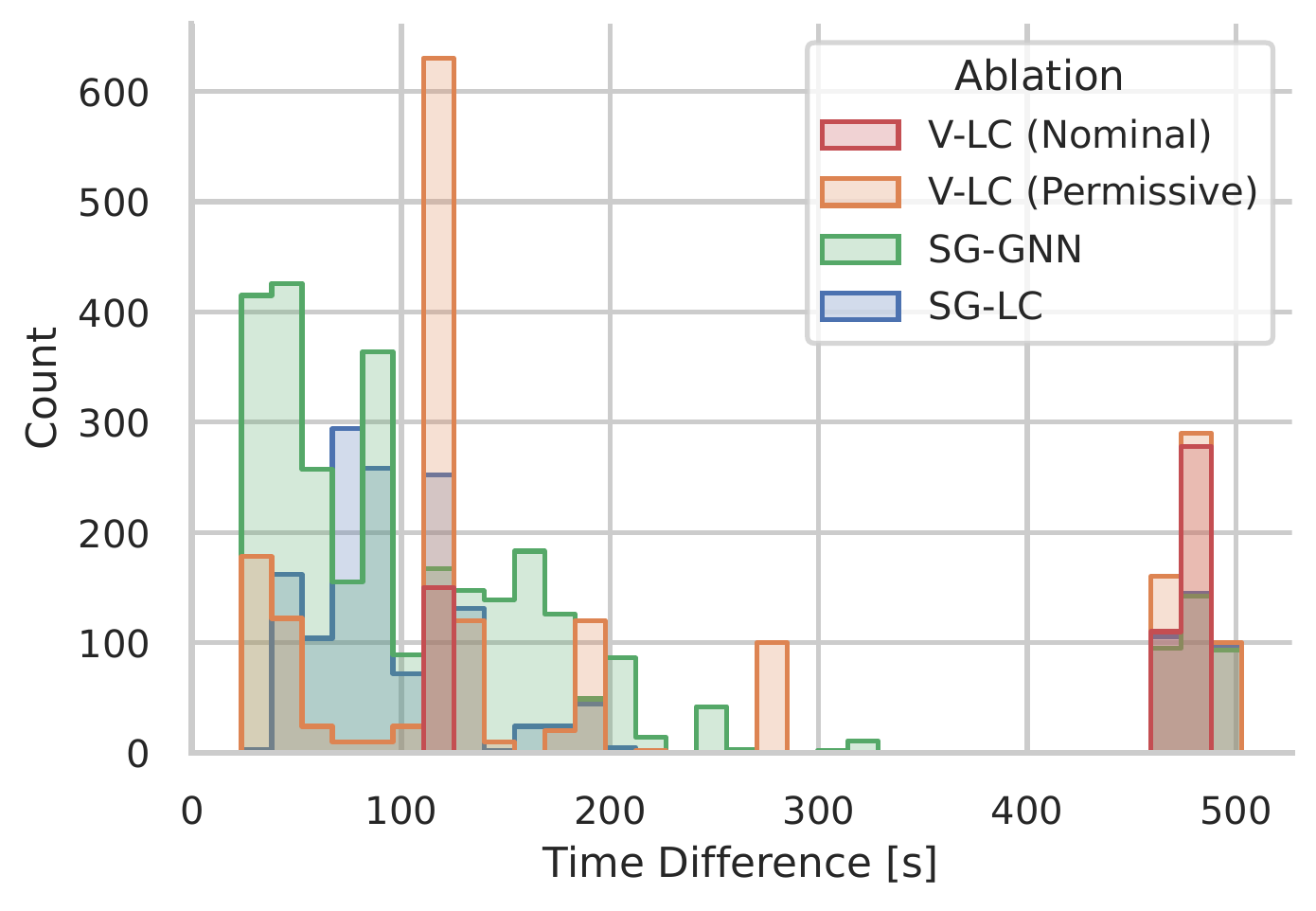}\label{fig:lcd_time_diff}
    } \\
    \subfloat[]{\centering
        \includegraphics[width=0.45\textwidth]{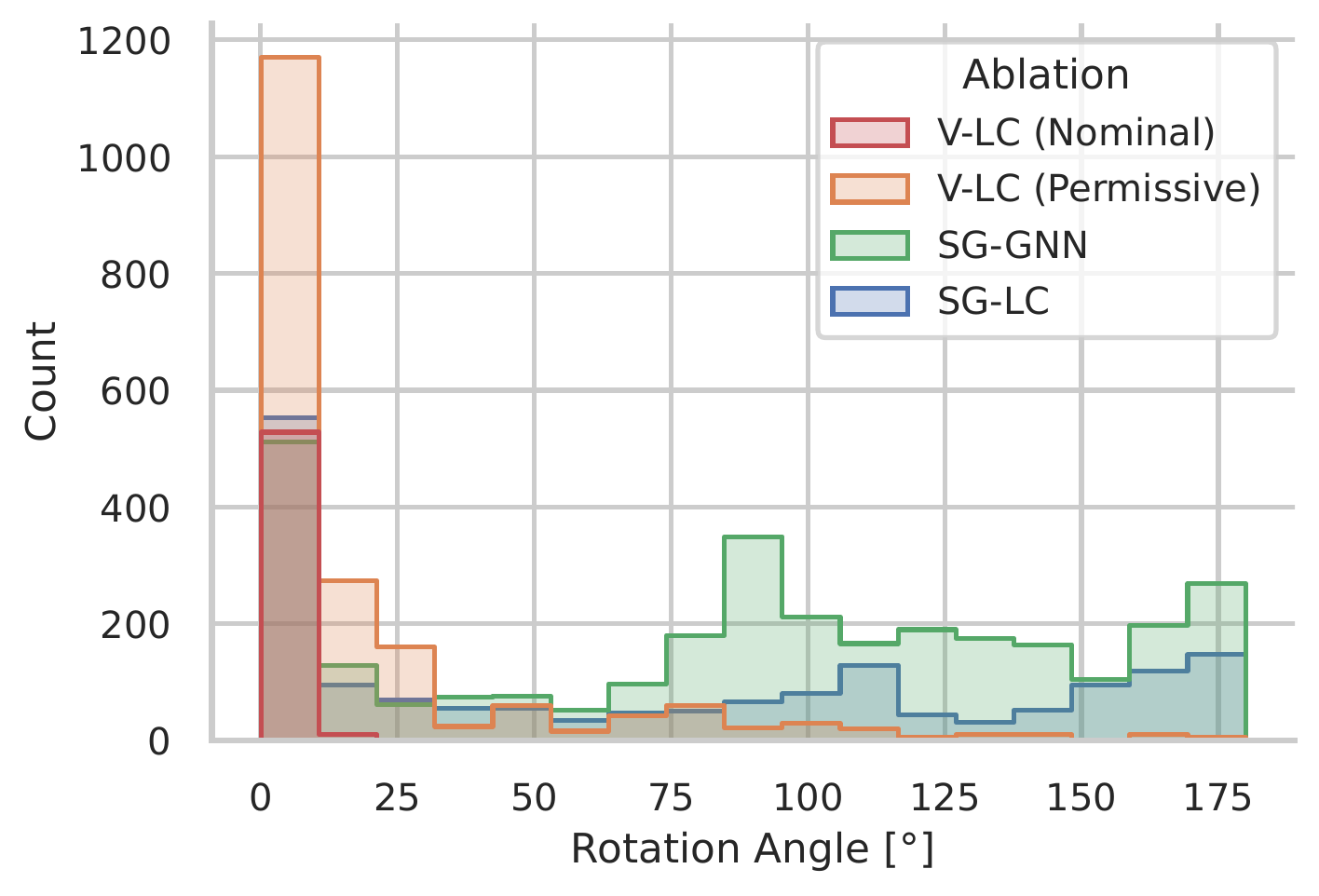}\label{fig:lcd_rot_angle}
    }
    \caption{(a) Distribution of time difference in seconds between current and matched agent poses for detected loop closures. (b) Distribution of rotation angle in degrees between current and matched agent poses for detected loop closures.}\label{fig:lcd_stats}
    \togglevspace{0mm}{0mm}
\end{figure}

\subsection{Additional Loop Closure Experiments}\label{app:lcd-stats}

Here we report additional results examining the properties of the loop closure detection configurations presented in \cref{fig:lcd_experiment}.
In particular, we show two different pieces of data for each loop closure that was detected over the five trials: (i) the timestamp difference in seconds between the current pose of the agent and the matched pose of the agent in \cref{fig:lcd_time_diff} and (ii) the rotation angle between the current pose of the agent and the matched pose of the agent in \cref{fig:lcd_rot_angle}.

We note two important trends.
One is that in \cref{fig:lcd_time_diff} the vision-based loop closure detection configurations (\ie{} ``V-LC (Nominal)'' and ``V-LC (Permissive)'') are centered around \SI{100}{\second} and \SI{500}{\second}, which correspond to the portions of the recorded data for the uHumans2 office scene where the robot revisits the same area in the scene with a similar viewpoint.
This is also supported by both vision-based configurations having a rotation angle between the current and matched pose centered around \SI{0}{\degree} in \cref{fig:lcd_rot_angle}.
In contrast, both ``SG-GNN'' and ``SG-LC'' have a wider distribution of time differences and rotation angles in \cref{fig:lcd_time_diff} and \cref{fig:lcd_rot_angle}, suggesting an improved viewpoint invariance.

Note that some of the loop closures resulting from ``SG-GNN'' and ``SG-LC'' have a small time difference (see left-most part of \cref{fig:lcd_time_diff}).
In these cases, the \subgraphs being registered often are not disjoint (\ie{} a subset of nodes is common to both), and the registration (via TEASER, as described in \cref{sec:LCD})  returns an identity pose between the two \subgraphs.
This relative pose already agrees with the current estimate of the relative pose between the current and the match agent node,  hence it does not provide any additional information to the optimization (but does act to constrain the deformation of the scene locally).
Regardless, ``SG-LC'' and ``SG-GNN'' outperform vision-only loop closures with respect to the overall system performance of \name{}.

The second trend to note is that ``SG-GNN'' has a wider distribution of time differences than ``SG-LC'', especially in the range of \SIrange{150}{250}{\second}.
While many of the loop closures detected by both of these configurations are from overlapping regions of the scene graph, ``SG-GNN'' does appear to provide loop closures over a slightly longer horizon than ``SG-LC''.

\end{document}